\documentclass[twoside]{article}

\usepackage[margin=1in]{geometry}

\usepackage{framed} 
\usepackage[colorlinks=true,linkcolor=red,filecolor=green,citecolor=red]{hyperref}
\usepackage{mathtools}
\usepackage[]{amsmath,amssymb,epsfig}
\usepackage{amsthm}
\usepackage{amsmath}
\usepackage{amssymb}
\usepackage{graphicx}
\usepackage{epstopdf}
\usepackage{comment}
\usepackage{array}
\usepackage{algorithm}
\usepackage{url}
\usepackage{ifthen}
\usepackage{wrapfig}
\usepackage{lscape}
\usepackage{algpseudocode}
\usepackage{setspace}
\usepackage{multicol}
\usepackage{multirow}
\usepackage{color}
\usepackage{colortbl}
\usepackage{xcolor}
\usepackage{rotating}
\usepackage{caption}
\usepackage{float}
\usepackage{ifthen}
\usepackage{placeins}
\usepackage{framed}
\usepackage{enumerate}

\usepackage[%
    font={small,sf},
    labelfont=bf,
    format=hang,    
    format=plain,
    margin=0pt,
    width=0.8\textwidth,
]{caption}
\usepackage[list=true]{subcaption}

\newtheorem{theorem}{Theorem}

\newtheorem{claim}{Claim}

\newtheorem{example}{Example}

\newtheorem{lemma}{Lemma}

\newtheorem{remark}{Remark}

\numberwithin{equation}{section}
\newcommand{\calR}{\ensuremath{\mathcal{R}}}

\newcommand{\calU}{\ensuremath{\mathcal{U}}}

\newcommand{\norm}[1]{\|{#1}\|}
\newcommand{\abs}[1]{\left|{#1}\right|}

\newcommand{\set}[1]{\left\{{#1}\right\}}
\newcommand{\dotprod}[2]{\langle#1,#2\rangle}
\newcommand{\est}[1]{\widehat{#1}}
\newcommand{\expec}{\ensuremath{\mathbb{E}}}
\newcommand{\matR}{\ensuremath{\mathbb{R}}}

\newcommand{\prob}{\ensuremath{\mathbb{P}}}

\newcommand{\mb}[1]{\mbox{\boldmath$#1$}}

\newcommand{\wtil}{\ensuremath{\widetilde{w}}}
\newcommand{\util}{\ensuremath{\widetilde{u}}}
\newcommand{\rtil}{\ensuremath{\widetilde{r}}}
\newcommand{\uhat}{\ensuremath{{\widehat{u}}}}
\newcommand{\Uhat}{\ensuremath{{\widehat{U}}}}

\newcommand{\Vhat}{\ensuremath{{\widehat{V}}}}
\newcommand{\Sighat}{\ensuremath{{\widehat{\Sigma}}}}
\newcommand{\sighat}{\ensuremath{{\widehat{\sigma}}}}

\newcommand{\Dbar}{\ensuremath{\bar{D}}}

\newcommand{\Hssym}{\ensuremath{H_{\text{ss}}}}
\newcommand{\Deltil}{\ensuremath{\widetilde{\Delta}}}
\newcommand{\Ztil}{\ensuremath{\widetilde{Z}}}
\newcommand{\pitil}{\ensuremath{\widetilde{\pi}}}
\newcommand{\minsep}{\ensuremath{\rho}}

\newcommand{\uinferr}{\ensuremath{\Upsilon}}
\newcounter{ale}

\newenvironment{liste}{\begin{itemize}}{\end{itemize}}
\newcommand{\aliste}{\begin{liste} \setcounter{ale}{1}}
\newcommand{\zliste}{\end{liste}}

\newcommand\x{\times}
\newcommand\bigzero{\makebox(0,0){\text{\huge0}}}
\newcommand*{\bord}{\multicolumn{1}{c|}{}}
 
\usepackage{array}  
\usepackage{multirow,bigdelim}

\newcommand{\ER}{\text{ Erd\H{o}s-R\'{e}nyi} }

\begin{document}

\title{Ranking and synchronization from pairwise measurements via SVD}
% and synchronization over the (real) line
% and time synchronization

\author{Alexandre d'Aspremont\footnotemark[1], Mihai~Cucuringu\footnotemark[2] \footnotemark[4], 
Hemant Tyagi \footnotemark[3]}

\renewcommand{\thefootnote}{\fnsymbol{footnote}}
\footnotetext[1]{CNRS \& Ecole Normale Sup\'erieure, Paris, France. E-mail address: aspremon@ens.fr}
\footnotetext[2]{Department of Statistics and Mathematical Institute, University of Oxford, Oxford, UK. Email: mihai.cucuringu@stats.ox.ac.uk}
\footnotetext[3]{INRIA Lille-Nord Europe, Lille, France; MODAL project-team.  Email: hemant.tyagi@inria.fr}
\footnotetext[4]{The Alan Turing Institute, London, UK. This work was supported by EPSRC grant EP/N510129/1.}
\renewcommand{\thefootnote}{\arabic{footnote}}
% CNRS & D.I., UMR 8548

\maketitle

\begin{abstract}
Given a measurement graph $G= (V,E)$ and an unknown signal $r \in \matR^n$, we investigate algorithms for recovering $r$ from pairwise measurements of the form $r_i - r_j$; $\set{i,j} \in E$. This problem arises in a variety of applications, such as ranking teams in sports data and time synchronization of distributed networks. 
Framed in the context of ranking, the task is to  recover the ranking of $n$ teams (induced by $r$) given a small subset of noisy pairwise rank offsets. We propose a simple SVD-based algorithmic pipeline for both the problem of time synchronization and ranking. We provide a detailed theoretical analysis in terms of robustness against both sampling sparsity and noise perturbations with outliers, using results from matrix perturbation and random matrix theory. Our theoretical findings are complemented by a detailed set of numerical experiments on both synthetic and real data, showcasing the competitiveness of our proposed algorithms with other  state-of-the-art methods.
\end{abstract}

%We consider a simple SVD-based algorithm for the problem of ranking with pairwise comparisons. 
%The setting we consider attributes an underlying strength to each player, and the available pairwise measurements 
%capture a noisy difference of the strength offsets. 
%We provide a detailed theoretical analysis in terms of robustness % of the algorithm 
%against both sampling sparsity and noise perturbations, for two different noise models, using results from matrix perturbation theory 
%and random matrix theory. We augment our theoretical findings with a detailed set of numerical experiments on both synthetic and real data, 
%showcasing the competitiveness of our simple \textsc{SVD-Rank} algorithm and its normalized extension, with other  state-of-the-art methods 
%that come with robustness guarantees. 

\textbf{Keywords:} ranking, angular  synchronization, spectral algorithms, matrix perturbation theory, singular value decomposition, random matrix theory, low-rank matrix completion.

{
  \hypersetup{linkcolor=black}
  \tableofcontents
}

%\vspace{-7mm}
\section{Introduction}

Let $r = (r_1,\dots,r_n)^T \in \matR^n$ be an unknown signal and $G = ([n],E)$ be 
an undirected measurement graph. Given a subset of noisy pairwise measurements of the form $r_i-r_j$ for each $\set{i,j} \in E$, the goal is to estimate the original vector $r$. Clearly, this is only possible only up to a global shift. Moreover, when measurements are exact without any measurement noise, one can recover the strength vector $r$ if and only if the graph $G$ is connected, by simply considering a spanning tree, fixing the value of the root node, and traversing the tree while propagating the information by summing the given offsets. For simplicity, we assume the graph is connected, otherwise it is not possible to estimate the offset values between nodes belonging to different connected components of the graph. 

Instantiations of the above problem are ubiquitous in engineering, machine learning and  computer vision, and have received a great deal of attention in the recent literature. There are two main classes of applications where this problem arises.
\begin{itemize}
 \item \emph{Time synchronization of wireless networks.} 
 % Here, $r_i$ denotes the local clock offset at a particular sensor in a sensor network.
 A popular application arises in engineering, and is known as time synchronization of distributed networks \cite{timeoneSync, timetwoSync}, where clocks measure noisy time offsets $r_i - r_j$, and the goal is to recover $r_1, \ldots, r_n \in \mathbb{R}$.
 \item \emph{Ranking.} A fundamental problem in information retrieval is that of recovering the ordering induced by the latent strengths or scores $r_1, \ldots, r_n \in \mathbb{R}$ of a set of $n$ players, that is best reflected by the given set of pairwise comparisons $r_i - r_j$. We refer the reader to \cite{syncRank} and references therein for a detailed overview.
\end{itemize}

A naive approach by sequentially propagating the information along spanning trees is doomed to fail in the presence of noise, due to accumulation of the errors. To this end, in order to increase the robustness to noise, one typically aims to simultaneously integrate all the pairwise measurements in a globally consistent framework. This line of thought appears in the literature in the context of the \textit{group synchronization} problem \cite{sync}, for recovering group elements from noisy pairwise measurements of their ratios. We briefly review the current literature in group synchronization in Section \ref{sec:relatedWork_sync}, which almost exclusively pertains to synchronization over compact groups. The problem we study in this paper can be construed as synchronization over the real line, hence a non-compact group. There exists a very rich literature on ranking, and it is beyond the scope of our work to provide an extensive review of it. Instead, in Section  \ref{sec:relatedWork_ranking}, we give a brief overview of the literature on ranking and time synchronization, and point out some of the methods that relate to our work. In particular, we focus on spectral methods and approaches that leverage low-rank matrix completion as a pre-processing step.

%------------------------
% Contribution
%-----------------------
%\vspace{-3mm}
% \paragraph{Our contribution is as follows. \noteMC{todo} }
\paragraph{Contributions.} We propose \textsc{SVD-RS}, a simple spectral algorithm for ranking and synchronization from pairwise comparisons, along with a normalized version denoted \textsc{SVD-NRS}\footnote{By normalized, we mean that the measurement matrix is normalized by the degree matrix of the graph, see Section \ref{sec:SVD_algo_analysis}.}. We provide a detailed theoretical consistency analysis for both algorithms for a random measurement model (see Section \ref{sec:problemSetup}) in terms of robustness against sampling sparsity of the measurement graph and noise level. Additionally, we provide extensive numerical experiments on both synthetic and real data, showing that in certain noise and sparsity regimes, our proposed algorithms perform comparable or better than state-of-the-art methods. 

On the theoretical side, our specific contributions can be summarized as follows (see also Section \ref{subsec:summary_theo_results}).
\begin{itemize}
\item For \textsc{SVD-RS}, we provide $\ell_2$  and $\ell_{\infty}$ recovery guarantees for the \textbf{score} vector $r$ (see Theorem \ref{thm:score_rec_main_thm}). For instance, in the setting when $r_i = i$,  $\Omega(n \log n)$ measurements suffice for $\ell_2$ recovery, while $\Omega(n^{4/3} (\log n)^{2/3} )$ measurements suffice for $\ell_{\infty}$ recovery. 

\item The $\ell_{\infty}$ analysis of \textsc{SVD-RS} leads to guarantees for \textbf{rank} recovery in terms of the maximum displacement error between the recovered ranking and the ground truth, as in Theorem \ref{thm:rank_recov_infty}. For e.g., when $r_i=i$,  $\Omega(n^{4/3} (\log n)^{2/3} )$ measurements suffice.

\item For \textsc{SVD-NRS}, we provide in Theorem \ref{thm:score_rec_l2_svdn}, $\ell_2$ recovery guarantees for the \textbf{score} vector $r$, and leave the $\ell_{\infty}$ guarantees for future work, essentially by following a similar, though more intricate, pipeline. Similar to \textsc{SVD-RS}, for $r_i = i$, $\Omega(n \log n)$ measurements suffice for $\ell_2$ recovery.
\end{itemize}

We remark that all the above recovery results hold with high probability. 

%Furthermore, in comparing our results with those from the existing %literature, our $\ell_{\infty}$ guarantees are in contrast with the %typical setting where one obtains $\ell_2$ consistency results. For %instance, in the case of ranking, these essentially lead to a bound %on the Kendall tau distance.

% That considers, for a given input matrix $C$, the following matrix  $ \tilde{L} = D^{-1} C $, with $D_{ii} = \sum_{j=1}^{n} | C_{ij} | $ (so , formula is similar to the Signed Laplacian) - see after real data experiments; make connection with the Connection Laplacian (todo). This one we do not analyze (for now) in some preliminary experiments, it seemed to perform particularly well! 
%
%% \vspace{-1mm} 
%\noindent  $\bullet$ We augment our theoretical analysis with extensive numerical experiments on both synthetic and real data, showing that in certain noise and sparsity regimes, our proposed algorithms perform comparable or better than state-of-the-art methods.
%% , while enjoying robustness guarantees.

%\vspace{-3mm}
%------------------------
% Outline of the paper
%-----------------------
% \paragraph{Outline of the paper.}
\paragraph{Outline.} 
The remainder of this paper is organized as follows.
Section \ref{sec:relatedWork} is a brief survey of the relevant literature, with a focus on group synchronization and ranking.  Section \ref{sec:problemSetup} starts with the formal setup of the problem, presents the gist of our SVD-based approach along with the two algorithms \textsc{SVD-RS}, \textsc{SVD-NRS}, and summarizes our main theoretical results. Section \ref{sec:SVD_algo_analysis}  details and interprets our theoretical results, with the main steps of the proofs outlined in Section \ref{sec:SVD_proofs}. Section \ref{sec:matrixCompletion} discusses the low-rank matrix completion problem and its applicability in the setting of this paper. Section \ref{sec:num_experiments} contains numerical experiments on various synthetic and real data sets. Finally, Section \ref{sec:conclusion} summarizes our findings along with future research directions. The Appendix contains additional technical details for the theoretical results, as well as further numerical experiments.

%------------
% Notation
%-----------
\paragraph{Notation.}
Vectors and matrices are denoted by lower case and upper case letters respectively. 
For a matrix $A \in \matR^{m \times n}$ where $m \geq n$, we denote its singular values by 
$\sigma_1 \geq \dots \geq \sigma_n$ and the corresponding left (resp. right) singular vectors by $u_i$ (resp. $v_i$).  
$\norm{A}_2$ denotes the spectral norm (largest singular value), $\norm{A}_{*}$ denotes the nuclear norm (sum of the singular values), and $\norm{A}_{\max} := \max_{i,j} \abs{A_{ij}}$ denotes the max-norm of $A$.
$\calR(A)$ denotes the range space of its columns. We denote $ e $ to be the all ones column vector. 
The symbol $\calU$ is used to denote the uniform distribution. For positive numbers $a,b$, 
we denote $a \lesssim b$ (resp. $a \gtrsim b$) to mean that there exists an absolute constant $C > 0$ such that $a \leq C b$ (resp. $a \geq C b$).

%\vspace{-3mm}
\section{Related work}   \label{sec:relatedWork}
This section is a brief survey of the relevant literature for the tasks we consider out in this paper. The first part focuses on the group synchronization problem, a relevant instance of it being that of synchronization of clocks arising often in engineering, while the second part surveys the very rich ranking literature, with an emphasis on spectral methods and approaches that leverage low-rank matrix completion.

\subsection{Synchronization literature}   \label{sec:relatedWork_sync}
Much of the engineering literature has focused on  least-squares approaches for solving the time synchronization problem. In its simplest terms, the approach can be formulated as follows. Let $m = |E|$ denote the number of edges in $G$, and $B$ denote the edge-vertex incidence matrix of size $m \times n$. For $l = \set{i,j} \in E$, with $i < j$, the entries $B_{li}$ and $B_{lj}$ are given by $B_{li} = 1$, $B_{lj} = -1$. If $i' \notin \set{i,j}$ then $B_{l i'} = 0$. 
\iffalse
\begin{equation}
B_{li} = \left\{
 \begin{array}{rll}
  1   &   \text{ if } \{i,j\} \in     E,    & \text{ and  } i > j  	\\
 -1   &   \text{ if } \{i,j\} \in     E,    & \text{ and  } i <  j  	\\
  0  &   \text{ if } \{i,j\}  \notin  E  &
  \end{array}
   \right.
\label{incidenceMtxL}
\end{equation}
\fi
Let $w \in \matR^m$ encode the pairwise rank measurements $C_{ij}$, for all edges $\set{i,j} \in E$. The least-squares solution to the ranking problem can be obtained by solving 
\begin{equation}
	% x = L \\ b;
	% \text{minimize}_{x \in \mathbb{R}^n } \;\;\;\;  || B x - w ||_2^2
	\underset{ x \in \mathbb{R}^n }{\text{ minimize } } \;\; || B x - w ||_2^2, 
\label{rank_LS}
\end{equation} 
where $B$ denotes the design matrix with two non-zero entries per row corresponding to the edge $\set{i,j}$ indexed by $l$.
In a related line of work, Hirani et al. \cite{hirani2010least} show that the problem of least-squares ranking on graphs has far-reaching rich connections with seemingly unrelated areas, such as spectral graph theory and multilevel methods for graph Laplacian systems, Hodge decomposition theory and random clique complexes in topology. 

% These connections are to theoretical computer science (spectral graph theory, and multilevel methods for graph Laplacian systems); numerical analysis (algebraic multigrid, and finite element exterior calculus); other mathematics (Hodge decomposition, and random clique complexes); and applications (arbitrage, and ranking of sports teams). Not all of these connections are explored in this paper, but many are. The underlying ideas are easy to explain, requiring only the four fundamental subspaces from elementary linear algebra. One of our aims is to explain these basic ideas and connections, to get researchers in many fields interested in this topic. Another aim is to use our numerical experiments for guidance on selecting methods and exposing the need for further development.

The problem we consider in this paper can also be framed in the context of the \textit{group synchronization} problem, of finding group elements from noisy measurements of their ratios. 
For example,  synchronization over the special orthogonal group $SO(d)$ 
consists of estimating a set of  $n$ unknown $d \times d$  matrices $R_1,\ldots,R_n \in $ SO($d$) from noisy measurements of a subset of the pairwise ratios $R_i R_j^{-1}$ via
\begin{align}
  \underset{R_1,\ldots,R_n \in SO(d)}{\text{minimize}}  \sum_{\set{i,j} \in E}^{ } w_{ij} \|  R_i^{-1} R_j - R_{ij} \|_{F}^{2},  
\end{align}
where $||\cdot||_F$ denotes the Frobenius norm, and $w_{ij}$ are non-negative weights representing the confidence in the  noisy pairwise measurements $R_{ij}$. Spectral and  semidefinite programming (SDP)
relaxations for solving an instance of the above synchronization problem were originally introduced and analyzed by Singer \cite{sync},  
% for angular synchronization over SO(2),   
in the context of angular synchronization over the group SO(2) of planar rotations. 
There, one is asked to estimate $n$ unknown angles $\theta_1,\ldots,\theta_n \in [0,2\pi)$ given $m$ noisy measurements of their offsets $\theta_i - \theta_j \mod 2\pi$. The difficulty of the problem is amplified, on one hand, by the amount of \textbf{noise} in the offset measurements, and on the other hand by \textbf{sparsity} - the fact that $m \ll {n \choose 2}$, i.e., only a very small subset of all possible pairwise offsets are measured. In general, one may consider other groups $\mathcal{G}$ (such as SO($d$), O($d$)) for which there are available noisy measurements $g_{ij}$ of ratios between the group elements
\begin{equation}
g_{ij} = g_i g_j^{-1}, g_i, g_j \in \mathcal{G}.  %  \nonumber
\end{equation}
The set $E$ of pairs $\set{i,j}$ for which a ratio of group elements is available can be realized as the edge set of a graph $G=(V,E)$, $|V|=n, |E|=m$, with vertices corresponding to the group elements $g_1,\ldots,g_n$, and edges to the available pairwise measurements $ g_{ij} = g_i g_j^{-1}$. 
As long as the group $\mathcal{G}$ is compact and has a real or complex representation, one may construct a real or Hermitian matrix (which may also be a matrix of matrices) where the element in the position $\set{i,j}$ is the matrix representation of the measurement $g_{ij}$ (possibly a matrix of size $1 \times 1$, as is the case for $\mathbb{Z}_2$), or the zero matrix if there is no direct measurement for the ratio of $g_i$ and $g_j$. For example, the rotation group SO(3) has a real representation using $3 \times 3$ rotation matrices, and the group SO(2) of planar rotations has a complex representation as points on the unit circle. 

%The seminal paper of Singer \cite{sync} considered the angular synchronization problem %over SO(2), which enjoys the property that $2 \times 2$ rotation matrices have a %convenient representation as complex numbers, where the goal is to recover the unknown %ground truth angles $\theta_1, \ldots, \theta_n  \in [0, 2 \pi)$, given  noisy pairwise %angle offsets
%$$ \Theta_{ij} = (\theta_i - \theta_j) \; \mod  \; 2 \pi. $$

%Unlike the examples above, 
The setting we consider in this paper, namely that of recovering points on the real line from a subset of noisy pairwise differences, is essentially synchronization over the non-compact group $\mathbb{R}$.
In recent work in the context of ranking from pairwise cardinal and ordinal measurements  \cite{syncRank}, the real line was compactified by wrapping it over the upper half of the unit circle,  making the problem amenable to standard synchronization over the compact group SO($2$). The estimated solution allowed for the recovery of the player rankings, after a post-processing step of modding out the best circular permutation. Note that the proposed approach only focused on recovering the individual rankings, and not the magnitude (i.e., strength) of each player, as we propose to do in this paper.

Very recently, Ozyesil et al. \cite{ozyesil2018synchronization} proposed an approach that allows one to consider the synchronization problem in the setting of non-compact groups. The authors leverage a compactification process, that relies on a mapping from a non-compact domain into a compact one, for  solving the synchronization problem. The contraction mapping enables one to transform measurements from a Cartan motion group to an associated compact group, whose synchronization solution provides a solution for the initial synchronization problem over the original domain.

%------------------------
% Ranking literature
%------------------------
\subsection{Ranking literature}   \label{sec:relatedWork_ranking}
Fogel et al. introduced \textsc{Serial-Rank} \cite{serialRank, fogel2016spectral}, a ranking algorithm that explicitly leverages the connection with seriation, a classical ordering problem that considers the setting where the user has available a similarity matrix (with $\pm 1$ entries) between a set of $n$ items, and assumes that there exists an underlying one-dimensional ordering such that the similarity between items decreases with their distance. 
%The pairwise similarity will be reflected in the proposed ranking solution, with similar items ending up close to each other in the ranking. The authors propose an efficient spectral algorithm with provable recovery and robustness guarantees. More specifically,
The authors propose a spectral algorithm \textsc{Serial-Rank}, and show that under an Erd\"os-Renyi  random graph model and a given noise model, it recovers the underlying ranking with $\Omega(n^{3/2} \log^4 n)$ comparisons where the estimation error is bounded in the $\ell_{\infty}$ norm. 
%in the maximum displacement error is certain assumptions on the pattern of noisy entries %is able to perfectly recover the underlying true ranking. 
%in scenarios where a fraction of the comparisons are either corrupted by noise and/or %completely missing.

A popular theme in the literature is the rank aggregation setting (where players meet in multiple matches), where there exists a latent probability matrix $P \in [0,1]^{n \times n}$, wherein $P_{ij}$ denotes the probability that player $i$ defeats player $j$, with $P_{ij} + P_{ji} = 1$. For each pair $(i,j)$, one observes $Y_{ij} \in \{0,1\} \sim \text{Bern}(P_{ij})$. A number of so-called Random Utility models have been considered in the literature, starting with the seminal work of Bradley-Terry-Luce (BTL) \cite{BradleyTerry1952}, which is by far the most popular model considered. In the most basic version of the BTL model, the probability that player $i$ beats player $j$ is given by $P_{ij} = \frac{w_i}{w_i + w_j}$, where the vector $w \in \mathbb{R}_{+}^n$ is the parameter vector to be inferred from the data, with $w_i$ being a proxy for the score or strength associated to player $i$. In the Thurstone model \cite{RUMmodel}, $P_{ij} = \Phi(s_i - s_j)$, where $\Phi$ denotes the standard normal cumulative distribution function (CDF), and $s \in \mathbb{R}^n$ is the score vector.  
Negahban et al. proposed  \textsc{Rank-Centrality} in the context of the rank aggregation problem from multiple ranking systems, under a Bradley-Terry-Luce (BTL) model \cite{RankCentrality, Negahban_RankCentrality_2017}. Rank-Centrality is an iterative algorithm that estimates the ranking  scores from the stationary distribution of a certain random walk on the graph of players. Edges encode the outcome of pairwise comparisons, and are directed $i \xrightarrow{} j $ towards the winner, where the weight  captures the proportion of times $j$ defeated $i$ in case of multiple direct matches. 
The authors show that under some assumptions on the connectivity of the underlying graph, their algorithm estimates the underlying score vector of the BTL model with $\Omega(n \text{poly}(\log n))$ comparisons. 
Very recently, Agarwal et al. \cite{Accelerated_Spectral_Ranking} propose Accelerated Spectral Ranking, a provably faster spectral ranking algorithm in the setting of the multinomial logit (MNL) and BTL models. The authors considered a random walk that has a faster mixing time than the random walks associated with previous algorithms (including Rank-Centrality \cite{Negahban_RankCentrality_2017}), along with improved sample complexity bounds (of the order $n \text{poly} (\log n)$) for recovery of the MNL and BTL parameters.
% http://proceedings.mlr.press/v80/agarwal18b.html

%In the rank aggregation setting, the user typically observes $k$ measurements of match %outcomes between players $i$ and $j$. Alternatively, in certain sports competitions, each %direct match or pairwise comparison is evaluated independently by a rating system or a %member of a jury of size $k$. A relevant line of work in this setting is that of Negahban %et al. \cite{RankCentrality}, who proposed an iterative algorithm for the rank %aggregation problem that considers the ranking induced by the stationary distribution of %a certain random walk on the graph of players, where each edge encodes the outcome of all %the $k$ different pairwise evaluations. At each iteration of the random walk, the %probability of transitioning from vertex $i$ to vertex $j$ is directly proportional to %how often player $j$ beat player $i$ across all the $k$ matches between the two players, %and is zero if the two players have never played before.  The stationary distribution of %the associated Markov Chain thus encodes the skill level of each player, since the random %walk has a higher chance of transitioning towards the most skillful players. Such an %interpretation can be traced back to earlier works in the network centrality literature, %such as the popular PageRank algorithm \cite{Pageetal98}.

Cucuringu \cite{syncRank} introduced \textsc{Sync-Rank}, formulating the problem of ranking with incomplete noisy information as an instance of the group synchronization problem over the group SO(2) of planar rotations. 
%, whose usefulness has been demonstrated in numerous applications in recent years in %computer vision and graphics, sensor network localization and structural biology. Its %least squares solution can be approximated by either a spectral or a SDP relaxation,  %followed by a rounding procedure. 
\textsc{Sync-Rank} starts by making the ansatz that items are embedded on the upper-half  of the unit circle in the complex plane, and transforms the pairwise input measurements into pairwise angle offsets. It then  solves the spectral or SDP relaxation of synchronization,  and finally chooses amongst the circular permutations of the ranking  induced by the synchronized solution by minimizing the number of upsets. The method was shown to compare favorably with a number 
of state-of-the-art methods, including \textsc{Rank-Centrality} \cite{RankCentrality}, \textsc{Serial-Rank} \cite{fogel2016spectral},  and a SVD-based ranking algorithm similar to the one we consider in this paper, but without the projection step and its normalized extension,  and without any theoretical guarantees (see Remark \ref{rem:comp_cucu_svd} for details). 
%(see Remark \ref{rem:OldSVD} for details).
Along the same lines, and inspired by the group synchronization framework, Fanu\"el and Suykens \cite{FANUEL2017_JACHA} propose a certain deformation of the combinatorial Laplacian, in particular, the so-called dilation Laplacian whose spectrum is leveraged for ranking in directed networks of pairwise comparisons. The method is reported to perform well on both synthetic and real data, and enjoys the property that it can place an emphasis of the top-$k$ items in the ranking.

More recently, De Bacco et al. \cite{CaterinaDeBacco_Ranking} proposed \textsc{Spring-Rank}, an efficient physically-inspired algorithm, based on solving a linear system of equations, for the task of inferring hierarchical rankings in directed networks, which also comes with a statistical significance test for the inferred hierarchy. The model considered also incorporates the assumption that interactions are more likely to occur between individuals with similar ranks, which is sometimes the case in real applications, such as sport competitions. 
The authors compare \textsc{Spring-Rank} to a wealth of algorithms, including the above-mentioned \textsc{Sync-Rank} and \textsc{Serial-Rank}, in terms of an accuracy defined as the fraction of edges whose direction is consistent with the inferred ranking, and conclude that, on average across various synthetic and real data sets, \textsc{Spring-Rank} and \textsc{Sync-Rank} have the highest accuracy.

% 3.
Volkovs and Zemel \cite{Volkovs_JMLR} consider two instances (in the unsupervised and  supervised settings) of the preference aggregation problem where the task is to combine multiple preferences (either in the form of binary comparisons or as score differences) over objects into a single consensus ranking. Both settings rely on a newly introduced Multinomial Preference model (MPM) which uses a multinomial generative process to model the observed preferences. 
% 4.
Dalal et al. \cite{dalal2012multi} consider a multi-objective rank aggregation problem (where pairwise information is in the form of score differences), and rely on the framework of the combinatorial Hodge decomposition \cite{jiang2011statistical}. In particular, they first formulate the problem of reducing global inconsistencies and then propose techniques for identifying local observations which can reduce such global inconsistencies.

%Finally, we also recall the approach of Chen and Cand{\`e}s, who considered in %\cite{chen2018projected} a setup similar to the one we  consider in this paper, of %recovering $n$ discrete variables $x_i \in {1, \ldots, m,  1 \leq i \leq n}$ given noisy %observation of  their \textit{modulo}  differences $ \{ x_i - x_j \mod m \}$. They %proposed a lifting procedure of each $x_i$ to higher dimensions, considered the resulting %constrained quadratic program which they optimize via a projected power method, following %an initial guess obtained via a spectral method/low-rank factorization.

%------------------------------------------------------
% Low-rank matrix completion approaches for ranking.
%------------------------------------------------------
\paragraph{Low-rank matrix completion approaches for ranking.}
% 1
Gleich and Lim \cite{gleich2011rank} proposed a new method for ranking a set of items from pairwise observations in the form $r_i - r_j$. They first use matrix completion to fill in a partial skew-symmetric matrix, followed by a simple row-sum to recover the rankings. Using standard recovery results from matrix completion literature, it follows that in the absence of noise, and under some coherence assumptions on the score vector $r$, the true ranking can be recovered exactly from a random subset of pairs of size $\Omega(n \text{poly}(\log n))$.  In case the observations are given as a rating matrix (users and items), they show how to transform it into a pairwise comparison matrix involving score differences so that their method can be applied. 
% 2
Ye et al. \cite{ye2012robust} proposed a rank-minimization approach to aggregate the predicted confidence scores of multiple models. The authors cast the score fusion problem as that of finding a shared rank-2 pairwise relationship matrix that renders each of the original score matrices from the different models to be decomposed into the common rank-2 matrix and the sparse error component.

%%  Ranking + Matrix Completion:
  % todo: further google for similar instances applied to ranking?
The recent work of  Levy et al \cite{levy2018ranking} relies on matrix completion in the context of the rank aggregation problem from a small subset of noisy pairwise comparisons, wherein the user observes $L$ repeated comparisons  between the same set of players. The authors work in the setting of the Bradley-Terry-Luce model (BTL) that assumes that a set of latent scores (or strengths) underlies all players, and each individual pairwise comparison is a probabilistic outcome, as a function of the underlying scores. The resulting pipeline is reported to improve over state-of-the-art in both simulated scenarios and real data. 
Prior to that, Kang et al.  \cite{kang2016top} also relied on matrix completion and proposed an algorithm for top-$N$ recommender systems. The authors filled in the user-item matrix based on a low-rank assumption, by  considering a nonconvex relaxation, as opposed to the usual nuclear  norm, and argued that it provides a better rank approximation and empirically leads to an accuracy superior to that of any state-of-the-art algorithm for the top-$N$ recommender problem, on a comprehensive set of experiments on real data sets.
%   title={Top-n recommender system via matrix completion},

% davenport20141
Another relevant line of work is that of Massimino and Davenport \cite{massimino2013one}, who consider an adaptation of the one-bit matrix completion framework \cite{davenport20141} to the setting of pairwise comparison matrices. One observes  measurements $Y_{ij}  = \pm 1$ with probability $f(M_{ij})$, where $M_{ij} = r_i - r_j$, and $f(M_{ij}) = \mathbb{P}(M_{ij} > 0)$; for example, the authors consider $f(x) = (1 + e^{-x})^{-1}$.
On a related note, Yang and Wakin \cite{Yang_Wakin} considered the rank aggregation setting, extending the work of Gleich and Lim \cite{gleich2011rank} to the setup of non-transitive matrices, for which it does not necessarily hold true that $ Y_{i,j} =  Y_{i,k} +  Y_{k,j}, \forall i,j,k$, as is the case in the special setting when $ Y_{i,j} = r_i - r_j $ (and thus $Y = r e^T - e r^T$), for a score vector $r$. In particular, Yang and Wakin were interested in modeling and recovering $Y$ itself as opposed to the one-dimensional ranking, and introduced a model for non-transitive pairwise comparisons $Y_{ij} = r_i a_j - r_j a_i$, where, for example, $r_i$ could denote offensive strength of player $i$, and $a_j$ defensive strength of player $j$, thus giving  $Y_{ij}$ the interpretation of the anticipated margin of victory for player $i$ over player $j$. 
The authors then propose a low-rank matrix completion approach based on alternating minimization, along with  recovery guarantees for the estimate $ \hat{Y} $ of the form $|| \hat{Y} - Y ||_F \leq  \epsilon$, after $\log 1/ \epsilon$ iterations.  
%\noteMC{------------ UPDATED -------------} 
% setting of recovering a non-transitive  pairwise comparison matrix, where $Y = r e^T - e r^T$, via  $\hat{Y}$ such that $|| \hat{Y} - Y ||_F \leq  \epsilon$ and the final output is given by $ \hat{S} = \frac{1}{n} Y e$. \noteHT{This paragraph could be written  more precisely with more detail since they are very relevant to the setting of our paper.}

The recent seminal work of Rajkumar and Agarwal \cite{rajkumar_agarwal_2016} considered the question of recovering rankings from $O(n \log n )$  comparisons between the items. Their proposed pipeline starts by mapping the input binary entries of the comparison matrix to a low-rank matrix. More specifically, a \textit{link function} is applied entry-wise to the input comparison matrix, which renders the resulting incomplete matrix to be of low rank. Standard low-rank matrix completion is applied to complete the matrix, followed by the inverse link function applied entry wise, which effectively maps the data back to binary pairwise measurements. Finally, the Copeland ranking procedure \cite{copeland} is applied, that ranks items by their Copeland scores, which is effectively given by the number of wins in the comparison  matrix after  thresholding the entries with respect to $0.5$. This approach can be seen as a generalization of the pipeline proposed by Gleich and Lim \cite{gleich2011rank}, in the sense that both lines of work rely on matrix completion followed by row-sum, the main difference being that  \cite{rajkumar_agarwal_2016} uses the link function that maps the binary input matrix to a low-rank one. 
% Learning from comparisons and choices by Neghaban, Oh et al. 
Borrowing tools from  matrix completion, restricted strong convexity and prior work on Random Utility Models, Negahban et al. \cite{negahban2017learning}  propose a nuclear norm regularized optimization problem for learning the  parameters of the MultiNomial Logit (MNL) model that best explains the data, typically given as user preference  in the form of choices and comparisons. The authors show that the convex relaxation for learning the MNL model is minimax optimal up to a logarithmic factor, by comparing its performance to a fundamental lower bound.

%It is important to stress that the majority of the papers above work in the setting where %each pair $\{i,j\}$ is observed $k$ times, for sufficiently large $k$, while the SVD-based %approach we propose and analyze in this paper only works with a single observation of each %of the available pairwise measurements. Furthermore, in our setting, we observe a subset of %noisy cardinal measurements $R_{ij} = r_i - r_j$, unlike in the typical scenarios arising %in ranking where one observes one or multiple ordinal measurements, i.e. binary $\{0,1\}$ %outcomes,  generated from an underlying probabilistic  model. \noteHT{This last para can be %removed I think...}

% HT: commented out above para as it doesnt say anything useful I think...
% One can potentially recover one type of measurement from the other, 

% Problem setup (Noise models etc)
%----------------------------------
% Problem setup and main results
%----------------------------------
\section{Problem setup and main results}  \label{sec:problemSetup}
Our formal setup is as follows. Consider an undirected graph $G = ([n], E)$ and 
an unknown vector $r \in \matR^n$, where $r_i$ is the score associated with node $i$. In particular, 
$G$ is assumed to be a $G(n,p)$, i.e., the popular Erd\H{o}s-R\'{e}nyi random graph model, where edges between 
vertices are present independently with probability $p$. Moreover, we assume $r_i$ to be bounded uniformly, i.e., 
$r_i \in [0,M]$ for each $i$. Hence $r_i - r_j \in [-M,M]$ for all $i,j$. 
Requiring $r_i \geq 0$ is only for convenience and w.l.o.g. Importantly, $M$ is \emph{not assumed} to be known to the algorithm.

For each $\set{i,j} \in E$, we are given noisy, independent measurements $R_{ij}$ where
\begin{equation}
\hspace{-3mm}  
R_{ij} = \left\{
 \begin{array}{rl}
 r_i - r_j; & \text{ w.p } \eta \\
 \sim \calU [-M,M];  & \text{ w.p } (1-\eta).	
     \end{array}
   \right.
\label{ERoutliers}
\end{equation}
The parameter $\eta \in [0,1]$ controls the level of noise; we will denote the noise level explicitly by $\gamma = 1-\eta$. It is important to note that the parameters $\eta,p$ are not assumed to be known to the algorithm. This model will be referred to 
as the \textbf{Erd\H{o}s-R\'{e}nyi Outliers} model, or in short, ERO($n,p, \eta$).
It was also considered in previous works in ranking \cite{syncRank} and  angular synchronization \cite{sync}, and is amenable to a theoretical analysis. Alternatives to this noise model include the multiplicative uniform noise model, as considered in  \cite{syncRank}.
Our goal is two fold - we would like to recover\footnote{Clearly, this is possible only up to a global shift.} the score vector $r$, and also the ranking $\pi$ induced by $r$. 
\begin{remark}
The above model is only for the purpose of theoretically analyzing the statistical performance of our methods. Other statistical models could also be considered of course, and we will see that our methods are completely model-independent. 
\end{remark}

%-------------------------------------------
% Main idea: SVD based spectral algorithms
%-------------------------------------------
\subsection{Main idea: SVD-based spectral algorithm} \label{sec:main_idea}
We start by forming the measurement matrix $H \in \matR^{n \times n}$, where
\begin{itemize}
  \setlength\itemsep{0em}
\item $H_{ii}=0, \forall i = 1, \ldots, n$, 
\item $H_{ij} = R_{ij} $ and $H_{ji} = -R_{ij}$, if $\set{i,j} \in E$, and
\item $H_{ij} = 0$, if $\set{i,j} \notin E$.
\end{itemize}
If $G$ is the complete graph and the measurements are noise free, it holds true that $H = r e^T - e r^T$, which is a rank $2$ skew-symmetric matrix. Denoting $\alpha = \frac{r^T e}{n}$, one can verify that the two non-zero left singular vectors are $u_1 = e/\sqrt{n}, u_2 =  \frac{r-\alpha e}{|| r-\alpha e ||_2}$ with equal non-zero singular 
values $\sigma_1 = \sigma_2 = \norm{r-\alpha e}_2 \sqrt{n}$ (see Lemma \ref{lem:sing_vals_C}). 
Therefore, given any orthonormal basis for span$\set{u_1,u_2}$, we can simply find a vector orthonormal to $e/\sqrt{n}$ and which lies in span$\set{u_1,u_2}$; this will give us candidate solutions $\pm \frac{r-\alpha e}{|| r-\alpha e ||_2}$. 
Multiplying these candidates by $\sigma_1/\sqrt{n}$ recovers the scale information of $r$, while the sign 
ambiguity is resolved by selecting the candidate which is most consistent with the measurements. 
If one is interested only in ranking the items, then there is of course no need to estimate the scale above.

If $G$ is not complete (thus there are missing edges) and the measurements are noisy, then
$H$ will typically not be rank $2$, but as will see shortly, it can be thought of as a perturbation of the rank-2 matrix $\eta p (r e^T - e r^T)$. 
Proceeding similarly to the noiseless case, we can find the top two singular vectors of $H$ (denoted by $\hat u_1, \hat u_2$),  project $u_1 = e/\sqrt{n}$ on to span$\set{\hat u_1, \hat u_2}$ to obtain $\bar u_1$, and 
then find a unit vector in span$\set{\hat u_1, \hat u_2}$ which is orthonormal to 
$\bar u_1$ (call this $\tilde u_2$). If the noise level is not too large and we have sufficiently many 
edges in $G$, then one can imagine that $\tilde u_2$ will be close (up to a sign flip) to $u_2$, and $\est{\sigma_i} \approx \eta p \sigma_i$. Hence we can recover the ranking from $\tilde u_2$ after resolving the sign ambiguity (as mentioned earlier).  
This also implies that the centered version of $\frac{\est{\sigma_1} \tilde u_2}{\eta p \sqrt{n}}$ will be close (up to a sign flip) to $r - \alpha e$. But since we do not know $\eta,p$, we will resort to other data-driven approaches for recovering the scale parameter, see Section \ref{sec:scaleRecovery}. 

The above approach is outlined formally as Algorithm \ref{algo:SVD_Rank_sync}, namely \textsc{SVD-RS} (SVD-Ranking and Synchronization) which is a spectral method 
for recovering the ranks and scores (up to a global shift) of a collection of $n$ items. Additionally, 
we also consider a ``normalized'' version of \textsc{SVD-RS}, wherein $H$ is replaced by 
$\Hssym = \Dbar^{-1/2} H \Dbar^{-1/2}$ with $\Dbar$ being a diagonal matrix and 
$\Dbar_{ii} = \sum_{j=1}^n \abs{H_{ij}}$. Clearly, $\Hssym$ is also skew symmetric. Such a normalization step 
is particularly useful when the degree distribution is skewed, and is commonly employed in other problem domains 
involving spectral methods, such as clustering (see for eg. \cite{kunegis2010spectral}). The ensuing 
algorithm, namely \textsc{SVD-NRS} (SVD-Normalized Ranking and Synchronization), is outlined as Algorithm \ref{algo:SVDN_Rank_sync}.

%---------------------------------------------------
% Algorithm outline
%---------------------------------------------------
%\vspace{-3mm}
\begin{algorithm}[!ht]
\caption{\textsc{SVD-RS}} \label{algo:SVD_Rank_sync} 
\begin{algorithmic}[1] 
\State \textbf{Input:} Measurement graph $G = ([n], E)$ and pairwise measurements $R_{ij}$ for $\set{i,j} \in E$. 
 
\State \textbf{Output:} Rank estimates: $\est{\pi}$ and score estimates $\est{r} \in \matR^n$.
%\hrulefill

%\textsc{// Stage 1: Compute the SVD}  

\State Form measurement matrix $H \in \matR^{n \times n}$ using $R_{ij}$ as outlined in Section \ref{sec:main_idea}.

\State Find the top two left singular vectors (resp. singular values) of $H$, namely $\hat u_1, \hat u_2$ (resp. $\est{\sigma}_1,\est{\sigma}_2$).  

%\textsc{// Stage 2: Projection }
% \textsc{// Stage 2: Projection and Recovery of the denoised real-valued strength/rankings} $f$.

\State Obtain vector $\bar{u}_1$ as the  orthogonal projection of $u_1 = e/\sqrt{n}$ on to $\text{span}\set{\hat u_1, \hat u_2}$.

\State Obtain a unit vector $\util_2 \in \text{span}\set{\hat u_1, \hat u_2}$ such that $\util_2 \perp \bar{u}_1$. \label{eq:u2tilde_algo}

\State \textbf{Rank recovery:} Obtain ranking $\widetilde{\pi}$ induced by $\tilde{u}_2$ (up to global sign ambiguity) and reconcile its global sign by minimizing the number of upsets. Output ranking estimate $\est{\pi}$. \label{step:rank_recv_svd}

\State \textbf{Score recovery:} Use $\tilde u_2$, $H$ to recover the scale $\tau \in \mathbb{R}$ as in Section \ref{sec:scaleRecovery}. Output $\est{r} = \tau \util_2 -  \frac{e^T (\tau \util_2)}{n} e$. \label{step:score_recv_svd}

%\State \textbf{Recovering $r$:} Obtain $\tilde{w} = \frac{\est{\sigma_1}}{\eta p \sqrt{n}} \tilde u_2$ and $\est{r} = \tilde{w} - \frac{e^T \tilde{w}}{n} e$.
%				Reconcile global sign of $\est{r}$ by forming $\est{R} \in \matR^{n \times n}$, where 
%				$\est{R}_{i,j} = \est{r}_{i} - \est{r}_j$, and choosing $\est{r}$ which minimizes $\sum_{i,j} \abs{\est{R}_{i,j} - H_{i,j}}^2$. \label{step:score_recv_svd}

\end{algorithmic}
\end{algorithm}

% 
%--------------------------------
% NormalizedAlgorithm: SVD-NRS
%--------------------------------
%\vspace{-4mm}
\begin{algorithm}[!ht]
\caption{\textsc{SVD-NRS}} \label{algo:SVDN_Rank_sync} 
\begin{algorithmic}[1] 
\State \textbf{Input:} Measurement graph $G = ([n], E)$ and pairwise measurements $R_{ij}$ for $\set{i,j} \in E$. 
 
\State \textbf{Output:} Rank estimates: $\est{\pi}$ and score estimates $\est{r} \in \matR^n$.
%\hrulefill

%\textsc{// Stage 1: Compute the SVD}  

\State Form $\Hssym = \Dbar^{-1/2} H \Dbar^{-1/2}$ with $H$ formed using $R_{ij}$ as outlined in Section \ref{sec:main_idea}. 
$\Dbar$ is a diagonal matrix with $\Dbar_{ii} = \sum_{j=1}^n \abs{H_{ij}}$.

\State Find the top two left singular vectors (resp. singular values) of $\Hssym$, namely $\hat u_1, \hat u_2$ (resp. $\est{\sigma}_1,\est{\sigma}_2$).  

%\textsc{// Stage 2: Projection }
% \textsc{// Stage 2: Projection and Recovery of the denoised real-valued strength/rankings} $f$.

\State Obtain vector $\bar{u}_1$ as the  orthogonal projection of $u_1 = \frac{\Dbar^{-1/2} e}{\norm{\Dbar^{-1/2} e}_2}$ on to $\text{span}\set{\hat u_1, \hat u_2}$.

\State Obtain a unit vector $\util_2 \in \text{span}\set{\hat u_1, \hat u_2}$ such that $\util_2 \perp \bar{u}_1$. \label{eq:u2tilde_algo_svdn}

\State \textbf{Rank recovery:} Obtain ranking $\widetilde{\pi}$ induced by $\Dbar^{1/2} \util_2$ (up to global sign ambiguity) and reconcile its global sign by minimizing the number of upsets. Output ranking estimate $\est{\pi}$. 

\State \textbf{Score recovery:} Use $\util_2$, $H$ to recover the scale $\tau \in \mathbb{R}$ as in Section \ref{sec:scaleRecovery}. Output $\est{r} = \tau \Dbar^{1/2} \util_2 -  \frac{e^T (\tau \Dbar^{1/2} \util_2)}{n} e$. \label{eq:score_rec_algo_svdn}

%\State \textbf{Recovering $r$:} Obtain $\tilde{w} = \frac{\est{\sigma_1}}{\eta p \norm{\Dbar^{-1/2} e}_2} \Dbar^{1/2} \tilde{u}_2$ and 
%$\est{r} = \tilde{w} - \frac{e^T \tilde{w}}{n} e$.
%Reconcile global sign of $\est{r}$ by forming $\est{R} \in \matR^{n \times n}$, where 
%$\est{R}_{i,j} = \est{r}_{i} - \est{r}_j$, and choosing $\est{r}$ which minimizes $\sum_{i,j} \abs{\est{R}_{i,j} - H_{i,j}}^2$. \label{eq:score_rec_algo_svdn}

\end{algorithmic}
\end{algorithm}

\begin{remark}
Algorithms \ref{algo:SVD_Rank_sync} and \ref{algo:SVDN_Rank_sync} are model-independent, and only require as input (a) the measurement graph $G$, and (b) the corresponding pairwise measurements $R_{ij}$. 
\end{remark}

\begin{remark} \label{rem:comp_cucu_svd}
Note that the SVD-based algorithm introduced in \cite{syncRank} considers the four possible rankings induced by the singular vectors $\set{\pm \est{u}_1, \pm \est{u}_2}$, and chooses the one that minimizes the number of upsets. In contrast, SVD-RS first computes $\bar{u}_1$ as the orthogonal projection of the all ones vector $e$ on to $\text{span}\set{\est{u}_1, \est{u}_2}$, finds a vector $\util_2 \in \text{span}\set{\est{u}_1, \est{u}_2}$ perpendicular to $\bar{u}_1$, and finally extracts the ranking induced by $\util_2$. Furthermore, unlike \cite{syncRank}, SVD-NRS introduces an additional normalization step of the measurement matrix $H$ to alleviate  potential issues arising from skewed degree distributions. Such a normalization step is fairly common for spectral methods in general. For example, this operator was considered in \cite{asap2d} in the context of angular synchronization \cite{sync} and the graph realization problem, and \cite{Singer_Hautieng_VDM} who introduced Vector Diffusion Maps for nonlinear dimensionality reduction and explored the interplay with the Connection-Laplacian operator for vector fields over manifolds.
\end{remark}

\paragraph{Computational complexity.}  Our proposed SVD-based methods rely on computing the top two singular vectors of the input matrix which in practice can be done very efficiently. In general, the leading singular values and singular vectors can be computed  using iterative techniques at a typical cost of $O(p n^2)$, using for example, a simple power method (since all iterations only require a matrix-vector product at a cost of $ p n^2 $, where $p$ is the sampling probability). In the sparse setting, the computational cost is essentially linear in the number of nonzero entries in the matrix, thus making the approach scalable to large measurement graphs. The data-driven approach for scale recovery will also require $O(p n^2)$ cost, as will be seen in the next Section \ref{sec:scaleRecovery}.

\begin{remark}
The computational complexity of our proposed SVD-based algorithms is on par with that of other competitive methods, detailed and compared against in the numerical experiments Section \ref{sec:num_experiments}. For example, the Serial-Rank approach \cite{serialRank} first computes a shifted version of the similarity matrix $H H^T$, then considers the graph Laplacian, and computes its Fiedler eigenvector, much like our SVD-based approach, which can be done at a typical cost of $O(p n^2)$. The Row-Sum method computes the sum of the entries in each row of $H$, which has complexity $O(p n^2)$, though Gleich and Lim \cite{gleich2011rank} first rely on a matrix completion step to fill in a partial skew-symmetric matrix, before computing the row sums. However, performing a low-rank matrix completion step can be computationally expensive, depending on the approach used; in general, alternating minimization has been proven empirically to be one of the most accurate and efficient methods for matrix completion \cite{jain2012lowrank}. 
In the numerical experiments reported in Section \ref{sec:num_experiments}, we ran into computational issues when running matrix completion for values of $n$ significantly larger than $1000$, even for very sparse graphs (eg, $p = 0.01$).

For the Least-Squared-based ranking approach, also considered in \cite{syncRank}, the setup is the following. Each (potentially noisy) pairwise comparison $r_i - r_j$ contributes with a row to a design matrix $T$ of size $m \times n$, where $m$ ($\approx pn^2$) denotes the number of edges in the graph $G$. Each row of $T$ has only two non-zero elements, in particular, a $+1$ in column $i$, and a $-1$ in column $j$. To estimate the score vector $s$ of length $n$, we solve the linear  system $ T s = b $ in the least-squares sense, where $b$ is a column vector with the outputs of the pairwise comparisons. One potential approach for solving such linear least-squares problems of the form $T s = b$ is to use conjugate gradient iterations applied to the normal equations $T^T T s = T^T b$ (which can be achieved  without explicitly performing the expensive calculation of the matrix $T^T T$). Note that the rate of convergence of the gradient iterations is dictated by the condition number $\kappa$ of the matrix $T^TT$, and the number of iterations required for convergence is $O(\sqrt{\kappa})$ \cite{trefethen97}. For matrices that are sparse or have exploitable structure (like the sparse matrix $T$ with only two nonzero entries per row), each conjugate gradient iteration has complexity as low as $O(m)$.  
% Overall, the complexity of the Least-Squares-based ranking approach is $O(m \sqrt{\kappa})$.

% highlight they have similar computational complexity, and where we do better, if any. 
\end{remark}

%
%------------------------------------------
% Data driven methods for scale recovery
%------------------------------------------
\subsection{Recovering the global scale: data-driven procedures}
\label{sec:scaleRecovery}

To recover the scaling factor $\tau \in \matR$,  we consider two possible approaches: a first one based on a median estimator, and a second one that relies on a least-squares estimator. We briefly review both methods, and remark that we only consider the median-based estimator throughout the rest of the paper, due to its additional robustness. We compare results on two problem instances of varying levels of noise, and also compare with the ground truth, as detailed below.

%\vspace{-1mm}
\begin{figure}[!ht]
% \small
% \captionsetup[subfigure]{labelformat=empty}
\captionsetup[subfigure]{skip=1pt}
% \hspace{-4mm}
\centering
\subcaptionbox[]{ Histogram of pairwise ratios in absolute value (on a log scale), for $\gamma = 0.02$.
% \label{subfig:sublabel3}
}[ 0.324\columnwidth ]
{\includegraphics[width=0.324\columnwidth] {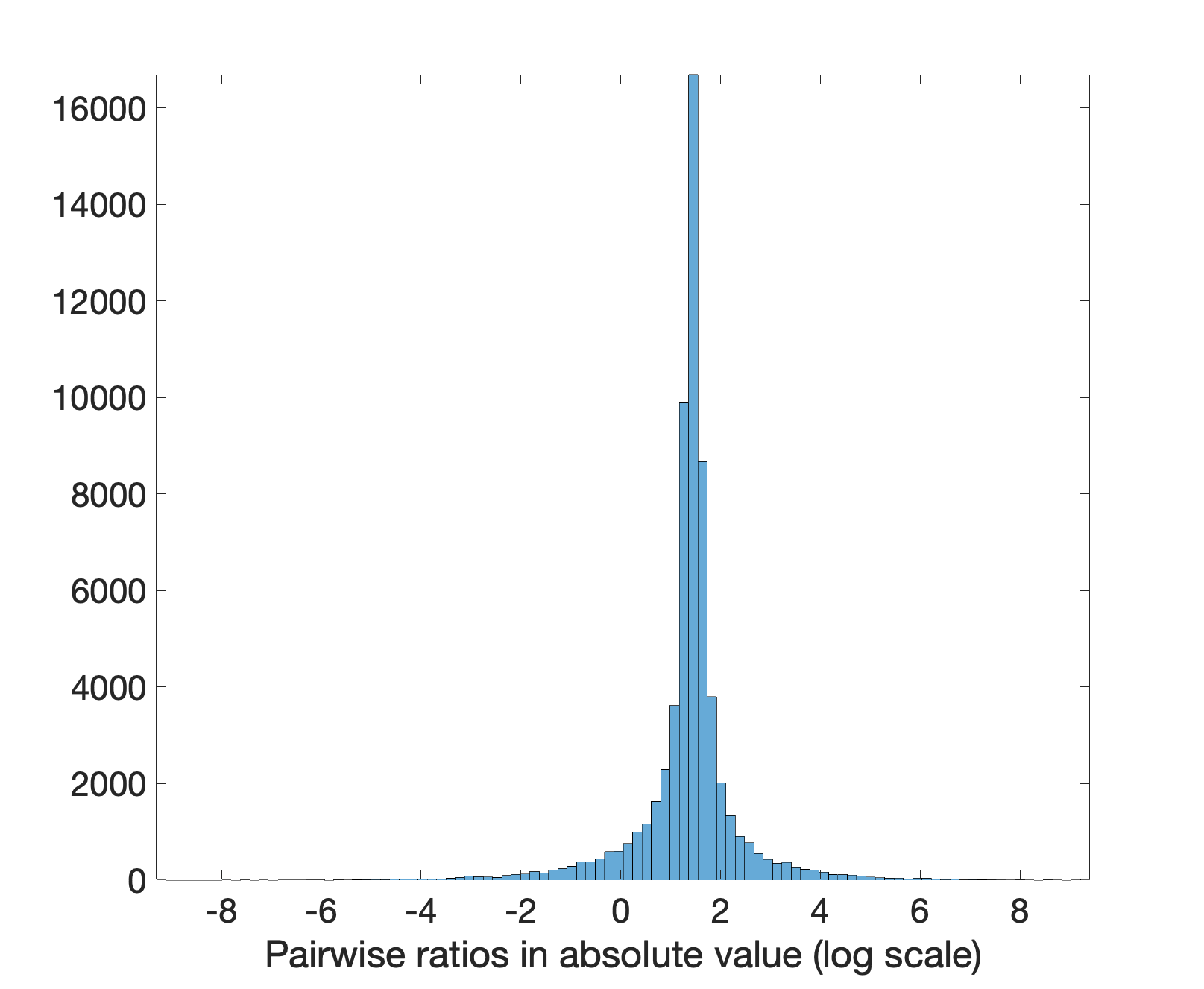} }
% \hspace{0.01\columnwidth} % seperation
%
\subcaptionbox[]{ Histogram of pairwise ratios in absolute value, for $\gamma = 0.02$.
% \label{subfig:sublabel3}
}[ 0.324\columnwidth ]
{\includegraphics[width=0.324\columnwidth] {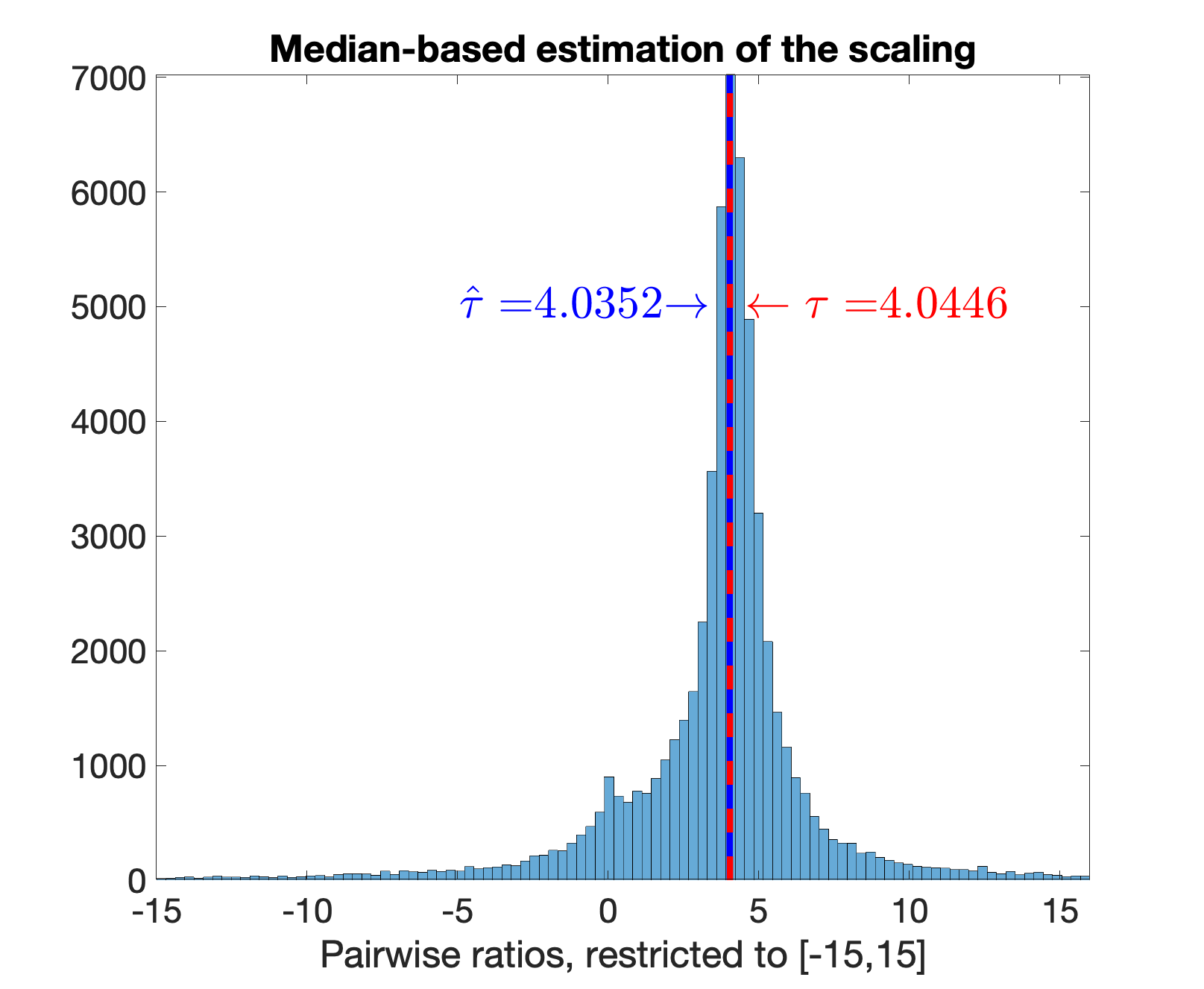} }
% \hspace{0.01\columnwidth} % seperation
%
\subcaptionbox[]{ Regression-based estimator of the scaling factor, for $\gamma = 0.02$.
% \label{subfig:sublabel3}
}[ 0.324\columnwidth ]
{\includegraphics[width=0.324\columnwidth] {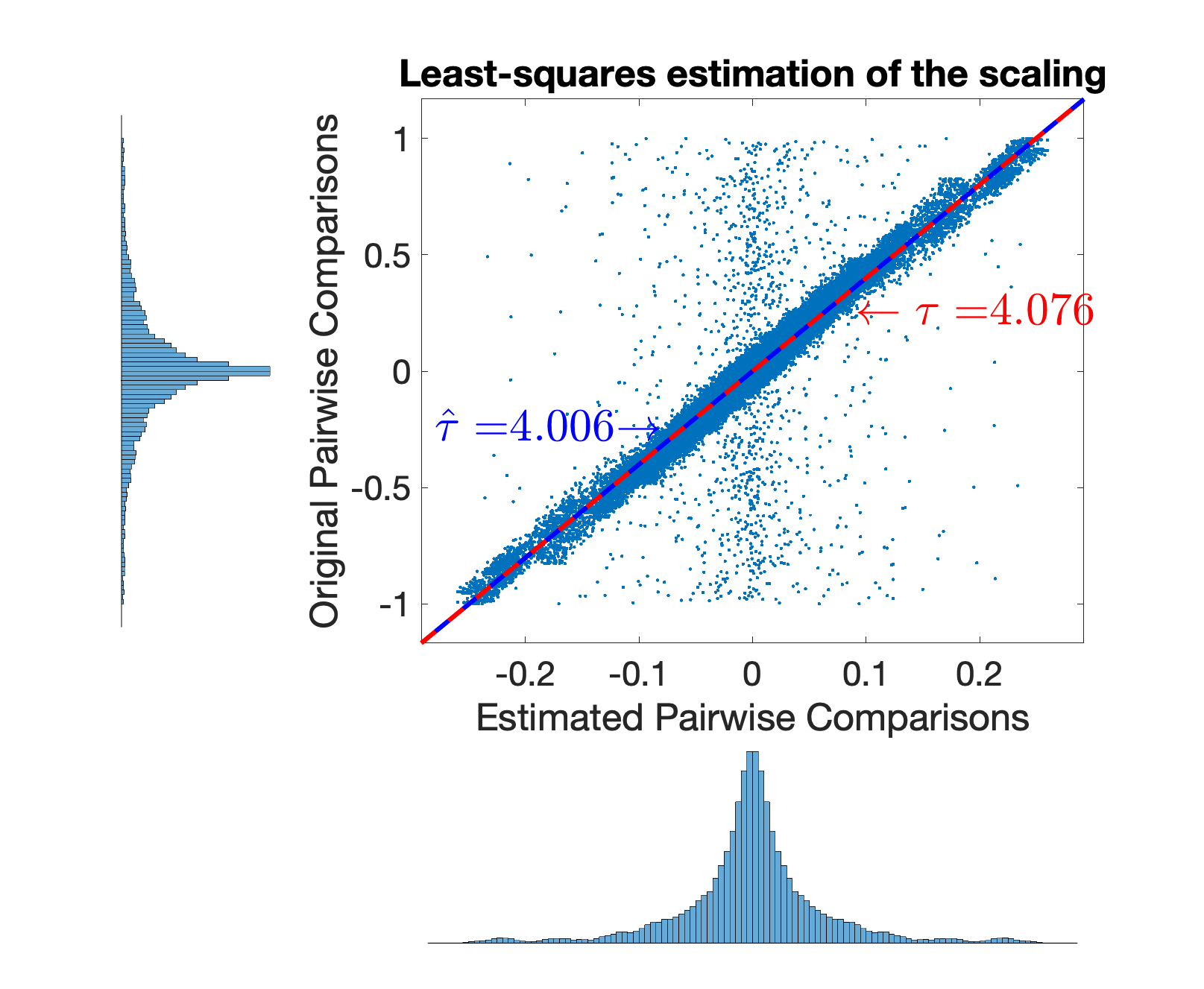} }
% \hspace{0.01\columnwidth} % seperation
%
\subcaptionbox[]{ Histogram of pairwise ratios in absolute value  (on a log scale), for $\gamma = 0.30$.
% \label{subfig:sublabel3}
}[ 0.324\columnwidth ]
{\includegraphics[width=0.324\columnwidth] {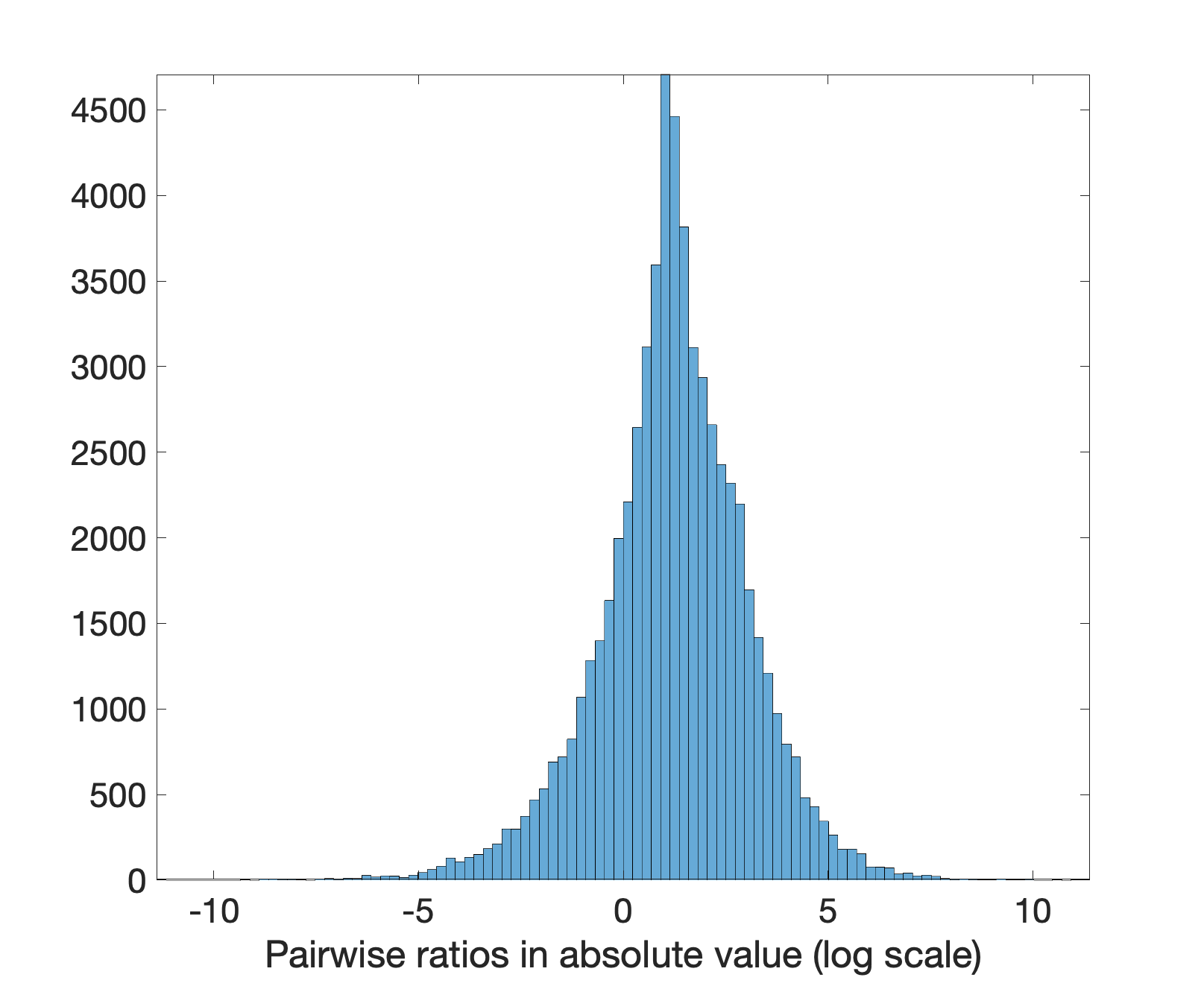} }
% \hspace{0.01\columnwidth} % seperation
%
\subcaptionbox[]{ Histogram of pairwise ratios (zoom in [-15,15]), with the median estimator, for $\gamma = 0.30$.
% \label{subfig:sublabel3}
}[ 0.324\columnwidth ]
{\includegraphics[width=0.324\columnwidth] {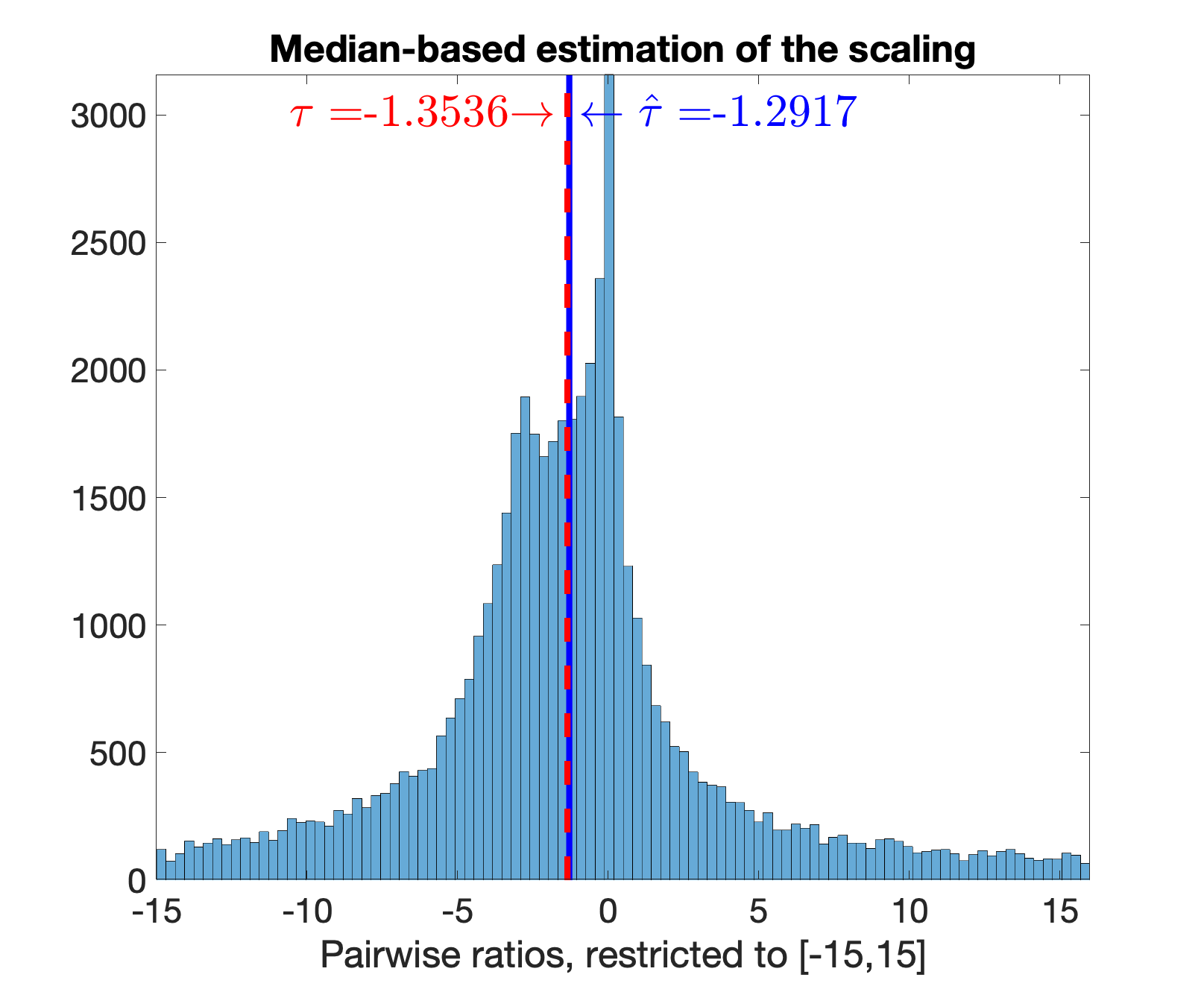} }
% \hspace{0.01\columnwidth} % seperation
%
\subcaptionbox[]{ Regression-based estimator of the scaling factor, for $\gamma = 0.30$.
% \label{subfig:sublabel3}
}[ 0.324\columnwidth ]
{\includegraphics[width=0.324\columnwidth] {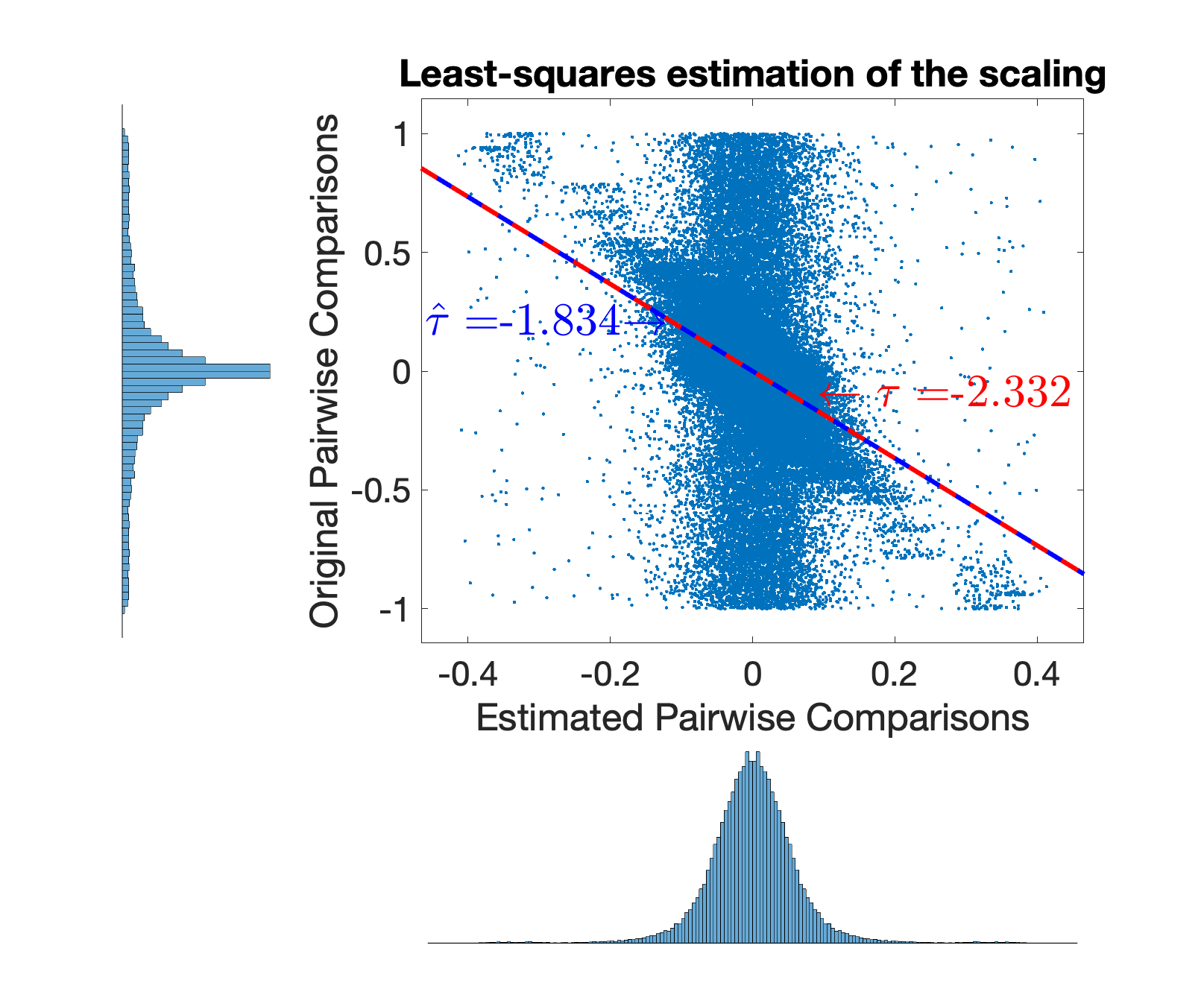} }
% \hspace{0.01\columnwidth} % seperation
%
%
\captionsetup{width=1.01\linewidth}
\caption[Short Caption]{Two instances (top: low-noise regime $\gamma=0.02$, bottom: high-noise regime  $\gamma=0.30$) showcasing the median and least-squares based estimators for scale recovery (shown in blue), along with their respective ground truth values (shown in red). Note that the bottom instance is also reconciliating the global sign.  Scores are gamma distributed, with $n=500$, and edge density $p = 0.25$.
}
\label{fig:ScalingExample}
\end{figure}
%\vspace{-1mm}

To begin with, we define the entrywise ratio of offsets/comparisons as 

\vspace{-2mm} 
\begin{equation}
\hspace{-3mm}  
\Pi_{i,j}= \left\{
 \begin{array}{rl}
  \frac{ H_{i,j} }{ S_{i,j} },  & \set{i,j} \in E \\
 0; & \text{ otherwise,}
     \end{array}
   \right.
\label{eq:PI_ij}
\end{equation}
\vspace{-3mm} 

\noindent  where $H_{i,j}$ denotes the pairwise comparison measurement initially available, and $S_{i,j}$ denotes the recovered/estimated pairwise offset given by $S_{i,j} = s_i - s_j$, where $s_i$ denotes the recovered score by either SVD-RS and SVD-NRS. In particular, $s_i = (\util_2)_i$ for SVD-RS, and $s_i = (\Dbar^{1/2} \util_2)_i$ for SVD-NRS. 
%
%\textcolor{red}{what did we agree to use here, if not a generic $s$? Just  $\tilde{u}_2$? If so, we should see notation wise, how do we write out the i-th entry of $\tilde{u}_2$?}. 
%
This scale recovery method is broadly applicable to other ranking methods we compare against, and we further detail this in the experiments section. If $|E|=m$, we next consider the vector $\rho$ of length $m$ with the pairwise ratios, i.e. nonzero entries in the matrix $\Pi_{i,j}$ in \eqref{eq:PI_ij} corresponding to the edges in the measurement graph $G$. 
The median estimator is given by 
% \vspace{-2mm}  
\begin{equation}   \label{eq:medianScaling}
 \hat{\tau}^{Median} = \text{median} \left( \rho_1, \ldots, \rho_m  \right), 
\end{equation}
% \vspace{-3mm} 
% \noindent 
as illustrated in the middle plot in Figure \ref{fig:ScalingExample}. 
From a computational perspective, the median-based estimator only requires forming the pairwise ratios for the $m$ edges in the graph, and computing the median element can be performed in linear time.

An alternative approach is to consider the least-squares formulation

\vspace{-2mm}
\begin{equation} \label{eq:fitScaling}
 \hat{\tau}^{LS} = \text{arg} \min_{\tau \in \mathbb{R} } \sum_{ \set{i,j} \in E} (H_{i,j}  - \tau S_{i,j})^2 = \frac{\sum_{ \set{i,j} \in E} H_{i,j}}{\sum_{ \set{i,j} \in E} S_{i,j}}, 
\end{equation}
\vspace{-2mm}

\noindent as shown in the rightmost column of Figure \ref{fig:ScalingExample}. In the low-noise regime, the least-squares estimator is close to the median-based estimator, and the discrepancy increases as the noise levels becomes larger. As shown in the left plot of  Figure \ref{fig:ScalingExample}, which contains a histogram of the nonzero entries $\Pi_{i,j} \neq 0, \{i,j\} \in E$ on a log-scale, gross outliers can affect the regression formulation \eqref{eq:fitScaling}. As expected, we have observed in our simulations that the median estimator \eqref{eq:medianScaling} leads to lower recovery errors compared to the regression-based estimator. In  % the plots of 
Figure \ref{fig:ScalingExample},  we show in red the recovered scaling when employing the ground truth data, in particular

\begin{equation*}   % \label{eq:PI_ij_ratios}
\hspace{-3mm}  
\Pi_{i,j}^{\text{(ground truth)}}   = \left\{
\begin{array}{ll}
\frac{ r_i - r_j }{ S_{i,j} },  &  \set{i,j} \in E \\
 0; & \text{ otherwise,} 
\end{array}
\right.
\end{equation*} 
%
%\textcolor{red}{Shall we introduce new notation for the ground truth version of $\Pi$, or abbreviate ''ground truth" by gt or GT or add superscript 0?}
%
Similarly to   \eqref{eq:medianScaling} and \eqref{eq:fitScaling},  we compute $\tau^{Median}$, respectively, $\tau^{LS}$, construed as the ground truth, and shown in red in the middle and right plots of Figure \ref{fig:ScalingExample}. Note that, in both instances, the median-based estimator is closer to the ground truth than the least-squares based estimator. In the low-noise regime $ \gamma = 0.02 $ (top plot of Figure \ref{fig:ScalingExample}), $\hat{\tau}^{Median}$ is within $0.23\%$ from its ground true counterpart, while $\hat{\tau}^{LS}$ is within $1.7\%$. However, the performance gap widens as we increase the noise level. In the high noise regime $\gamma = 0.30$ (bottom plot of Figure \ref{fig:ScalingExample}), $\hat{\tau}^{Median}$ is off by $4.5\%$ with respect to the ground truth, while $\hat{\tau}^{LS}$ is off by $27.1 \%$ from its ground truth counterpart. Note that in the latter problem instance, the recovered scaling is negative, emphasizing that the global sign has been reconciled. Due to its increased robustness, for the remainder of the numerical experiments in this paper, we henceforth rely only on the median-based estimator to perform the scale recovery step.

\FloatBarrier

%
%
%----------------
% Main results
%----------------
\subsection{Summary of main theoretical results} \label{subsec:summary_theo_results}
Broadly speaking, the bulk of our analysis revolves around bounding the distance between the unit norm vectors $u_2 = \frac{r-\alpha e}{\norm{r - \alpha e}_2}$ and $\util_2$ (up to a global sign), the latter obtained in SVD-RS (Step \ref{eq:u2tilde_algo}) and SVD-NRS (Step \ref{eq:u2tilde_algo_svdn}). For SVD-RS, these error estimates are stated in Theorem \ref{thm:main_thm_l2_ERO} (resp. Theorem \ref{thm:main_thm_linfty_ERO}) for the $\ell_2$ (resp. $\ell_{\infty}$) norm, while Theorem \ref{thm:main_svdn_l2_ERO} states the corresponding $\ell_2$ error bound for SVD-NRS. Below, we outline what this implies for rank and score recovery.
\begin{enumerate}
\item  Theorem \ref{thm:rank_recov_infty} gives guarantees for rank recovery (up to a global order reversal) for \textsc{SVD-RS} in terms of the maximum displacement error (defined in \eqref{eq:def_max_disp_rank}) between the ranking $\widetilde{\pi}$ obtained in Step \ref{step:rank_recv_svd} of Algorithm \ref{algo:SVD_Rank_sync}, and the ground truth ranking $\pi$. Denoting $\alpha = \frac{r^T e}{n}$ to be the average score, it states that if 
\begin{equation} \label{eq:mainres_tmp1}
\norm{r-\alpha e}_2 \gtrsim  \frac{M}{\eta \sqrt{p}} , \quad p \gtrsim \frac{\log n}{n},      
\end{equation}
and if $n$ is large enough, then $\norm{\widetilde{\pi} - \pi}_{\infty} \lesssim \frac{\norm{r - \alpha e}_2}{\minsep} \uinferr(n,M,\eta,p,\varepsilon,r)$ w.h.p. Here $\uinferr(n,M,\eta,p,\varepsilon,r)$ is the bound on the $\ell_{\infty}$ error between $\util_2$ and $u_2$ (up to a global sign) in Theorem \ref{thm:main_thm_linfty_ERO}. For concreteness, if $r_i = i$, we show in Example \ref{ex:linf_SVD_RS_rank} that for a fixed $\delta \in (0,1)$, if 
$p \gtrsim \frac{1}{\eta^2\delta^{2/3}}\left(\frac{(\log n)^{2/3}}{n^{2/3}}\right)$ and $n$ is large enough, then 
$\norm{\pitil - \pi}_{\infty} \lesssim \delta \frac{n}{\log n}$ w.h.p.
Note that this implies that $\Omega(n^{4/3} (\log n)^{2/3})$ measurements suffice.

\item  Theorem \ref{thm:score_rec_main_thm} provides $\ell_2$  and $\ell_{\infty}$ bounds (up to a global sign and shift) for \textsc{SVD-RS} for the score vector $r$ and a vector $\est{r}$ (derived from $\util_2$) where 
\begin{equation} \label{eq:svdrs_ideal_rhat}
 \tilde{w} = \frac{\est{\sigma_1}}{\eta p \sqrt{n}} \tilde u_2;  \quad \rtil := \tilde{w} - \frac{e^T \tilde{w}}{n} e.   
\end{equation}
If one knew the product $\eta p$, then \eqref{eq:svdrs_ideal_rhat} corresponds to choosing the scale $\tau = \frac{\est{\sigma_1}}{\eta p \sqrt{n}}$ in Step \ref{step:score_recv_svd}.
%
%obtained in Step \ref{step:score_recv_svd} of Algorithm \ref{algo:SVD_Rank_sync} (up to a global shift and sign), and essentially states the following. 
%
%
\begin{enumerate}
    \item  If the conditions in \eqref{eq:mainres_tmp1} hold, then for $\rtil$ as in \eqref{eq:svdrs_ideal_rhat}, $\exists \beta \in \set{-1,1}$ such that 
\begin{align*} %\label{eq:l2_score_rec}
 \norm{\rtil - \beta(r - \alpha e)}_2 
 \leq \frac{M}{\eta\sqrt{p}} + \sqrt{\frac{M \norm{r - \alpha e}_2}{\eta p^{1/2}}}
\end{align*}
holds w.h.p. If $r_i = i$, we show in Example \ref{ex:score_rec_main_thm} that
if $p \gtrsim \frac{\log n}{\eta^2\delta^2 n}$ for a fixed $\delta \in (0,1)$, then 
$ \norm{\rtil - \beta(r - \alpha e)}_2 \lesssim \frac{n^{3/2}}{(\log n)^{1/4}} \sqrt{\delta}$ w.h.p. Hence $\Omega(n \log n)$ measurements suffice for $\ell_2$ recovery.

\item  If the conditions in \eqref{eq:mainres_tmp1} hold and $n$ is large enough, then for $\est{r}$ as in \eqref{eq:svdrs_ideal_rhat}, there exists $\beta \in \set{-1,1}$ such that 
\begin{align*} %\label{eq:linf_score_rec}
 \norm{\rtil - \beta(r - \alpha e)}_{\infty} 
 \lesssim \norm{r-\alpha e}_2 [ C(n,M,\eta,p,r) + \sqrt{n} C^2(n,M,\eta,p,r)] + \frac{M(M-\alpha)}{\eta \sqrt{p} \norm{r-\alpha e}_2}, 
\end{align*}
holds with $ C(n,M,\eta,p,r)$ as defined in  \eqref{eq:C_linfty_ERO_exp}. When $r_i = i$, we show in Example \ref{ex:score_rec_main_thm} that if 
$p \gtrsim \frac{1}{\eta^2\delta^{2/3}}\left(\frac{(\log n)^{2/3}}{n^{2/3}}\right)$ for a fixed $\delta \in (0,1)$, and $n$ is large enough, then  $\norm{\rtil - \beta(r - \alpha e)}_{\infty} \lesssim \frac{n}{\log n} \delta$ holds w.h.p.
\end{enumerate}

\item In the same spirit as Theorem \ref{thm:score_rec_main_thm}, Theorem \ref{thm:score_rec_l2_svdn} provides a $\ell_2$ error bound (up to a global sign and shift) for \textsc{SVD-NRS} for the score vector $r$ and a vector $\rtil$ (derived from $\util_2$ in Step \ref{eq:u2tilde_algo_svdn}) where
\begin{equation} \label{eq:svdnrs_ideal_rhat}
 \tilde{w} = \frac{\est{\sigma_1}}{\eta p \norm{\Dbar^{-1/2} e}_2} \Dbar^{1/2} \tilde{u}_2;  \quad \rtil := \tilde{w} - \frac{e^T \tilde{w}}{n} e.   
\end{equation}
Again, if one knew the product $\eta p$, then \eqref{eq:svdnrs_ideal_rhat} corresponds to choosing the scale $\tau = \frac{\est{\sigma_1}}{\eta p \norm{\Dbar^{-1/2} e}_2}$ in Step \ref{step:score_recv_svd}.
While the conditions on $p$ are more convoluted than those in Theorem \ref{thm:score_rec_main_thm} -- in large part due to the normalization -- the essence is the same. In the setting where $r_i = i$, we can derive a simplified version of the theorem as shown in Example \ref{ex:score_rec_l2_svdn}. Let us now denote $\alpha = \frac{r^T (\expec[\Dbar])^{-1} e}{e^T (\expec[\Dbar])^{-1} e}$ and $\alpha' = \frac{e^T(r-\alpha e)}{n}$. Theorem \ref{thm:score_rec_l2_svdn} states that if 
$p \gtrsim  \frac{\log n}{\eta^2 \delta^2 n}$ for fixed $\delta \in (0,1)$,  then for $\rtil$ as in \eqref{eq:svdnrs_ideal_rhat}, there exists $\beta \in \set{-1,1}$ such that w.h.p,
$$\norm{\rtil - \beta(r - (\alpha + \alpha') e)}_2  \lesssim  n^{3/2} \eta^{1/4} \delta^{1/2} + \delta^{1/2} \frac{n^{3/2}}{(\log n)^{1/4}}.$$
\end{enumerate}

The closest works to our setting are those of \cite{gleich2011rank} and \cite{Yang_Wakin} which rely on matrix completion, followed by a row sum to compute rankings. In our experiments, we also consider matrix completion as a preprocessing step (see Section \ref{sec:matrixCompletion}) before applying our SVD-based algorithms, and find that in general it does improve the performance of our methods. However, this comes at the cost of an additional computational overhead; our SVD-based methods rely on computing the top two singular vectors of the input matrix, which in practice can be done very efficiently. The theoretical guarantees we provide for our problem setting, are new to the best of our knowledge, and have been lacking in the ranking and synchronization literatures currently.

% Analysis of SVD rank algorithm
%
%-----------------------------------------------------
% SVD-Rank:Theoretical results for SVD-RS and SVD-NRS
%-----------------------------------------------------
%  
\section{Theoretical results for \textsc{SVD-RS} and  \textsc{SVD-NRS} } \label{sec:SVD_algo_analysis}
We begin by detailing the theoretical results for \textsc{SVD-RS} in Section  
\ref{subsec:analysis_svd_ERO}, followed by those for \textsc{SVD-NRS} 
in Section \ref{subsec:analysis_svdn_ERO}. Throughout, we instantiate the main theorems for the special setting where $r_i = i$, for ease of interpretability.
%\vspace{-3mm}
%------------------------------------------
% Analysing the algorithm under ERO model
%------------------------------------------
\subsection{Analysis of Algorithm \ref{algo:SVD_Rank_sync} (\textsc{SVD-RS})} \label{subsec:analysis_svd_ERO}
%\vspace{-2mm}
We will assume throughout this 
section that for $\set{i,j} \in E$, the measurement $R_{ij}$ corresponds to $i < j$. 
This is clearly without loss of generality. From \eqref{ERoutliers}, we can 
write the entry $H_{ij}$ (for $i < j$) as the following mixture model
%

%\vspace{-6mm}
\begin{equation*}
H_{ij} = \left\{
 \begin{array}{rl}
		r_i - r_j & \text{ w.p } \eta p \\
		N_{ij} \sim \calU [-M,M] & \text{ w.p } (1-\eta) p	\\
		0  & \text{ w.p } 1-p,  	\\
 \end{array}
\right.
%\label{Hdef_ERoutliers}
\end{equation*}
with $(H_{ij})_{i < j}$ being independent random variables. 
Since $H_{ii} = 0$ and $H_{ij} = -H_{ji}$ (by construction of $H)$, therefore for all $i,j \in [n]$, we obtain 
$\expec[H_{ij}] = \eta p (r_i - r_j)$. In particular, we have that 
%

%\vspace{-5mm}
\begin{equation} \label{eq:expecH}
  \expec[H] = \eta p  C,
\end{equation}
%\vspace{-6mm}
%
where $C = r e ^T - e r^T$ is a skew-symmetric matrix of rank $2$. We can decompose $H$ as
%\vspace{-3mm}
\begin{equation*} % \label{eq:lowRankDecomp}
 H =  \expec[H] + Z = \eta p  C + Z,
\end{equation*}
%\vspace{-6mm}
where $Z$ is a random noise matrix with $Z_{ii}=0, \forall i=1,\ldots,n$. 
For $1 \leq i < j \leq n$ the entries of $Z$ 
are defined by the following mixture model

%\vspace{-7mm}
\begin{equation} \label{eq:Zmatrix_def_ERO}
 Z_{ij}= \left\{
\begin{array}{rl}
( r_i - r_j) - \eta p ( r_i - r_j) \quad ; & \text{w.p }  \eta p  \\
N_{ij} - \eta p ( r_i - r_j) \quad ; & \text{w.p }  (1-\eta) p \\
- \eta p ( r_i - r_j) \quad ; & \text{w.p }  (1- p).
\end{array} \right. %\quad ; %\quad   1 \leq i < j \leq n.
\end{equation}
%\vspace{-6mm}

Note that $(Z_{ij})_{i < j}$ are independent\footnote{Since $Z_{ii} = 0$, clearly 
$(Z_{ij})_{i \leq j}$ are independent as well.} random variables and $Z_{ij} = - Z_{ji}$. 
Hence the matrices $H, C, Z$ are all skew-symmetric. 
%The above decomposition  \eqref{eq:lowRankDecomp} renders our approach amenable to a theoretical 
%analysis using tools from matrix perturbation theory and random matrix theory, and leads to the 
%following theorem.
We now proceed to show in Theorem \ref{thm:main_thm_l2_ERO} that, provided $p$ is large enough, it holds true that  $\util_2$ (obtained in Step \ref{eq:u2tilde_algo} of \textsc{SVD-RS}) 
is close to $u_2 =  \frac{r-\alpha e}{|| r-\alpha e ||_2}$ in the $\ell_2$ norm (up to a global sign) with high probability. The proof is outlined in Section \ref{sec:proofOutline_thm_l2_ERO}.

%-------------------------------------------
% Main theorem: $l_2$ recovery, ERO model 
%-------------------------------------------
\begin{theorem}[$\ell_2$ recovery, ERO model] \label{thm:main_thm_l2_ERO} 
Denoting $\alpha = \frac{r^T e}{n}$, for a given $ 0 < \varepsilon \leq 1/2 $, let 
$\norm{r-\alpha e}_2 \geq  \frac{24 M}{\eta}  \sqrt{\frac{5}{3p}} (2 + \varepsilon)$. 
%where $$ A(\eta,p) = \eta p (1-\eta p )^2 + (1-\eta) p \left(\frac{1}{3}+\eta^2p^2\right) + (1-p) \eta^2 p^2. $$ 
Let $\tilde{u}_2 \in \matR^n$ be the vector obtained in Step \ref{eq:u2tilde_algo} of \textsc{SVD-RS} (with $\norm{\tilde{u}_2}_2 = 1$) and 
$u_2 =  \frac{r-\alpha e}{|| r-\alpha e ||_2}$. 
Then, there exists $\beta  \in \{ -1,1 \}$ and a constant $c_{\varepsilon} > 0$ depending only on $\varepsilon$ such that 
%\vspace{-6mm}
\begin{equation*} %\label{eq:l2_bd_ERO}
||  \tilde{u}_2 - \beta u_2 ||_2^2   
\leq \frac{120 M}{\eta} \sqrt{\frac{5}{3p}} \frac{(2+\varepsilon)}{\norm{r-\alpha e}_2}, 
\end{equation*}
%\vspace{-5mm}
with probability  at least  $1- 2n \exp \left(- \frac{ 20 p n }{ 3 c_{\varepsilon} } \right) $.
\end{theorem}
Theorem \ref{thm:main_thm_l2_ERO} says that for $\delta \in (0,1)$, 
if $p \gtrsim \max \set{\frac{M^2 \log n}{\eta^2 \delta^2 \norm{r - \alpha e}_2^2}, \frac{\log n}{n}}$, then 
with high probability, we have 
\begin{equation*}
 \norm{\util_2 - \beta u_2}_2^2 \lesssim \frac{\delta}{\sqrt{\log n}}.
\end{equation*}
Note that the bounds involved are invariant to the scaling of $r$, and essentially depend on the variance of the 
normalized entries of $r$, wherein each $\frac{r_i}{M} \in [0,1]$. We can see that as $\frac{\norm{r - \alpha e}_2^2}{M^2}$ becomes 
small, then the corresponding condition on $p$ becomes more stringent.
\begin{example}
Consider the case where $r_i = i$ for $i=1,\dots,n$. Then $M=n$, $\alpha = \frac{n+1}{2}$ and 
$\norm{r - \alpha e}_2 = \Theta(n^{3/2})$. Hence Theorem \ref{thm:main_thm_l2_ERO} now says that if 
$p \gtrsim \frac{\log n}{\eta^2\delta^2 n}$, 
then with high probability, we have $\norm{\util_2 - \beta u_2}_2^2 \lesssim \frac{\delta}{\sqrt{\log n}}$.

\end{example} 
If in addition, $n$ is also large enough, then we show in Theorem \ref{thm:main_thm_linfty_ERO} that $\util_2$ is close to $u_2$ (up to a global sign) in the 
$\ell_{\infty}$ norm as well. The proof is outlined in Section \ref{sec:proofOutline_thm_linf_ERO}.
%
%
%------------------------------------------------
% Main theorem: $l_infty$ recovery, ERO model 
%------------------------------------------------
\begin{theorem}[$\ell_{\infty}$ recovery, ERO model] \label{thm:main_thm_linfty_ERO} 
With the same notation as in Theorem \ref{thm:main_thm_l2_ERO}, for a given $ 0 < \varepsilon \leq 1/2 $, let 
$\norm{r-\alpha e}_2 \geq  \frac{24 M}{\eta}  \sqrt{\frac{5}{3p}} (2 + \varepsilon)$. 
Assume $p \geq \max\set{\frac{1}{2n}, \frac{2\log n}{15 n}}$. Choose $\xi > 1$, $0 < \kappa < 1$ and 
define $\mu = \frac{2}{\kappa + 1}$. Let $n$ satisfy $\frac{16}{\kappa} \leq (\log n)^{\xi}$. Then, there 
 exists $\beta  \in \{ -1,1 \}$ and constants $C_{\varepsilon}, c_{\varepsilon} > 0$ depending only on $\varepsilon$ 
such that with probability at least 
$$1- 2n \exp \left(- \frac{ 20 p n }{ 3 c_{\varepsilon} } \right) - \frac{4}{n} - 2n^{1-\frac{1}{4}(\log_{\mu} n)^{\xi-1} (\log_{\mu} e)^{-\xi}}$$
we have that 
\begin{equation*} 
||  \tilde{u}_2 - \beta u_2 ||_{\infty}   
\leq 4 (2 + \sqrt{2}) C(n,M,\eta,p,\varepsilon,r) + 4\sqrt{n} C^2(n,M,\eta,p,\varepsilon,r),
\end{equation*}
where 
\begin{equation} \label{eq:C_linfty_ERO_exp}
 C(n,M,\eta,p,\varepsilon,r) = C_{\varepsilon} \left[\left(\frac{M \sqrt{\log n}}{\eta \sqrt{p} \norm{r-\alpha e}_2} + \frac{M^2 (\log n)^{2\xi}}{\eta^2 p \norm{r - \alpha e}_2^2}\right)  
	\left(\frac{1}{\sqrt{n}} + \frac{M-\alpha}{\norm{r-\alpha e}_2}\right) + \frac{M^3}{\eta^3 p^{3/2} \norm{r-\alpha e}_2^3} \right].
\end{equation}
\end{theorem}
%
%
%------------------------------------------------------
% Interpreting Theorem \ref{thm:main_thm_linfty_ERO}
%------------------------------------------------------
Let us look at bounding the admittedly complicated looking term $ C(n,M,\eta,p,\varepsilon,r)$. For $\delta \in (0,1)$, 
say $p$ satisfies
\begin{equation} \label{eq:p_cond_compl_1}
p \gtrsim 
\frac{M^2}{\norm{r - \alpha e}_2^2}\max\set{\frac{n \log^3 n}{\eta^2\delta^2 } \left(\frac{1}{\sqrt{n}} + \frac{M-\alpha}{\norm{r-\alpha e}_2}\right)^2, 
\frac{\sqrt{n} (\log n)^{2\xi + 1}}{\eta^2 \delta}\left(\frac{1}{\sqrt{n}} + \frac{M-\alpha}{\norm{r-\alpha e}_2}\right), 
\frac{(\sqrt{n} \log n)^{2/3}}{\eta^2 \delta^{2/3}}}.
\end{equation}
Then we have that $C(\cdot) \lesssim \frac{\delta}{\sqrt{n} \log n}$. Hence if additionally 
$p \gtrsim \max \set{\frac{M^2}{\eta^2 \norm{r - \alpha e}_2^2}, \frac{\log n}{n}}$ and $n = \Omega(1)$ hold, then 
Theorem \ref{thm:main_thm_linfty_ERO} says that 
\begin{equation*}
\norm{\tilde{u}_2 - \beta u_2}_{\infty} \lesssim \frac{\delta}{\sqrt{n} \log n}
\end{equation*}
holds w.h.p. Note that the condition on $p$ is stricter than that in Theorem \ref{thm:main_thm_l2_ERO}.
\begin{example} Let us revisit the case where $r_i = i$ for all $i$. The condition \eqref{eq:p_cond_compl_1} then simplifies to
\begin{align} \label{eq:pcond_exam_thm_inf}
p \gtrsim \max \set{\frac{1}{\eta^2\delta^2}\left(\frac{(\log n)^3}{n}\right), \frac{1}{\eta^2\delta}\left(\frac{(\log n)^{2\xi + 1}}{n}\right), 
\frac{1}{\eta^2\delta^{2/3}}\left(\frac{(\log n)^{2/3}}{n^{2/3}}\right)}.
\end{align}
Since $\frac{M^2}{\norm{r - \alpha e}_2^2} = \frac{1}{n}$, we conclude that for a fixed $\delta \in (0,1)$, if 
$p \gtrsim \frac{1}{\eta^2\delta^{2/3}}\left(\frac{(\log n)^{2/3}}{n^{2/3}}\right)$ and $n = \Omega(1)$, then 
$\norm{\tilde{u}_2 - \beta u_2}_{\infty} \lesssim \frac{\delta}{\sqrt{n} \log n}$ holds with high probability.
\end{example}
%
%
%-----------------------------------------------
% Recovering ranks in the $\ell_{\infty}$ norm
%-----------------------------------------------
\paragraph{Recovering ranks in the $\ell_{\infty}$ norm.} 
Let us assume that $\util_2$ is aligned with $u_2$, i.e., 
$\norm{\util_2 - u_2}_{\infty} \leq \uinferr(n,M,\eta,p,\varepsilon,r)$ where $\uinferr(\cdot)$ is the bound in Theorem \ref{thm:main_thm_linfty_ERO}. For any $s \in \matR^n$, we say that the permutation $\pi : [n] \rightarrow [n]$ is \emph{consistent} with $s$ if for all pairs $(i,j)$, $\pi(i) < \pi(j)$ implies $s_i \geq s_j$. 
For the purpose of recovering the rankings, we will consider the entries of $r$ to be pairwise distinct, i.e. $r_i \neq r_j$ for all $i \neq j$. Hence there is a unique permutation $\pi$ which is consistent with $r$. Let $\pitil$ be a permutation that is consistent with $\util_2$, note that this is not unique since some entries of $\util_2$ could have the same value. 
Our goal is to bound the maximum displacement error between $\pi$ and $\pitil$ defined as 
%we will do so for the following two commonly used distance measures.
%
%\begin{enumerate}
%\item \textbf{Maximum displacement error.} This quantifies the %maximum displacement in the position of an item relative to $\pi$ %and is defined as
%
\begin{equation}  
\label{eq:def_max_disp_rank}
\norm{\pitil - \pi}_{\infty} := \max_i \left(\sum_{\pi(j) > \pi(i)} \mb{1}_{\pitil(j) < \pitil(i)} + \sum_{\pi(j) < \pi(i)} \mb{1}_{\pitil(j) > \pitil(i)} \right). 
\end{equation}
%
%This quantifies the maximum displacement in the position of an item %relative to $\pi$.

%\item \textbf{Kendall tau distance.} This simply counts the total number of pairwise disagreements between the two rankings, i.e., 
%
%\begin{equation} \label{eq:kend_tau_dist}
%\kentau(\pi,\pitil) := \sum_{\pi(i) < \pi(j)} \mb{1}_{\set{\pitil(i) %> \pitil(j)}}.    
%    
%\end{equation}
%\end{enumerate}
%
%
To this end, we have the following theorem the proof of which is provided in Section \ref{sec:linf_rank_svdrs_proof}. Notice that bounding $\norm{\pi - \pitil}_{\infty}$ requires a bound on $\norm{\util_2 - u_2}_{\infty}$ which is obtained from Theorem \ref{thm:main_thm_linfty_ERO}. The proof technique is essentially the same as that of \cite[Theorem 24]{fogel2016spectral}.
%as compared to bounding $\kentau(\pi,\pitil)$ which requires a bound %on $\norm{\util_2 - u_2}_2$.
%
%
\begin{theorem} \label{thm:rank_recov_infty}
Assuming $r_i \neq r_j$ for all $i \neq j$, let $\pi$ denote the (unique) ranking consistent with $r$ 
and also define $\minsep := \min_{i \neq j} \abs{r_i - r_j}$.  
%
%
%\begin{enumerate}
%\item (\textbf{Maximum displacement error}) 
Assuming $\beta = 1$, denote $\uinferr(n,M,\eta,p,\varepsilon,r)$ to be the bound on $\norm{\util_2 - u_2}_{\infty}$ in 
Theorem \ref{thm:main_thm_linfty_ERO}. Then under the notation and assumptions in Theorem \ref{thm:main_thm_linfty_ERO}, we have with probability at least $1- 2n \exp \left(- \frac{ 20 p n }{ 3 c_{\varepsilon} } \right) - \frac{4}{n} - 2n^{1-\frac{1}{4}(\log_{\mu} n)^{\xi-1} (\log_{\mu} e)^{-\xi}}$ that
\begin{equation} \label{eq:rank_inf_bound}
   \norm{\pitil - \pi}_{\infty} \leq \frac{4\norm{r - \alpha e}_2}{\minsep} \uinferr(n,M,\eta,p,\varepsilon,r)
\end{equation}
holds for all rankings $\pitil$ which are consistent with $\util_2$.

%\item (\textbf{Kendall tau distance}) Assuming $\beta = 1$, let %$\nu(M,\eta,p,\varepsilon, r)$ denote the bound on $\norm{\util_2 - %u_2}_2$ in Theorem \ref{thm:main_thm_l2_ERO}. Then under the %notation and assumptions in Theorem \ref{thm:main_thm_l2_ERO}, we %have with probability at least $1- 2n \exp \left(- \frac{ 20 p n }{ %3 c_{\varepsilon}}\right)$ that 
%
%\begin{equation} \label{eq:rank_kt_bound}
%   \kentau(\pi,\pitil) \leq \frac{32 \norm{r-\alpha %e}_2^2}{\minsep^2} \nu^2(M,\eta,p,\varepsilon, r)
%\end{equation}
%
%holds for all rankings $\pitil$ which are consistent with $\util_2$.
%\end{enumerate}
%
%
\end{theorem}
\begin{example}  \label{ex:linf_SVD_RS_rank}
Let us examine the bound in \eqref{eq:rank_inf_bound} for the case where $r_i = i$ for $i=1,\dots,n$. 
We have seen that $\norm{r - \alpha e}_2 = \Theta(n^{3/2})$. Theorem \ref{thm:main_thm_linfty_ERO} tells us that for a fixed $\delta \in (0,1)$, if 
$p \gtrsim \frac{1}{\eta^2\delta^{2/3}}\left(\frac{(\log n)^{2/3}}{n^{2/3}}\right)$ and $n = \Omega(1)$, then 
$\norm{\util_2 - u_2}_{\infty} \leq \uinferr(\cdot) \lesssim \frac{\delta}{\sqrt{n} \log n}$ w.h.p. Since $\minsep = 1$, we obtain from \eqref{eq:rank_inf_bound} that
\begin{equation*}
   \norm{\pitil - \pi}_{\infty} \lesssim \delta \frac{n}{\log n}.
\end{equation*}
\end{example} 
%
%------------------- 
% Recovering scores 
%-------------------
\paragraph{Recovering scores in the $\ell_2$ and $\ell_{\infty}$ norms.} 
We now bound the error between the score vector estimate $\rtil$ as in \eqref{eq:svdrs_ideal_rhat}, and $r$ up to a global shift and sign. This is shown in the following theorem for the $\ell_2$ and $\ell_{\infty}$ norms. The proof is outlined in Section \ref{subsec:proof_score_rec_svd}. Note that both $\rtil$ and $r-\alpha e$ are centered vectors.
\begin{theorem} \label{thm:score_rec_main_thm}
Recall $\rtil \in \matR^n$ as defined in \eqref{eq:svdrs_ideal_rhat}.
\begin{enumerate}
\item Under the notation and assumptions of Theorem \ref{thm:main_thm_l2_ERO}, there exists $\beta \in \set{-1,1}$ 
such that 
\begin{align} \label{eq:l2_score_rec}
 \norm{\rtil - \beta(r - \alpha e)}_2 
 \leq \frac{8M}{\eta} \sqrt{\frac{5}{3p}} (2+\varepsilon) + \sqrt{\frac{120 M (2+\varepsilon) \norm{r - \alpha e}_2}{\eta p^{1/2}}} \left(\frac{5}{3}\right)^{1/4}
\end{align}
with probability at least $1- 2n \exp \left(- \frac{ 20 p n }{ 3 c_{\varepsilon} } \right) $.

\item Under the notation and assumptions of Theorem \ref{thm:main_thm_linfty_ERO}, there exists 
$\beta \in \set{-1,1}$ such that 
\begin{align} \label{eq:linf_score_rec}
 \norm{\rtil - \beta(r - \alpha e)}_{\infty} 
 &\leq \frac{16}{3}\norm{r-\alpha e}_2 [(2 + \sqrt{2}) C(n,M,\eta,p,\varepsilon,r) + \sqrt{n} C^2(n,M,\eta,p,\varepsilon,r)] \nonumber \\
 &+ 8\sqrt{\frac{5}{3}} (2+\varepsilon) \frac{M(M-\alpha)}{\eta \sqrt{p} \norm{r-\alpha e}_2}
\end{align}
with probability at least
$$1- 2n \exp \left(- \frac{ 20 p n }{ 3 c_{\varepsilon} } \right) - \frac{4}{n} - 2n^{1-\frac{1}{4}(\log_{\mu} n)^{\xi-1} (\log_{\mu} e)^{-\xi}}.$$
\end{enumerate}
\end{theorem}
\begin{example}  \label{ex:score_rec_main_thm}
When $r_i = i$ for all $i$, the bound in \eqref{eq:l2_score_rec} is of the form
\begin{equation} \label{eq:l2_score_rec_temp1}
 \norm{\rtil - \beta(r - \alpha e)}_2 \lesssim \frac{n}{\eta\sqrt{p}} + \frac{n^{5/4}}{\sqrt{\eta} p^{1/4}}.
\end{equation}
Recall from Theorem \ref{thm:main_thm_l2_ERO} that $p$ is required to satisfy 
$p \gtrsim \frac{\log n}{\eta^2\delta^2 n}$ 
for $\delta \in (0,1)$. Then, \eqref{eq:l2_score_rec_temp1} simplifies to
\begin{equation*} 
 \norm{\rtil - \beta(r - \alpha e)}_2 \lesssim \frac{n^{3/2}}{(\log n)^{1/4}} \sqrt{\delta}.
\end{equation*}
Turning our attention to \eqref{eq:linf_score_rec}, recall from Theorem \ref{thm:main_thm_linfty_ERO} that we require  
$p \gtrsim \frac{1}{\eta^2\delta^{2/3}}\left(\frac{(\log n)^{2/3}}{n^{2/3}}\right)$ and $n = \Omega(1)$. This leads to   
\begin{align*}
\norm{\rtil - \beta(r - \alpha e)}_{\infty} 
\lesssim \frac{n}{\log n} \delta + \frac{\sqrt{n}}{\eta \sqrt{p}} 
\lesssim \frac{n}{\log n} \delta + \frac{n^{5/6}}{(\log n)^{1/3}} \delta^{1/3} 
\lesssim \frac{n}{\log n} \delta.
\end{align*}
\end{example} 
%
%----------------------------------------------------------
% Analysing the normalized algorithm under ERO model
%----------------------------------------------------------
\subsection{Analysis of Algorithm \ref{algo:SVDN_Rank_sync} (\textsc{SVD-NRS})} \label{subsec:analysis_svdn_ERO}
We now analyze Algorithm \ref{algo:SVDN_Rank_sync}, namely SVD-NRS, under the ERO model. 
The random matrix $\Dbar^{-1/2} H \Dbar^{-1/2}$ will concentrate around $(\expec[\Dbar])^{-1/2} \expec[H] (\expec[\Dbar])^{-1/2}$. 
Recall from \eqref{eq:expecH} that $\expec[H] = \eta p (r e^T - e r^T)$, hence $(\expec[\Dbar])^{-1/2} \expec[H] (\expec[\Dbar])^{-1/2}$ is a rank-$2$ matrix. It is easy to verify that 
$$\expec[\Dbar_{ii}] = \eta p \sum_{j=1}^n \abs{r_i - r_j} + (1-\eta) p \frac{M}{2}.$$
It will be useful to denote
\begin{align}
\max_{i} \expec[\Dbar_{ii}] = p\left[\eta \max_i \left(\sum_{j=1}^n \abs{r_i - r_j} \right) + (1-\eta) \frac{M}{2}\right] = p \lambda_{\max}, \label{eq:lamax} \\
\min_{i} \expec[\Dbar_{ii}] = p\left[\eta \min_i \left(\sum_{j=1}^n \abs{r_i - r_j} \right) + (1-\eta) \frac{M}{2} \right] = p \lambda_{\min}. \label{eq:lamin}
\end{align}
Recall $\tilde{u}_2 \in \matR^n$ obtained in Step \ref{eq:u2tilde_algo_svdn} of \textsc{SVD-NRS}, 
and denote $u_2 =  \frac{(\expec[\Dbar])^{-1/2} (r-\alpha e)}{\norm{(\expec[\Dbar])^{-1/2} (r-\alpha e)}_2}$. 
The following theorem bounds the $\ell_2$ error between $\util_2$ and $u_2$ up to a global sign, its proof is given in Section \ref{sec:proofOutline_thm_svdn_l2_ERO}.
%------------------------------------------------
% Main theorem: $l_2$ recovery, ERO model 
%------------------------------------------------
\begin{theorem}[$\ell_{2}$ recovery, ERO model] \label{thm:main_svdn_l2_ERO} 
Denote 
\begin{equation*}
\alpha = \frac{r^T (\expec[\Dbar])^{-1} e}{e^T (\expec[\Dbar])^{-1} e}, \ A(\eta,M) = \eta M^2 + (1-\eta) \frac{M}{2} \text{ and } 
C_1(\eta,M) = 4 A^{1/4}(\eta,M). 
\end{equation*}
For $\varepsilon > 0$, denote 
\begin{align*}
\sigma_{\min} &= \frac{\eta \norm{r-\alpha e}_2 \sqrt{n}}{\lambda_{\max}}, \  \sigma_{\max} = \frac{\eta \norm{r-\alpha e}_2 \sqrt{n}}{\lambda_{\min}}, \\
\Deltil &= 16M\sqrt{\frac{5}{3} pn} \frac{(2+\varepsilon)}{p\lambda_{\min}} + \frac{ C_1(\eta,M) (np\log n)^{1/4} 
\sigma_{\max}}{p^{3/2} \lambda_{\min}^{3/2}} \left(\frac{C_1(\eta,M) (n p \log n)^{1/4}}{\sqrt{p \lambda_{\min}}} + 2\sqrt{2} \right).
\end{align*} 
Let $\tilde{u}_2 \in \matR^n$ be the vector obtained in Step \ref{eq:u2tilde_algo_svdn} of \textsc{SVD-NRS} (with $\norm{\tilde{u}_2}_2 = 1$) 
and denote $u_2 =  \frac{(\expec[\Dbar])^{-1/2} (r-\alpha e)}{\norm{(\expec[\Dbar])^{-1/2} (r-\alpha e)}_2}$. 
If
\begin{equation} \label{eq:pconds_svdnrs_1}
p \geq \frac{M^2}{9 A(\eta,M)} \frac{\log n}{n}, \  
 p\geq \frac{16(\sqrt{2} + 1)^2 A(\eta,M) n \log n}{\lambda_{\min}^2} \text{ and } \Deltil \leq \frac{\sigma_{\min}}{3}
\end{equation}
hold, then there exists $\beta  \in \{ -1,1 \}$ and a constant $c_{\varepsilon} > 0$ depending only on $\varepsilon$ such that 
\begin{equation*}
||  \tilde{u}_2 - \beta u_2 ||_2^2 \leq \frac{15 \Deltil}{\sigma_{\min}}
\end{equation*}
holds with probability  at least  $1- \frac{2}{n} - 2n \exp \left(- \frac{ 20 p n }{ 3 c_{\varepsilon} } \right) $.
\end{theorem}
The quantities $\sigma_{\min},\sigma_{\max}$ are lower and upper bounds on the non-zero singular values of the rank $2$ matrix $(\expec[\Dbar])^{-1/2} \expec[H] (\expec[\Dbar])^{-1/2}$, while $\Deltil$ is an upper bound on $\norm{\Dbar^{-1/2} H \Dbar^{-1/2} - (\expec[\Dbar])^{-1/2} \expec[H] (\expec[\Dbar])^{-1/2}}_2$. The theorem essentially states that if $p$ is sufficiently large, then 
$\util_2$ is close to $u_2$ up to a sign flip. In order to get a sense of the precise scaling of the quantities involved, we consider the special case $r_i = i$.

\begin{example}
Consider $r_i = i$ for all $i$ with $\eta \in (0,1)$ fixed, hence $M = n$. Denoting 
$S_i = \sum_{j=1}^n \abs{r_i - r_j}$, one can verify that $S_i = i^2 - i(n+1) + \frac{n^2 + n}{2}.$ 
This implies that $\lambda_{\max}, \lambda_{\min} = \Theta(\eta n^2) = A(\eta, M)$, and so, 
$C(\eta,M) = \Theta(\eta^{1/4} n^{1/2})$. Now let us find the scaling of $\alpha$. Since $\expec[\Dbar_{ii}] = \Theta(p \eta n^2)$, we obtain
\begin{equation*}
    \alpha = \frac{\sum_{i=1}^n i (\expec[\Dbar_{ii}])^{-1}}{\sum_{i=1}^n (\expec[\Dbar_{ii}])^{-1}} = \Theta\left(\frac{\sum_{i=1}^n i}{n}\right) = \Theta(n).
\end{equation*}
This in particular implies that $\norm{r - \alpha e}_2^2 = \sum_{j=1}^n (j - \alpha)^2 = \Theta(n^3)$.
We then obtain $\sigma_{\min}, \sigma_{\max} = \Theta(1)$; plugging the aforementioned scalings in the expression for $\Deltil$ leads to
\begin{equation*}
    \Deltil = O\left( \frac{1}{\eta \sqrt{np}} + \frac{(\log n)^{1/4}}{(\eta p)^{5/4} n^{9/4}} \left( \frac{(\log n)^{1/4}}{(\eta n p)^{1/4}} + 2\sqrt{2} \right) \right).
\end{equation*}
The first two conditions on $p$ in \eqref{eq:pconds_svdnrs_1}  translate to $p \gtrsim \frac{1}{\eta} \frac{\log n}{n}$. Since $\sigma_{\min} = \Theta(1)$, hence the third condition in \eqref{eq:pconds_svdnrs_1} holds provided $\Deltil \lesssim 1$. In fact, for any $\delta \in (0,1)$, we have $\Deltil \lesssim \frac{\delta}{\sqrt{\log n}}$ if $p$ additionally satisfies
\begin{align*}
    \frac{1}{\eta \sqrt{np}} \lesssim \frac{\delta}{\sqrt{\log n}} \Leftrightarrow p \gtrsim \frac{\log n}{\eta^2 \delta^2 n} \quad  
    \text{and} \quad  \frac{(\log n)^{1/4}}{(\eta p)^{5/4} n^{9/4}} \lesssim \frac{\delta}{\sqrt{\log n}} \Leftrightarrow p \gtrsim \frac{(\log n)^{3/5}}{\eta n^{9/5} \delta^{4/5}}.
\end{align*}
Clearly, the first condition on $p$ dominates since $\frac{\log n}{\eta^2 \delta^2 n} \gtrsim \frac{(\log n)^{3/5}}{\eta n^{9/5} \delta^{4/5}}$. To summarize, Theorem  \ref{thm:main_svdn_l2_ERO} states that if $p$ satisfies
\begin{equation} \label{eq:pcond_svdnrs_2}
    p \gtrsim  \frac{\log n}{\eta^2 \delta^2 n}, 
\end{equation}
then $\exists \beta \in \set{-1,1}$ such that $||\tilde{u}_2 - \beta u_2 ||_2^2 \lesssim \delta$.
\end{example} 
%
%
%----------------------------------------------
% Recovering scores in the $\ell_2$ norm
%----------------------------------------------
\paragraph{Recovering scores in the $\ell_2$ norm.} The following theorem bounds the $\ell_2$ error between the score vector estimate $\rtil$ defined in \eqref{eq:svdnrs_ideal_rhat}, and $r$ up to a global sign and shift. Its proof is outlined in Section \ref{subsec:proof_score_rec_l2_svdn}. Note that both $\rtil$ and $r - (\alpha + \alpha') e$ are centered vectors.
\begin{theorem} \label{thm:score_rec_l2_svdn}
Recall $\rtil \in \matR^n$ as defined in \eqref{eq:svdnrs_ideal_rhat}. Under the notation and assumptions of Theorem \ref{thm:main_svdn_l2_ERO}, and denoting $\alpha' = \frac{e^T(r-\alpha e)}{n}$, there exists $\beta \in \set{-1,1}$ such that 
\begin{align*}
    \norm{\rtil - \beta(r - (\alpha + \alpha') e)}_2 &\leq \frac{2}{\eta p} \Biggl[\sqrt{\frac{3}{n}} \left(\sigma_{\max} + \Deltil \right) (2\sqrt{2} + 1)^{1/2} (A(\eta,M) np \log n)^{1/4} \left(\sqrt{p\lambda_{\max}} + \frac{\lambda_{\max}}{\lambda_{\min}}\right) \\ 
%---------------    
    &+ \frac{\Deltil p \lambda_{\max}}{\sqrt{n}} + \sqrt{\frac{15 \Deltil}{\sigma_{\min} n}} \sigma_{\max} p \lambda_{\max} \Biggr]
\end{align*}
holds with probability  at least  $1- \frac{2}{n} - 2n \exp \left(- \frac{ 20 p n }{ 3 c_{\varepsilon} } \right) $.
\end{theorem}
\begin{example} \label{ex:score_rec_l2_svdn} 
Consider the scenario where $r_i = i$ for each $i$. Then as discussed above for Theorem \ref{thm:main_svdn_l2_ERO}, if $p$ satisfies \eqref{eq:pcond_svdnrs_2}, we obtain the bound
\begin{align*}
\norm{\rtil - \beta(r - (\alpha + \alpha') e)}_2 &\lesssim \frac{1}{\eta p} \left[\frac{(\eta n^3 p \log n)^{1/4}}{\sqrt{n}} (\sqrt{p \eta n^2} + 1) + \frac{p \eta n^2 \delta}{\sqrt{n \log n}} + \sqrt{\frac{\delta}{n \sqrt{\log n}}} p \eta n^2 \right] \\
&\lesssim \frac{1}{\eta p} \left[\eta^{3/4} n^{5/4} (\log n)^{1/4} p^{3/4} + \eta p \delta^{1/2} \frac{n^{3/2}}{(\log n)^{1/4}} \right] \\
&= \frac{n^{5/4} (\log n)^{1/4}}{(\eta p)^{1/4}} + \delta^{1/2} \frac{n^{3/2}}{(\log n)^{1/4}} \\
&\lesssim n^{3/2} \eta^{1/4} \delta^{1/2} + \delta^{1/2} \frac{n^{3/2}}{(\log n)^{1/4}} \quad \text{(using \eqref{eq:pcond_svdnrs_2})}.
\end{align*}
\end{example}

% Proof outlines for main theorems
%
%------------------------------------
% SVD-Rank: Proofs
%------------------------------------
%  
\section{Proofs} \label{sec:SVD_proofs}
% 
%---------------------------------------------------------------------
% Proof outline of our theorem for ell_2 error bound (ERO)
%---------------------------------------------------------------------
\subsection{Proof of Theorem \ref{thm:main_thm_l2_ERO}}  \label{sec:proofOutline_thm_l2_ERO}
We now outline the proof of Theorem \ref{thm:main_thm_l2_ERO} 
by highlighting the steps involved, along with the intermediate Lemmas (the proofs of which are in Appendix \ref{sec:proofThm_l2_ERO}).
The proof is broken into the following steps. 
%--------------------------------------------------------------------------
% Step 1: Analysis of the singular values and singular vectors of $C$
%--------------------------------------------------------------------------
%
%\vspace{-1mm}
\paragraph{Step 1: Singular values and singular vectors of $C$.} 
We begin by finding the singular values and singular vectors of $C$.
\begin{lemma} \label{lem:sing_vals_C}
%[ ] 
For $C = r e^T - e r^T$, it holds true that $C = \sigma_1 u_1 v_1^T +  \sigma_2 u_2 v_2^T$, with 
$\sigma_1 = \sigma_2 = \norm{r - \alpha e}_2 \sqrt{n}$, and
%
%\vspace{-6mm}
\begin{equation*}
u_1 = v_2 = \frac{e}{\sqrt{n}}; \;  v_1 = -\frac{r-\alpha e}{ || r-\alpha e ||_2 }; \;  u_2 =  \frac{r-\alpha e}{ || r-\alpha e ||_2 },  
\end{equation*}
%\vspace{-6mm}
%
where $\alpha = \frac{r^T e}{n}$. 
\end{lemma}
%
%

%------------------------------------------
% Step 2: Bounding the spectral norm of Z
%------------------------------------------
\paragraph{Step 2: Bounding the spectral norm of Z.}
Recall that $(Z_{ij})_{i \leq j}$ are independent zero-mean random variables with $Z_{ii} = 0$ and moreover, 
$Z_{ji} = -Z_{ij}$. Denote $Z = S - S^T$, where  $S$ is an upper-triangular matrix with $S_{ij} = Z_{ij}$ for 
$i \leq j$.
%
\iffalse
\begin{equation*}
 S =  \left(
    \begin{array}{ccccc}
    \x    & \x       & \x    & \x    & \x \\ \cline{1-1}
    \bord & \x       & \x    & \x    & \x \\ \cline{2-2}
          & \bord    & \x    & \x    & \x \\ \cline{3-3}
          & \bigzero & \bord & \x    & \x \\ \cline{4-4}
          &          &       & \bord & \x \\ \cline{5-5}
  \end{array}\right) .
\end{equation*}
\fi
%
Note that $|| Z ||_2 \leq  || S ||_2 + \underbrace{ || S^T ||_2 }_{ = || S ||_2}   \leq  2 || S ||_2 $ . 
We define a symmetric matrix
\begin{equation*}
 \tilde{S} := 
              \left[
  \begin{array}{cc}
              0 & S \\ 
              S^T & 0
\end{array}
 \right]  \in \mathbb{R}^{2n \times 2n} .
\end{equation*}
It is easy to check that $ || \tilde{S} ||_2 = || S ||_2 $, and thus  
$ || Z ||_2 \leq 2 || \tilde{S} ||_2 $. Since $\tilde{S}$ is a random symmetric matrix with the 
entries $(\tilde{S}_{ij})_{i \leq j}$ being centered, independent and bounded random variables, we can bound  
$ || \tilde{S} ||_2 $ using a recent result of Bandeira and van Handel \cite[Corollary 3.12, Remark 3.13]{bandeira2016} 
that gives non-asymptotic bounds on the spectral norm of such random symmetric matrices. 
This altogether leads to the following Lemma.
%
%------------------------------
% Bounding spectral norm of Z
%------------------------------
\begin{lemma} \label{lem:specnorm_Z_ERO}
For $Z \in \matR^{n \times n}$ as defined in \eqref{eq:Zmatrix_def_ERO}, we have that 
\begin{equation} \label{eq:Delta_bd}
  || Z ||_2 \leq  8 M \sqrt{\frac{5}{3} pn}  (2 + \varepsilon) 
\end{equation}
holds with probability at least 
$1 - 2n \exp \left(  - \frac{ 20 pn }{3 c_{\varepsilon} }  \right) $.
\end{lemma}
%
%
%------------------------------
% Step 3: Using Wedin's bound
%------------------------------ 
\paragraph{Step 3: Using Wedin's bound.} %
Let us denote $U = [u_1 \ u_2] \in \matR^{n \times 2}$ (recall $u_1,u_2$ from Lemma \ref{lem:sing_vals_C}). 
Denote $\hat U = [\hat u_1 \ \hat u_2] \in \matR^{n \times 2}$ where $\hat u_1 , \hat u_2$ are 
the left singular vectors corresponding to the top two singular values of $H$. The following Lemma states that $\calR(U)$ is close to $\calR(\hat{U})$ if $\norm{Z}_2$ is small. The proof uses Wedin's bound \cite{Wedin1972} for perturbation of singular subspaces. 
%
%----------------------------
% Step 3 lemma: Wedins bound
%----------------------------
\begin{lemma} \label{lem:wedins_bd_ero}
Given $H = \eta  p  C  +  Z $, where $|| Z ||_2 \leq \Delta$, if $\Delta <  \eta p || r-\alpha e ||_2 \sqrt{n}$, then 
\begin{equation} \label{eq:sintheta_bd_1}
    ||  (I - \hat{U} \hat{U}^T) U  ||_2  
 \leq  \frac{ \Delta }{ \eta p  ||r-\alpha e||_2 \sqrt{n} - \Delta }  \ (=: \delta).
\end{equation}
\end{lemma}

%\vspace{-3mm}
%--------------------------------------------
% Step 4: Analyzing the projection step
%--------------------------------------------
\paragraph{Step 4: Analyzing the projection step.} 
Next, we project $u_1$ on span($\hat{U}$) to obtain $\bar{u}_{1}$, 
find a unit vector orthogonal to $\bar{u}_1$ lying in $\calR(\hat{U})$  (denote this by $ \tilde{u}_2 $), 
and show that $ \tilde{u}_2 $ is close to $u_2$ up to a sign flip. 
%This is shown formally in the following Lemma.
%
%
\begin{lemma} \label{lem:proj_analysis_Step_ERO}
With $\delta $ as defined in Lemma \ref{lem:wedins_bd_ero} let $\delta \leq 1/2$ hold. 
Then, there exists $ \beta \in \{ -1, 1 \}$ such that 
% 

%\vspace{-4mm}
\begin{equation*} 
        || \tilde{u}_2 - \beta  u_2  ||_2^2 \leq 10 \delta.
\end{equation*}
\end{lemma}
%\vspace{-4mm}
%
%

%\vspace{0mm}
%------------------------------
% Step 5: Putting it together
%------------------------------
\paragraph{Step 5: Putting it together.} 
We now use the above Lemma's to readily obtain Theorem \ref{thm:main_thm_l2_ERO} (see appendix for details). 
Using $\norm{Z}_2 \leq \Delta = 8 M \sqrt{\frac{5}{3} pn}  (2 + \varepsilon) $ in the expression for $\delta$, one can verify that 
%
%\vspace{-6mm}
\begin{equation} \label{eq:Delta_bd_ERO_outline}
\delta \leq 1/2 \Leftrightarrow 3 \Delta  \leq   \eta p ||r-\alpha e||_2 \sqrt{n}  
\end{equation}
%\vspace{-2mm}
%
holds if $\norm{r-\alpha e}_2 \geq  \frac{24 M}{\eta}  \sqrt{\frac{5}{3p}} (2 + \varepsilon)$.  
By using the bound on $\triangle$ in \eqref{eq:Delta_bd_ERO_outline} in the expression for $\delta$, we finally obtain
%
%
%\vspace{-5mm}
\begin{equation*}
\delta \leq \frac{12 M}{\eta} \sqrt{\frac{5}{3p}} \frac{(2+\varepsilon)}{\norm{r-\alpha e}_2}.
\end{equation*}
%\vspace{-5mm}
%
%
Plugging this in Lemma \ref{lem:proj_analysis_Step_ERO} yields the stated $\ell_2$ bound and completes the proof.

%-----------------------------------------------------
% Proof outline of theorem for l_infty bound (ERO)
%-----------------------------------------------------
\subsection{Proof of Theorem \ref{thm:main_thm_linfty_ERO}} \label{sec:proofOutline_thm_linf_ERO}
Recall that $U=[u_1 \ u_2], \Uhat = [\hat{u}_1 \ \hat{u}_2] \in \matR^{n \times 2}$ correspond to the two largest 
singular vectors of $\expec[H]$, and $H$ respectively. 
The proof is broken down into three main steps. The first step involves 
bounding $\norm{\util_2 - \beta u_2}_{\infty}$ in terms of $\norm{\Uhat - UO^*}_{\max}$ 
-- where $O^*$ is a $2 \times 2$ orthogonal matrix that ``aligns'' $U$ with $\Uhat$ -- provided $\norm{Z}_2$ is small. 
The second step involves bounding $\norm{\Uhat - UO^*}_{\max}$, and in the third step we combine the results from the previous steps. The proofs of all lemmas are provided in Appendix \ref{sec:proofs_lems_linf_svdrs}.
%
%------------
% Step 1
%------------
\paragraph{Step 1: Bounding $\norm{\util_2 - \beta u_2}_{\infty}$ in terms of 
$\norm{\Uhat - UO^*}_{\max}$.} The following Lemma states this formally.
\begin{lemma} \label{linf_ero_step1}
For $\Delta$ as defined in Lemma \ref{lem:wedins_bd_ero}, let $\Delta \leq \frac{\eta p ||r-\alpha e||_2 \sqrt{n}}{3}$ hold. 
Then for any orthogonal matrix $O \in \matR^{2 \times 2}$, there exists $\beta \in \set{-1,1}$ depending on $O$ such that
\begin{equation} \label{eq:ero_linfty_step1}
  \norm{\util_2 - \beta u_2}_{\infty} \leq 4\norm{\Uhat - UO}_{\max}\left(2 + \frac{\sqrt{n}(M-\alpha)}{\norm{r-\alpha e}_2} \right) + 4 \sqrt{n} \norm{\Uhat - UO}_{\max}^2.
\end{equation} 
Moreover, there exists an orthogonal matrix $O^* \in \matR^{2 \times 2}$ such that
\begin{equation} \label{eq:O_align_infbd}
\norm{\Uhat - UO^*}_2 
\leq \frac{3\Delta}{\eta p \sqrt{n} \norm{r-\alpha e}_2}.
\end{equation}
\end{lemma}
%
%
%------------
% Step 2
%------------
\paragraph{Step 2: Bounding $\norm{\Uhat - UO^*}_{\max}$.} Since $H^T = -H$, therefore $-H^2 = H H^T$. Using the SVD 
$H = \Uhat \Sighat \Vhat$, we arrive at
\begin{align}
HH^T \Uhat &= -H^2 \Uhat = \Uhat \Sighat^2 \nonumber \\
\Rightarrow \Uhat &= -H^2 \Uhat \Sighat^{-2} = -(\expec[H]^2 + \expec[H] Z + Z\expec[H] + Z^2) \Uhat \Sighat^{-2}. \label{eq:uhat_exp}
\end{align}
For convenience of notation, denote $\sigma = \sigma_i(\expec[H]) = \eta p \sqrt{n} \norm{r-\alpha e}_2$; $i = 1,2$. 
Since $\expec[H]^T = -\expec[H]$, we get 
\begin{align}
- \expec[H]^2 &= \expec[H] \expec[H]^T = U \Sigma^2 U^T = \sigma^2 (U O^*) (U O^*)^T \nonumber \\
U O^* &= -\frac{1}{\sigma^2} (\expec[H])^2 (UO^*).  \label{eq:uo_exp}
\end{align}
Subtracting \eqref{eq:uo_exp} from \eqref{eq:uhat_exp} leads to 
\begin{align}
  \Uhat - U O^* &= \underbrace{-\expec[H]^2 (\Uhat - U O^*) \Sighat^{-2}}_{E_1} + \underbrace{\expec[H]^2 UO^{*} (\sigma^{-2} I - \Sighat^{-2})}_{E_2} 
	- \underbrace{\expec[H] Z \Uhat \Sighat^{-2}}_{E_3} %\nonumber \\ 
	- \underbrace{Z \expec[H] \Uhat \Sighat^{-2}}_{E_4} - \underbrace{Z^2 \Uhat \Sighat^{-2}}_{E_5} \nonumber \\
	&= E_1 + E_2 - E_3 - E_4 - E_5 \nonumber \\
	\Rightarrow \norm{\Uhat - U O^*}_{\max} &\leq \sum_{i=1}^5 \norm{E_i}_{\max}. \label{eq:uhat_uo_bd}
\end{align}
We now proceed to bound each term in the RHS of \eqref{eq:uhat_uo_bd}. This is stated precisely in the following Lemma.
\begin{lemma} \label{lem:linf_ero_step2a}
If $\Delta \leq \frac{\sigma}{3} = \frac{\eta p \sqrt{n} \norm{r-\alpha e}_2}{3}$, then the following holds true.
\begin{enumerate}
\item $\norm{E_1}_{\max} \leq \frac{27 \Delta}{4 \sigma} \left(\frac{1}{\sqrt{n}} + \frac{M-\alpha}{\norm{r-\alpha e}_2}\right).$

\item $\norm{E_2}_{\max} \leq \frac{21\Delta}{4\sigma} \left(\frac{1}{\sqrt{n}} + \frac{M-\alpha}{\norm{r-\alpha e}_2}\right).$

\item $\norm{E_3}_{\max} \leq \frac{9\Delta}{4\sigma} \left(\frac{1}{\sqrt{n}} + \frac{M-\alpha}{\norm{r-\alpha e}_2}\right).$
\item $\norm{E_4}_{\max} \leq \frac{9}{4\sigma}(\norm{Z u_1}_{\infty} + \norm{Z u_2}_{\infty}).$
\item $\norm{E_5}_{\max} \leq \frac{27 \Delta^3}{4\sigma^3} + \frac{9}{4\sigma^2}(\norm{Z^2 u_1}_{\infty} + \norm{Z^2 u_2}_{\infty}).$
\end{enumerate}
\end{lemma}
%
%-----------------

Next, we will bound $\norm{Z u_i}_{\infty}$ and $\norm{Z^2 u_i}_{\infty}$ for $i=1,2$. Bounding $\norm{Z u_i}_{\infty}$ is 
a relatively straightforward consequence of Bernstein's inequality \cite[Corollary 2.11]{conc_book}, we state this in the 
following Lemma.
%
%----------------------------
% First concentration bound
%----------------------------
\begin{lemma} \label{lem:inf_conc_bd1}
If $p \geq \frac{2 \log n}{15n}$ then the following holds.
\begin{enumerate}
\item $\prob(\norm{Z u_1}_{\infty} \geq \frac{2\sqrt{2} + 4}{3} M \sqrt{15p \log n}) \leq \frac{2}{n}.$

\item $\prob(\norm{Z u_2}_{\infty} \geq \frac{2\sqrt{2} + 4}{3} M \sqrt{15p B n \log n}) \leq \frac{2}{n}$ where $B = \frac{(M - \alpha)^2}{\norm{r-\alpha e}_2^2}$.
\end{enumerate}
\end{lemma} 
Bounding $\norm{Z^2 u_i}_{\infty}$ is a considerably more challenging task since the entries within a row of $Z^2$ will not be 
independent now. However, leveraging a recent entry-wise concentration result for the product of a random 
matrix (raised to a power) and a fixed vector \cite[Theorem 15]{Eldridge2018} (see Theorem \ref{thm:eldridge_randmat_pow}), we are 
able to do so and arrive at the following Lemma.
%
%
%----------------------------
% Second concentration bound
%----------------------------
\begin{lemma} \label{lem:inf_conc_bd2}
Let $p \geq \frac{1}{2n}$ hold. Choose $\xi > 1, 0 < \kappa < 1$ and define $\mu = \frac{2}{\kappa + 1}$. 
Then if $\frac{16}{\kappa} \leq (\log n)^{\xi}$, the following holds true.
\begin{enumerate}
\item $\prob\left(\norm{Z^2 u_1}_{\infty} \geq 8 \sqrt{n} p M^2 (\log n)^{2\xi}\right) \leq n^{1-\frac{1}{4}(\log_{\mu} n)^{\xi-1} (\log_{\mu} e)^{-\xi}}.$

\item $\prob\left(\norm{Z^2 u_2}_{\infty} \geq 8 n p M^2 \frac{M-\alpha}{\norm{r - \alpha e}_2} (\log n)^{2\xi}\right) \leq n^{1-\frac{1}{4}(\log_{\mu} n)^{\xi-1} (\log_{\mu} e)^{-\xi}}.$
\end{enumerate}
\end{lemma}
%
%
%------------
% Step 3
%------------
\paragraph{Step 3: Putting it together.} We now combine the results of the preceding Lemmas to derive the final 
approximation bound. Recall from Lemma \ref{lem:specnorm_Z_ERO} that 
$\Delta =  8 M \sqrt{\frac{5}{3} pn}  (2 + \varepsilon)$. Also, $\sigma = \eta p \sqrt{n} \norm{r-\alpha e}_2$ and so
\begin{equation} \label{eq:Del_sigma_exp}
\frac{\Delta}{\sigma} = \frac{8M}{\eta} \sqrt{\frac{5}{3p}} \frac{(2+\varepsilon)}{\norm{r-\alpha e}_2}.
\end{equation} 
Using Lemmas \ref{lem:linf_ero_step2a},\ref{lem:inf_conc_bd1},\ref{lem:inf_conc_bd2} and \eqref{eq:Del_sigma_exp} in \eqref{eq:uhat_uo_bd}, we get
\begin{align*} 
  \norm{\Uhat - UO^{*}}_{\max} 
	&\leq 
	\left(\frac{114 M}{\eta} \sqrt{\frac{5}{3p}} \frac{(2+\varepsilon)}{\norm{r - \alpha e}_2}\right) \left(\frac{1}{\sqrt{n}} + \frac{M-\alpha}{\norm{r-\alpha e}_2}\right) \\
	& + \left(\frac{3\sqrt{15} (2\sqrt{2} + 4)}{4}\right) \left(\frac{M \sqrt{\log n}}{\eta \sqrt{p} \norm{r-\alpha e}_2}\right) \left(\frac{1}{\sqrt{n}} + \frac{M-\alpha}{\norm{r-\alpha e}_2}\right) \\
	&+ \frac{18 M^2 (\log n)^{2\xi}}{\eta^2 p \norm{r - \alpha e}_2^2} \left(\frac{1}{\sqrt{n}} + \frac{M-\alpha}{\norm{r-\alpha e}_2}\right) 
	+ \frac{27}{4} \left(\frac{8\sqrt{5}}{3}\right)^3 \frac{M^3}{\eta^3 p^{3/2}} \frac{(2+\varepsilon)^3}{\norm{r-\alpha e}_2^3} \\
	&\leq C_{\varepsilon} \left[\left(\frac{M \sqrt{\log n}}{\eta \sqrt{p} \norm{r-\alpha e}_2} + \frac{M^2 (\log n)^{2\xi}}{\eta^2 p \norm{r - \alpha e}_2^2}\right)  
	\left(\frac{1}{\sqrt{n}} + \frac{M-\alpha}{\norm{r-\alpha e}_2}\right) + \frac{M^3}{\eta^3 p^{3/2} \norm{r-\alpha e}_2^3} \right] \\
	&= C(n,M,\eta,p,\varepsilon,r)
\end{align*}
where $C_{\varepsilon} > 0$ is a universal constant depending only on $\varepsilon$. Plugging this in \eqref{eq:ero_linfty_step1}, we arrive at the stated $\ell_{\infty}$ bound in the Theorem. The lower bound on the success probability follows readily via a union bound on the events stated in Lemmas \ref{lem:specnorm_Z_ERO}, \ref{lem:inf_conc_bd1}, \ref{lem:inf_conc_bd2}.

%-----------------------------------------
% Proof of theorem for recovering ranks
%-----------------------------------------
\subsection{Proof of Theorem \ref{thm:rank_recov_infty}} \label{sec:linf_rank_svdrs_proof}
We assume w.l.o.g that $\pi$ is identity so that $\pi(i) = i$. For any $i \in [n]$, we have  for all rankings $\pitil$ which are consistent with $\util_2$ that 
\begin{align*}
\abs{i - \pitil(i)} = \sum_{j > i} \mb{1}_{\set{\pitil(j) < \pitil(i)}} + \sum_{j < i} \mb{1}_{\set{\pitil(j) > \pitil(i)}} 
\leq \sum_{j > i} \mb{1}_{\set{\util_{2,j} \geq \util_{2,i}}} + \sum_{j < i} \mb{1}_{\set{\util_{2,j} \leq \util_{2,i}}},
\end{align*}
leading to the bound 
\begin{equation*}
\norm{\pi - \pitil}_{\infty} = \max_{i} \abs{i - \pitil(i)} \leq \sum_{j > i} \mb{1}_{\set{\util_{2,j} \geq \util_{2,i}}} + \sum_{j < i} \mb{1}_{\set{\util_{2,j} \leq \util_{2,i}}}. 
\end{equation*}
Now for a given $i' \in [n]$, we can decompose $\util_{2,i'} - \util_{2,j}$ as
\begin{equation*}
\util_{2,i'} - \util_{2,j} = (\util_{2,i'} - u_{2,i'}) + (u_{2,i'} - u_{2,j}) + (u_{2,j} - \util_{2,j}),
\end{equation*}
which in turn implies $\util_{2,i'} - \util_{2,j} \leq (u_{2,i'} - u_{2,j}) + 2\norm{\util_2 - u_2}_{\infty}$. Hence if $j$ is such that $i' > j$, then $\norm{\util_2 - u_2}_{\infty} \leq \frac{\abs{u_{2,i'} - u_{2,j}}}{2}$ implies 
$$\util_{2,i'} - \util_{2,j} \leq \underbrace{(u_{2,i'} - u_{2,j})}_{\leq 0} + \abs{u_{2,i'} - u_{2,j}} = 0.$$ Using this, we obtain
\begin{align*}
\sum_{j < i'} \mb{1}_{\set{\util_{2,j} < \util_{2,i'}}} 
&\leq \sum_{j < i'} \mb{1}_{\set{\norm{\util_2 - u_2}_{\infty} > \frac{\abs{u_{2,i'} - u_{2,j}}}{2}}} \\
&\leq \sum_{j < i'} \mb{1}_{\set{\uinferr(n,M,\eta,p,\varepsilon,r) > \frac{\abs{u_{2,i'} - u_{2,j}}}{2}}} \quad \text{ (since $\norm{\util_2 - u_2}_{\infty} \leq \uinferr(n,M,\eta,p,\varepsilon,r)$)} \\
&= \sum_{j < i'} \mb{1}_{\set{2 \uinferr(n,M,\eta,p,\varepsilon,r) \norm{r - \alpha e}_2 > \abs{r_{i'} - r_{j}}}} \\
&= \left\lfloor \frac{2 \uinferr(n,M,\eta,p,\varepsilon,r) \norm{r - \alpha e}_2}{\rho} \right\rfloor \leq \frac{2 \uinferr(n,M,\eta,p,\varepsilon,r) \norm{r - \alpha e}_2}{\rho}.
\end{align*}
The same bound holds for $\sum_{j > i'} \mb{1}_{\set{\util_{2,j} > \util_{2,i'}}}$ and hence the statement of the Theorem follows.

%-----------------------------------------
% Proof of theorem for recovering scores
%-----------------------------------------
\subsection{Proof of Theorem \ref{thm:score_rec_main_thm}} \label{subsec:proof_score_rec_svd}
Recall that $\wtil = \frac{\est{\sigma}_1}{\eta p \sqrt{n}} \util_2$. Let us denote
\begin{equation*}
w = \frac{\sigma}{\eta p \sqrt{n}} \beta u_2 = \norm{r - \alpha e}_2 \beta u_2 = (r-\alpha e)\beta.
\end{equation*}
For any norm $\norm{\cdot}$ on $\matR^n$, we have by triangle inequality that 
\begin{equation*}
 \norm{\rtil - w} \leq \norm{\wtil - w} + \abs{\frac{e^T \wtil}{n}} \norm{e}.
\end{equation*}
Since $e^T w = 0$, therefore we obtain via Cauchy-Schwarz that
\begin{equation*}
 \abs{\frac{e^T \wtil}{n}} = \abs{\frac{e^T (\wtil - w)}{n}} \leq \frac{\norm{\wtil - w}_2}{\sqrt{n}}.
\end{equation*}
Hence we have that 
\begin{align*}
 \norm{\rtil - w}_2 \leq 2\norm{\wtil - w}_2 \quad \text{and} \quad 
 \norm{\rtil - w}_{\infty} \leq \norm{\wtil - w}_{\infty} + \frac{\norm{\wtil - w}_2}{\sqrt{n}} \leq 2\norm{\wtil - w}_{\infty}.
\end{align*}
It remains to bound $\norm{\wtil - w}_2$  and  $ \norm{\wtil - w}_{\infty}$. Our starting point will be
\begin{align}
\norm{\wtil - w} 
&= \frac{1}{\eta p \sqrt{n}} \norm{\est{\sigma}_1 \util_2 - \sigma_1 \beta u_2} \nonumber \\
&= \frac{1}{\eta p \sqrt{n}} \norm{(\est{\sigma}_1 - \sigma_1)\util_2 + \sigma_1(\util_2 - \beta u_2)} \nonumber \\
\Rightarrow \norm{\wtil - w} 
&\leq \frac{1}{\eta p \sqrt{n}}(\abs{\est{\sigma}_1 - \sigma} \norm{\util_2} + \sigma_1 \norm{\util_2 - \beta u_2}) \nonumber \\
&\leq \frac{1}{\eta p \sqrt{n}}(\Delta \norm{\util_2} + \eta p \norm{r-\alpha e}_2 \sqrt{n} \norm{\util_2 - \beta u_2}). \label{eq:wtil_w_genbd}
\end{align}
\begin{enumerate}
\item \textbf{Bounding $\norm{\wtil - w}_2$.} From \eqref{eq:wtil_w_genbd}, we have that 
\begin{equation*}
 \norm{\wtil - w}_2 \leq \frac{1}{\eta p \sqrt{n}} (\Delta + \eta p \norm{r-\alpha e}_2 \sqrt{n} \norm{\util_2 - \beta u_2}_2)
\end{equation*}
Plugging the expression for $\Delta$ from Lemma \ref{lem:specnorm_Z_ERO}, and the bound on $\norm{\util_2 - \beta u_2}_2$ 
from Theorem \ref{thm:main_thm_l2_ERO}, we readily obtain the bound in part (i) of the Theorem. 

\item \textbf{Bounding $\norm{\wtil - w}_{\infty}$.} From \eqref{eq:wtil_w_genbd}, we have that 
\begin{align}
\norm{\wtil - w}_{\infty} 
&\leq \frac{1}{\eta p \sqrt{n}} [\Delta \norm{\util_2}_{\infty} + \sigma_1 \norm{\util_2 - \beta u_2}_{\infty}] \nonumber \\
&\leq \frac{1}{\eta p \sqrt{n}} [\Delta (\norm{\util_2 - \beta u_2}_{\infty} + \norm{u_2}_{\infty}) + \sigma_1 \norm{\util_2 - \beta u_2}_{\infty}] \nonumber \\
&= \frac{1}{\eta p \sqrt{n}} \left[(\sigma_1 + \Delta)\norm{\util_2 - \beta u_2}_{\infty} + \Delta \frac{(M-\alpha)}{\norm{r - \alpha e}_2}\right]. \label{eq:wtil_w_infbd}
\end{align}
Plugging the expression for $\sigma_1$ (from Lemma \ref{lem:sing_vals_C}) and $\Delta$ (from Lemma \ref{lem:specnorm_Z_ERO}), followed by some simplification, 
we arrive at the bound in part (ii) of the Theorem.
\end{enumerate}
%

%---------------------------------------------------------------------
% Proof of our theorem for ell_2 error bound for SVD NRS (ERO)
%---------------------------------------------------------------------
\subsection{Proof of Theorem \ref{thm:main_svdn_l2_ERO}}  \label{sec:proofOutline_thm_svdn_l2_ERO}
The outline is similar to the proof of Theorem \ref{thm:main_thm_l2_ERO} with some technical changes. We first note that the 
SVD of $(\expec[\Dbar])^{-1/2} \expec[H] (\expec[\Dbar])^{-1/2}$ is given by $\sigma (u_1 v_1^T + u_2 v_2^T)$ where 
\begin{align} 
\sigma &= \eta p \norm{(\expec[\Dbar])^{-1/2} (r-\alpha e)}_2 \norm{(\expec[\Dbar])^{-1/2} e}_2, \ v_2 = u_1 = \frac{(\expec[\Dbar])^{-1/2} e}{\norm{(\expec[\Dbar])^{-1/2} e}_2}, \label{eq:svdn_decomp_1} \\
v_1 &= -u_2 = -\frac{(\expec[\Dbar])^{-1/2} (r - \alpha e)}{\norm{(\expec[\Dbar])^{-1/2} (r - \alpha e)}_2}, \nonumber
\end{align}
and $\alpha = \frac{r^T (\expec[\Dbar])^{-1} e}{e^T (\expec[\Dbar])^{-1} e}$. This is verified easily by proceeding as in the proof of Lemma \ref{lem:sing_vals_C}. One can also readily 
see that $\sigma_{\min} \leq \sigma \leq \sigma_{\max}$. Now let us write 
\begin{equation*}
\Dbar^{-1/2} H \Dbar^{-1/2} = (\expec[\Dbar])^{-1/2} \expec[H] (\expec[\Dbar])^{-1/2} + \Ztil
\end{equation*}
where $\Ztil = \Dbar^{-1/2} (H - \expec[H]) \Dbar^{-1/2} + \Dbar^{-1/2} \expec[H] \Dbar^{-1/2} - (\expec[\Dbar])^{-1/2} \expec[H] (\expec[\Dbar])^{-1/2}$. 
In order to bound $\norm{\Ztil}_2$, we will first need to establish the concentration of $\Dbar$ around $\expec[\Dbar]$.
\begin{lemma} \label{lem:Dbar_conc}
Denote $A(\eta,M) = \eta M^2 + (1-\eta) \frac{M}{2}$. If $p \geq \frac{M^2}{9 A(\eta,M)} \frac{\log n}{n}$ then, 
\begin{equation*}
\prob(\norm{\Dbar - \expec[\Dbar]}_2 \geq 2(\sqrt{2} + 1) \sqrt{A(\eta,M) np \log n}) \leq 2/n.
\end{equation*}
\end{lemma}
The proof is deferred to Appendix \ref{sec:proof_app_N_Thm_l2_ERO}. Conditioned on the event in Lemma \ref{lem:Dbar_conc}, we have that 
\begin{equation*}
\Dbar_{ii} \in [\expec[\Dbar_{ii}] \pm 2(\sqrt{2} + 1) \sqrt{A(\eta,M) np \log n}], \quad \forall i=1,\dots,n.
\end{equation*}
In particular, if $2(\sqrt{2} + 1) \sqrt{A(\eta,M) np\log n} \leq \frac{1}{2} \min_{i} \expec[\Dbar_{ii}] = \frac{p\lambda_{\min}}{2}$, then 
\begin{equation} \label{eq:Dbarmin_bd}
 \min_{i} \Dbar_{ii} \geq p\lambda_{\min}/2.
\end{equation}
Using triangle inequality, we obtain the bound 
\begin{align*}
\norm{\Ztil}_2 &\leq \norm{\underbrace{\Dbar^{-1/2} (H - \expec[H]) \Dbar^{-1/2}}_{\Ztil_1}}_2 
+ \norm{\underbrace{\Dbar^{-1/2} \expec[H] \Dbar^{-1/2} - (\expec[\Dbar])^{-1/2} \expec[H] (\expec[\Dbar])^{-1/2}}_{\Ztil_2}}_2.  
\end{align*}
From sub-multiplicativity of the spectral norm and also \eqref{eq:Dbarmin_bd}, we have that 
\begin{equation}
	\norm{\Ztil_1}_2 \leq \frac{\norm{H - \expec[H]}_2}{\min_i \Dbar_{ii}} \leq \frac{2\Delta}{p\lambda_{\min}}. \label{eq:ztil1_bd_1}
\end{equation}
Here, we used the bound $\norm{H - \expec[H]}_2 \leq \Delta$ from Lemma \ref{lem:specnorm_Z_ERO}. 
In order to bound $\norm{\Ztil_2}_2$, we add and subtract $\expec[\Dbar]$ from $\Dbar$ and apply triangle inequality. This yields
\begin{align}
\norm{\Ztil_2}_2 
&\leq \norm{(\Dbar^{-1/2} - (\expec[\Dbar])^{-1/2}) \expec[H] (\Dbar^{-1/2} - (\expec[\Dbar])^{-1/2})}_2 \nonumber \\
&+ \norm{(\Dbar^{-1/2} - (\expec[\Dbar])^{-1/2}) \expec[H] \Dbar^{-1/2}}_2 + \norm{\Dbar^{-1/2} \expec[H] (\Dbar^{-1/2} - (\expec[\Dbar])^{-1/2})}_2 \nonumber \\
&\leq \norm{\Dbar^{-1/2} - (\expec[\Dbar])^{-1/2}}_2^2 \norm{\expec[H]}_2 + 2\norm{\Dbar^{-1/2} - (\expec[\Dbar])^{-1/2}}_2 \norm{\expec[H]}_2 \norm{\Dbar^{-1/2}}_2. \label{eq:ztil2_bd_1}
\end{align}
Note that 
\begin{align*}
\norm{(\Dbar^{-1/2} - (\expec[\Dbar])^{-1/2})}_2 
&= \norm{\Dbar^{-1/2}(\Dbar^{1/2} - (\expec[\Dbar])^{1/2}) (\expec[\Dbar])^{-1/2}}_2 \\
&\leq \norm{\Dbar^{-1/2}}_2 \norm{(\Dbar^{1/2} - (\expec[\Dbar])^{1/2})}_2 \norm{(\expec[\Dbar])^{-1/2}}_2.
\end{align*}
Moreover, we can bound $\norm{(\Dbar^{1/2} - (\expec[\Dbar])^{1/2})}_2 \leq \norm{\Dbar - \expec[\Dbar]}^{1/2}_2$ since $\Dbar, \expec[\Dbar] \succ 0$ 
and $(\cdot)^{1/2}$ is operator monotone (see \cite[Theorem X.1.1]{bhatia_book}). 
Using $\norm{\Dbar^{-1/2}}_2 \leq \sqrt{\frac{2}{p\lambda_{\min}}}$, $\norm{(\expec[\Dbar])^{-1/2}}_2 \leq \sqrt{\frac{1}{p\lambda_{\min}}}$ and the bound 
from Lemma \ref{lem:Dbar_conc} in \eqref{eq:ztil2_bd_1}, we get
\begin{align}
\norm{\Ztil_2}_2 \leq C_1^{2}(\eta,M) \frac{(n p \log n)^{1/4}}{p\lambda_{\min}} \norm{\expec[H]}_2 
+ \frac{2\sqrt{2} C_1(\eta,M) \norm{\expec[H]}_2 (n p \log n)^{1/4}}{(p\lambda_{\min})^{3/2}},  \label{eq:ztil2_bd_2}
\end{align}
where $C_1(\eta,M) = 4 A^{1/4}(\eta,M)$. From \eqref{eq:ztil1_bd_1}, \eqref{eq:ztil2_bd_2}, we get 
\begin{align*}
\norm{\Ztil}_2 &\leq \frac{2\Delta}{p\lambda_{\min}} + \frac{C_1(\eta,M) \norm{\expec[H]}_2 (n p \log n)^{1/4}}{(p\lambda_{\min})^{3/2}} 
\left[\frac{C_1(\eta,M) (n p \log n)^{1/4}}{\sqrt{p\lambda_{\min}}} + 2\sqrt{2} \right] \\
&\leq 16M\sqrt{\frac{5}{3} pn} \frac{(2+\varepsilon)}{p\lambda_{\min}} + \frac{ C_1(\eta,M) (np\log n)^{1/4} 
\sigma}{(p\lambda_{\min})^{3/2}} \left(\frac{C_1(\eta,M) (n p \log n)^{1/4}}{\sqrt{p\lambda_{\min}}} + 2\sqrt{2} \right) \quad \text{ (using Lemma \ref{lem:specnorm_Z_ERO},\eqref{eq:svdn_decomp_1}) } \\
&= \Deltil. 
\end{align*}
Let $\Uhat = [\uhat_1 \ \uhat_2]$ denote the top two left singular vectors of $\Dbar^{-1/2} H \Dbar^{-1/2}$. Then from Wedin's bound, 
we know that 
\begin{equation}
\Deltil < \sigma \Rightarrow \norm{(I - \Uhat\Uhat^T) U}_2 \leq \frac{\Deltil}{\sigma - \Deltil} = \delta.
\end{equation}
Finally, Lemma \ref{lem:proj_analysis_Step_ERO} can be used here unchanged. Hence if $\delta \leq 1/2 \Leftrightarrow \Deltil \leq \sigma/3$, 
then there exists $ \beta \in \{ -1, 1 \}$ such that $\norm{\tilde{u}_2 - \beta  u_2}_2^2 \leq 10 \delta$. 
Since $\Deltil \leq \sigma_{\min}/3 \Leftrightarrow \delta \leq \frac{3\Deltil}{2\sigma_{\min}} \Rightarrow \Deltil \leq \sigma/3$, 
we obtain the stated bound on $\norm{\tilde{u}_2 - \beta  u_2}_2^2$. The lower bound on the success probability follows by applying the union bound to the events in Lemmas \ref{lem:specnorm_Z_ERO}, \ref{lem:Dbar_conc}.

%----------------------------------------
% Proof of l_2 score recovery for SVDN
%----------------------------------------
\subsection{Proof of Theorem \ref{thm:score_rec_l2_svdn}} \label{subsec:proof_score_rec_l2_svdn}
Recall that $\tilde{w} = \frac{\est{\sigma_1}}{\eta p \norm{\Dbar^{-1/2} e}_2} \Dbar^{1/2} \tilde{u}_2$ and $\rtil = \tilde{w} - \frac{e^T \tilde{w}}{n} e$. Let us denote 
\begin{align*}
    w &= \frac{\sigma}{\eta p \norm{(\expec[\Dbar])^{-1/2} e}_2} (\expec[\Dbar])^{1/2} \beta u_2
    = (r - \alpha e) \beta, \\ 
    w' &= w - \frac{e^T w}{n} e = \beta(r - (\alpha + \alpha')e),
\end{align*}
where $\alpha' = \frac{e^T(r-\alpha e)}{n}$. 
Then using triangle inequality, we can bound $\norm{\rtil - w'}_2$ as
$$\norm{\rtil - w'}_2 \leq \norm{w - \tilde{w}}_2 + \abs{\frac{e^T(\tilde{w} - w)}{n}}\norm{e}_2 \leq 2\norm{w - \tilde{w}}_2.$$ 
In order to bound $\norm{w - \tilde{w}}_2$, note that we can write $\tilde{w} - w = \frac{1}{\eta p}(a_1 + a_2 + a_3 + a_4)$ where
\begin{align*}
a_1 &= \frac{\est{\sigma_1}}{ \norm{\Dbar^{-1/2} e}_2} \Dbar^{1/2} \tilde{u}_2 - \frac{\est{\sigma_1}}{ \norm{\Dbar^{-1/2} e}_2} (\expec[\Dbar])^{1/2} \tilde{u}_2,    \\ 
a_2 &= \frac{\est{\sigma_1}}{ \norm{\Dbar^{-1/2} e}_2} (\expec[\Dbar])^{1/2} \tilde{u}_2 - 
\frac{\est{\sigma_1}}{ \norm{(\expec[\Dbar])^{-1/2} e}_2} (\expec[\Dbar])^{1/2} \tilde{u}_2, \\
a_3 &= \frac{\est{\sigma_1}}{ \norm{(\expec[\Dbar])^{-1/2} e}_2} (\expec[\Dbar])^{1/2} \tilde{u}_2 - \frac{\sigma}{ \norm{(\expec[\Dbar])^{-1/2} e}_2} (\expec[\Dbar])^{1/2} \tilde{u}_2, \\
a_4 &= \frac{\sigma}{ \norm{(\expec[\Dbar])^{-1/2} e}_2} (\expec[\Dbar])^{1/2} \tilde{u}_2 - \frac{\sigma}{ \norm{(\expec[\Dbar])^{-1/2} e}_2} (\expec[\Dbar])^{1/2} \beta u_2.
\end{align*}
Since $\norm{\tilde{w} - w}_2 \leq \frac{1}{\eta p} \sum_{i=1}^4 \norm{a_i}_2$ we will now bound $\norm{a_i}_2$ for each $i$. Before proceeding, recall from the proof of Theorem \ref{sec:proofOutline_thm_svdn_l2_ERO} that $\min_i \Dbar_{ii} \geq \frac{p \lambda_{\min}}{2}$. 
We also have $$\max_{i} \Dbar_{ii} \leq p\lambda_{\max} + \frac{p \lambda_{\min}}{2} \leq \frac{3p\lambda_{\max}}{2},$$ and $\est{\sigma_1} \leq \sigma + \Deltil$ where the latter is due to Weyl's inequality.
\paragraph{Bounding $\norm{a_1}_2$.} Since $\sigma \leq \sigma_{\max}$, we have that
\begin{align} \label{eq:temp_recsvdn_1}
    \norm{a_1}_2 \leq \frac{\Deltil + \sigma_{\max}}{\norm{\Dbar^{-1/2} e}_2} \norm{\Dbar^{1/2} - (\expec[\Dbar])^{1/2}}_2 \leq
    \frac{\Deltil + \sigma_{\max}}{\norm{\Dbar^{-1/2} e}_2} \norm{\Dbar - \expec[\Dbar]}^{1/2}_2
\end{align}
where the second inequality is due to $\Dbar, \expec[\Dbar] \succ 0$ 
and since $(\cdot)^{1/2}$ is operator monotone (see \cite[Theorem X.1.1]{bhatia_book}). Using the bound $\norm{\Dbar^{-1/2} e}_2 \geq \sqrt{\frac{2n}{3p \lambda_{\max}}}$, along with the bound for $\norm{\Dbar - \expec[\Dbar]}_2$ (from Lemma \ref{lem:Dbar_conc}) in \eqref{eq:temp_recsvdn_1}, we arrive at 
\begin{align*}
    \norm{a_1}_2 \leq \sqrt{\frac{3}{2n} p \lambda_{\max}} (\sigma_{\max} + \Deltil) (2\sqrt{2} + 1)^{1/2} (A(\eta,M) np \log n)^{1/4}.
\end{align*}

\paragraph{Bounding $\norm{a_2}_2$.} We begin by noting that  
\begin{equation*}
    \norm{a_2}_2 = \est{\sigma_1} \abs{\frac{1}{\norm{\Dbar^{-1/2} e}_2} - \frac{1}{\norm{(\expec[\Dbar])^{-1/2} e}_2}} \norm{(\expec[\Dbar])^{1/2} \tilde{u}_2}_2 
    \leq (\sigma_{\max} + \Deltil) \abs{\frac{\norm{(\expec[\Dbar])^{-1/2} e}_2 - \norm{\Dbar^{-1/2} e}_2}{\norm{\Dbar^{-1/2} e}_2 \norm{(\expec[\Dbar])^{-1/2} e}_2}}.
\end{equation*}
Since $\norm{\Dbar^{-1/2} e}_2 \geq \sqrt{\frac{2n}{3p \lambda_{\max}}}$,  $\norm{(\expec[\Dbar])^{-1/2} e}_2 \geq \frac{\sqrt{n}}{\sqrt{p \lambda_{\max}}}$, and 
\begin{align*}
    \abs{\norm{(\expec[\Dbar])^{-1/2} e}_2 - \norm{\Dbar^{-1/2} e}_2} 
    &\leq \norm{(\Dbar^{-1/2} - (\expec[\Dbar])^{-1/2} ) e}_2 \\
    &\leq \norm{\Dbar^{-1/2}}_2 \norm{\Dbar^{1/2} - (\expec[\Dbar])^{1/2}}_2 \norm{(\expec[\Dbar])^{-1/2}}_2 \sqrt{n} \\
    &\leq \frac{\sqrt{2}}{p\lambda_{\min}} (2\sqrt{2} + 1)^{1/2} (A(\eta,M) np \log n)^{1/4},
\end{align*}
we can bound $\norm{a_2}_2$ as
\begin{equation*}
    \norm{a_2}_2 \leq \sqrt{\frac{3}{n}} (\sigma_{\max} + \Deltil) \frac{\lambda_{\max}}{\lambda_{\min}} (2\sqrt{2} + 1)^{1/2} (A(\eta,M) np \log n)^{1/4}.
\end{equation*}

\paragraph{Bounding $\norm{a_3}_2$.} This is easily achieved by noting that 
\begin{equation*}
    \norm{a_3}_2 = \frac{\abs{\est{\sigma_1} - \sigma}}{\norm{(\expec[\Dbar])^{-1/2} e}_2} \norm{(\expec[\Dbar])^{1/2} \tilde{u}_2}_2 \leq \frac{\Deltil \sqrt{p \lambda_{\max}}}{\sqrt{n}/\sqrt{p\lambda_{\max}}} = \frac{\Deltil p \lambda_{\max}}{\sqrt{n}}.
\end{equation*}
\paragraph{Bounding $\norm{a_4}_2$.} This is also easily achieved by noting that
\begin{align*}
\norm{a_4}_2 &\leq \frac{\sigma}{ \norm{(\expec[\Dbar])^{-1/2} e}_2} \norm{(\expec[\Dbar])^{1/2}}_2 \norm{\tilde{u}_2 - \beta u_2}_2 \\
&\leq \frac{\sigma_{\max} \sqrt{p \lambda_{\max}}}{\sqrt{n}/\sqrt{p \lambda_{\max}}} \left( \sqrt{15} \sqrt{\frac{\Deltil}{\sigma_{\min}}}\right) \\
&\leq \sqrt{\frac{15 \Deltil}{\sigma_{\min} n}} \sigma_{\max} p \lambda_{\max}.
\end{align*}
The stated bound now follows from $\norm{\rtil - \beta(r - (\alpha + \alpha') e)}_2 \leq \frac{2}{\eta p} \sum_{i=1}^4 \norm{a_i}_2$.

% Numerical experiments: synthetic and real %% HT: Update this section !!

%--------------------------------------------------------
% Matrix completion as a preprocessing step for ranking
%
\section{Matrix completion as a preprocessing step for ranking}  \label{sec:matrixCompletion}

\textit{Low-rank matrix completion} is the problem of recovering the missing entries of a low-rank 
matrix given a subset of its entries. This line of research started with the results in \cite{candes_recht} and \cite{candes_Tao_MC} which showed that given a rank-$r$ matrix $C$ of size $n_1 \times n_2$ (with $n = \max\set{n_1,n_2}$), one can recover it by observing only $O(n r \;  \text{polylog}(n))$ randomly selected entries (under some assumptions on $C$) via a simple convex optimization algorithm. This was partly inspired by similar approaches used previously in the compressed sensing literature  \cite{candes_Romberg_Tao, Candes2006}.
This problem has received tremendous attention in the last decade, due to a number of applications such as in phase retrieval \cite{phase_retrieval_candes_MC}, computer vision \cite{tomasi1992shape,ozyecsil2017survey} and sensor network localization \cite{asap2d} to name a few.

More formally, let us assume for simplicity that $r = O(1)$ and the SVD of $C = \sum_{i \in [r]} \sigma_i u_i v_i^T$ satisfies
\begin{equation} \label{eq:coherence_term}
    \norm{u_i}_{\infty} \leq \sqrt{\frac{\mu}{n_1}}, \quad  \norm{v_i}_{\infty} \leq \sqrt{\frac{\mu}{n_2}}; \quad \forall i,j,
\end{equation}
for some $\mu \geq 1$. Here $\mu$ is a measure of how spread out the entries of $u_i, v_i$ are - the smaller the value of $\mu$, the better. Say we observe $m$ entries of $C$ on a subset $\Omega \subset [n_1] \times [n_2]$, sampled uniformly at random. Denoting $P_{\Omega}: \matR^{n_1 \times n_2} \rightarrow \matR^{n_1 \times n_2}$ to be the projection operator on $\Omega$, it was shown in \cite{candes_Tao_MC} that the solution $\est{C}$ of 
\begin{equation} \label{eq:nuc_norm_MC}
    \min \norm{X}_{*} \quad \text{s.t} \quad P_{\Omega}(X) = P_{\Omega}(C)
\end{equation}
equals $C$ with high probability, provided $m = \Omega(\mu^4 n \log^2 n)$. The nuclear norm minimization problem \eqref{eq:nuc_norm_MC} is a SDP and hence can be solved in polynomial time using, for eg., interior point methods. In fact, the result holds under a different Bernoulli sampling model too,  wherein each entry of $M$ is observed independently with a certain probability \cite[Section 4.1]{candes_recht}. Moreover, say that the observations are noisy, i.e, we observe %
\begin{equation*}
    Y_{i,j} = C_{i,j} + Z_{i,j}; \quad (i,j) \in \Omega \quad \Leftrightarrow  \quad P_{\Omega}(Y) = P_{\Omega}(C) + P_{\Omega}(Z), 
\end{equation*}
where $Z$ is the noise matrix. Say $\norm{P_{\Omega}(Z)}_F \leq \delta$. Then, it was shown in \cite{Candes10MCnoise} (under certain additional conditions) that the solution of
\begin{equation} \label{eq:nuc_norm_MC_noise}
    \min \norm{X}_{*} \quad \text{s.t} \quad \norm{P_{\Omega}(X) - P_{\Omega}(Y)}_F \leq \delta
\end{equation}
is stable, i.e., the estimation error $\norm{\est{C} - C}_{F}$ is bounded by a term proportional to $\delta$.

In our setting, we observe (noisy versions of) a subset of the entries of the $n \times n$ matrix $C = r e^T - e r^T$ where each off-diagonal entry of $C$ is revealed with probability $p$. Since $C$ has rank $2$, it is natural to consider estimating $C$ via matrix completion as a preprocessing step, and then subsequently applying \textsc{SVD-RS} or \textsc{SVD-NRS} on the obtained estimate of $C$ for recovering the underlying ranks and scores. In order to understand the sample complexity for successful matrix completion, we need to express $\mu$ defined in \eqref{eq:coherence_term} in terms of the score vector $r$. To this end, we see from Lemma \ref{lem:sing_vals_C} that
\begin{equation*}
 \norm{u_2}_{\infty}, \norm{v_1}_{\infty} = \frac{\norm{r - \alpha e}_{\infty}}{\norm{r - \alpha e}_2} \leq \frac{M - \alpha}{\norm{r - \alpha e}_2}   
\end{equation*}
and $\norm{u_1}_{\infty}, \norm{v_2}_{\infty} = 1/\sqrt{n}$. Hence it follows that
\begin{equation*}
\mu = \max\set{\frac{(M-\alpha)\sqrt{n}}{\norm{r - \alpha e}_2}, 1}.   
\end{equation*}
In order to get a sense of the scaling of $\mu$, consider the setting $r_i = i$. As seen before, we then have $\norm{r - \alpha e}_2 = \Theta(n^{3/2})$, $M = n$, and $\alpha = \Theta(n)$. This in turn implies that $\mu = \Theta(1)$ which is the ideal scenario for matrix completion. In our numerical experiments, we rely on the TFOCS software library  \cite{tfocs}, that allows for construction of first-order methods for a variety of convex optimization problems \cite{becker2011templates}. 
In our implementation, we also set the diagonal entries of $X$ to be equal to zero since the same is true for $C$. We do not enforce the skew-symmetry constraints $X_{ij} = - X_{ji}$ for $(i,j) \in \Omega$,  so the solution $\est{C}$ to $C$ is not guaranteed to be skew-symmetric. Instead, we output $\frac{\est{C} - \est{C}^T}{2}$ as the final (skew-symmetric) estimate obtained from the preprocessing step. Note that since $C^T = -C$, we have
\begin{equation*}
    \left\|\frac{\est{C} - \est{C}^T}{2} - C \right\|_F = \left\|\frac{\est{C} - \est{C}^T - C + C^T}{2}\right\|_F \leq \norm{C - \est{C}}_F.
\end{equation*}

\begin{remark}
Note that in the setting of matrix completion, we perform the scaling recovery procedure outlined in Section \ref{sec:scaleRecovery} as follows. 
We build the matrix $\Pi$ from equation \eqref{eq:PI_ij}  by only consider the entries/edges $\{i,j\} \in E$ from the original measurement graph $G$, and do not include the entries filled in during the matrix completion step.
\end{remark}

Finally, we end with a motivating discussion for the applicability of the matrix completion approach in the setting of ranking from pairwise comparisons. 
The recent work \cite{udell_SIMODS_lowRank} of Udell and Townsend provides a compelling argument on why big data matrices are approximately low-rank. The authors consider a simple generative model for matrices, assuming that each row or column is associated to a (possibly high dimensional) bounded latent variable, with the individual matrix entries being generated by applying a piecewise analytic function to these latent variables.  While the initial resulting matrices are typically full rank, the authors show that one can approximate every entry of an $m \times  n$ matrix drawn from the above model to within a fixed absolute error by a low-rank matrix whose rank grows as $O (\log(m + n))$. In other words, any sufficiently large matrix from such a latent variable model can be approximated by a low-rank matrix (up to a small entrywise error).
The paradigm that \textit{``nice latent variables models are of log-rank''} \cite{udell_SIMODS_lowRank} is also applicable in the ranking setting, where one assumes that the final ranking (or skill) of each player varies smoothly as a function of covariate information, which is typically available in many real word applications. For example, in sport or clinical data, covariates may provide additional information about both the ranked players and the rankers/judges, which can overall lead to better aggregated results.

\section{Numerical Experiments}  \label{sec:num_experiments}

We compare the performance of Algorithm \ref{algo:SVD_Rank_sync} \textsc{SVD-RS} (\textsc{SVD} in the figure legends for brevity) and Algorithm \ref{algo:SVDN_Rank_sync} \textsc{SVD-NRS} (\textsc{SVD-N} in the figure legends) with that of seven other algorithms from the literature, namely 
% which we briefly summarize below
% \subsection{Summary of algorithms we compare against}
\textsc{RowSum Ranking} (\textsc{RSUM})  \cite{gleich2011rank}, 
\textsc{Least-Squares-Rank} (\textsc{LS}) as considered in \cite{syncRank}, 
\textsc{Serial-Rank} (\textsc{SER})  \cite{fogel2016spectral}, 
\textsc{Spring-Rank} (\textsc{SPR})  \cite{CaterinaDeBacco_Ranking}, 
\textsc{Bradley–Terry } (\textsc{BTL})  \cite{BradleyTerry1952}, 
\textsc{Page-Rank} (\textsc{PGR})   \cite{Pageetal98}, and  
\textsc{Sync-Rank} (\textsc{SYNC})  \cite{syncRank}. We refer the reader to Section \ref{sec:relatedWork_ranking} for a brief survey of the ranking literature, including the above algorithms.
% as an instance of the group synchronization problem
We compare the performance of all algorithms on synthetic data in Section \ref{sec:num_synthetic}, and on real data in Section \ref{sec:num_real}. We consider a variety of performance metrics summarized further below, altogether highlighting the competitiveness of our proposed SVD-based algorithms with that of state-of-the-art methods. 
% We defer to the appendix results on a synthetic multiplicative uniform noise model, as well as a comparison of algorithms on ordinal measurements (when only the sign information is available).

% \textsc{Serial-Rank} (\textsc{SER}) along with its extension to the generalized linear model setting (\textsc{GLM}) \cite{serialRank}, as well as Row-Sum (\textsc{RS}) and a least-squared based formulation (\textsc{LS}) as considered in \cite{syncRank}. This is done for synthetic data, under the ERO noise model in Section \ref{sec:num_synthetic}, and real data in Section \ref{sec:num_real}. 

%From a computational point of view, we remark that in most practical %applications where the matrix of pairwise comparisons is sparse, the %spectral methods run in almost linear time. Every iteration of the power %method is linear in the number of edges in the graph $G$ (i.e., the %number of pairwise comparisons), and the number of iterations is greater %than $O(1)$, as it depends on the spectral gap.

\subsection{Synthetic data} \label{sec:num_synthetic}
This section details the outcomes of synthetic numerical experiments, under the ERO model, where the measurement graph is \ER and the noise comes in the form of outliers from a uniform distribution.

We consider two sets of experiments, a first one where the strengths $r_i$ are uniformly distributed in $[0,1]$, and a second one where they are Gamma distributed with shape parameter $a=0.5$ and scale parameter $b=1$. This choice of parameters for the Gamma distribution leads to a skewed distribution of player strengths, and subsequently a skewed distribution of node degrees, which is a setting realistic in practice but known to be challenging for spectral methods. We fix the number of nodes ($n$), and vary the edge density ($p$), and noise level ($\gamma$) as follows.  

We broadly consider two main experimental settings.
\begin{itemize}
\item In Figure \ref{fig:unif_ERO_n1000} (uniform scores) and Figure \ref{fig:gamma_ERO_n1000} (Gamma scores) we consider a synthetic model with $n=1000$, with sparsity parameters $p=0.05$ (column 1) and $p=1$  (column 3). For the $p=0.05$ scenario, we also show the results after running all algorithms on top of a reprocessing step that applies low-rank matrix completion (column 2). 

\item In Figure \ref{fig:uniform_ERO_n3000} (uniform scores)  and Figure \ref{fig:gamma_ERO_n3000} (Gamma scores), we consider a synthetic model with $n=3000$, and sparsity parameter $p \in  \set{0.01, 0.05, 0.1}$, indexing the columns.   
\end{itemize}
Across all the above experiments, we consider three different performance metrics, as we vary the noise level $\gamma$ on the x-axis. 
% Figure \ref{fig:unif_ERO_n1000}, respectively Figure \ref{fig:gamma_ERO_n1000}, pertain to the experiment where the strengths are Gamma distributed without, respectively with, the matrix completion pre-processing  step. 
% Similarly, Figure \ref{fig:unif_ERO_nMC}, respectively Figure \ref{fig:unif_ERO_yMC}, concern the scenario where the strengths are drawn from a uniform distribution, without, respectively with, the matrix completion step.
Whenever ground truth is available, we plot in the top row of Figures \ref{fig:unif_ERO_n1000}, \ref{fig:gamma_ERO_n1000}, \ref{fig:uniform_ERO_n3000}, \ref{fig:gamma_ERO_n3000} the \textit{Kendall Distance} between the recovered strength vector $\hat{r}$ and the ground truth $r$, for different noise levels $\gamma$. The Kendall distance counts the number of pairs of candidates that are ranked in different order (flips), in the original ranking versus the recovered one.
The middle row of each Figure  plots the \textit{Correlation Score}, computed as the Pearson correlation between the ground truth $r$ and the recovered $\hat{r}$. %
Finally, the bottom row in each Figure plots the \textit{RMSE} error, defined as $ \sqrt{ \frac{1}{n} || r - \hat{r} ||_2} $, after having centered $r$ and $\hat{r}$. 
Note that the low-rank matrix completion is illustrated in the middle column of Figures \ref{fig:unif_ERO_n1000} and \ref{fig:gamma_ERO_n1000}.

\paragraph{Performance comparison.} For the remainder of this section, we compare and contrast the performance of our algorithms with that of the other methods, across four synthetic data sets, detailed below. 

\textbullet \; 
In Figure \ref{fig:unif_ERO_n1000} (uniform scores with $n=1000$): in the setting $p=0.05$, our methods perform better than \textsc{SER} but are in general outperformed by the other methods in terms of Kendall Distance (KD) and Correlation Score (CS) (except at very low levels of noise), and they perform on par with all the methods in terms of RMSE (in particular for $\gamma > 0.30$ our methods outperform \textsc{SER},  \textsc{PGR}, \textsc{BTL}, and  \textsc{SPR}). After the matrix completion step,   both \textsc{SVD} and  \textsc{SVD-N} outperform all other methods except  \textsc{LS} and \textsc{RSUM} in terms of RMSE, to which they are comparable (visually indistinguishable) up until $\gamma < 0.50$, and only slightly outperformed for higher noise. In terms of KD and CS, our two SVD-based algorithms  perform on par with most other methods, and clearly outperform both \textsc{SER} and \textsc{PGR}. Finally, for the complete graph case $p=1$, our methods perform at the top of the rankings in terms of KD, CS, and RMSE, and are only matched in performance or slightly outperformed depending on the noise regime, by \textsc{SYNC}, \textsc{LS}, and \textsc{RSUM}.

\textbullet \; 
In Figure \ref{fig:gamma_ERO_n1000} (Gamma scores with $n=1000$): for the very sparse regime \textsc{SYNC} clearly outperforms all other  methods, with the SVD-based methods outperforming only \textsc{SER}. After the matrix completion step, our proposed algorithms outperform all other methods except \textsc{LS} and \textsc{RSUM}. For the complete graph case, \textsc{SYNC} is by far the winner in terms of KD, with the remaining methods performing comparably, except \textsc{PGR}.  In terms of RMSE, our methods perform best along with \textsc{LS} and \textsc{RSUM}, and clearly outperform \textsc{SYNC}, \textsc{SPR}, \textsc{PGR}, \textsc{BTL} and \textsc{SER}.

\textbullet \; 
In Figure \ref{fig:uniform_ERO_n3000}  (uniform scores with $n=3000$): for $p=0.01$ there is very clear ordering of the methods, mostly consistent across all three performance metrics. We observe that  \textsc{SYNC}, \textsc{SPR}, \textsc{RSUM}, and  \textsc{LS}, are the top three best performing methods, followed by \textsc{BTL}, \textsc{PGR}, \textsc{SVD}, \textsc{SVD-N}, and  \textsc{SER}. At slightly higher edge density  $p=0.05$, our SVD methods perform better than  \textsc{SER}, \textsc{PGR}, and \textsc{BTL}, for the first part of the noise spectrum, while in terms of RMSE, \textsc{SVD} and \textsc{SVD-N}  clearly outperform    \textsc{SER}, \textsc{PGR}, \textsc{BTL}, \textsc{SPR}, and \textsc{SYNC}, and are surpassed only by \textsc{LS} and \textsc{RSUM}. The relative performance is similar for  $p=0.10$.

\textbullet \; 
Finally, we comment on the results from Figure  \ref{fig:gamma_ERO_n3000}  (Gamma scores with $n=3000$): for $p=0.01$, \textsc{SYNC} is the clear winner in terms of KD, CS, and RMSE, while \textsc{SER}, \textsc{SVD} and \textsc{SVD-N} are at the bottom of the ranking. The relative ordering is roughly preserved for higher edge densities, with the comment that \textsc{SER} is clearly the worst performer, while \textsc{SYNC} is by far the most accurate method, especially in terms of KD, where it outperforms all other methods by a large margin.

\newcommand{\wid}{2.2in} 
\newcolumntype{C}{>{\centering\arraybackslash}m{\wid}}

% \newcolumntype{C}{>{\centering\arraybackslash}m{\wid}} 
\begin{table*}\sffamily 
\hspace{-9mm} 
\begin{tabular}{l*3{C}@{}}
& $p=0.05$ & $p=0.05$ + Matrix Completion  & $p=1$  \\
& \includegraphics[width=0.35\columnwidth]{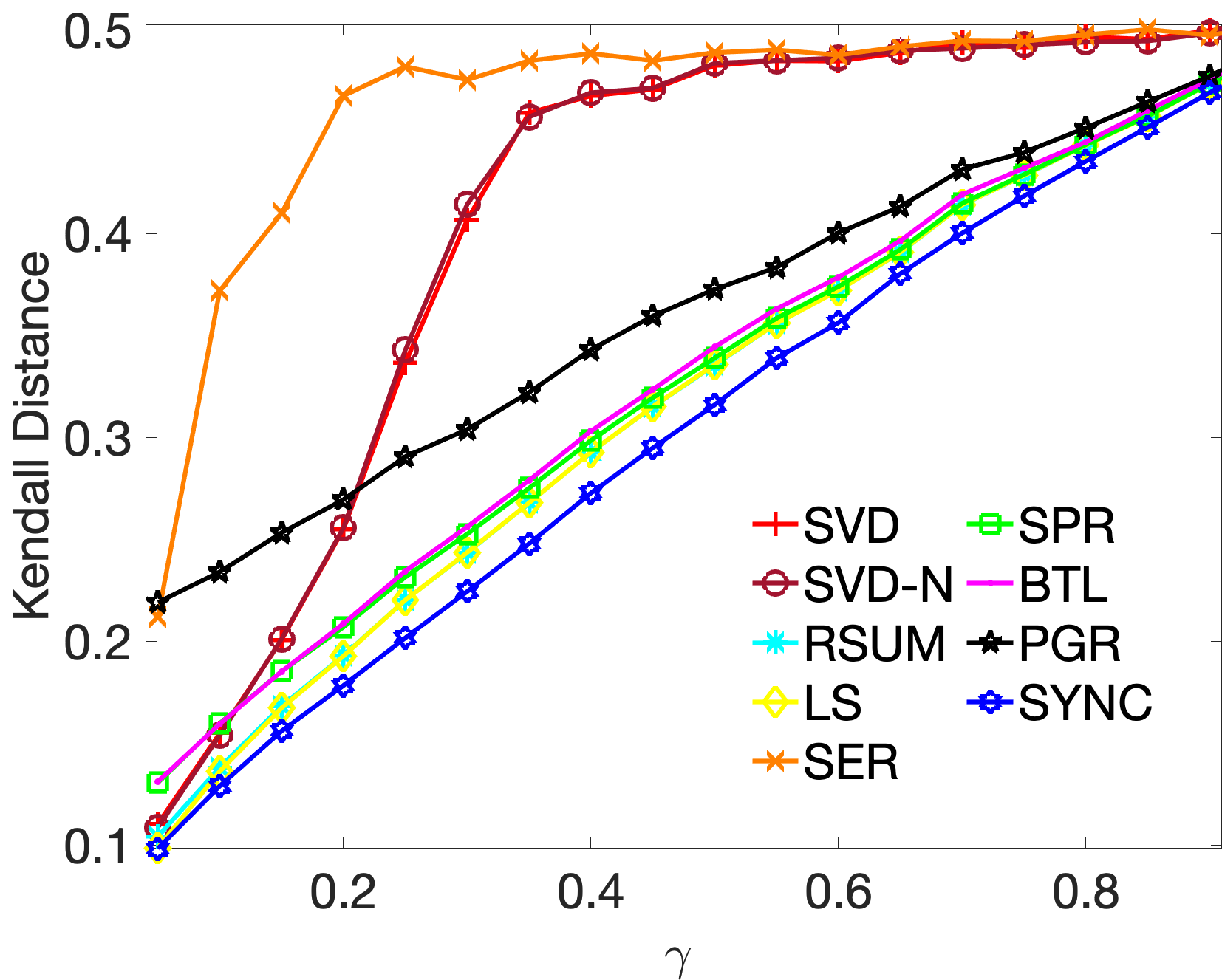}
& \includegraphics[width=0.35\columnwidth]{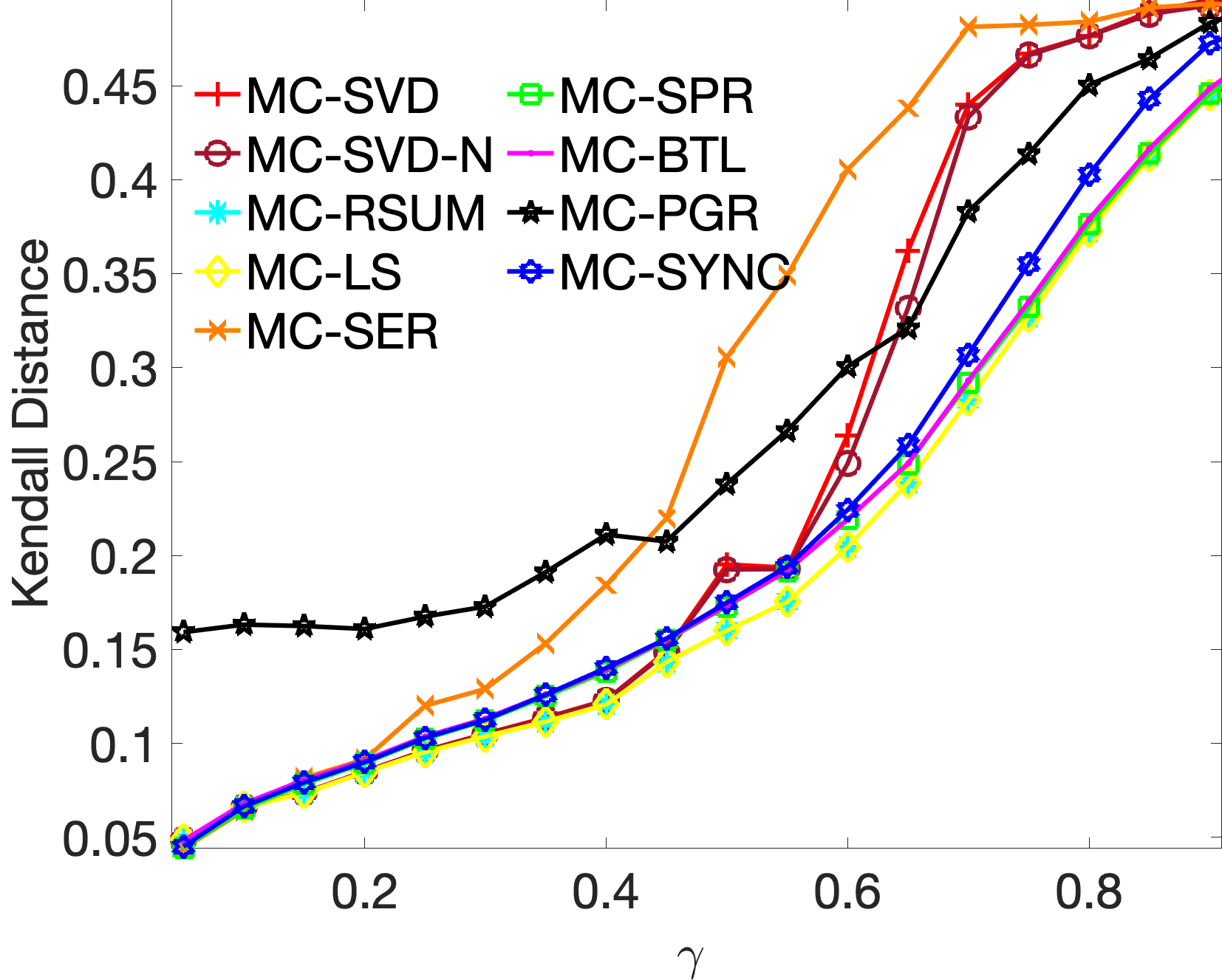}
& \includegraphics[width=0.35\columnwidth]{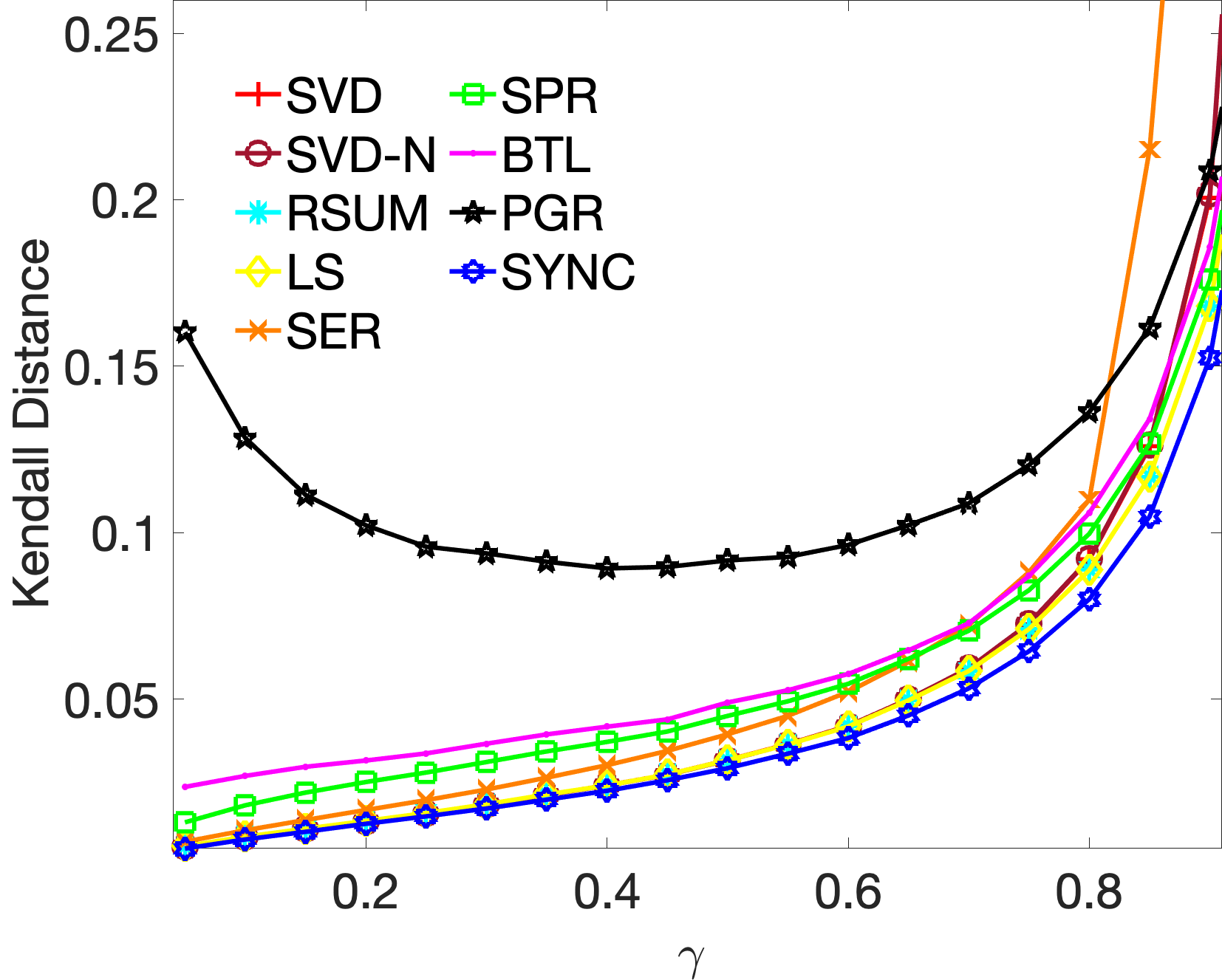}  \\ 
& \includegraphics[width=0.35\columnwidth]{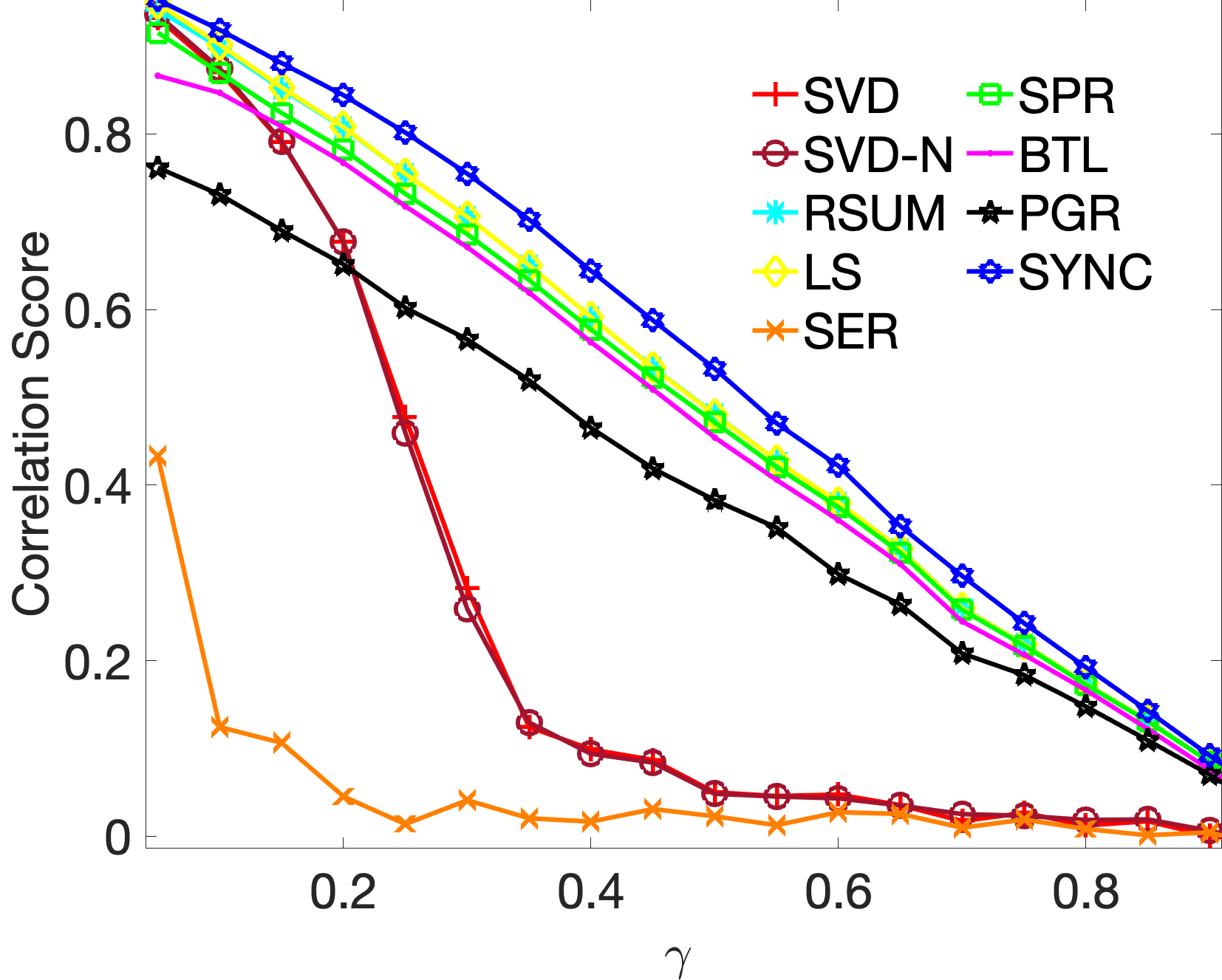}
& \includegraphics[width=0.35\columnwidth]{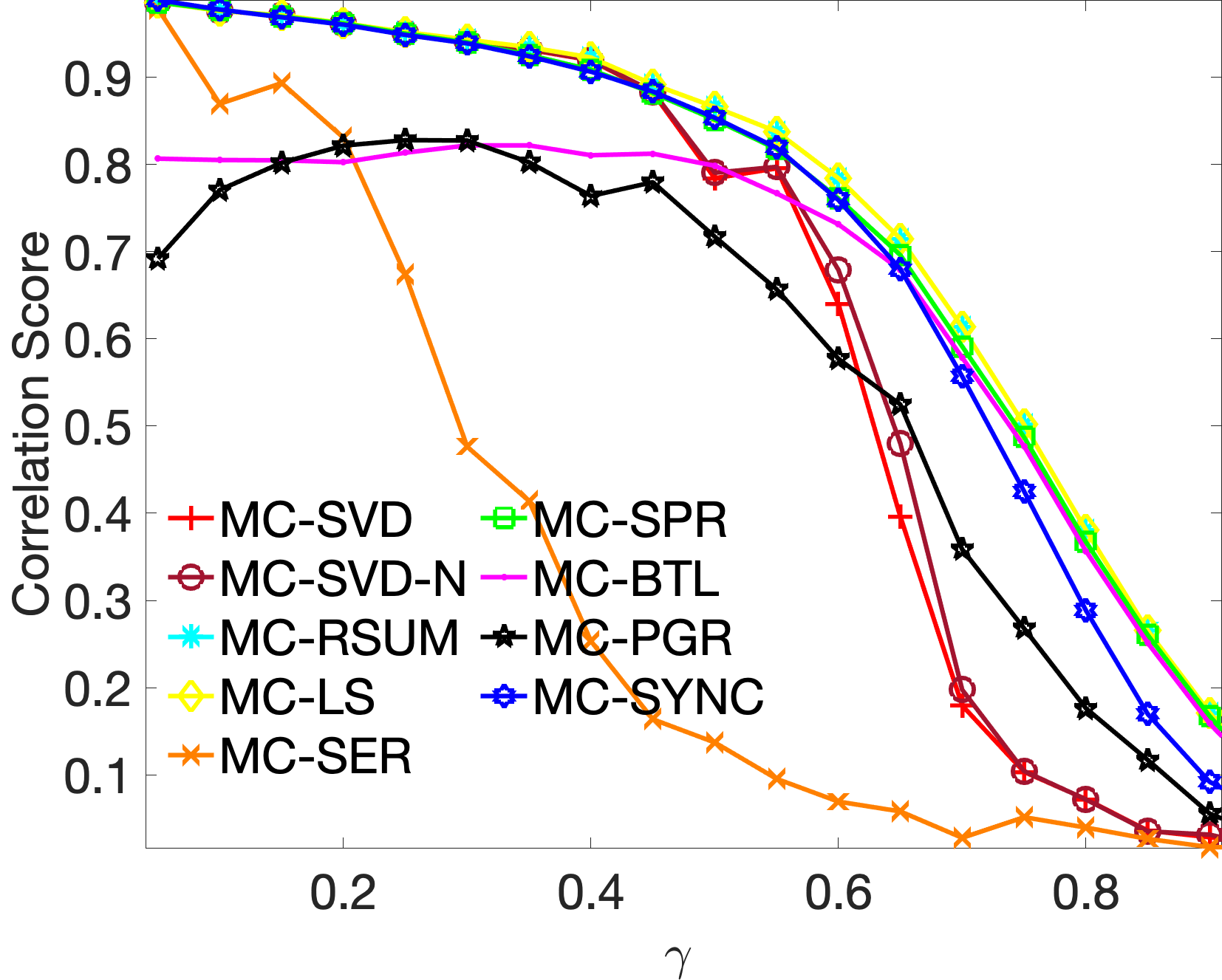}
& \includegraphics[width=0.35\columnwidth]{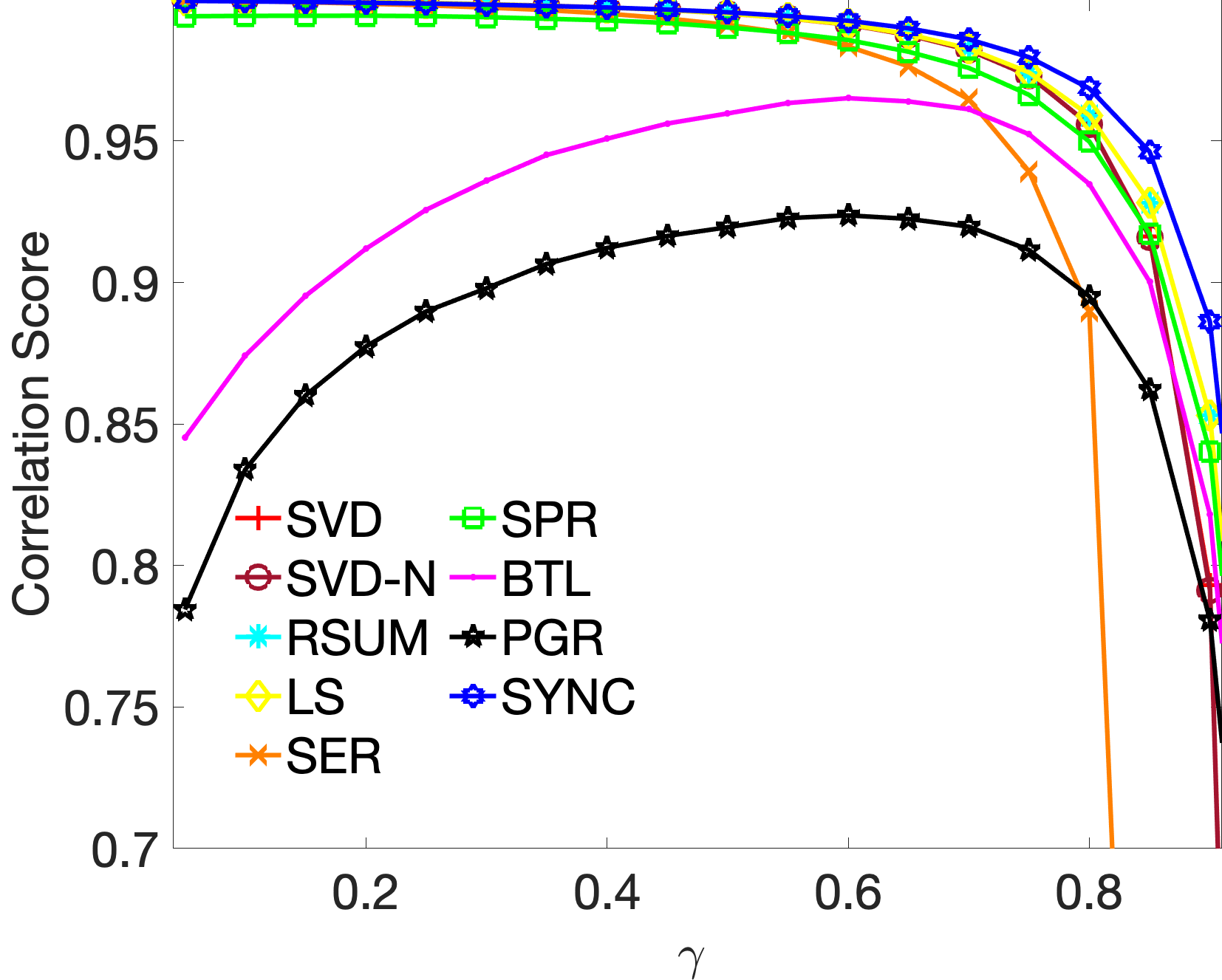} \\
& \includegraphics[width=0.35\columnwidth]{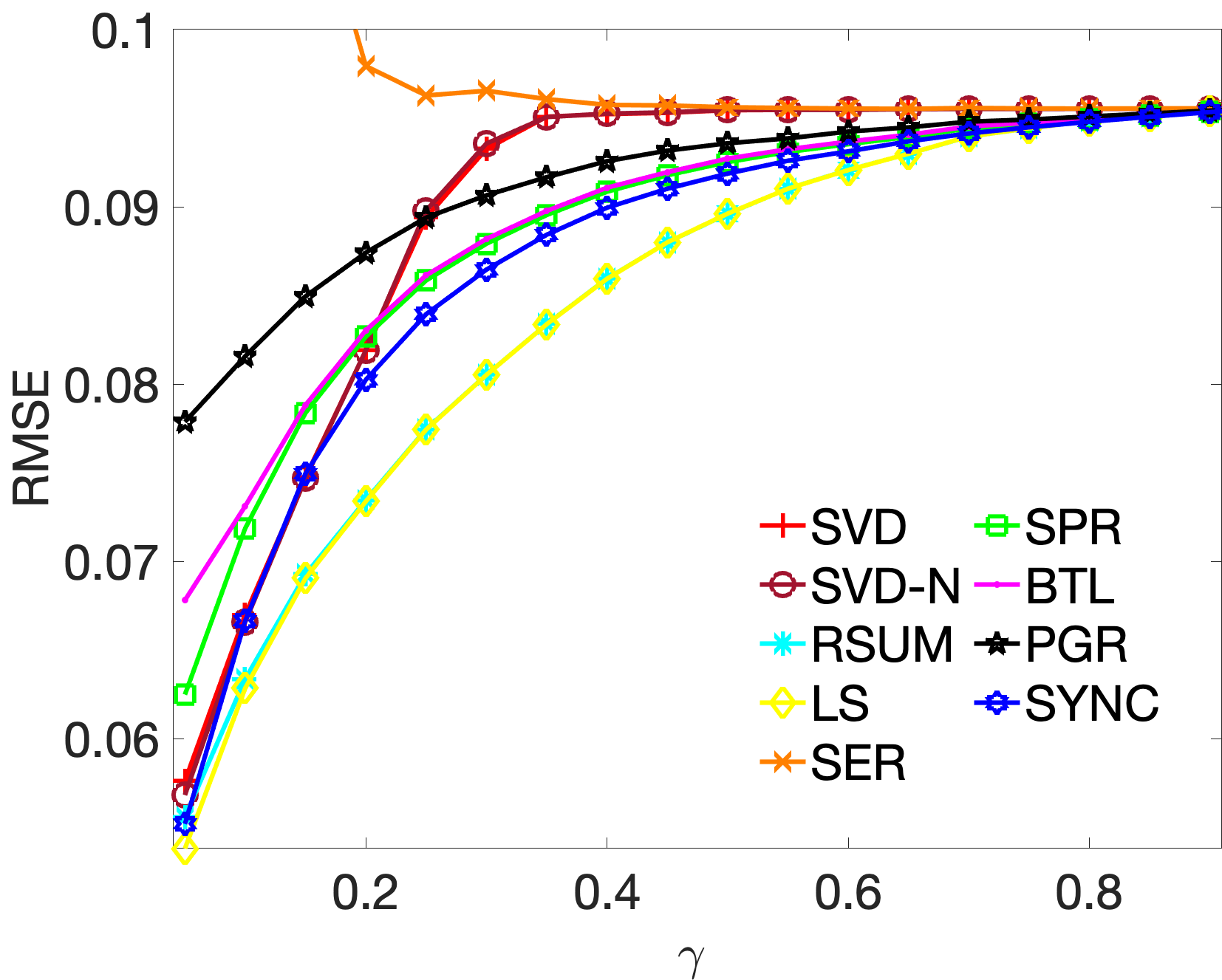}
& \includegraphics[width=0.35\columnwidth]{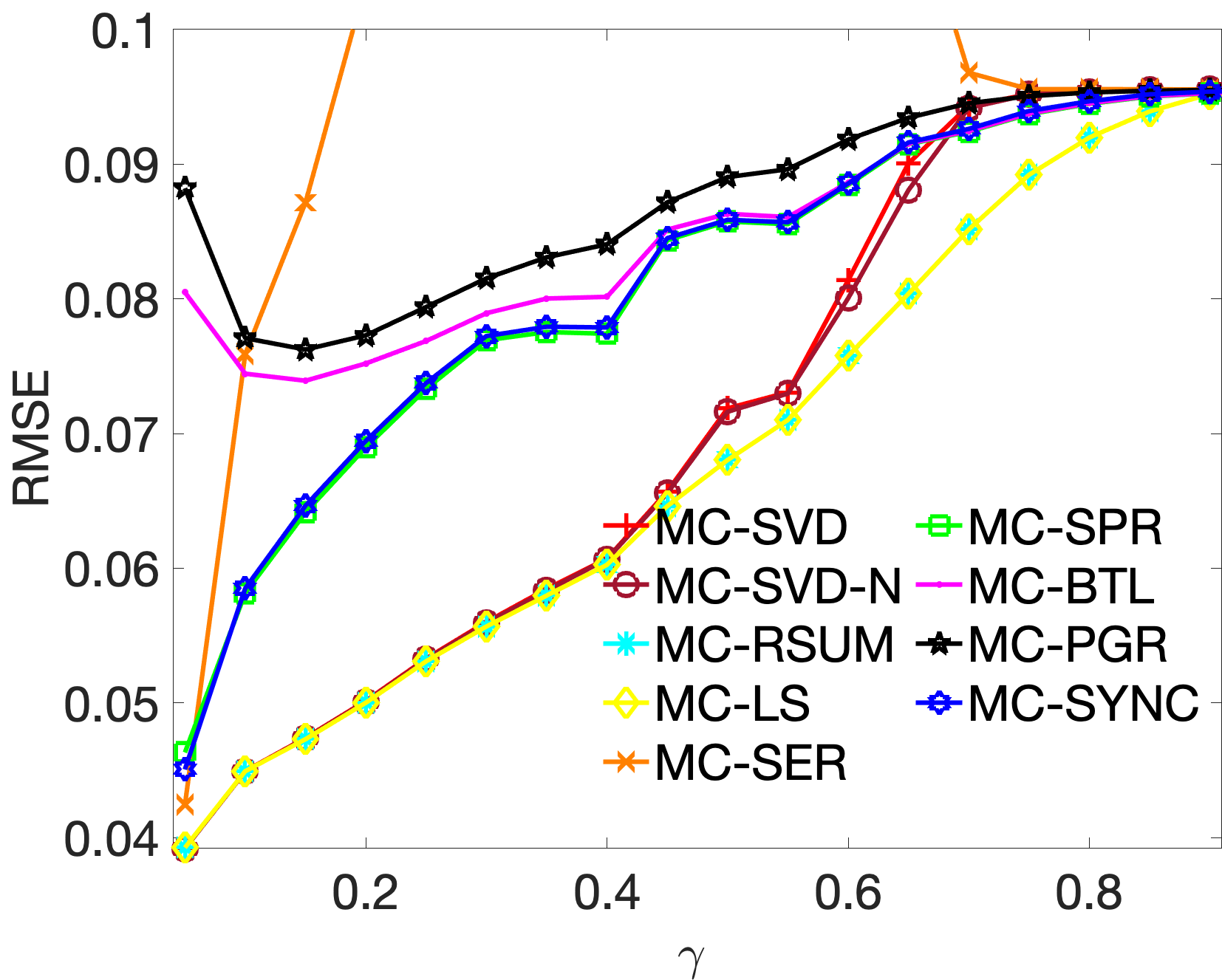}
& \includegraphics[width=0.35\columnwidth]{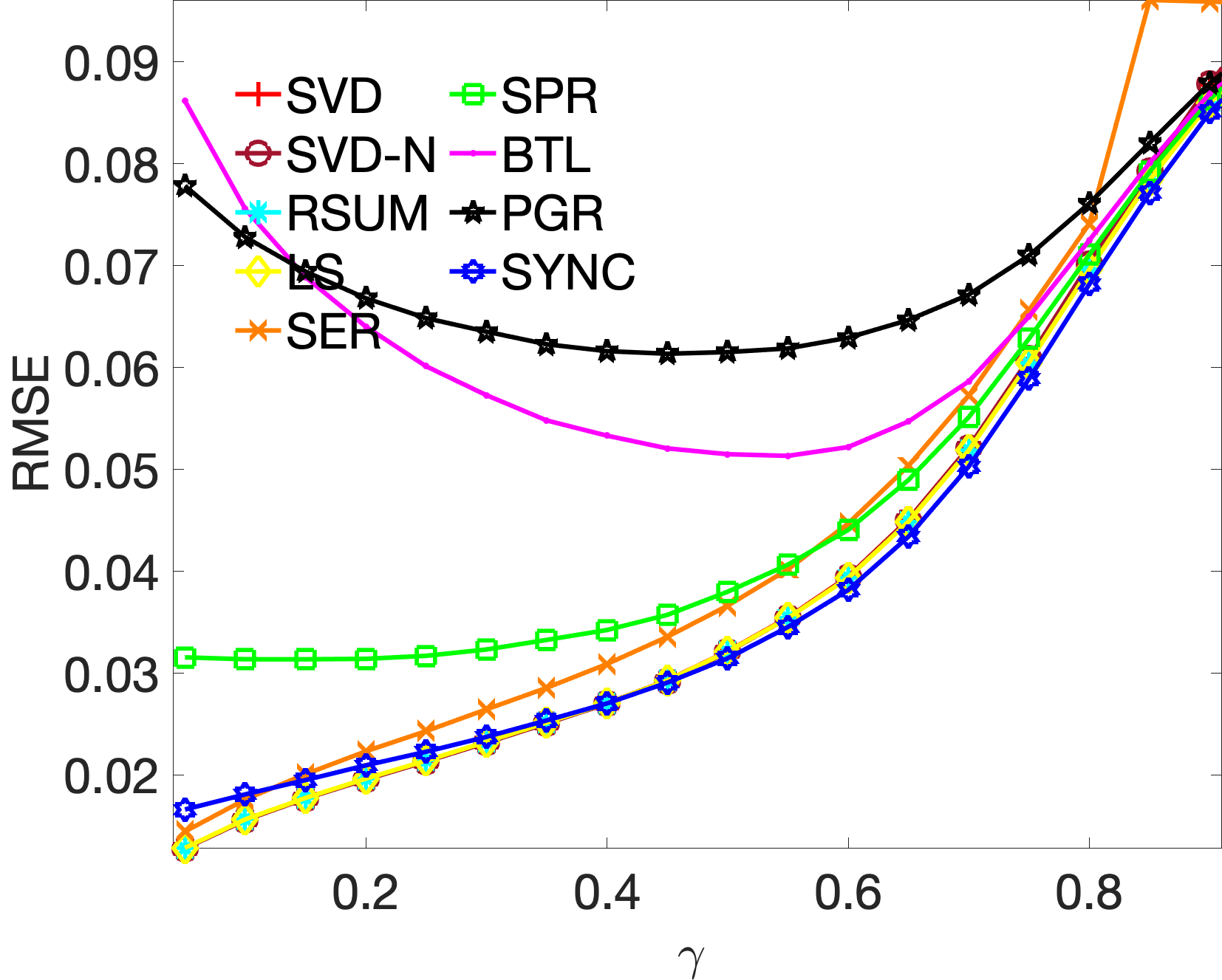}
\end{tabular}
\captionsetup{width=0.99\linewidth}
\vspace{-2mm}
\captionof{figure}{Performance statistics in terms of Kendall distance (top row; lower is better), correlation score (middle row; higher is better), and RMSE (bottom row; lower is better), for synthetic data with scores generated from a \textbf{Uniform} distribution  with $n=1000$. The first and third columns pertain to the case of sparsity $p=0.05$, respectively $p=1$,  without the matrix completion step, while the middle column, for $p=0.05$, employs a matrix-completion preprocessing step. Results are averaged over 20 runs.
}
\label{fig:unif_ERO_n1000}
\end{table*}

% Example nice table with rownames and column names.
% needs \usepackage{array}  and maybe also  \usepackage{multirow,bigdelim}

\begin{table*}\sffamily 
\hspace{-9mm} 
\begin{tabular}{l*3{C}@{}}
& $p=0.05$ & $p=0.05$ + Matrix Completion  & $p=1$  \\
& \includegraphics[width=0.35\columnwidth]{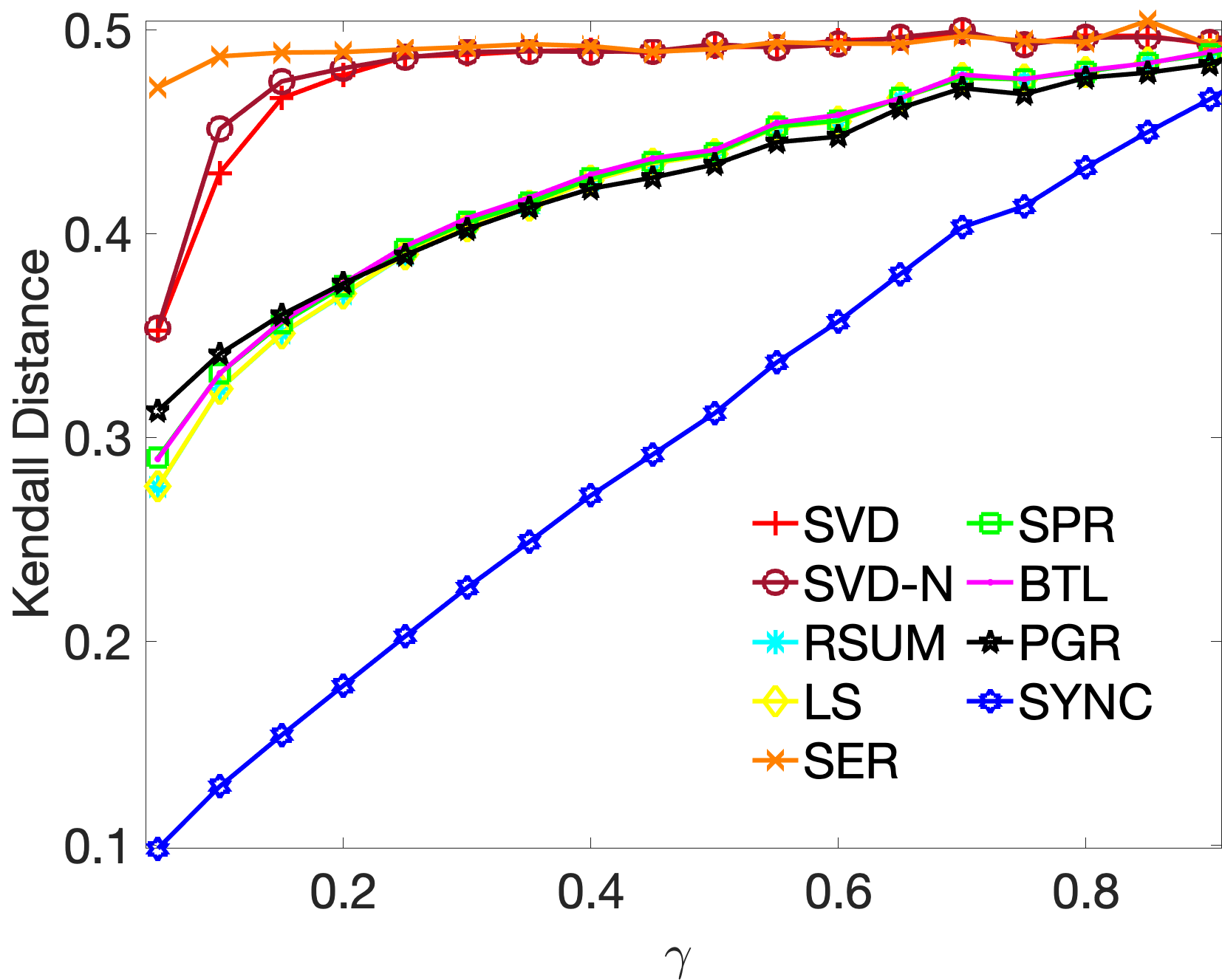}
& \includegraphics[width=0.35\columnwidth]{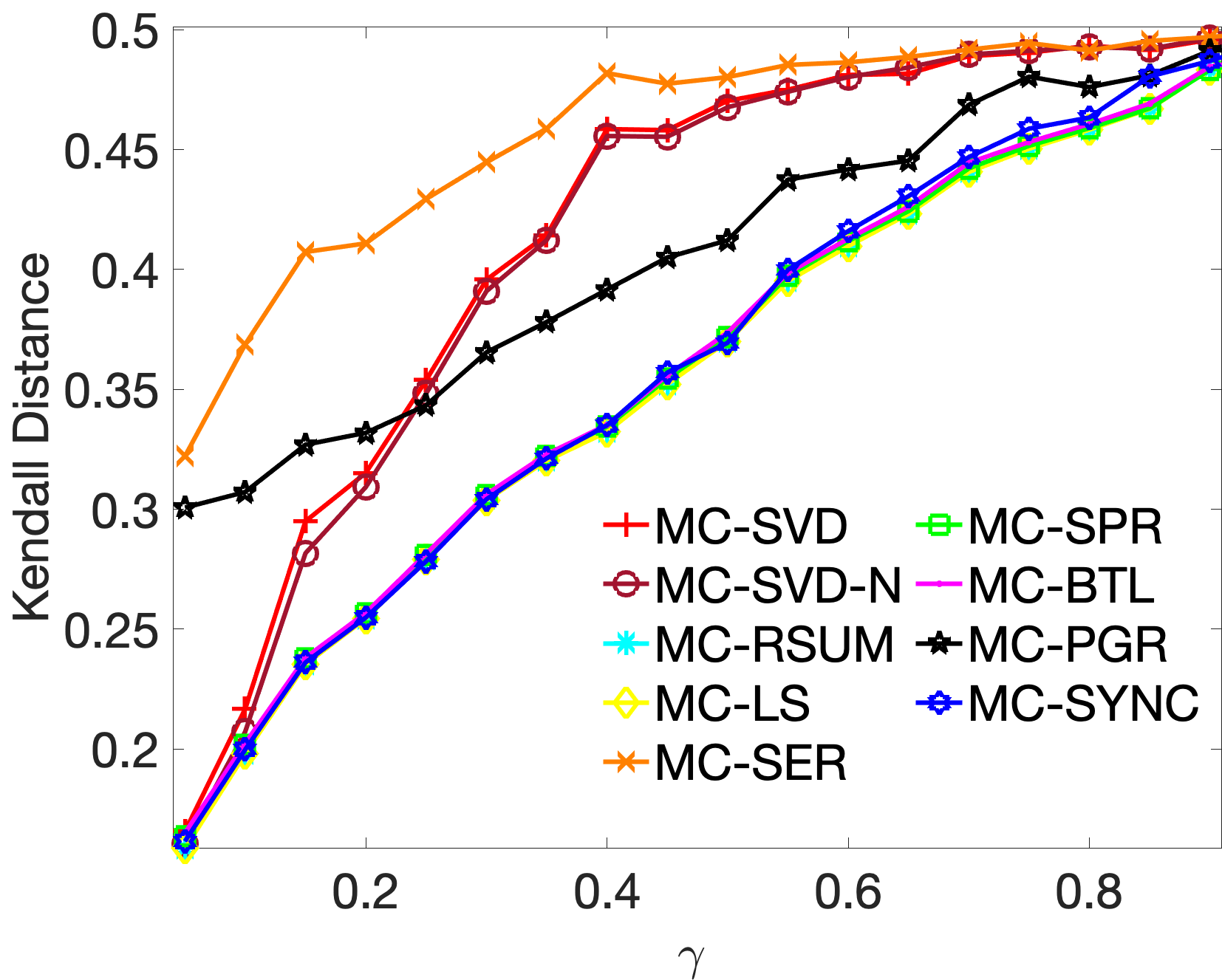}
& \includegraphics[width=0.35\columnwidth]{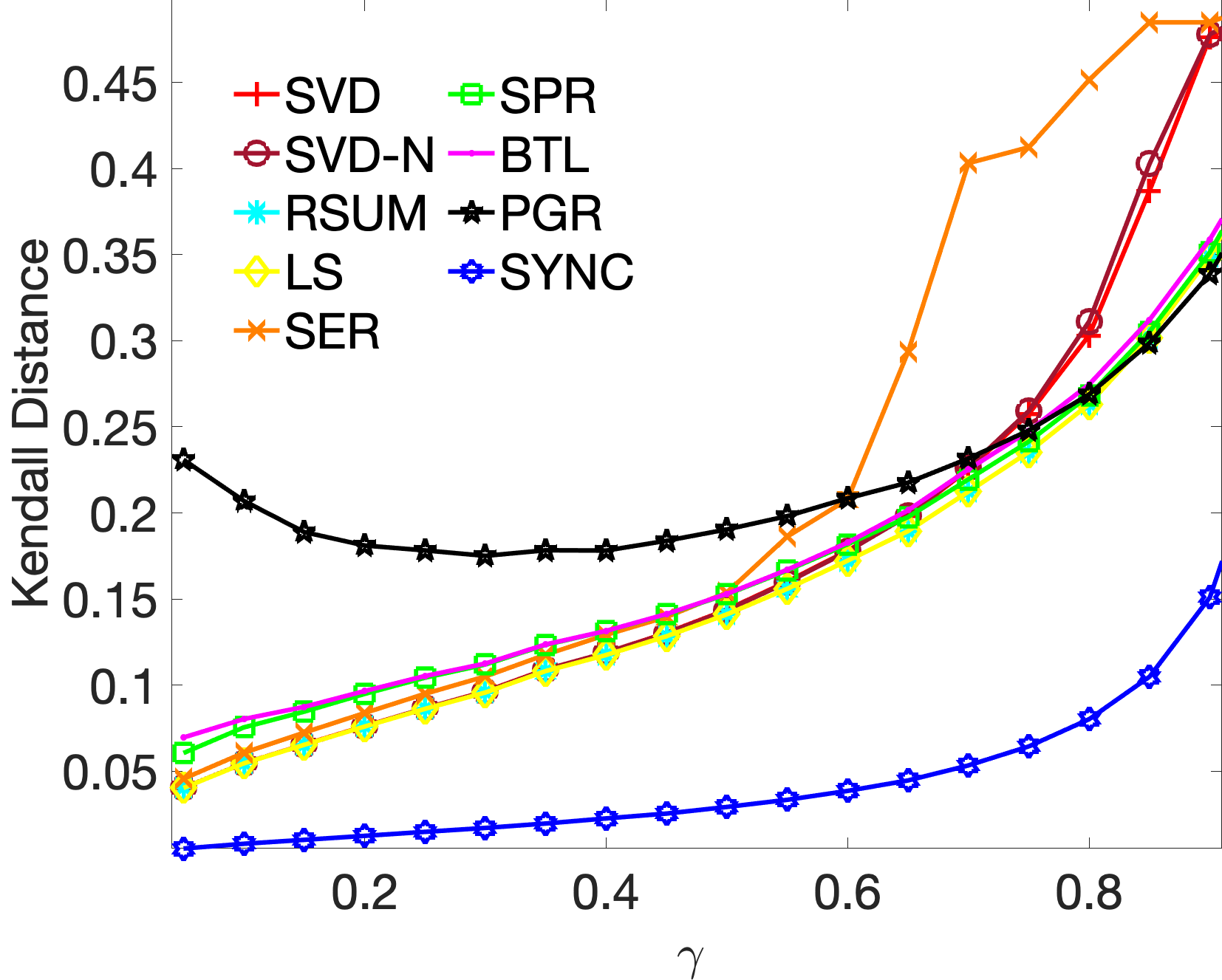}  \\ 
& \includegraphics[width=0.35\columnwidth]{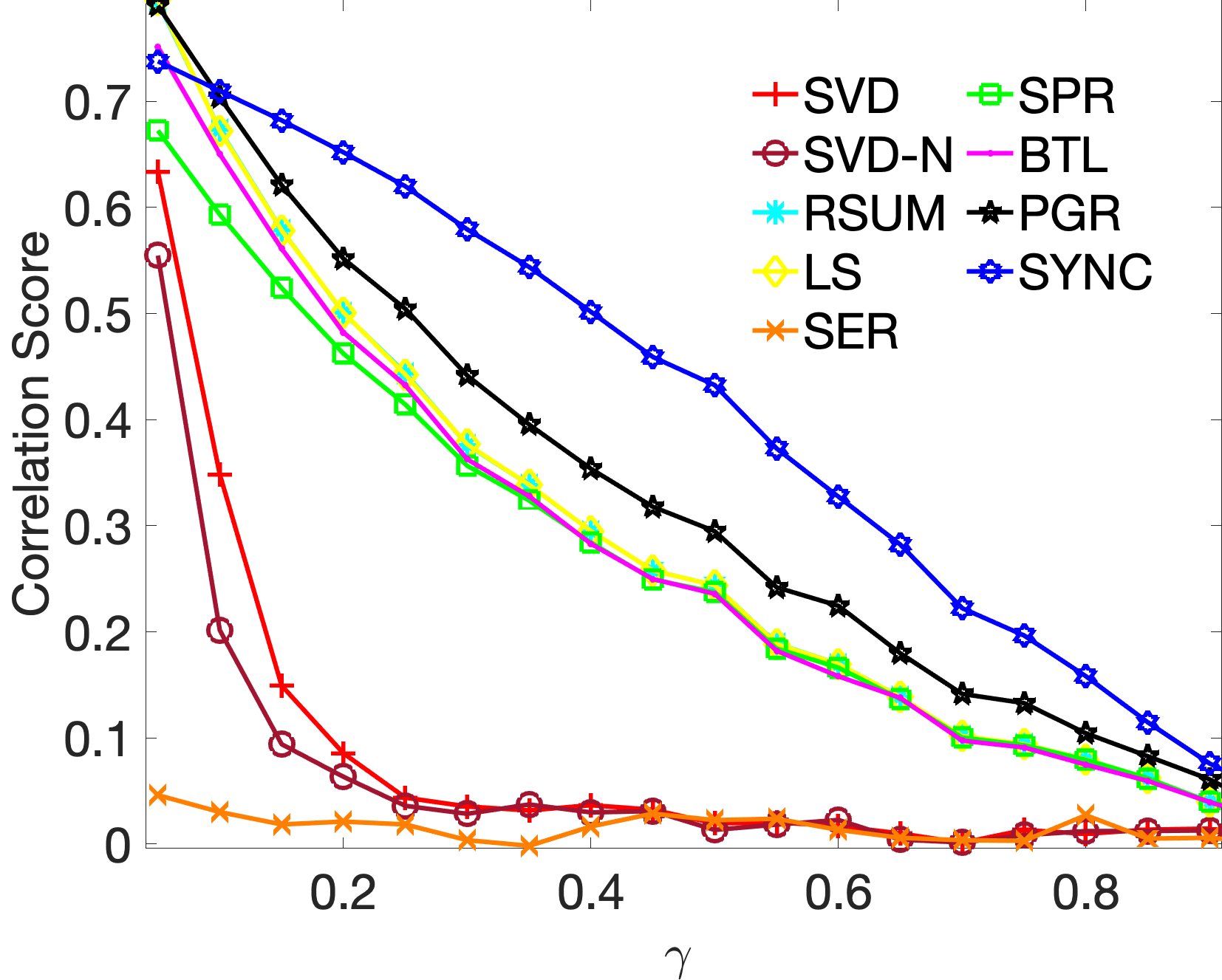}
& \includegraphics[width=0.35\columnwidth]{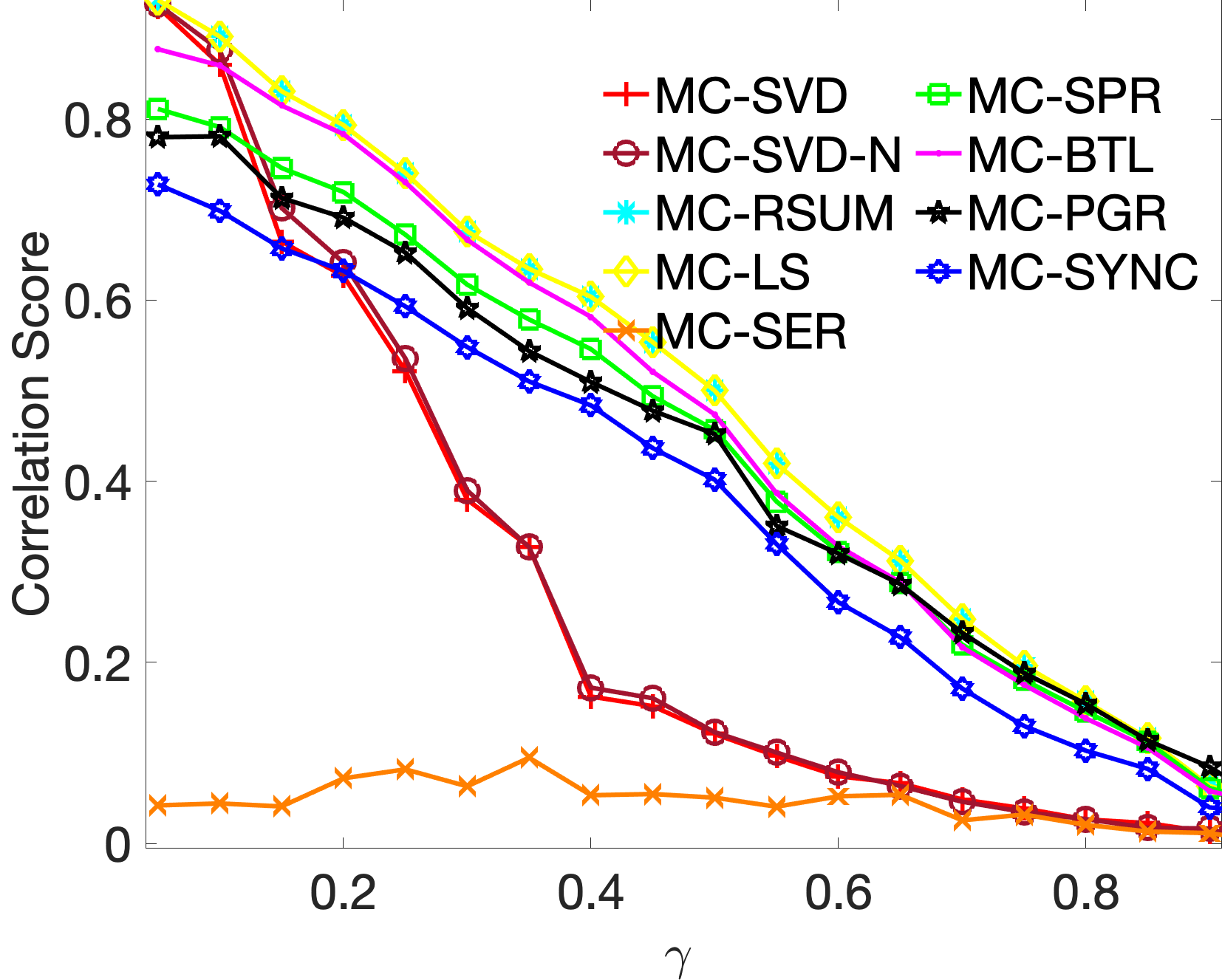}
& \includegraphics[width=0.35\columnwidth]{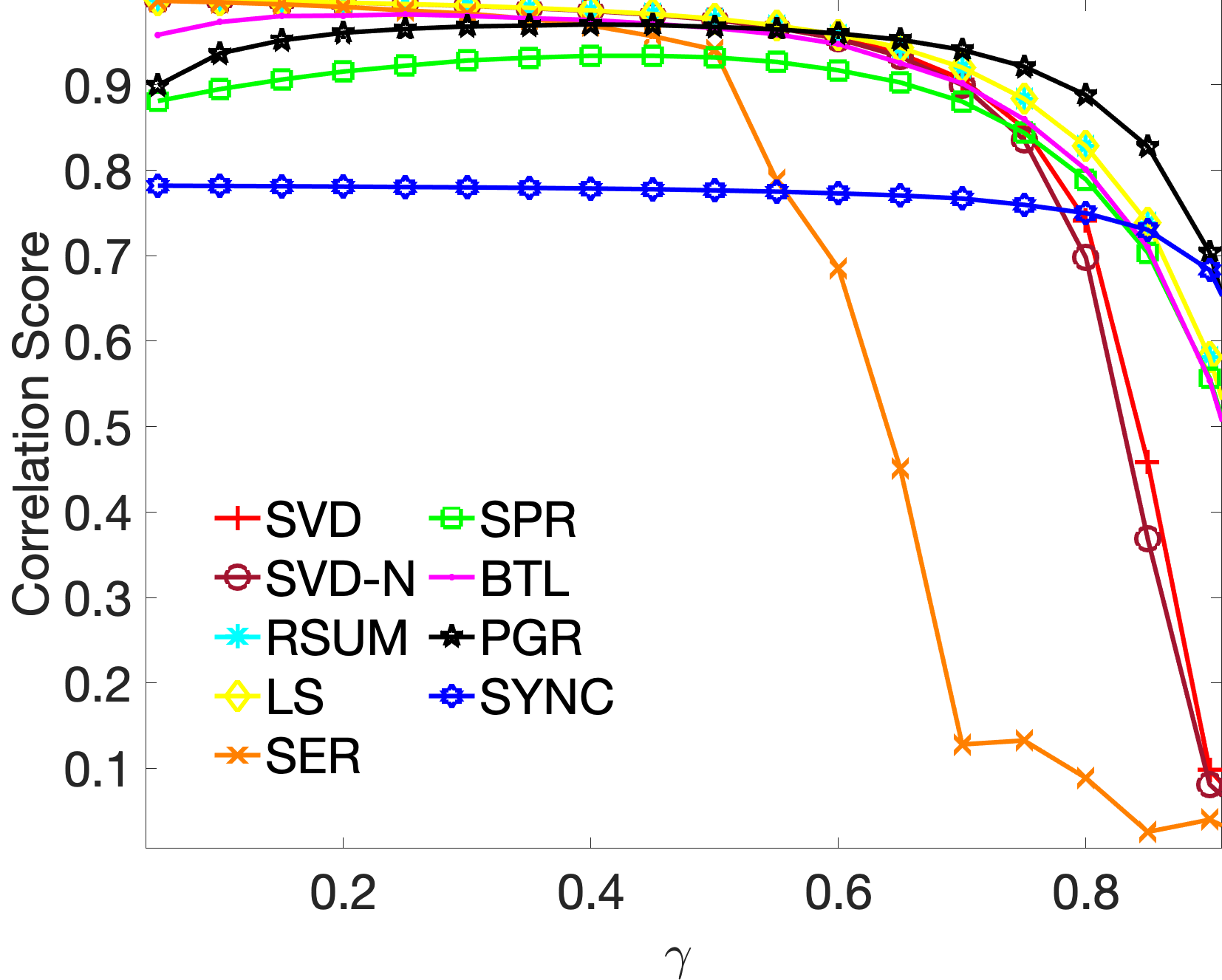} \\
& \includegraphics[width=0.35\columnwidth]{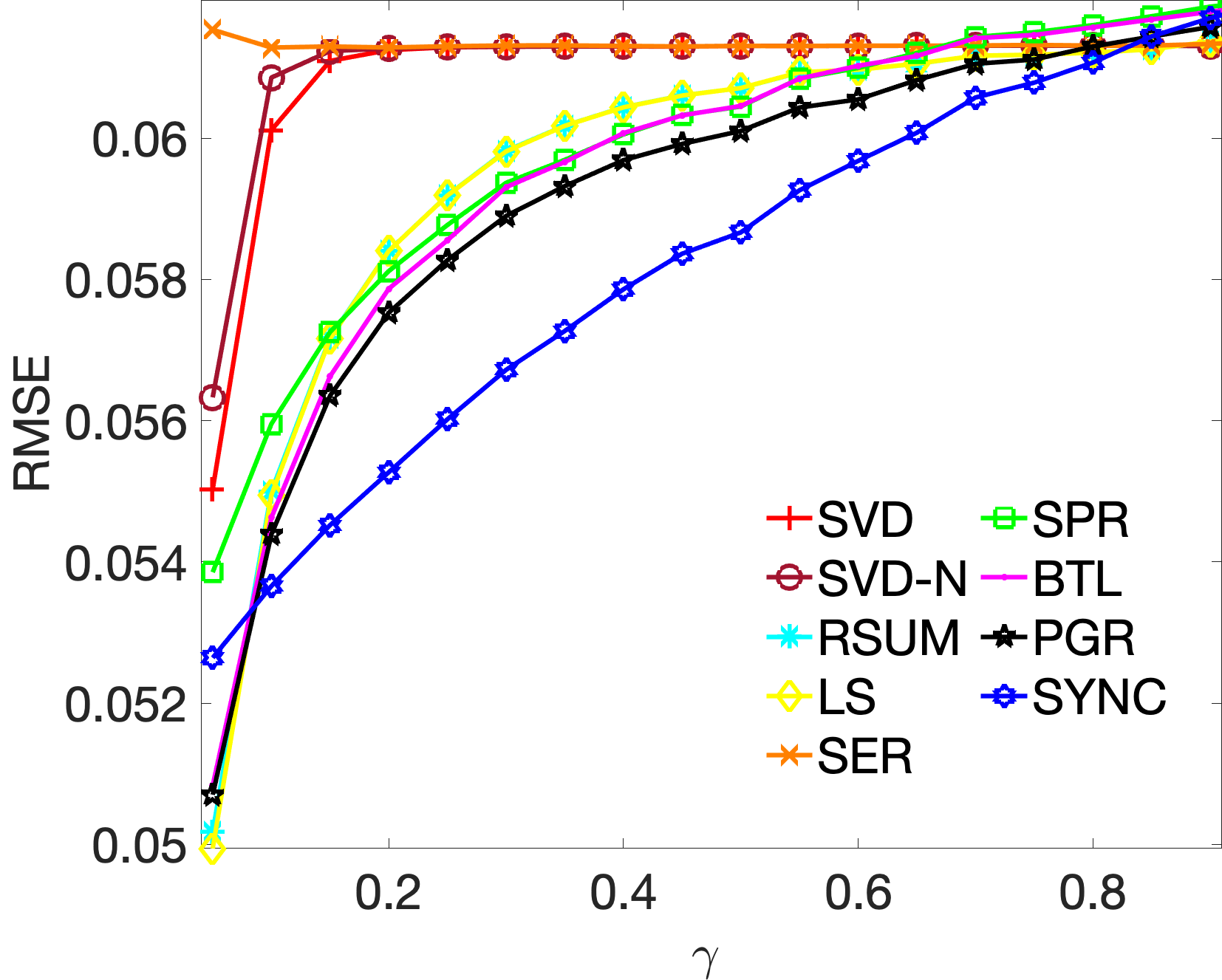}
& \includegraphics[width=0.35\columnwidth]{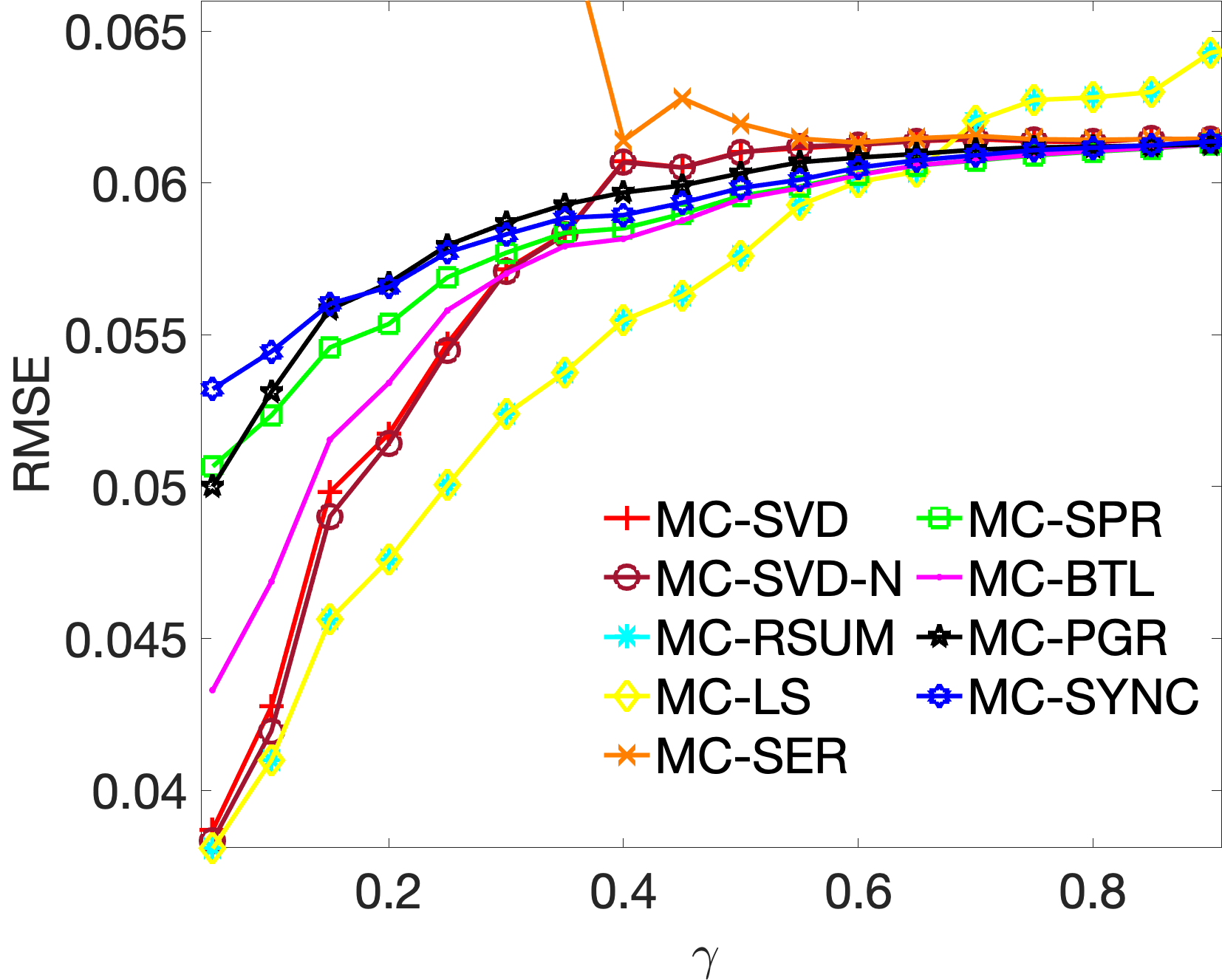}
& \includegraphics[width=0.35\columnwidth]{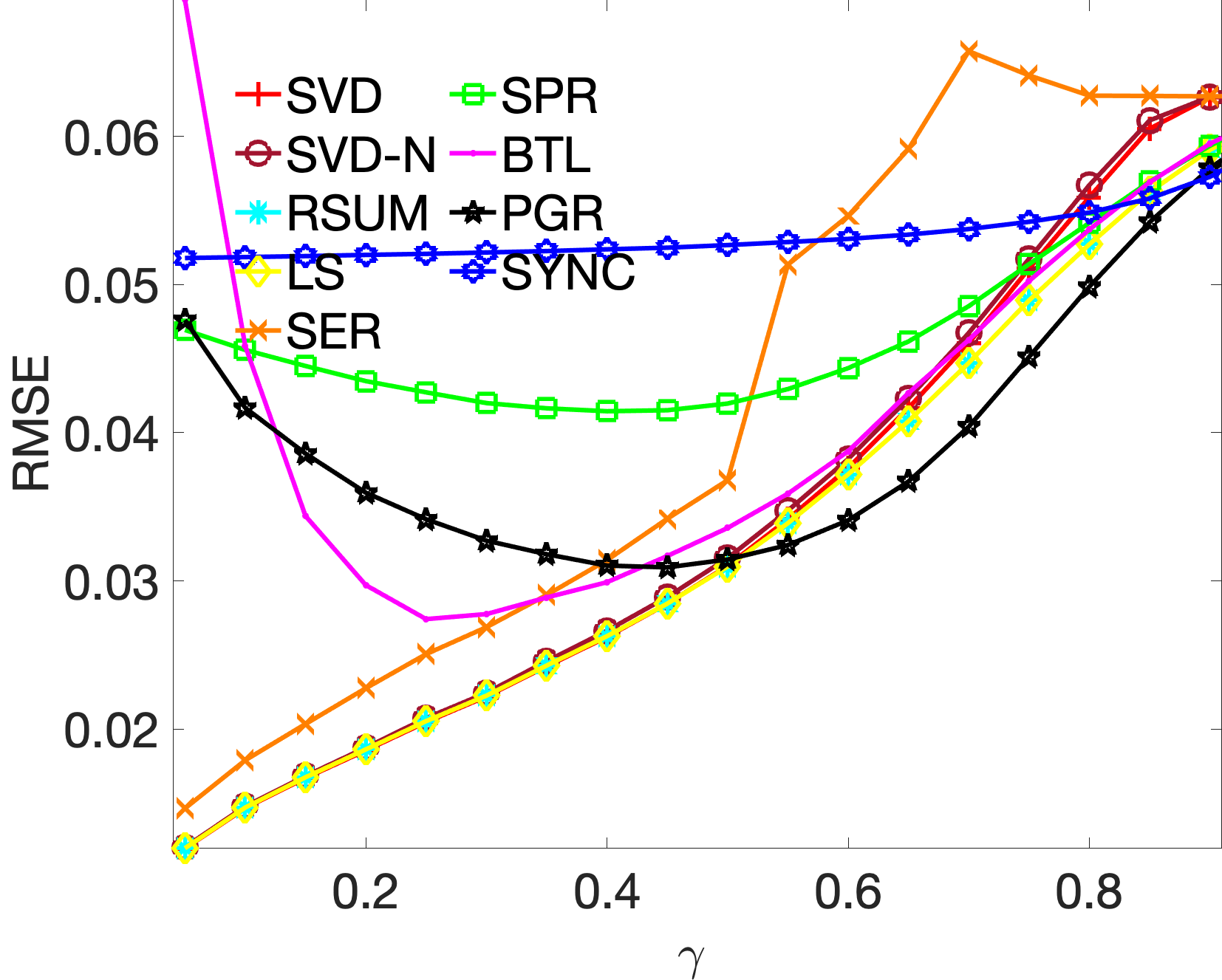}
\end{tabular}
\captionsetup{width=0.99\linewidth}
\vspace{-2mm}
\captionof{figure}{Performance statistics in terms of Kendall distance (top row), correlation score (middle row), and RMSE (bottom row), for synthetic data with scores generated from a \textbf{Gamma} distribution with $n=1000$. The first and third columns pertain to the case of sparsity $p=0.05$, respectively $p=1$ without the matrix completion step, while the middle column, for $p=0.05$, employs a matrix-completion preprocessing step. Results are averaged over 20 runs.
}
\label{fig:gamma_ERO_n1000}
\end{table*}

\newcolumntype{C}{>{\centering\arraybackslash}m{\wid}} 

\begin{table*}\sffamily 
\hspace{-9mm} 
\begin{tabular}{l*3{C}@{}}
& $p=0.01$ & $p=0.05$ & $p=0.1$  \\
& \includegraphics[width=0.35\columnwidth]{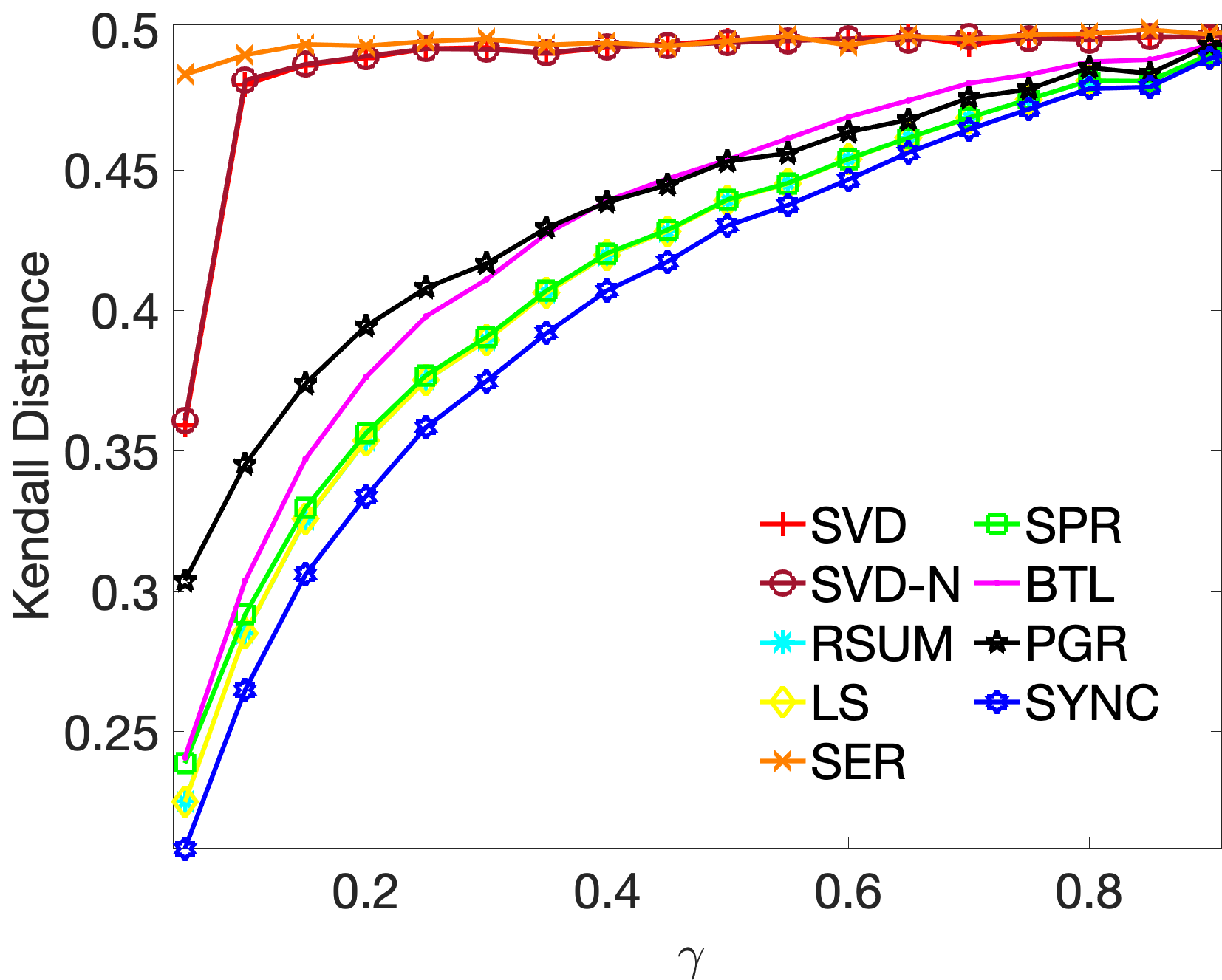}
& \includegraphics[width=0.35\columnwidth]{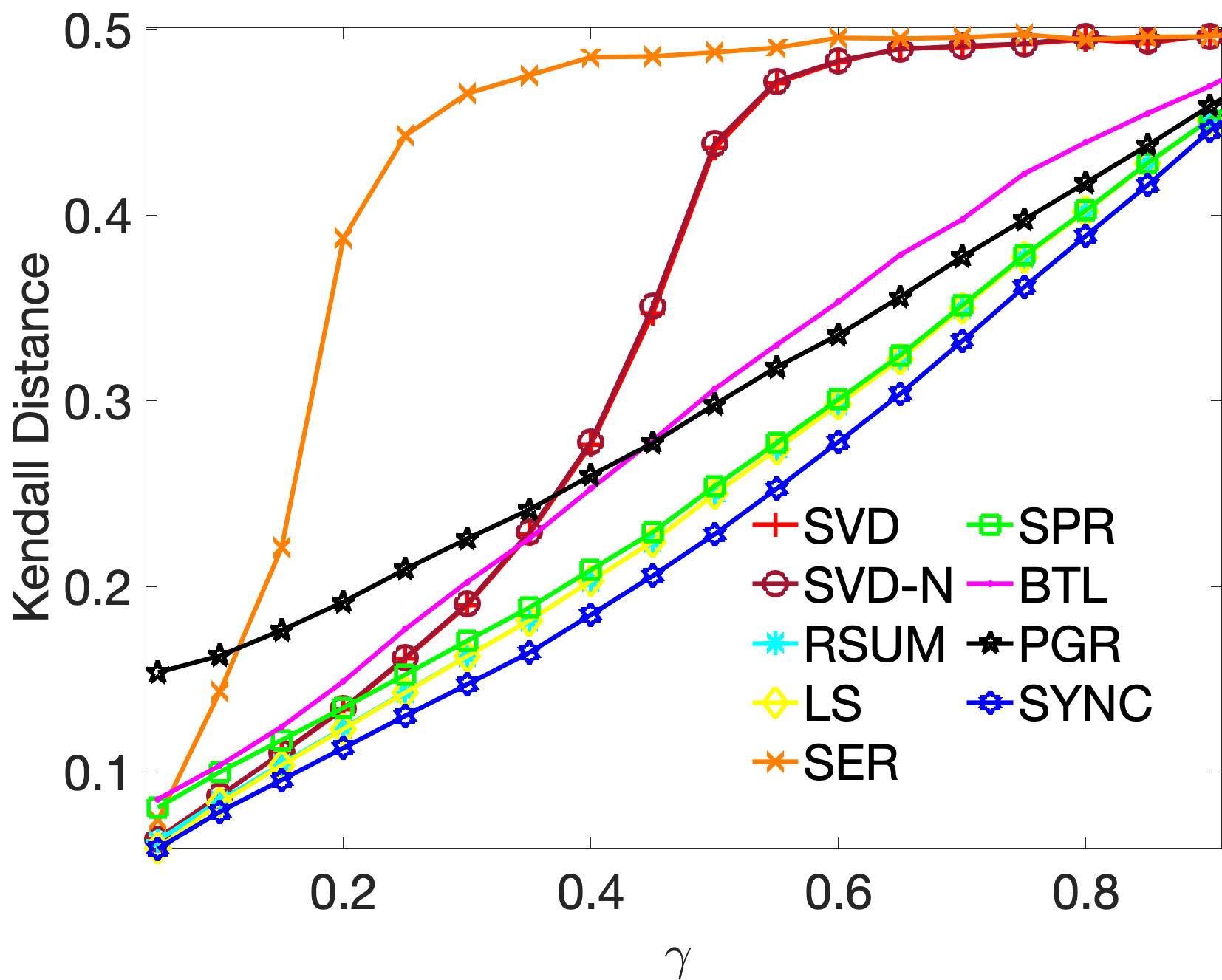}
& \includegraphics[width=0.35\columnwidth]{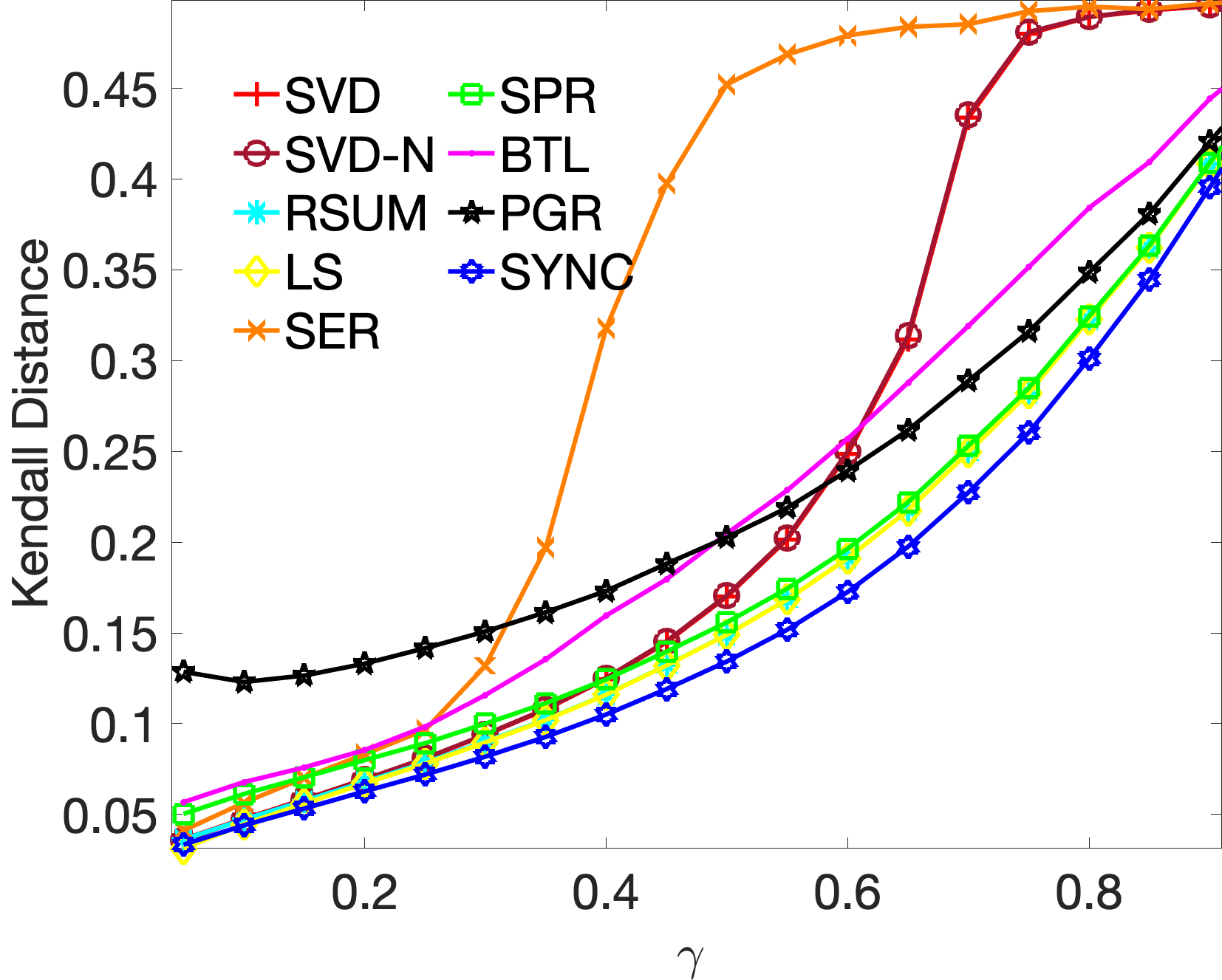}  \\ 
& \includegraphics[width=0.35\columnwidth]{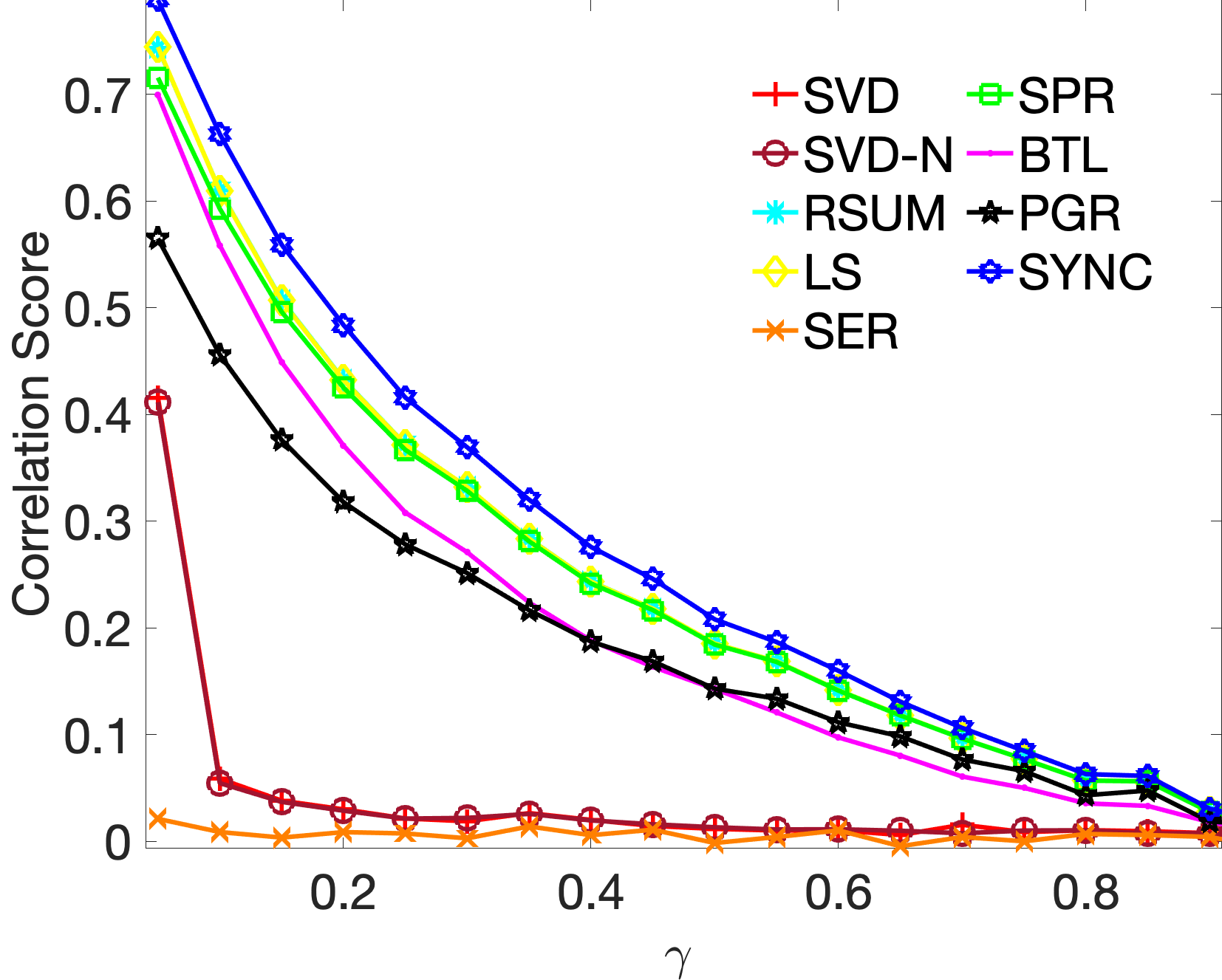}
& \includegraphics[width=0.35\columnwidth]{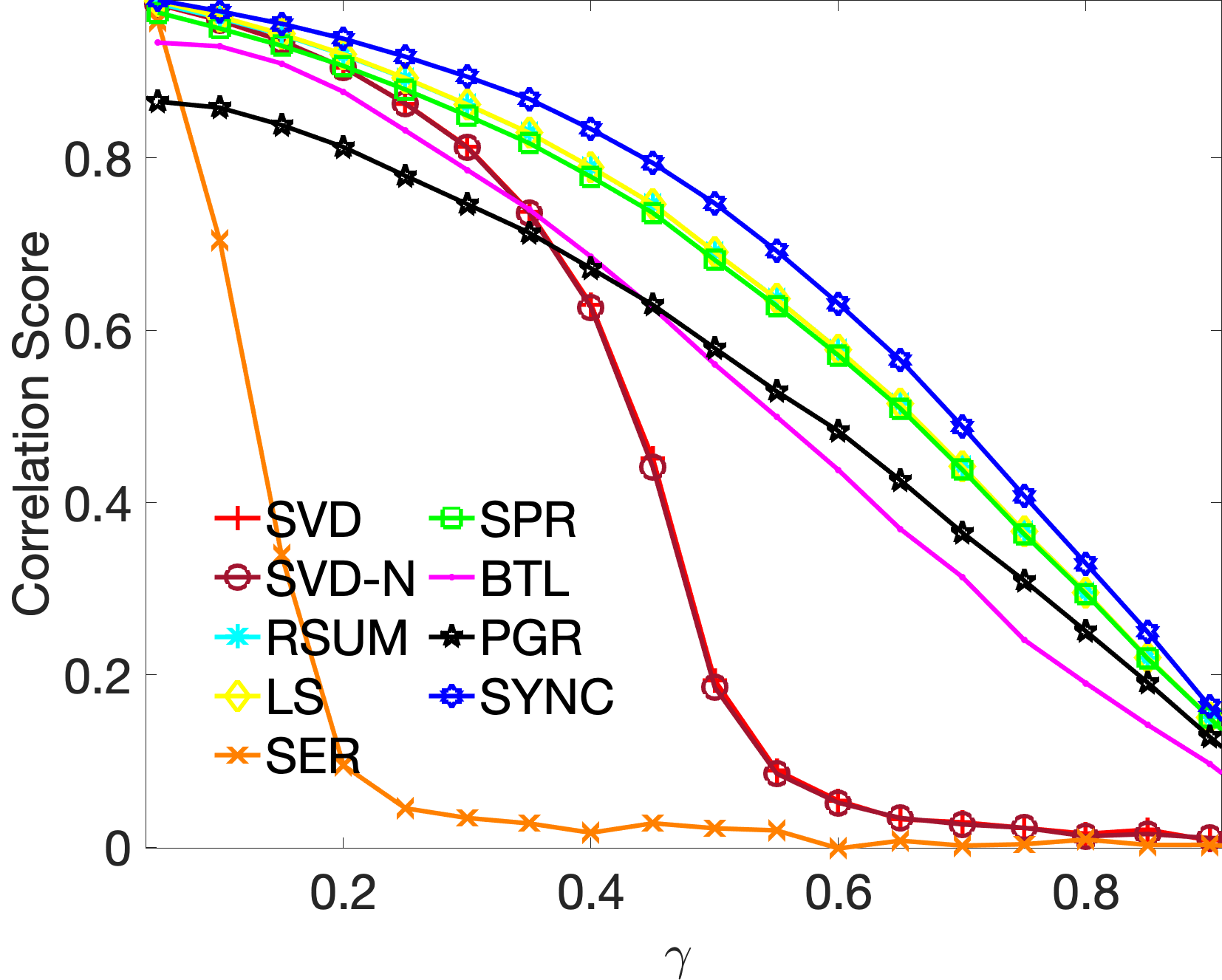}
& \includegraphics[width=0.35\columnwidth]{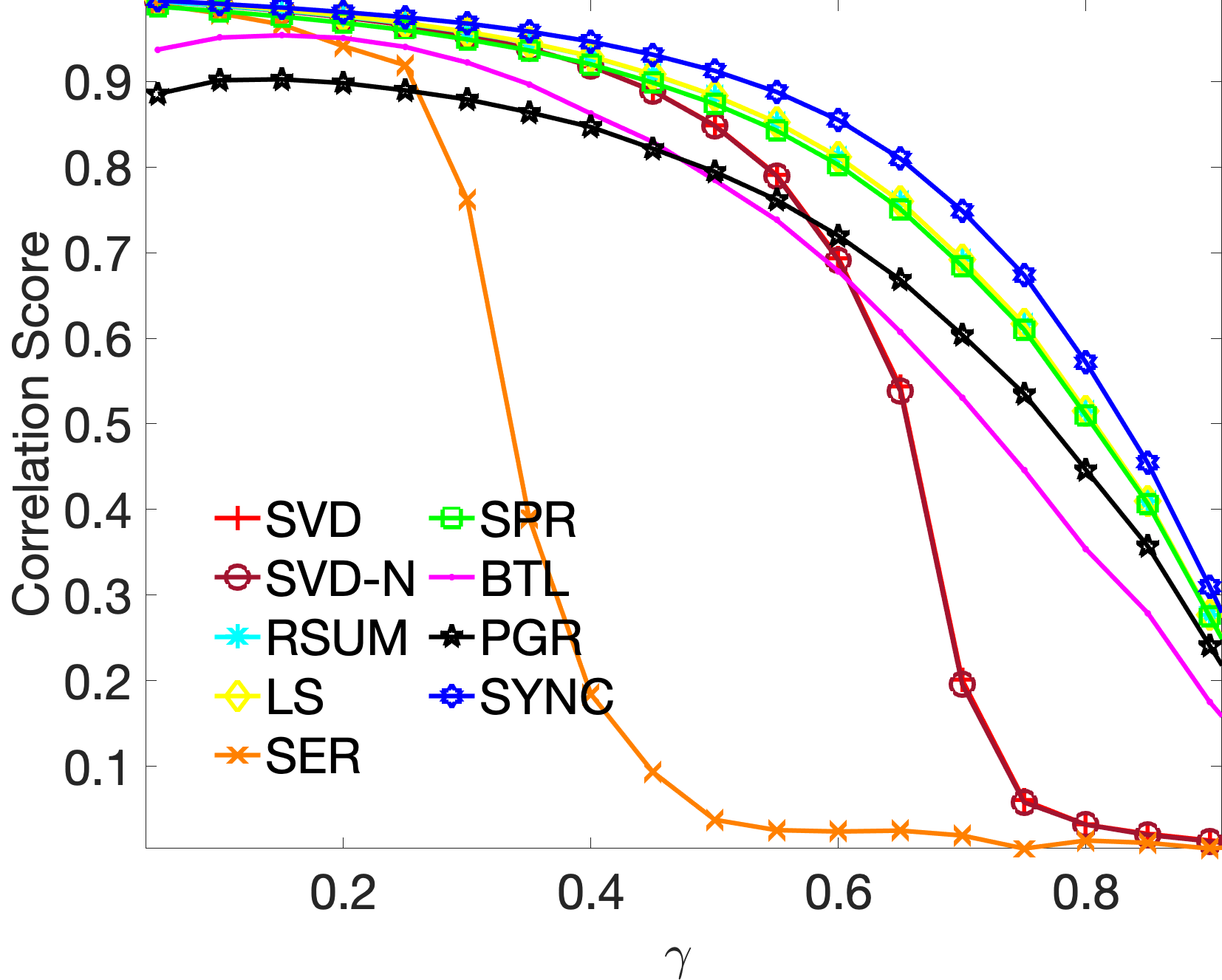} \\
& \includegraphics[width=0.35\columnwidth]{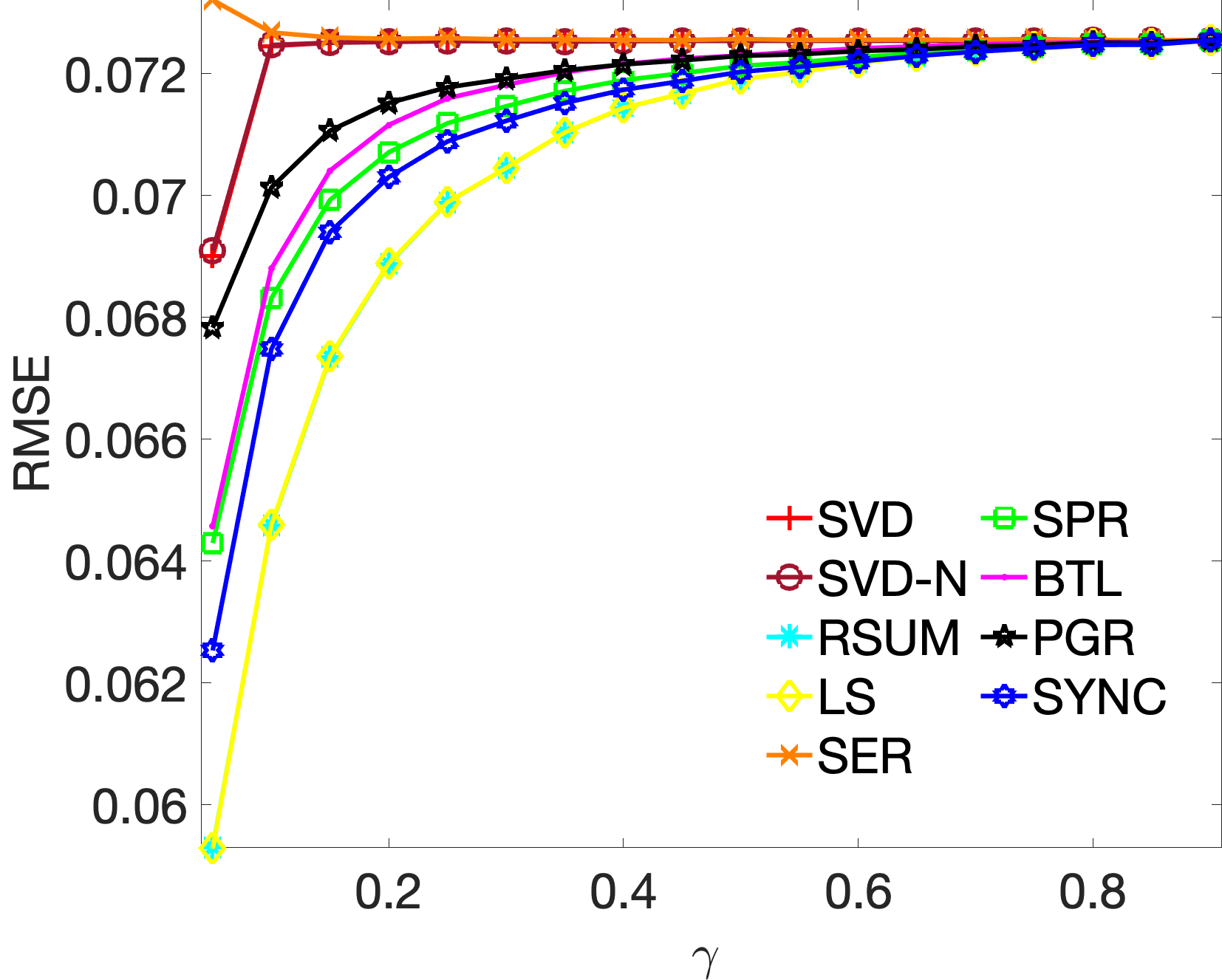}
& \includegraphics[width=0.35\columnwidth]{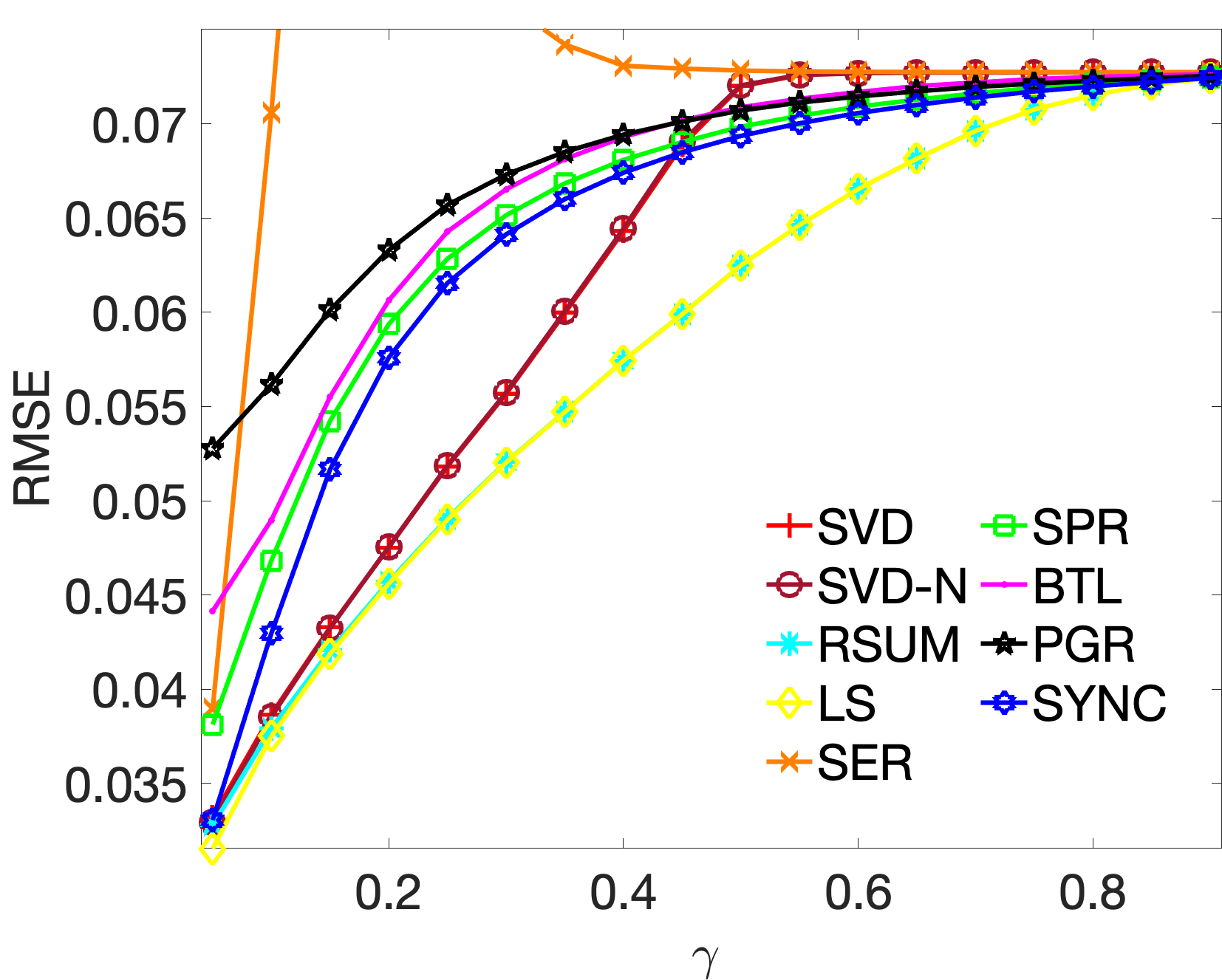}
& \includegraphics[width=0.35\columnwidth]{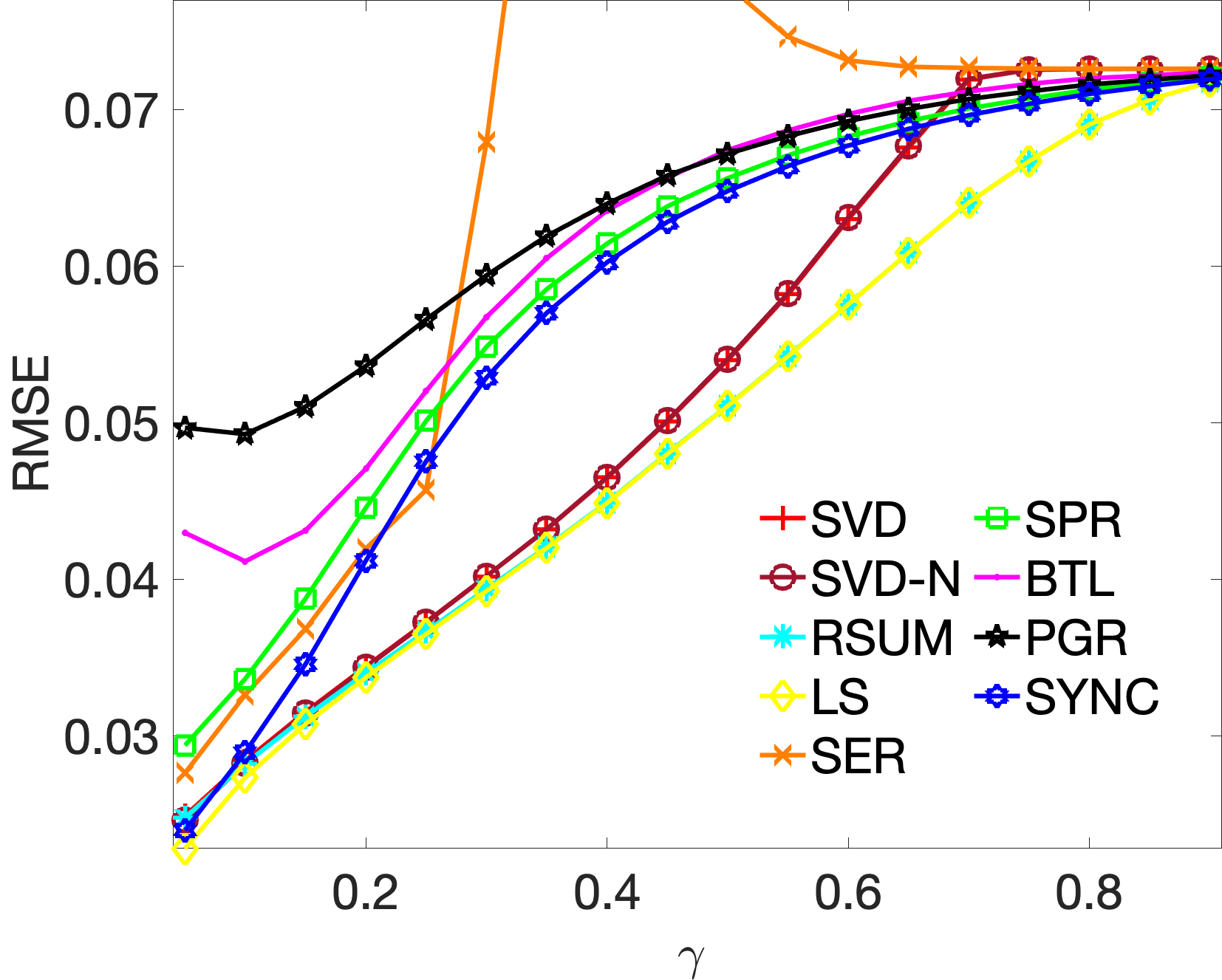}
\end{tabular}
\captionsetup{width=0.99\linewidth}
\vspace{-2mm}
\captionof{figure}{Performance statistics in terms of Kendall distance (top row), correlation score (middle row), and RMSE (bottom row), for synthetic data with scores generated by a \textbf{Uniform} distribution with $n=3000$, and sparsity $p \in \{0.01, 0.05, 0.1\}$, without the matrix completion preprocessing step. Note that  $\log(3000) / 3000 = 0.00267$. Results are averaged over 20 runs. 
}
\label{fig:uniform_ERO_n3000}
\end{table*}

\newcolumntype{C}{>{\centering\arraybackslash}m{\wid}} 
\begin{table*}\sffamily 
\hspace{-9mm} 
\begin{tabular}{l*3{C}@{}}
& $p=0.01$ & $p=0.05$ & $p=0.1$  \\
& \includegraphics[width=0.35\columnwidth]{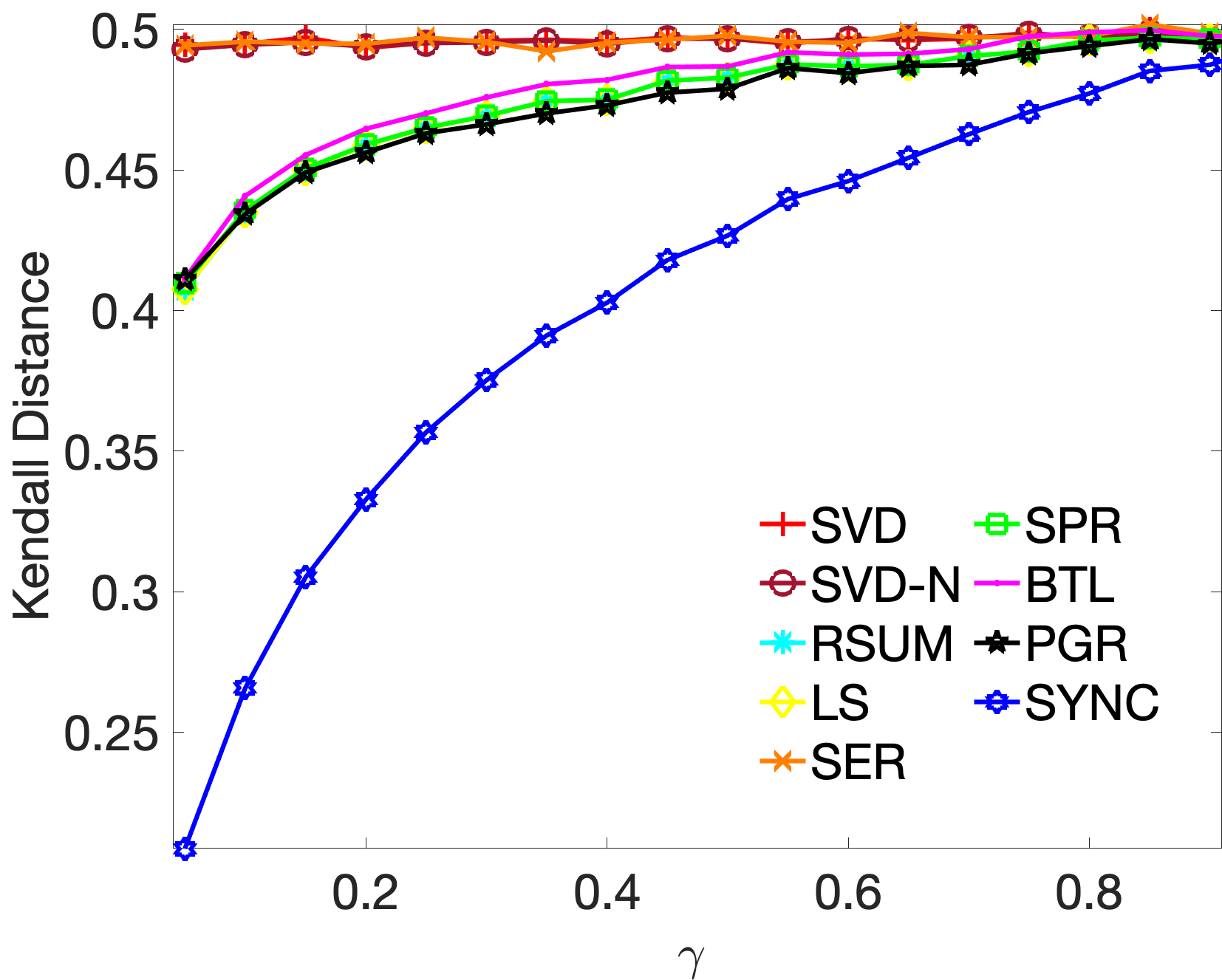}
& \includegraphics[width=0.35\columnwidth]{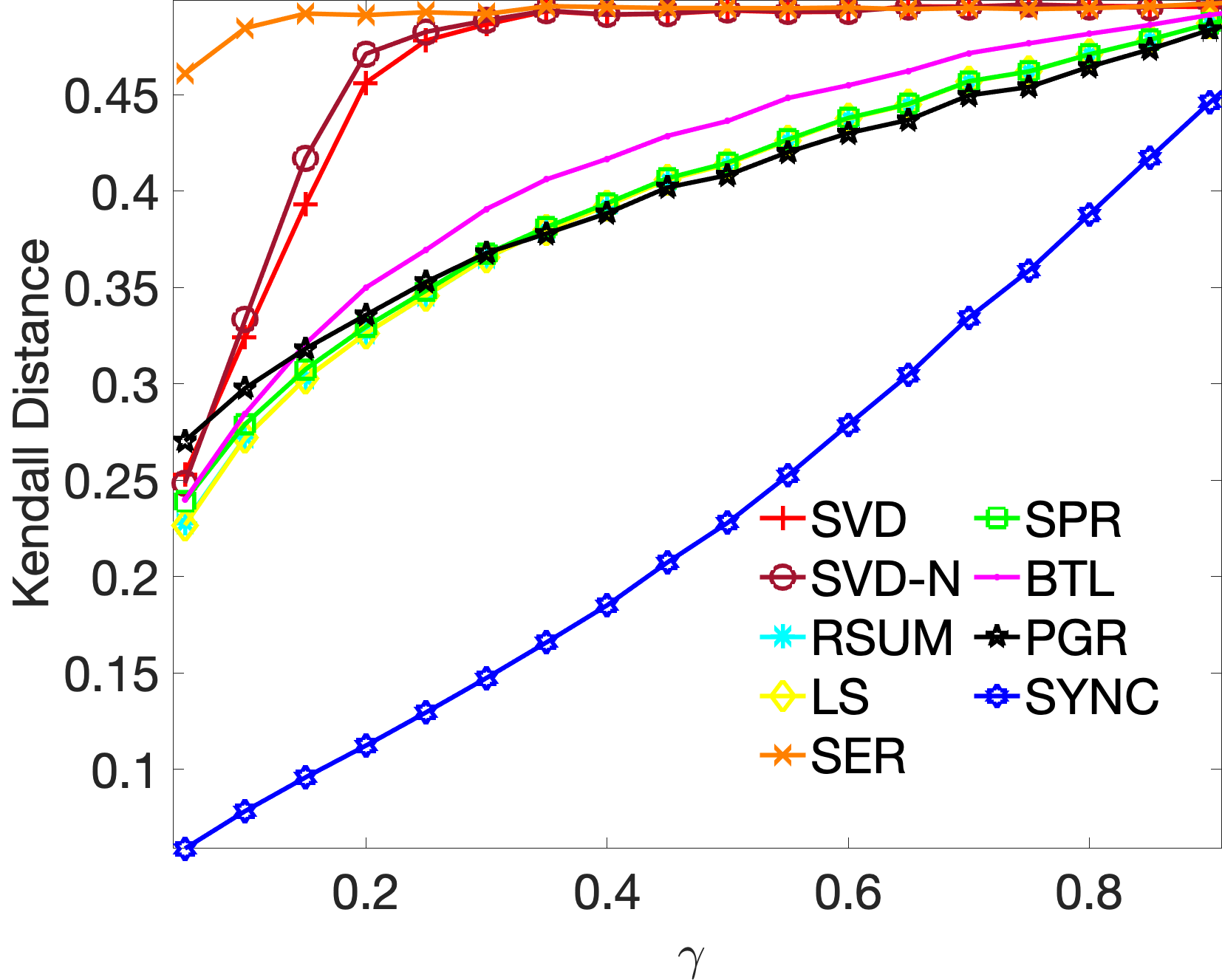}
& \includegraphics[width=0.35\columnwidth]{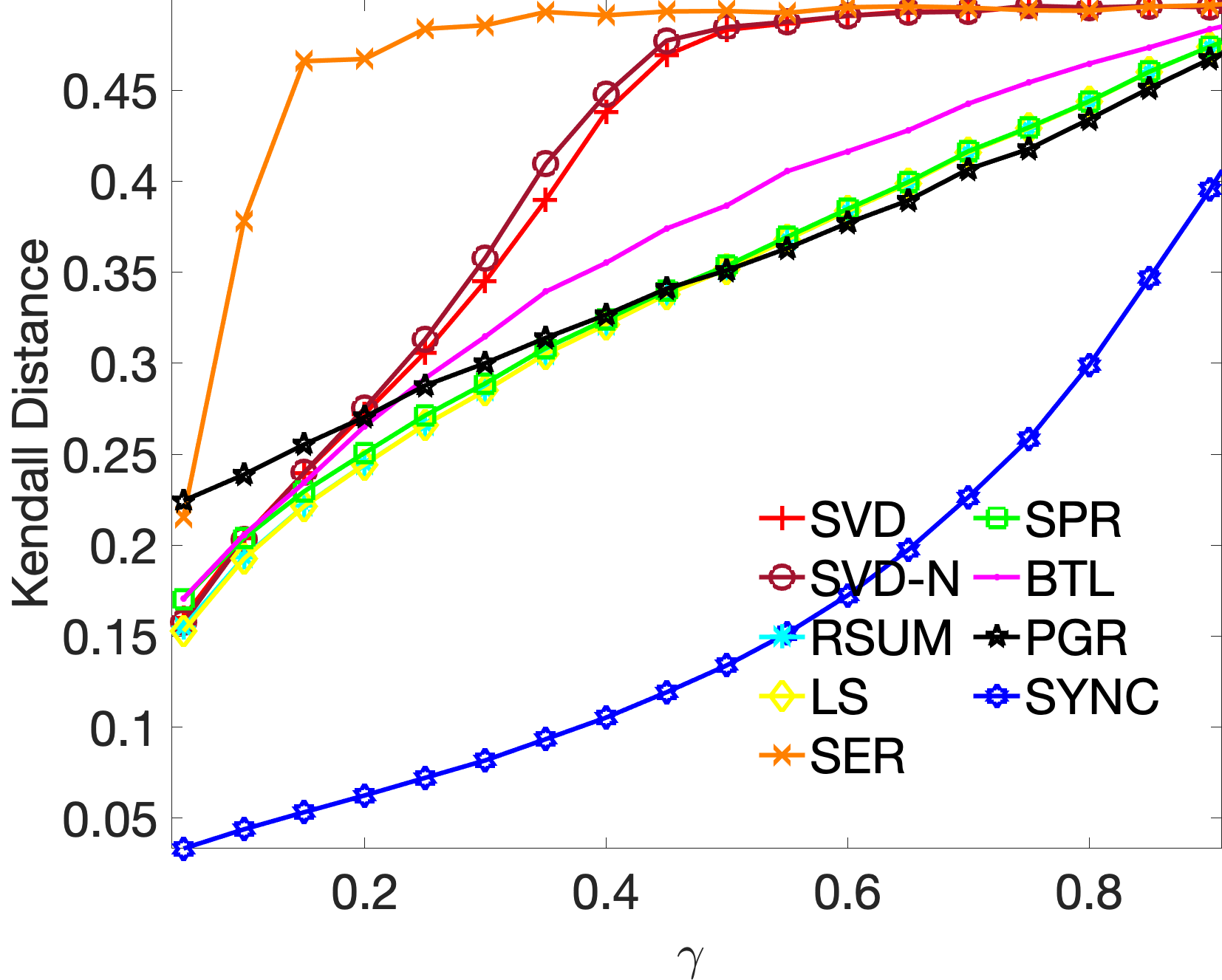}  \\ 
& \includegraphics[width=0.35\columnwidth]{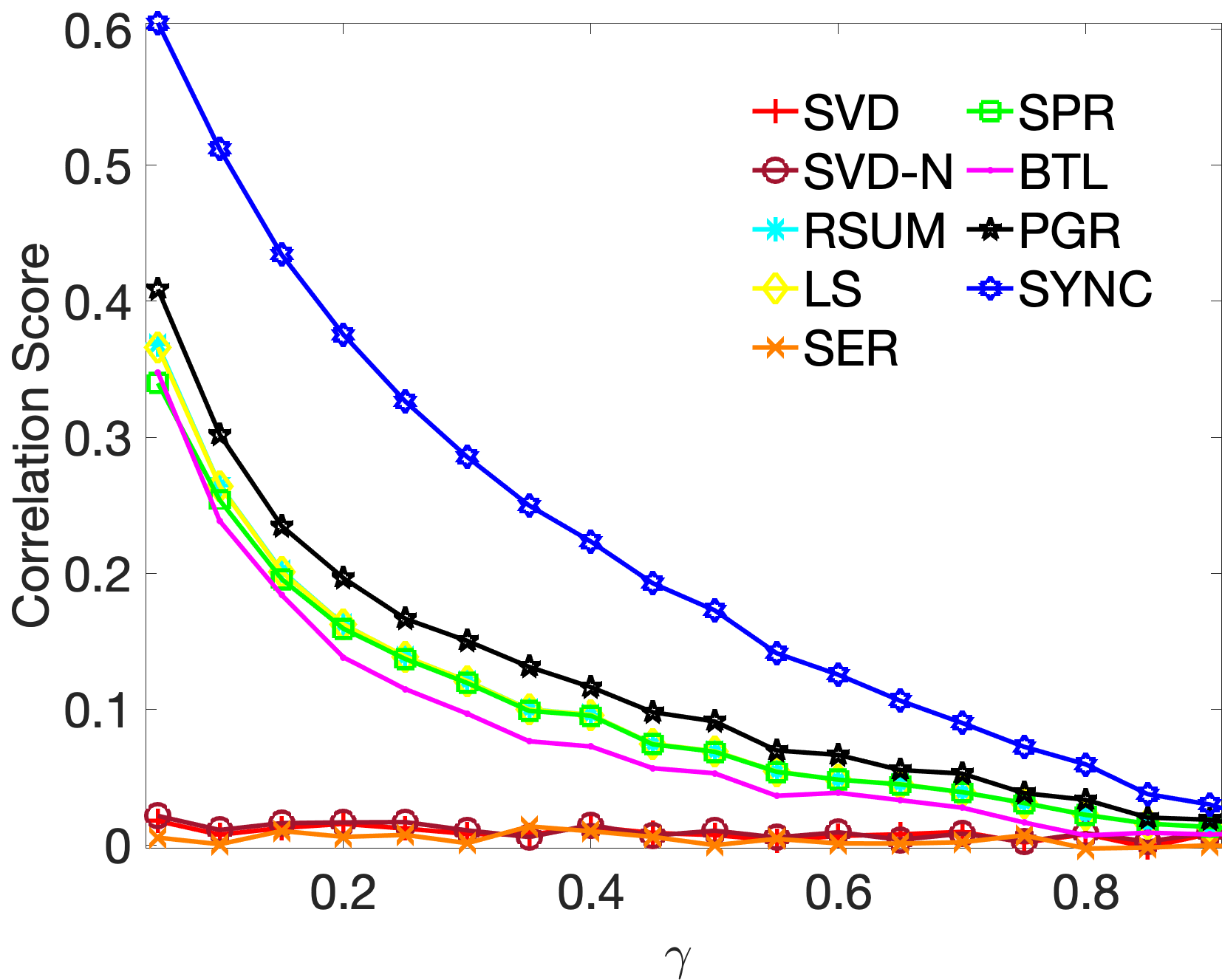}
& \includegraphics[width=0.35\columnwidth]{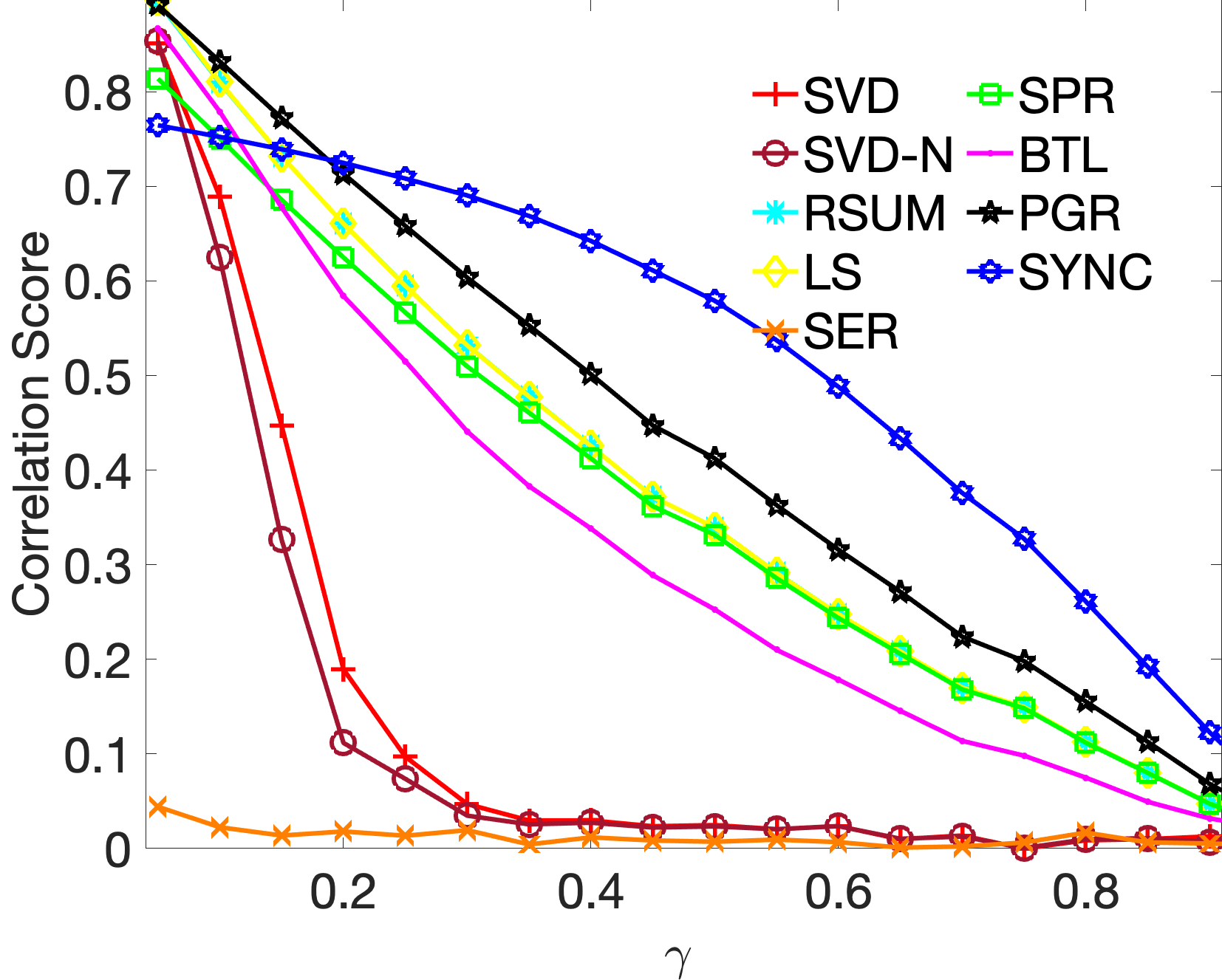}
& \includegraphics[width=0.35\columnwidth]{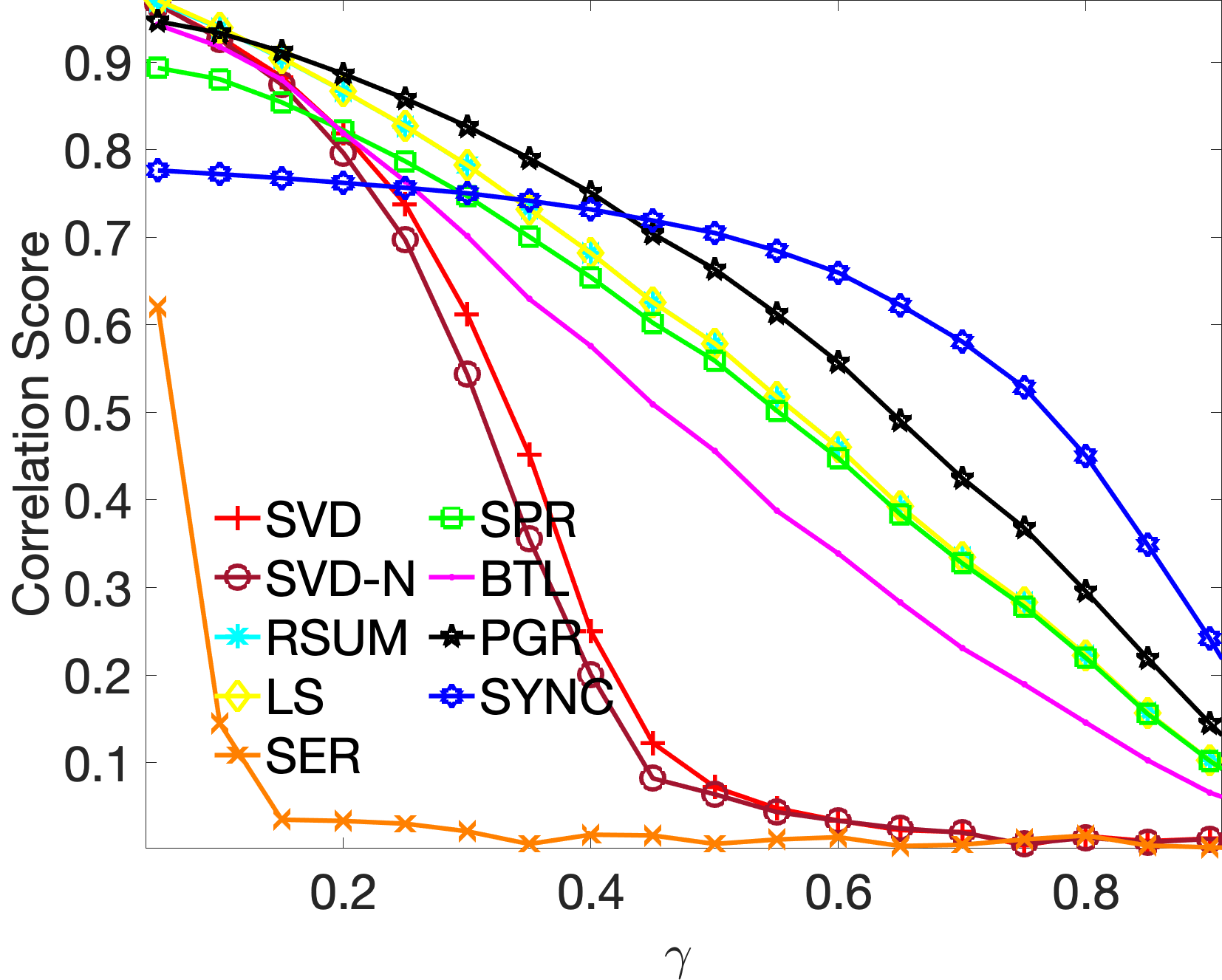} \\
& \includegraphics[width=0.35\columnwidth]{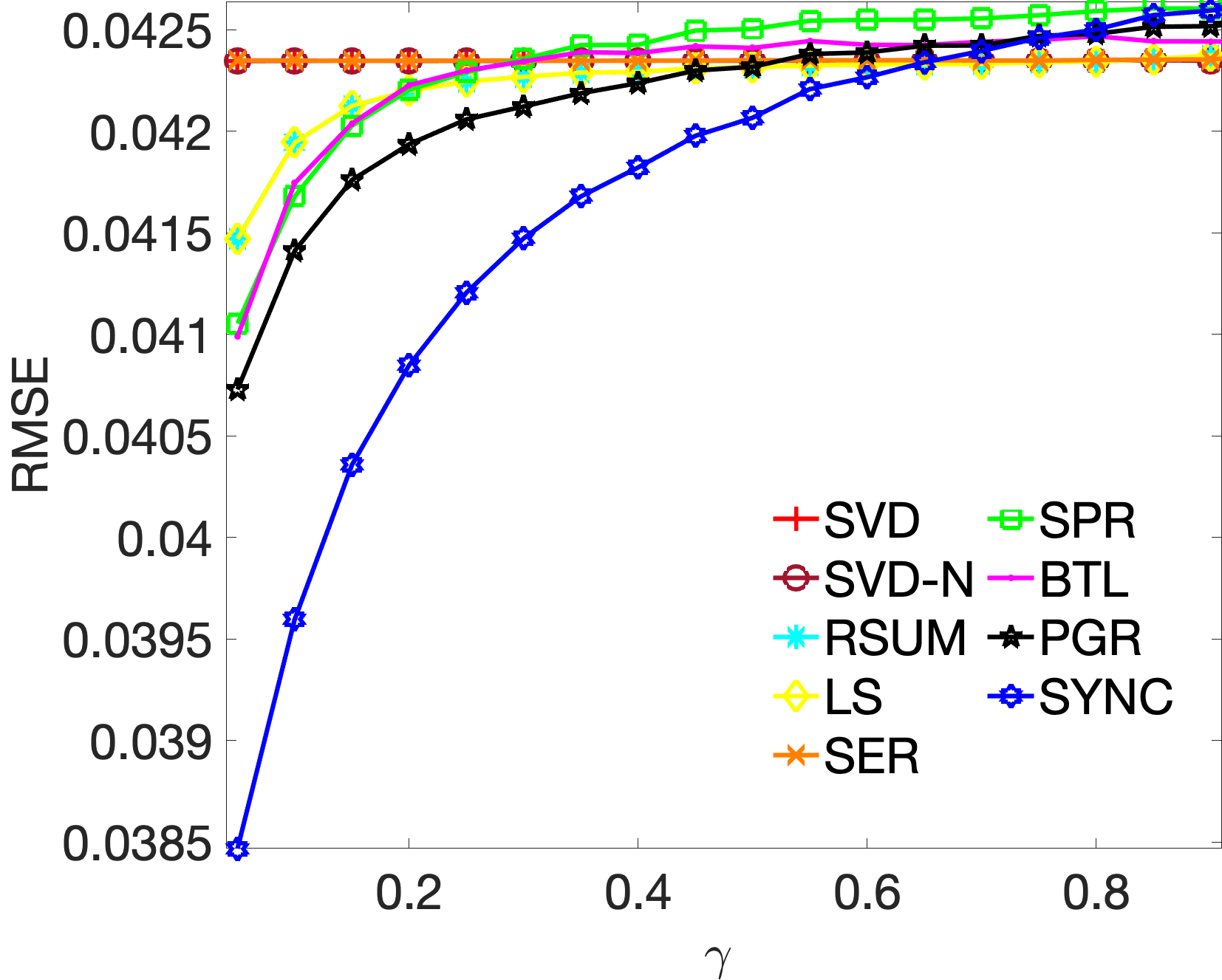}
& \includegraphics[width=0.35\columnwidth]{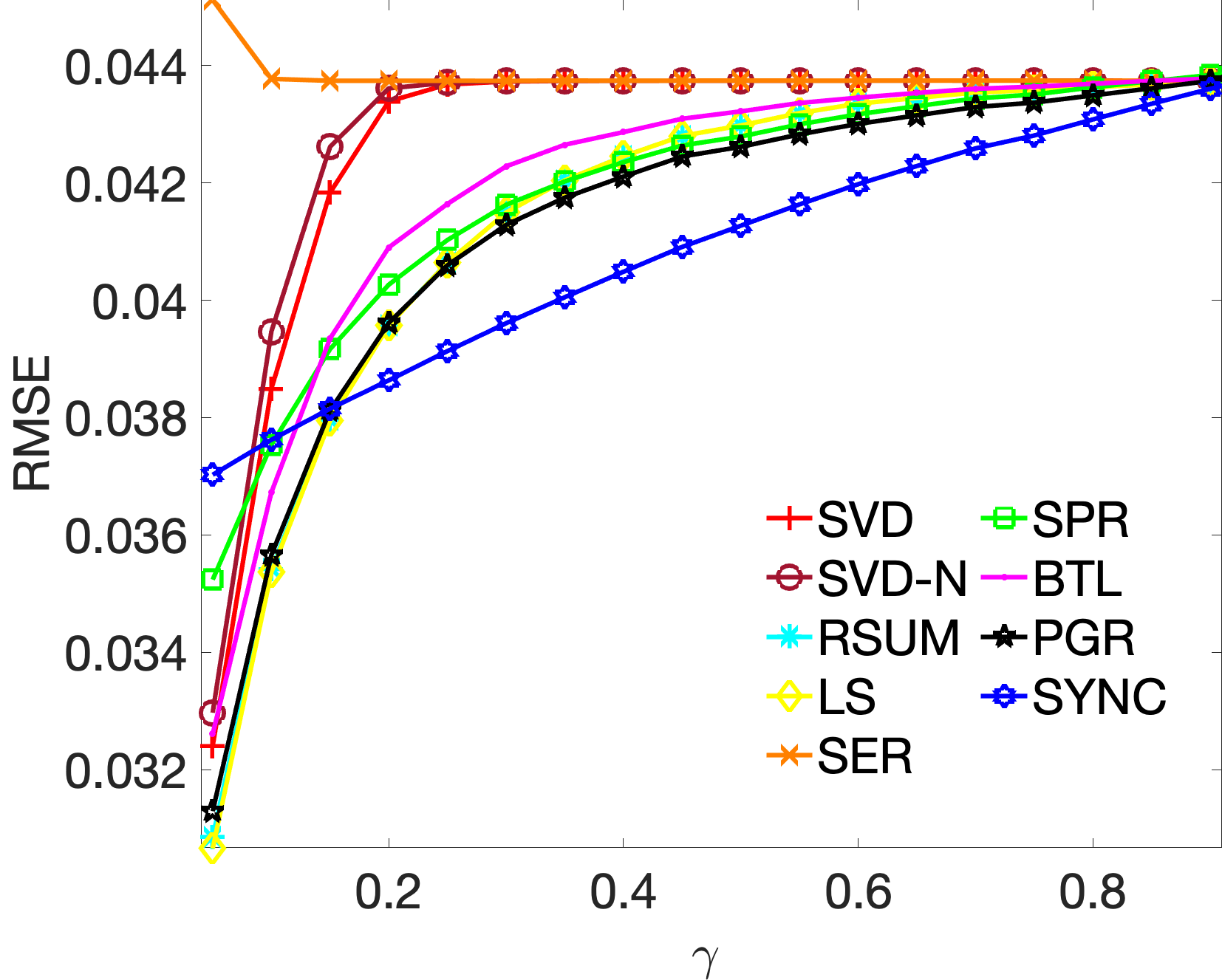}
& \includegraphics[width=0.35\columnwidth]{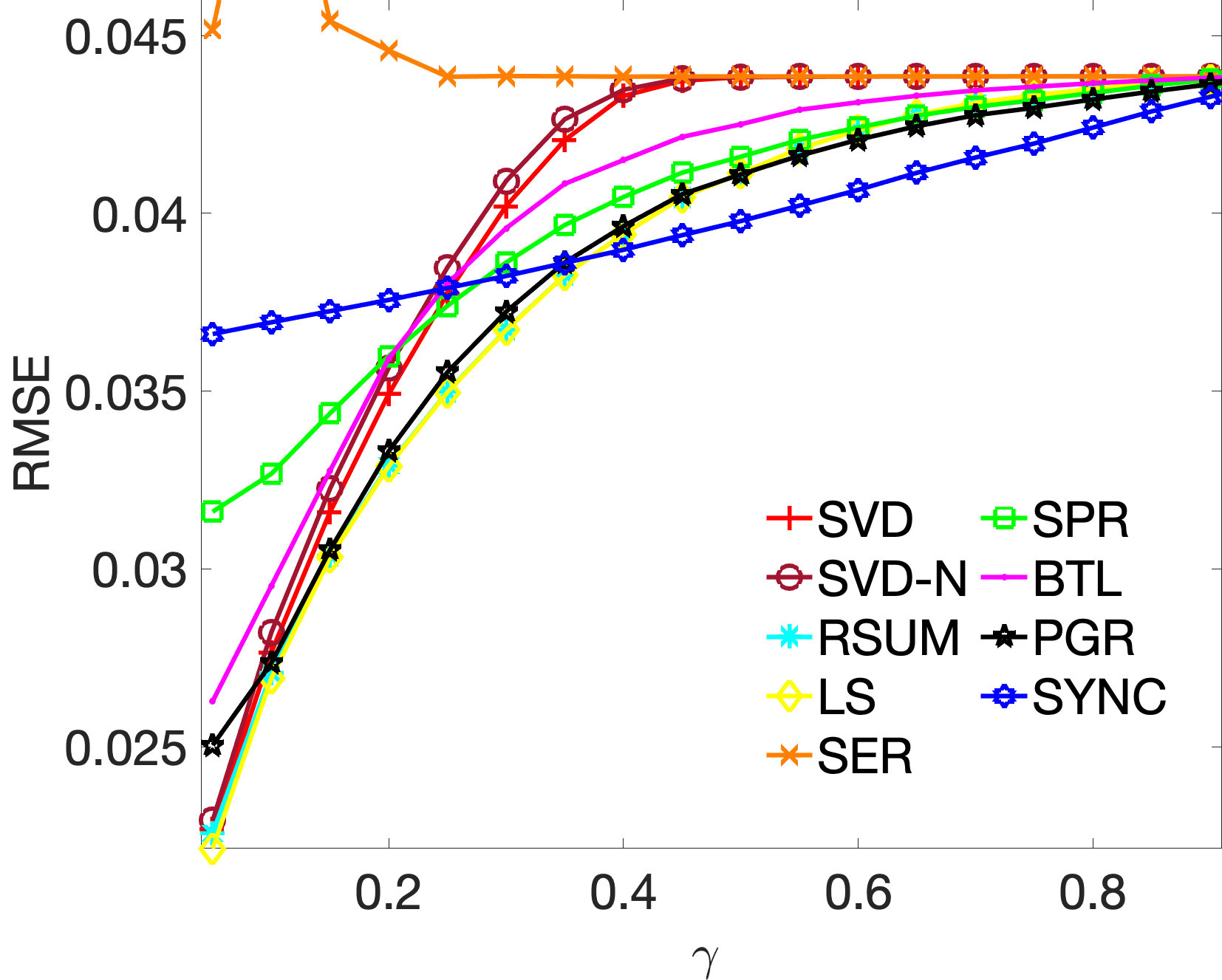}
\end{tabular}
\captionsetup{width=0.99\linewidth}
\vspace{-2mm}
\captionof{figure}{Performance statistics in terms of Kendall distance (top row), correlation score (middle row), and RMSE (bottom row), for synthetic data with scores generated by a \textbf{Gamma} distribution with $n=3000$, and sparsity $p \in \{0.01, 0.05, 0.1\}$, without the matrix completion preprocessing step. Note that  $\log(3000) / 3000 = 0.00267$. Results are averaged over 20 runs. 
}
\label{fig:gamma_ERO_n3000}
\end{table*}

% \FloatBarrier

\subsection{Real data}  \label{sec:num_real}

We also apply all ranking algorithms to a variety of real world networks, and measure the performance of the outcome by the number of upsets in the final ranking, computed with respect to the given pairwise measurements. To this end, we consider two types of upsets, as detailed below. Let  $\hat{r}_i$ denote the estimated score of item $i$ as computed by the method of choice. Note that higher values of $\hat{r}_i$ correspond to higher ranks, meaning stronger players or more preferred items. Next, we construct the induced (possibly incomplete) matrix of estimated/denoised pairwise rank-offsets
\begin{equation} \nonumber
\hat{R}_{ij} = \begin{cases}
 \hat{r}_i - \hat{r}_j;  & \text{if } \{i,j\} \in E \\
  0; & \text{if } (i,j) \notin  E, 
\end{cases}
\label{recoveredC}
\end{equation}
with $\hat{R}_{ij}  > 0$ denoting that the rank of player $i$ is higher than the rank of player $j$.
To measure the accuracy of a proposed reconstruction, we rely on the following two metrics. First, we use the popular metric 
% that counts the number of \textit{upsets} (lower is better)
\begin{equation}
	  \text{Number of upsets} = \sum_{i=1}^{n-1} \sum_{j=i+1}^{n} \mb{1}_{ \{ \text{sign}( R_{ij} \hat{R}_{ij}) = -1 \} }, 
\end{equation}
which counts the number of \textit{upsets} (lower is better), i.e., the number of disagreeing ordered comparisons. It contributes a $+1$ to the summation whenever the ordering in the provided data contradicts the ordering in the induced ranking. Next, we define a weighted version of the upset criterion (lower is better) between the given rank comparison data, and the one induced by the recovered solution
\begin{equation}
	\text{Weighted upsets} = \sum_{i=1}^{n-1}   \sum_{j=i+1}^{n}  | R_{ij} - \hat{R}_{ij} |. 
\end{equation}

For ease of visualization and comparison, in all the barplot figures shown in this section, we color in red (resp. blue) the outcomes attained by \textsc{SVD} (resp. \textsc{SVD-N}). For brevity, for each real data set, we primarily comment on the performance of the two SVD-based methods relative to the other seven methods. We compare performance both before and after the matrix completion step. In each scenario, we compute the above two performance metrics, hence the four columns in the five Figures \ref{fig:NBA}, \ref{fig:AnimalNetworks}, \ref{fig:Hiring_ALL_FM}, \ref{fig:MSFT_Halo}, \ref{fig:PremierLeague}  corresponding to the five real data sets we consider.

%\vspace{-2mm}
\paragraph{NCAA College Basketball.} Our first and most comprehensive real data set comes from sports, and contains the outcomes of NCAA College Basketball matches during the regular season, for the years $1985$ - $2014$. Each separate season provided a pairwise comparison matrix on which we evaluate all algorithms. This data set can essentially be construed as 30 separate problem instances. The input matrix corresponding to each season contains the point difference of the direct match between a pair of teams. If the same pair of teams played multiple matches against each other, the point differences are added up. The resulting pairwise comparison matrix is thus skew-symmetric by construction.

The experimental results shown in the left column of Figure \ref{fig:NBA} pertain to the setting without the matrix completion step, while the right column compares the outcome of all methods after  the low-rank matrix completion pre-processing step. The first row shows performance in terms of the number of upsets across time for each season, while the barplot in the second row shows the average attained by all methods across all seasons. Before matrix completion, \textsc{SVD-N} and \textsc{SVD} are ranked 6th, respectively 8th, while after matrix completion, they are ranked 6th, respectively 7th, out of the 9 algorithms considered. The third and fourth row depict similar results, but the performance metric used is the weighted upsets, which is shown on a log scale for ease of visualization,  primarily due to the poor performance of \textsc{BTL}, which would otherwise distort the $y$-axis. In this case,  \textsc{SVD-N} and \textsc{SVD} are ranked 5th, respectively 6th, while after matrix completion, they are ranked 4th, respectively 5th, out of 9 methods, visibly outperforming the rest of the methods, and attaining a comparable performance to that of the top three methods.

The heatmaps shown in the fifth row of Figure  \ref{fig:NBA} contain the pairwise correlations between the rankings attained by all methods over the 30 seasons. In other words, for any pair of methods, we compute the % Kendall 
correlation between the respective rankings attained in each season, which we then average over the entire set of 30 seasons. To the best of our knowledge, such a correlation study has not been performed in the literature, and provides insights on the potentially different latent rankings inherent in each data set, while still minimizing the number of upsets. \textsc{SYNC}, the best performing method in terms of the number of upsets, has a comparable performance to \textsc{BTL}, and a correlation of 86\%.   \textsc{PGR} and  \textsc{SVD-N} attain a comparable number of upsets, but their Kendall score correlation is only 70\%. The bottom barplot of the same Figure  \ref{fig:NBA} show the average correlation of each of the methods, highlighting that   \textsc{SPR}, \textsc{LS} and  \textsc{BTL}  yield the most correlated rankings with the other methods. A similar correlation analysis is performed on all rankings algorithms after the low-rank matrix completion step, which reveals a number of clusters. As expected,  \textsc{SVD} and  \textsc{SVD-N}  produce very related rankings to each others, and similarly for  \textsc{RSUM} and  \textsc{LS}, and finally,  \textsc{SPR} and  \textsc{BTL}.

% Figure \ref{fig:NBA} (a)  shows the number of upsets across time (without matrix completion), while plot (b) shows the average over the entire period, for each algorithm.  On average, it is typically the case that a pair of teams play each other at most once. We remark that the number of upsets shown in Figure \ref{fig:NBA} (a) and (c) is ascending since the number of teams grew over the years.  For example, during the $1985$ season, the league contained $282$ teams and $3737$ games took place, while in $2014$, there were $351$ teams and $5362$ games. 
%%% , with each team playing on average about 30 games.
% We remark that, on average, the best performer was \textsc{LS}, followed by \textsc{SVD-N}, and that  \textsc{RS}, while  \textsc{GLM} was by far the worst performer. After matrix completion, the top three best performing methods are \textsc{RS}, \textsc{LS} and \textsc{SER}. 

\newcommand{\widB}{3.3in} 
\newcolumntype{C}{>{\centering\arraybackslash}m{\widB}} 
\begin{table*}\sffamily 
\hspace{-9mm} 
\begin{tabular}{l*2{C}@{}}
 & Without Matrix Completion  &   With Matrix Completion  \\
 \hline
% \rotatebox{90}{\textsc{Upsets}}  
\parbox{1.5cm}{  Number \\ of Upsets}
& \includegraphics[width=0.48\columnwidth]{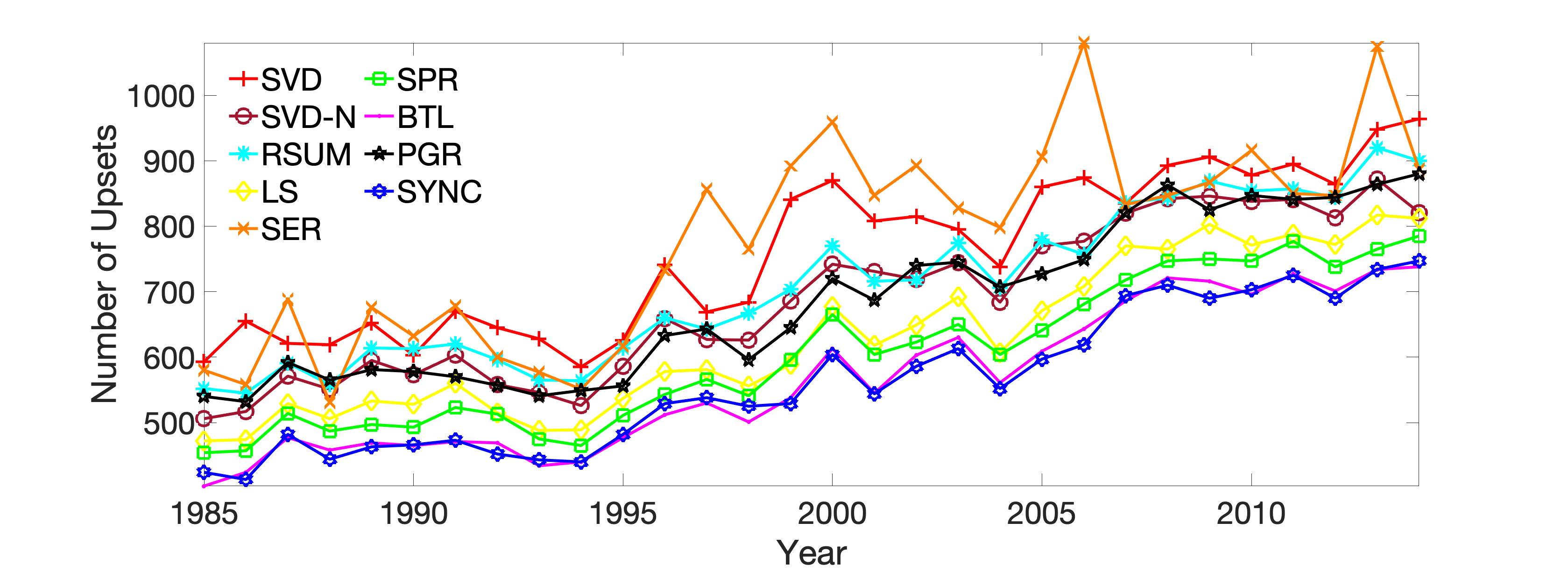}
& \includegraphics[width=0.48\columnwidth]{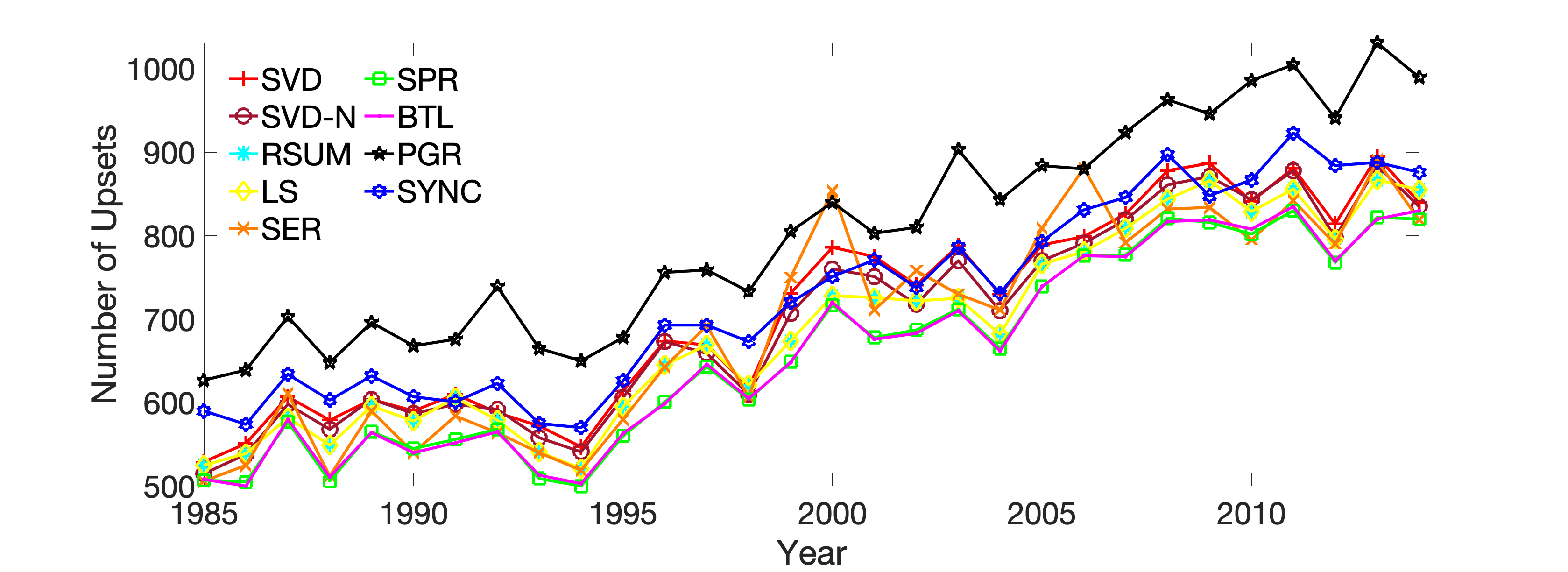}  \\ 
%
% \rotatebox{90}{\textsc{Upsets}}  
\vspace{-2mm}
\parbox{1.5cm}{  Number  \\ of  Upsets}
& \includegraphics[width=0.48\columnwidth]{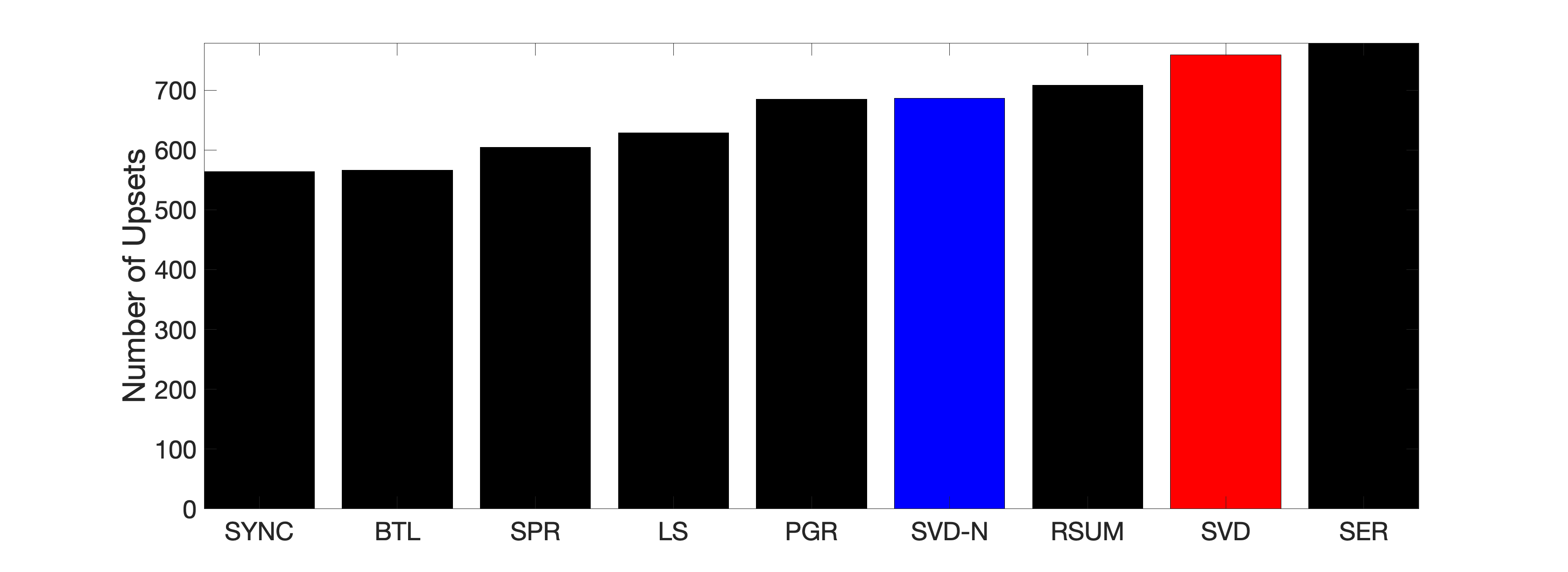}
& \includegraphics[width=0.48\columnwidth]{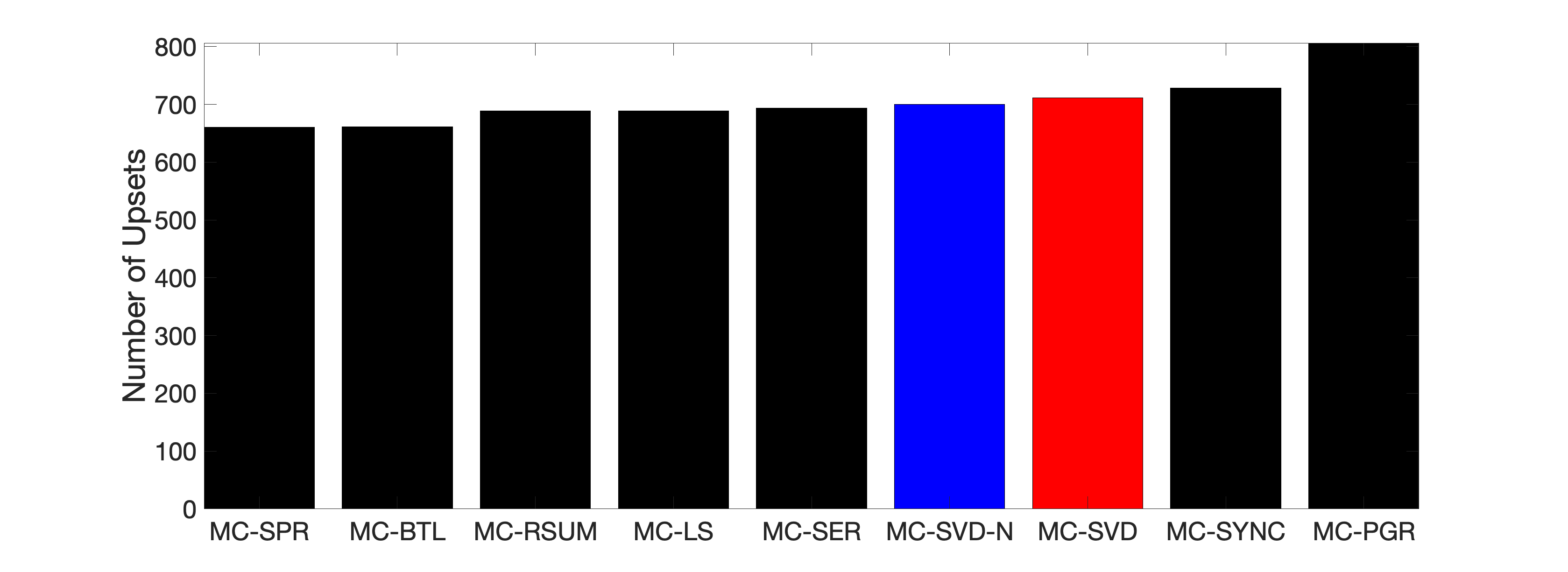} \\
\hline
% \rotatebox{90}{\footnotesize  Weighted Upsets}
% \rotatebox{90}{\parbox{4cm}{Weighted  Upsets}} 
% \rotatebox{90}{\parbox{1cm}{Weighted \\  Upsets}} 
\parbox{1cm}{  Weighted \\  Upsets}
& \includegraphics[width=0.48\columnwidth]{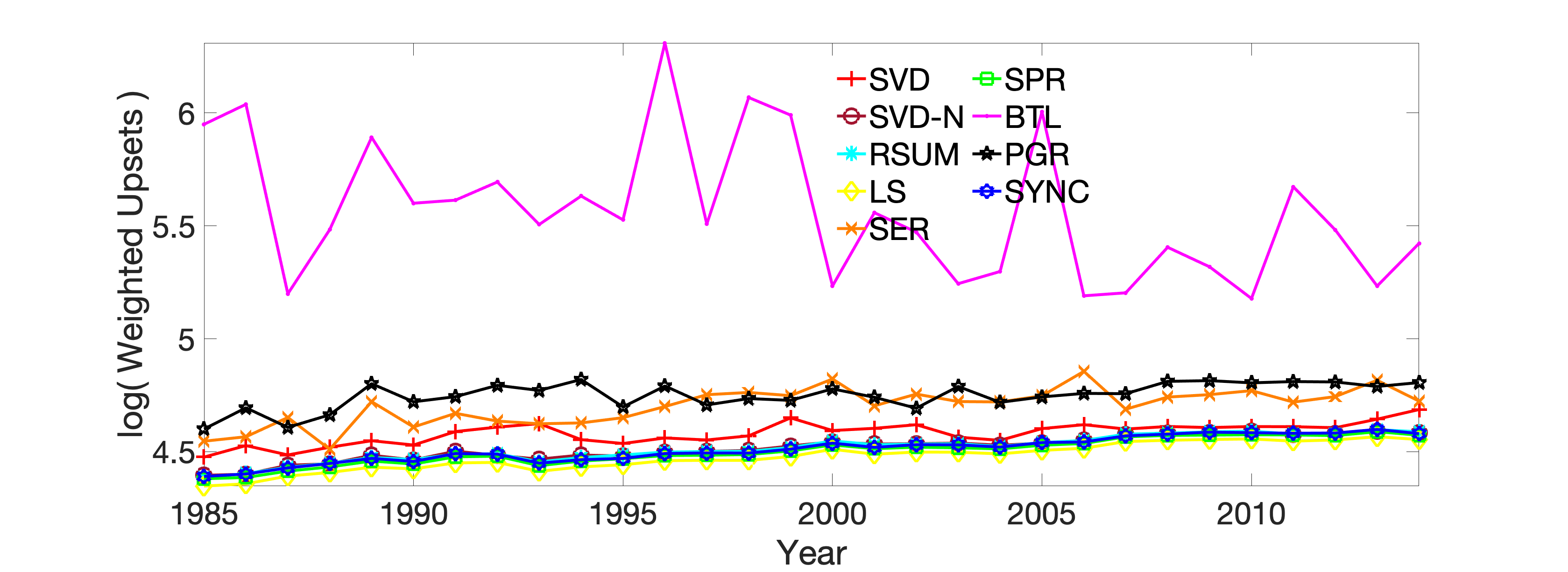}
& \includegraphics[width=0.48\columnwidth]{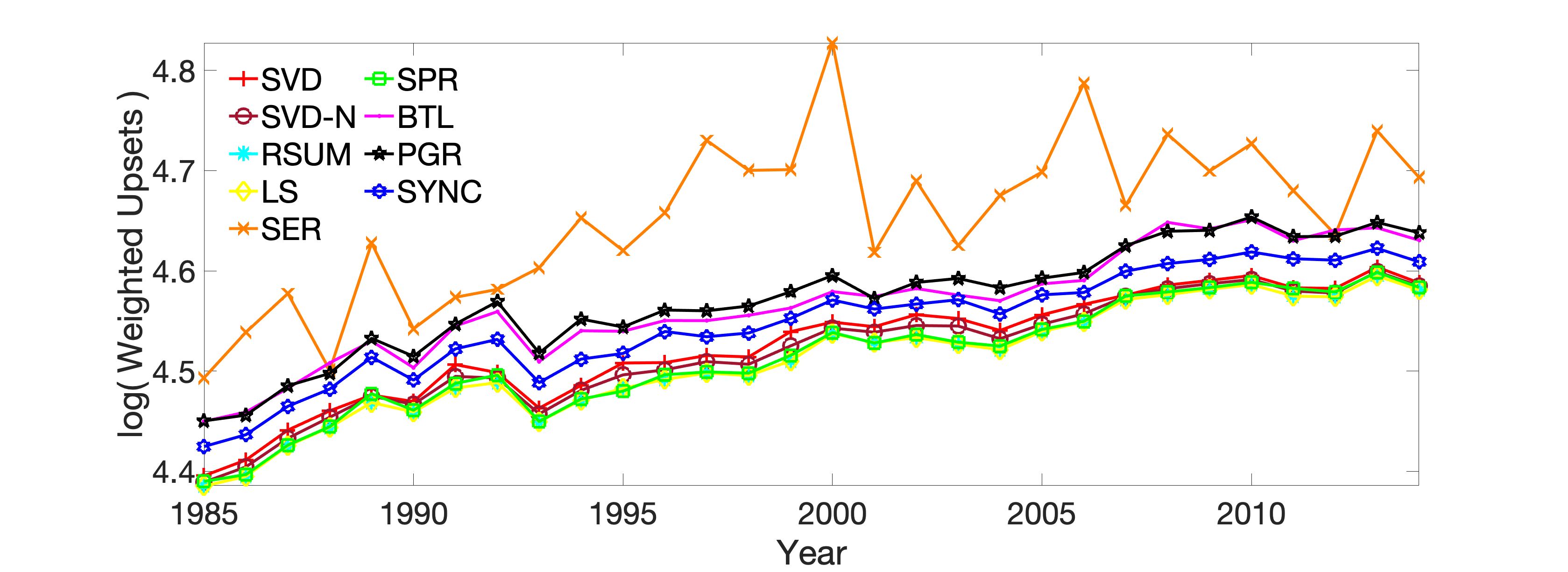}  \\ 
\vspace{-2mm}
% Weighted Upsets   \hspace{-5mm}  
\parbox{1cm}{  Weighted \\  Upsets}
& \includegraphics[width=0.48\columnwidth]{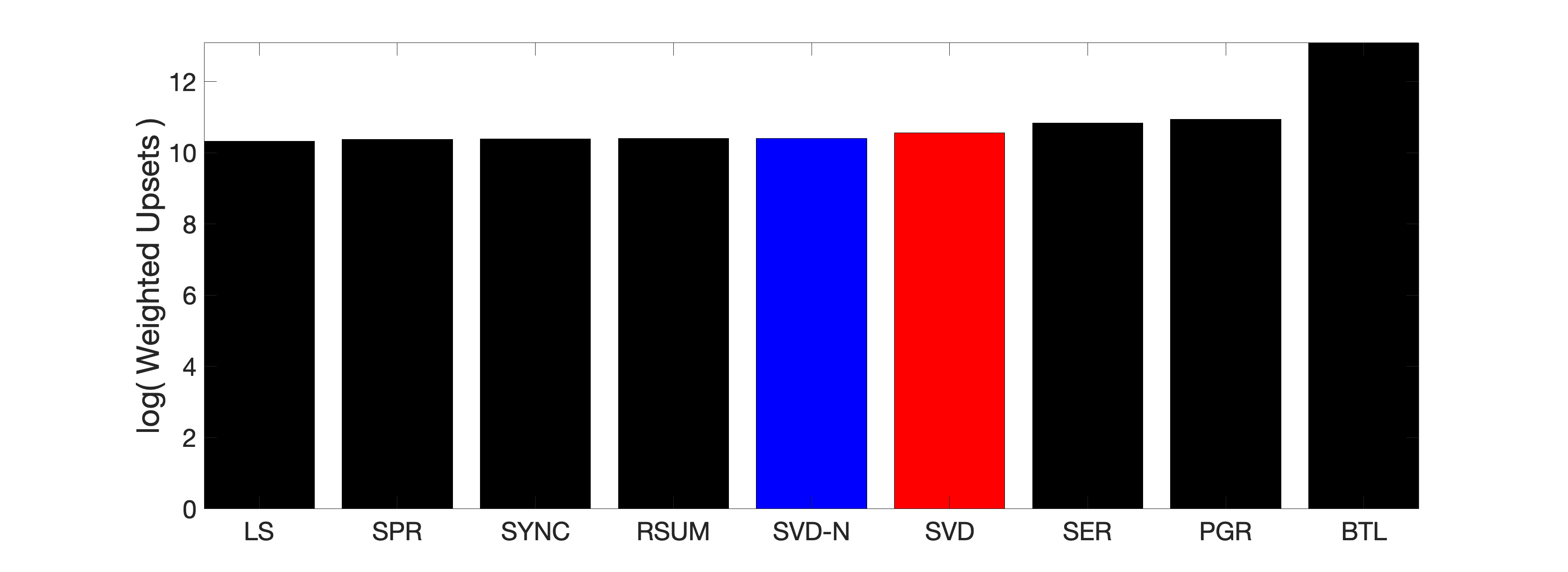}
& \includegraphics[width=0.48\columnwidth]{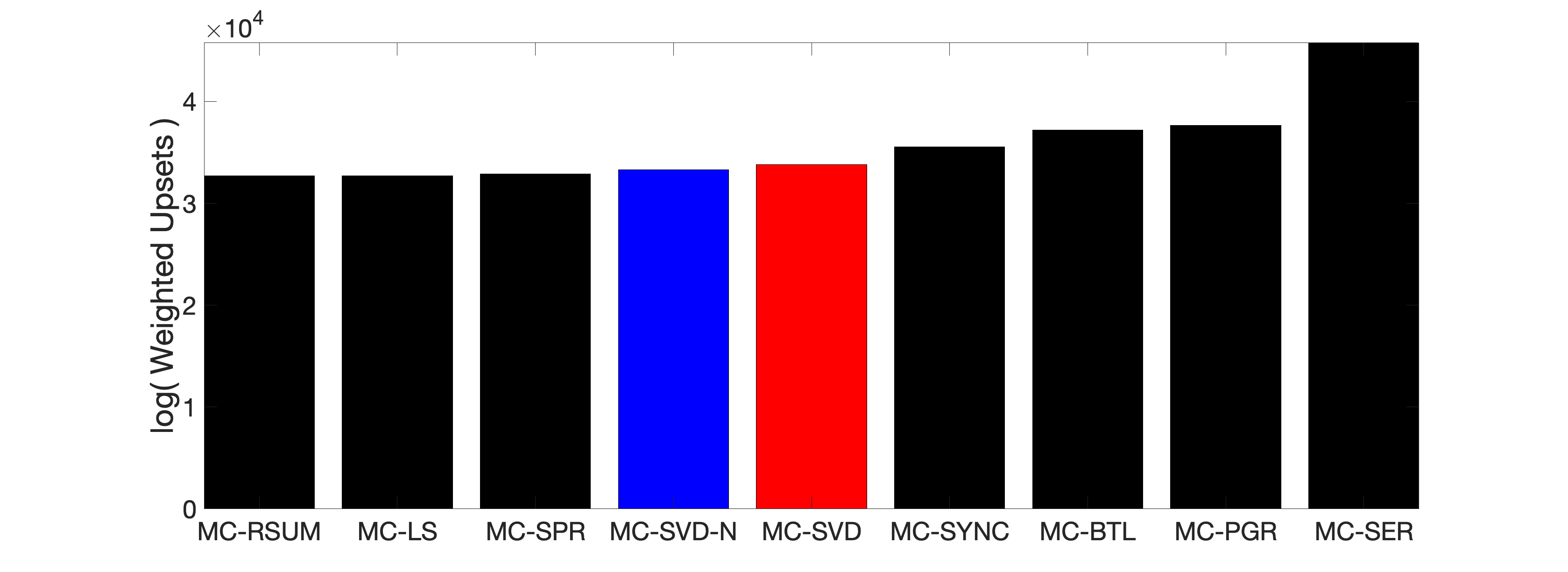} \\
\hline
\parbox{1cm}{  Correlation \\  across \\ methods}
& \includegraphics[width=0.48\columnwidth]{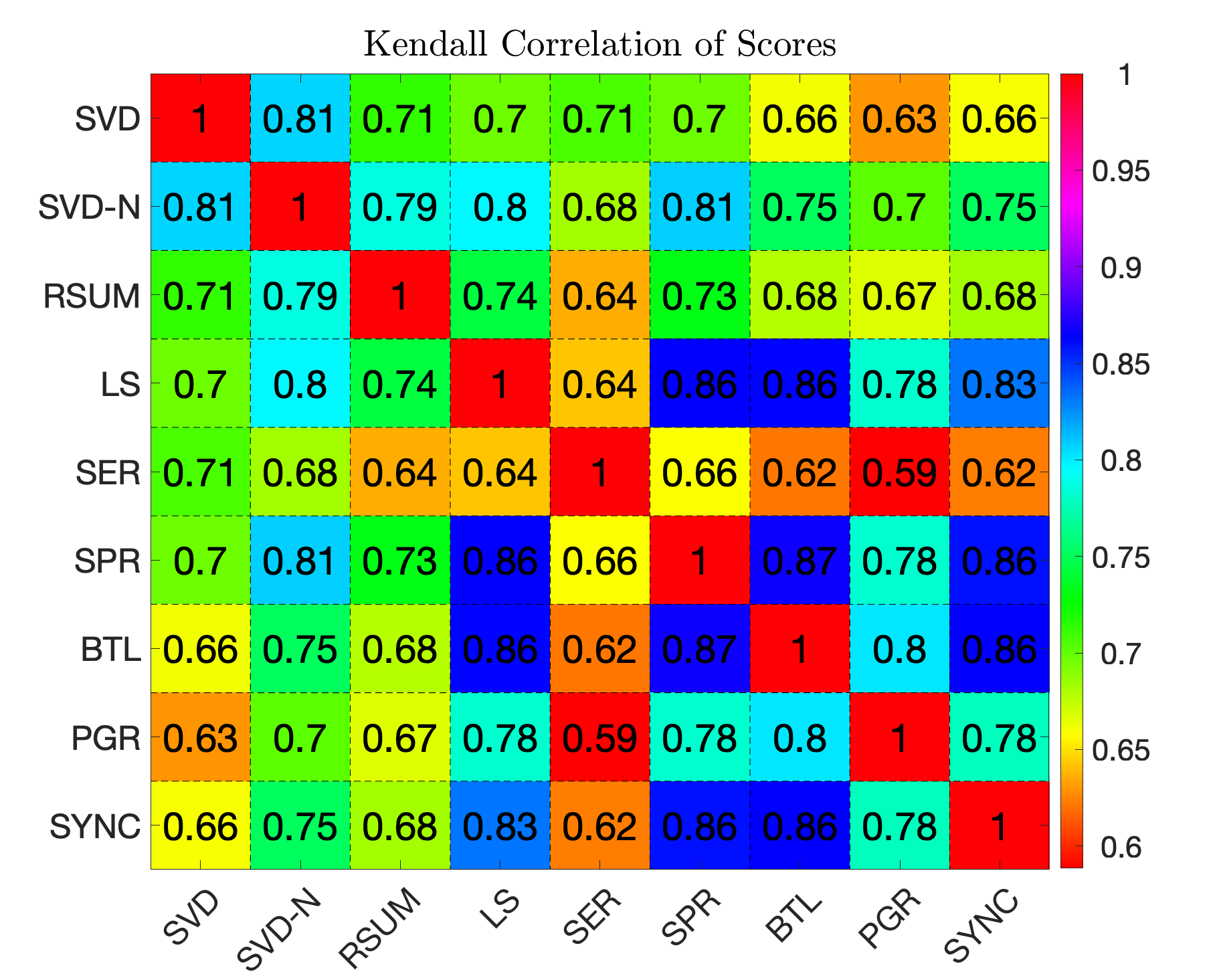}
& \includegraphics[width=0.48\columnwidth]{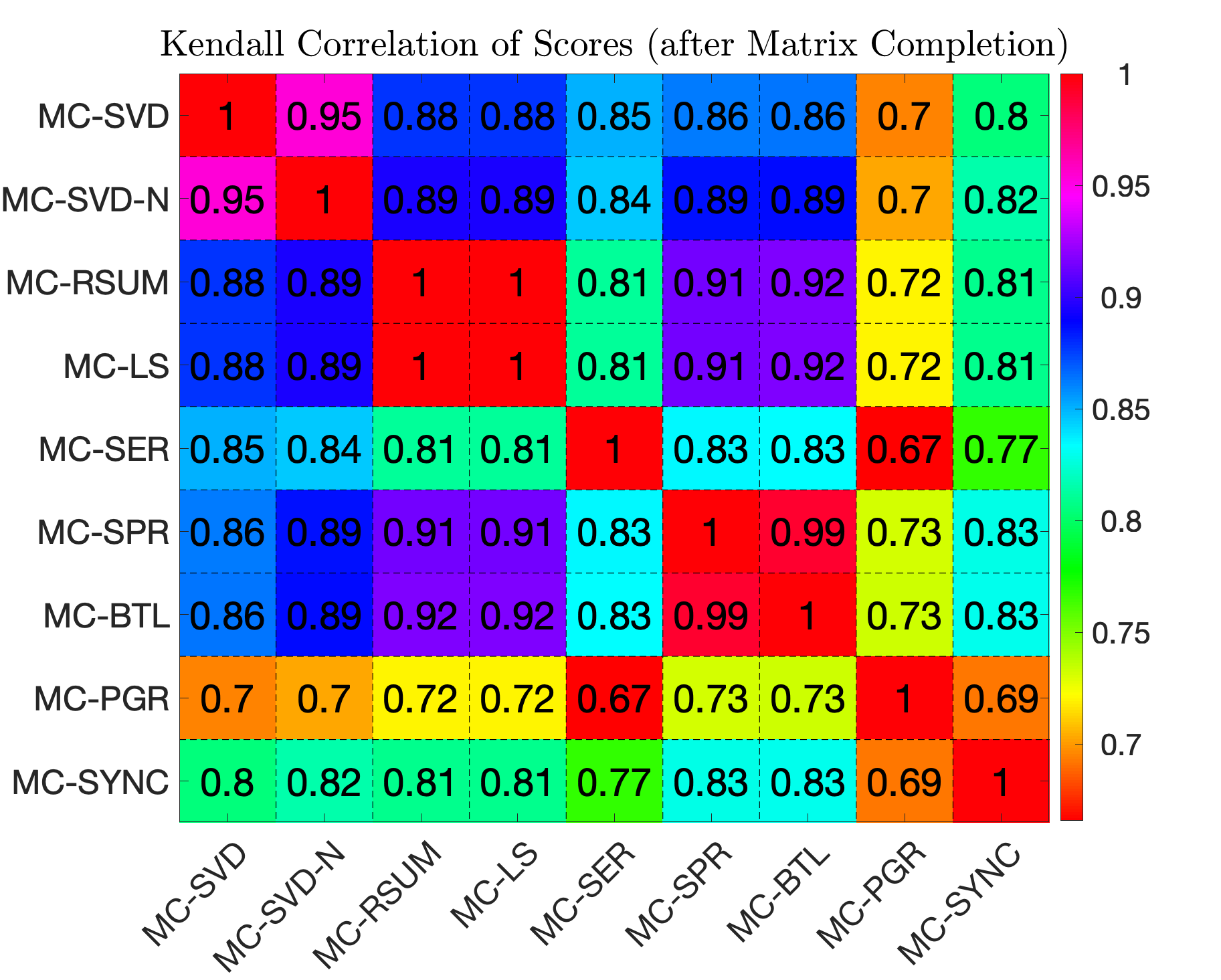} \\
\vspace{-2mm}
\parbox{1.5cm}{  Average \\  correlation \\ with other \\ methods}
& \includegraphics[width=0.48\columnwidth]{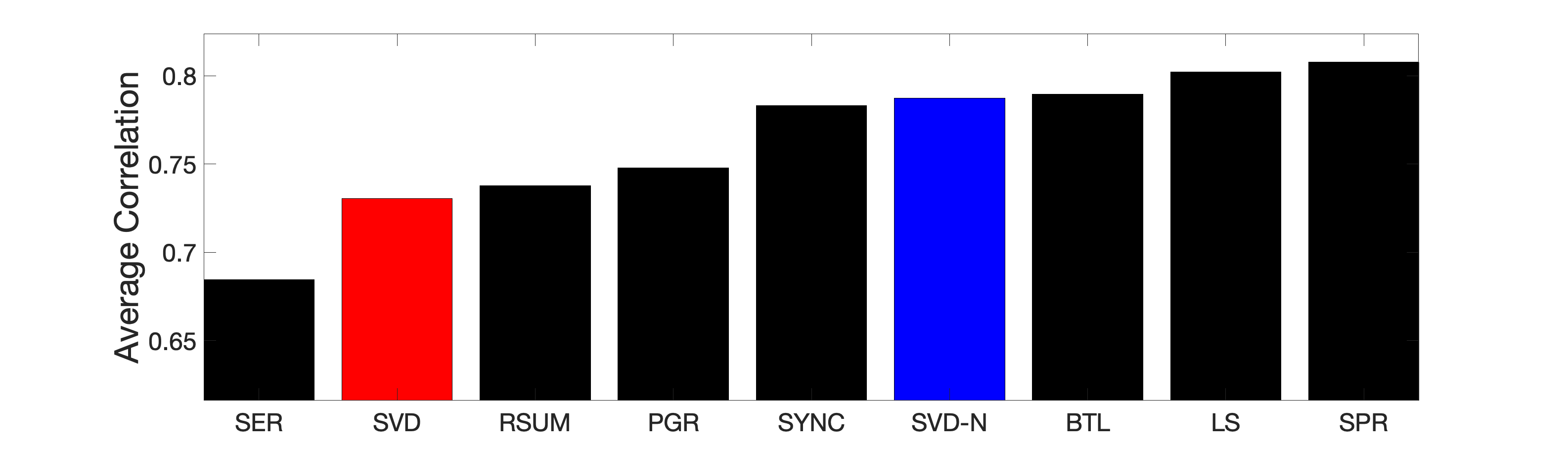}
& \includegraphics[width=0.48\columnwidth]{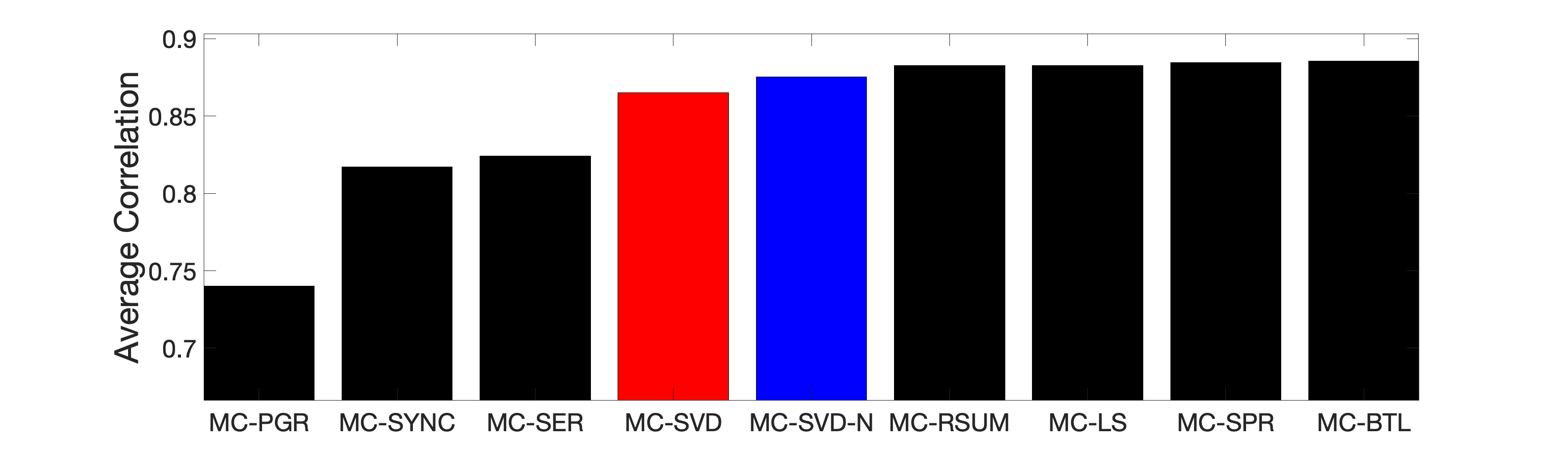}
\vspace{-2mm}
\end{tabular}
\captionsetup{width=0.99\linewidth}
\vspace{-2mm}
\captionof{figure}{Comparison of the algorithms for the NCAA College Basketball data set (1984 - 2014), before (left column) and after (right column) the matrix completion step. We analyze 30 measurement graph instances, one for each season, and estimate rankings for each problem instance, for all the algorithms we considered. The first row shows the number of upsets for each season, and the second row displays the averages across all the seasons.  Similar plots are shown in the third and fourth row, but this time we compare performance in terms of  weighted upsets. The heatmaps in the fifth row show the average (across seasons) correlation between the rankings estimated by each pair of methods, with the average correlation degree depicted in the bottom row.}
\label{fig:NBA}
\end{table*}

% \FloatBarrier

\newcommand{\widfour}{1.3in} 
\newcolumntype{C}{>{\centering\arraybackslash}m{\widfour}}

\vspace{-1mm}
\paragraph{Animal dominance network.}  Our second  real world example is a set of two networks of animal dominance  among captive monk parakeets \cite{hobson2015social}. The input matrix $H$ is skew-symmetric, with $H_{ij}$ denoting the number of net aggressive wins of animal $i$ toward animal $j$. The study  covers two groups (of 21, respectively, 19 individuals), and spans across four quarters. 
We follow the same approach as in \cite{CaterinaDeBacco_Ranking}, and only tested the algorithms on Quarters 3 and 4, based on guidance from the behavioral ecologist who collected and analyzed the data. Since the animals did not know each other well at the start of the experiment, the hierarchies were still in the transient process of forming. Figure  \ref{fig:AnimalNetworks} compares performance in  terms of the number of upsets and weighted upsets, with and without the matrix completion step. Overall, we observe that \textsc{SVD-N} outperforms  \textsc{SVD}, and the two methods typically rank halfway in the performance ranking, except in the leftmost column, where  \textsc{SVD} performs worst. 

% For two out of the four data sets, \textsc{SVD-N} is the best performer, without matrix completion. After matrix completion, \textsc{SVD} ranks second for two of the data sets, while \textsc{SVD-N} ranks first on one data set and comes third on two of the other data sets. 

% This file contains the number of aggressive wins for parakeets in two study groups across four study quarters.  % Study species: Myiopsitta monachus

\begin{table*}\sffamily 
\hspace{-9mm} 
\begin{tabular}{l*4{C}@{}}
\parbox{1.2cm}{Group\\Quarter}  & Number of Upsets & Number of Upsets (Matrix Completion) & Weighted Upsets & Weighted Upsets (Matrix Completion)  \\
 \hline
% \rotatebox{90}{\textsc{G1-Q3}}  
% \parbox{1.2cm}{ G1-Q3}
\scriptsize{G1-Q3} \hspace{-5mm}
& \includegraphics[width=0.248\columnwidth]{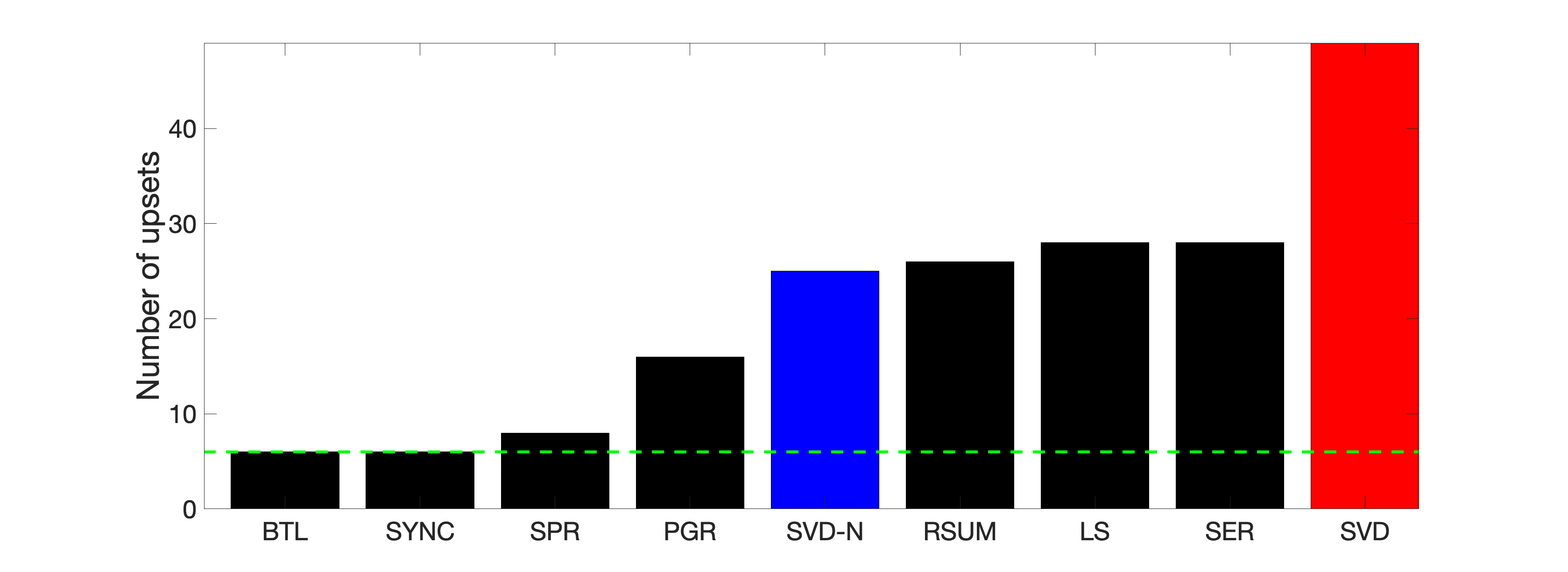}
& \includegraphics[width=0.248\columnwidth]{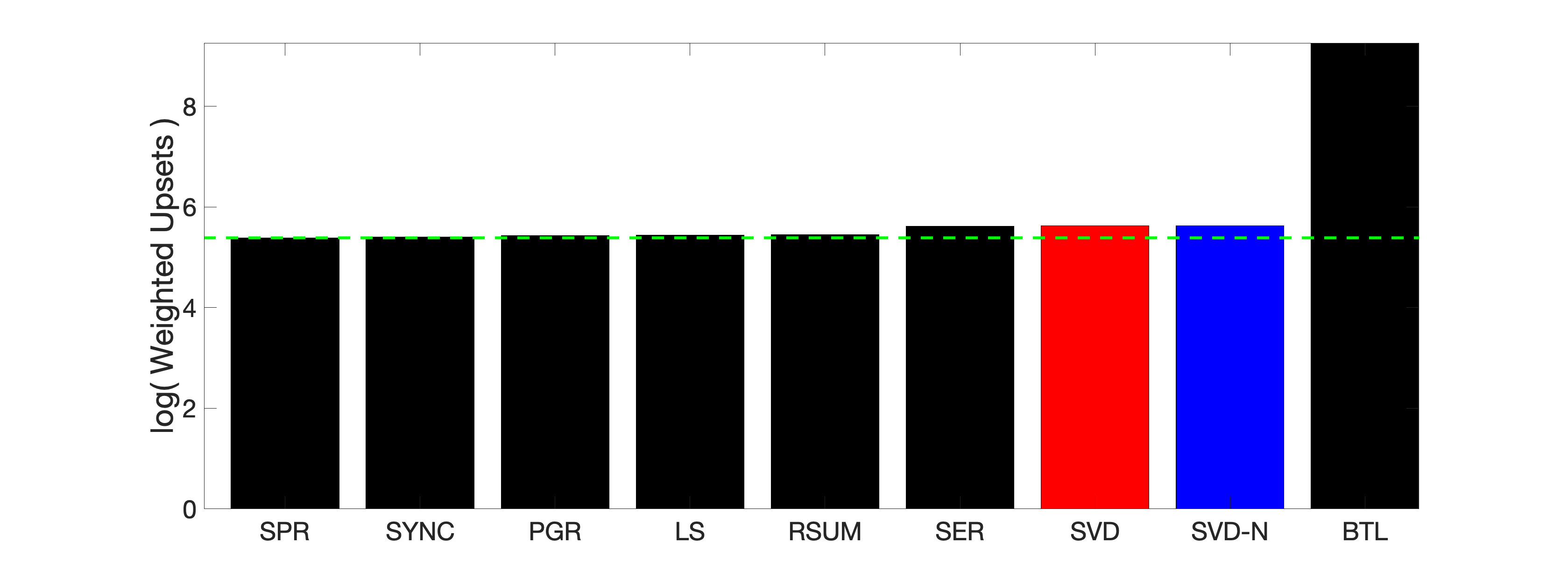}  
& \includegraphics[width=0.248\columnwidth]{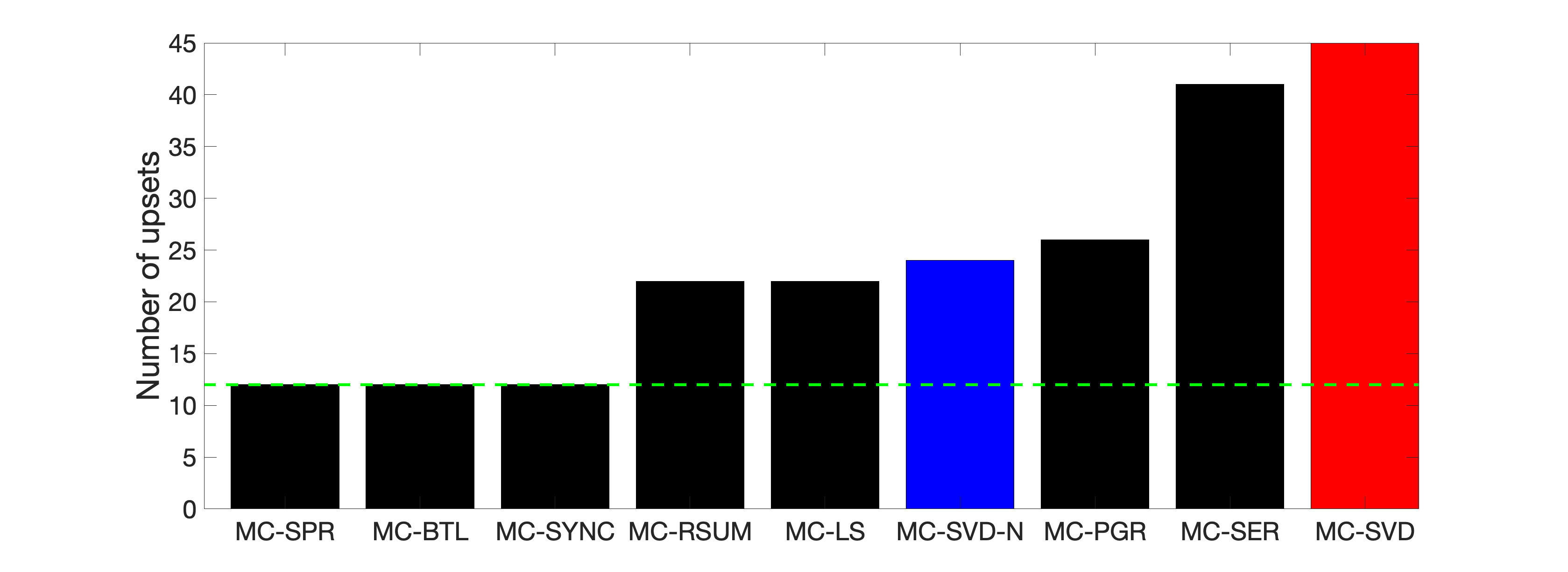}
& \includegraphics[width=0.248\columnwidth]{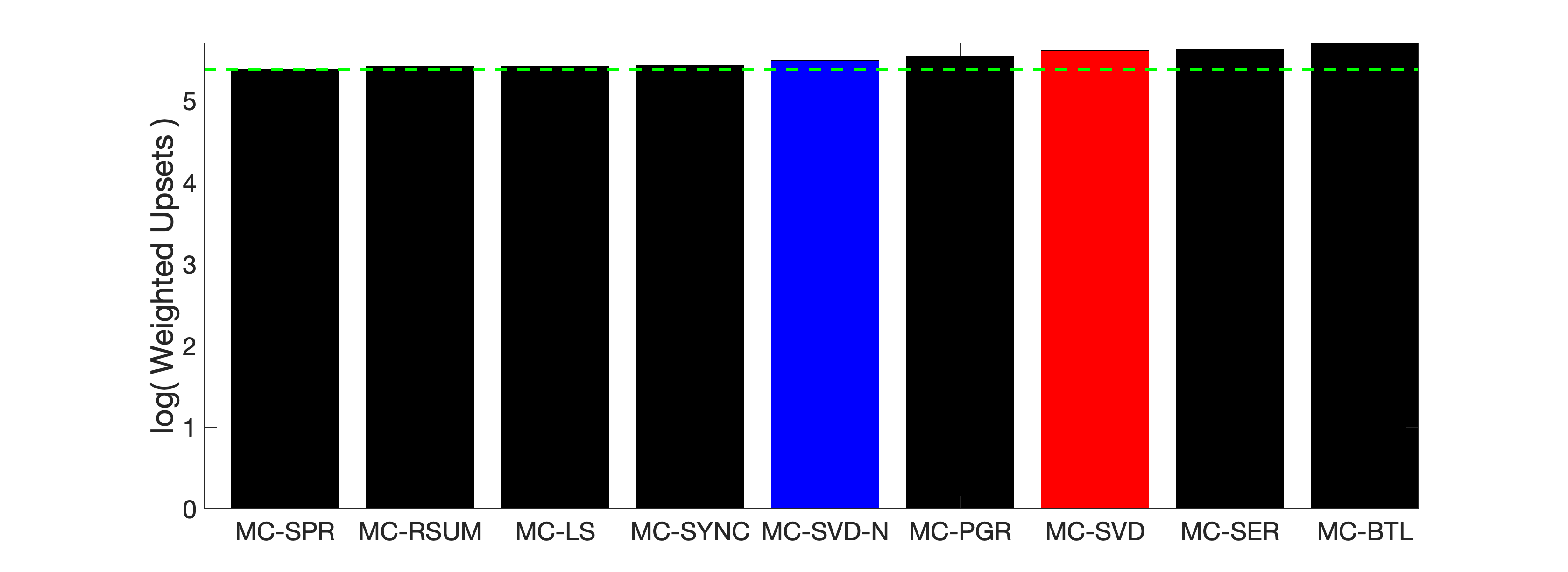}  \\ 
\scriptsize{G1-Q4} \hspace{-5mm}
& \includegraphics[width=0.248\columnwidth]{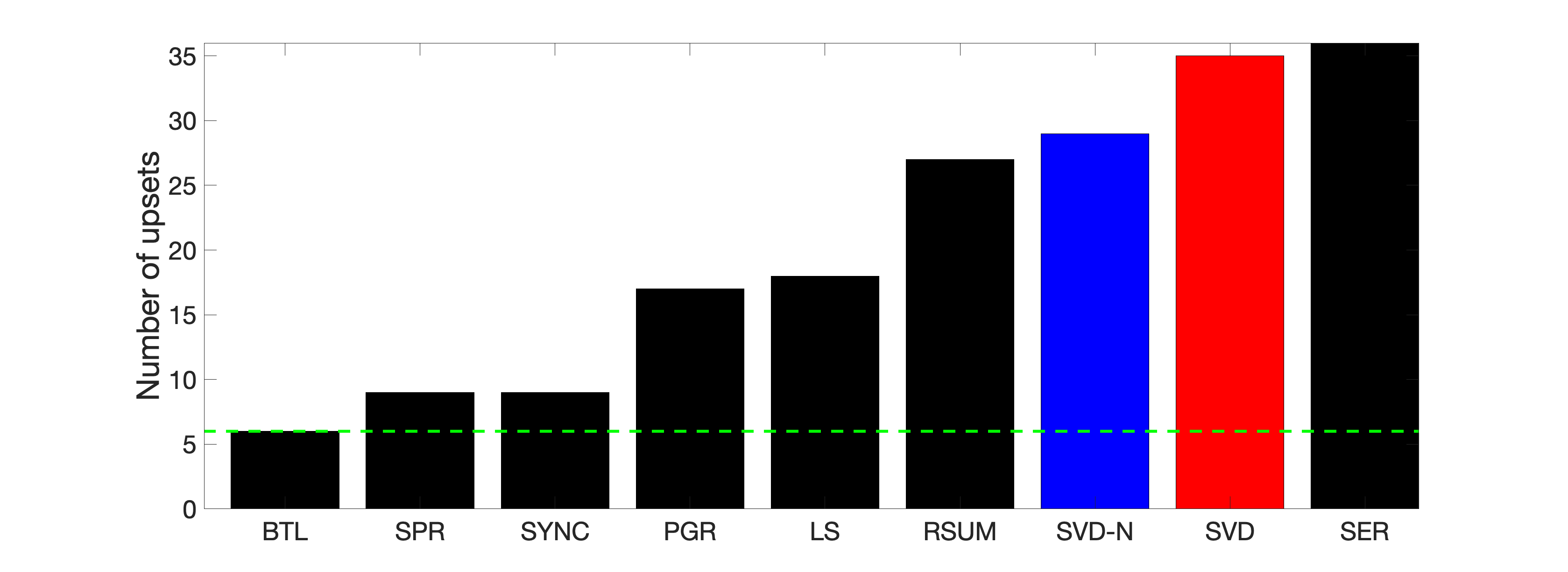}
& \includegraphics[width=0.248\columnwidth]{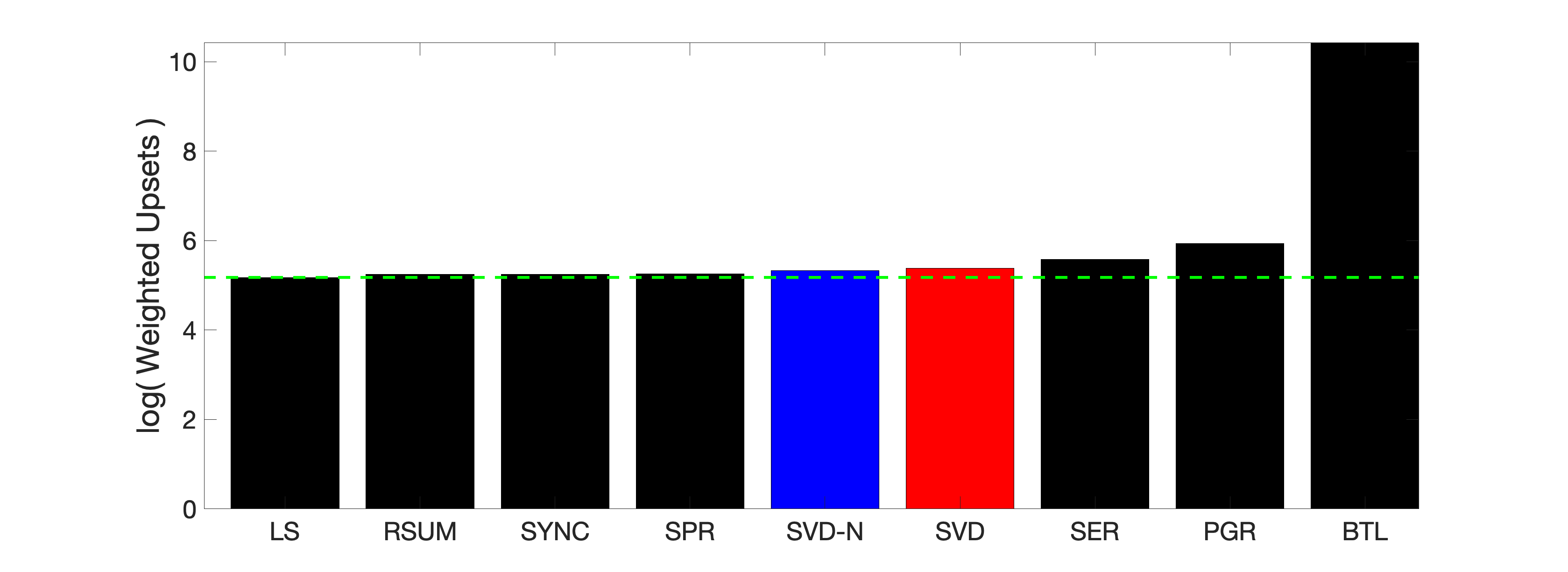}  
& \includegraphics[width=0.248\columnwidth]{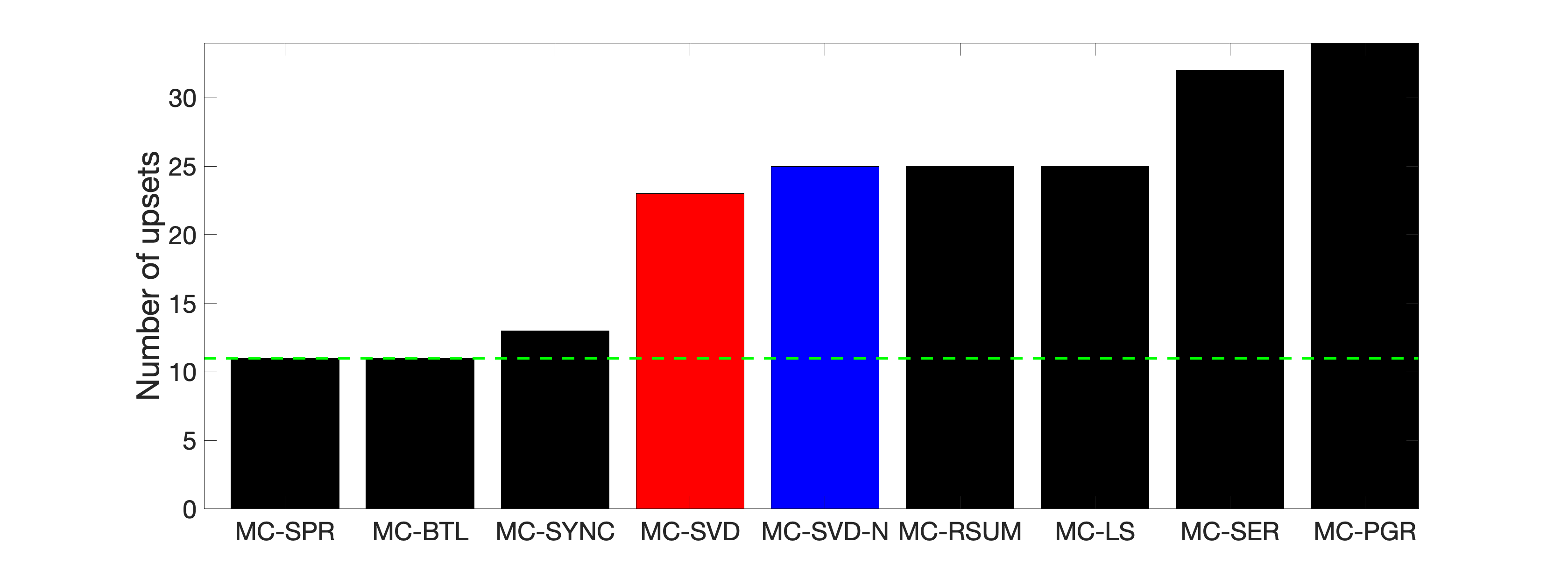}
& \includegraphics[width=0.248\columnwidth]{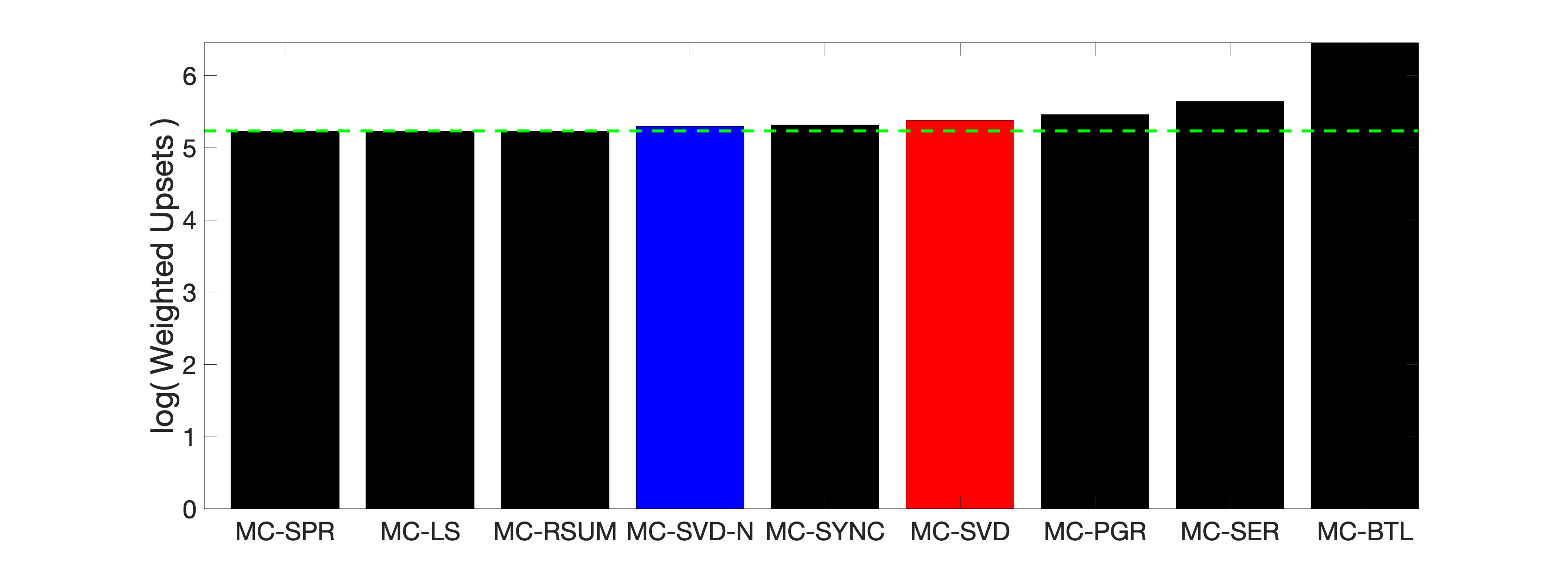}  \\ 
\scriptsize{G2-Q3} \hspace{-5mm}
& \includegraphics[width=0.248\columnwidth]{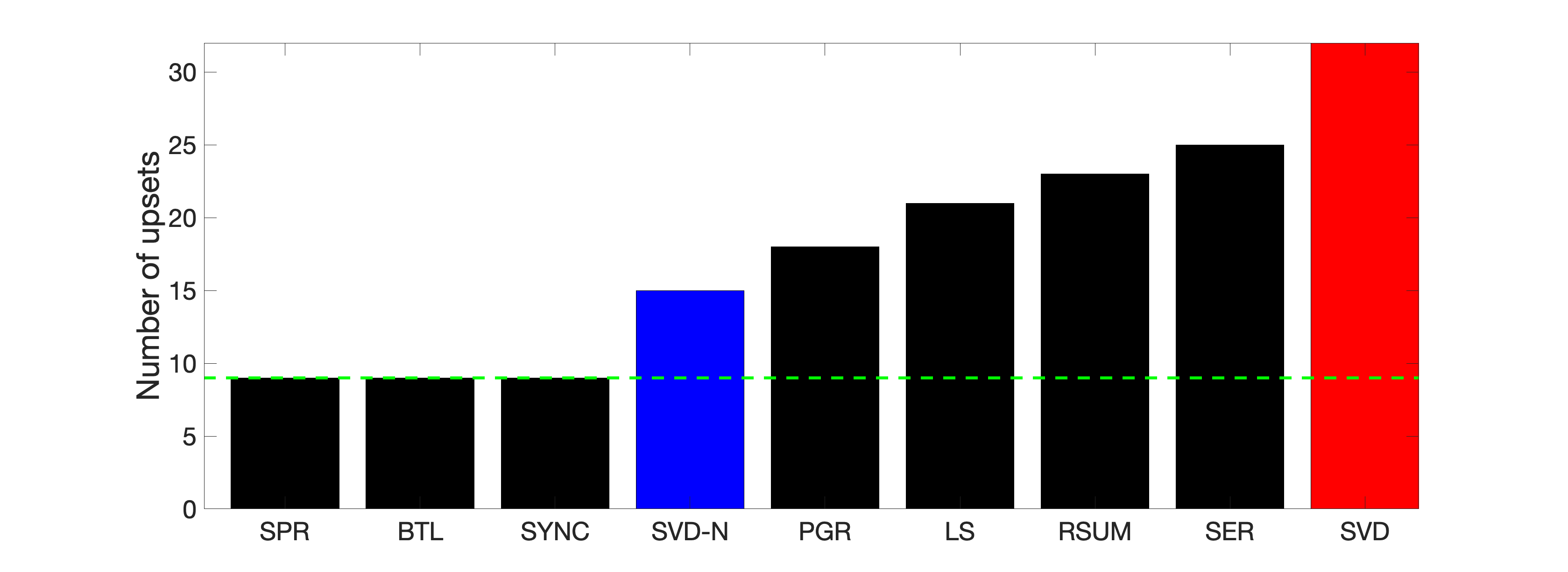}
& \includegraphics[width=0.248\columnwidth]{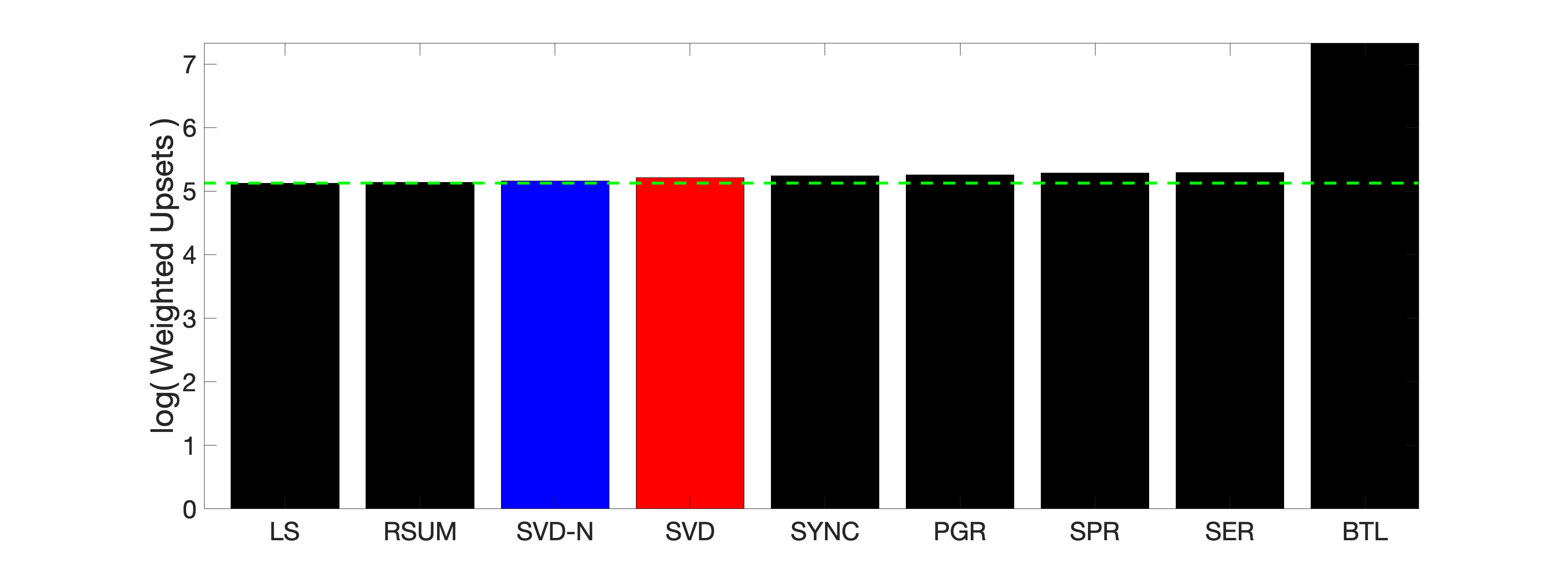}  
& \includegraphics[width=0.248\columnwidth]{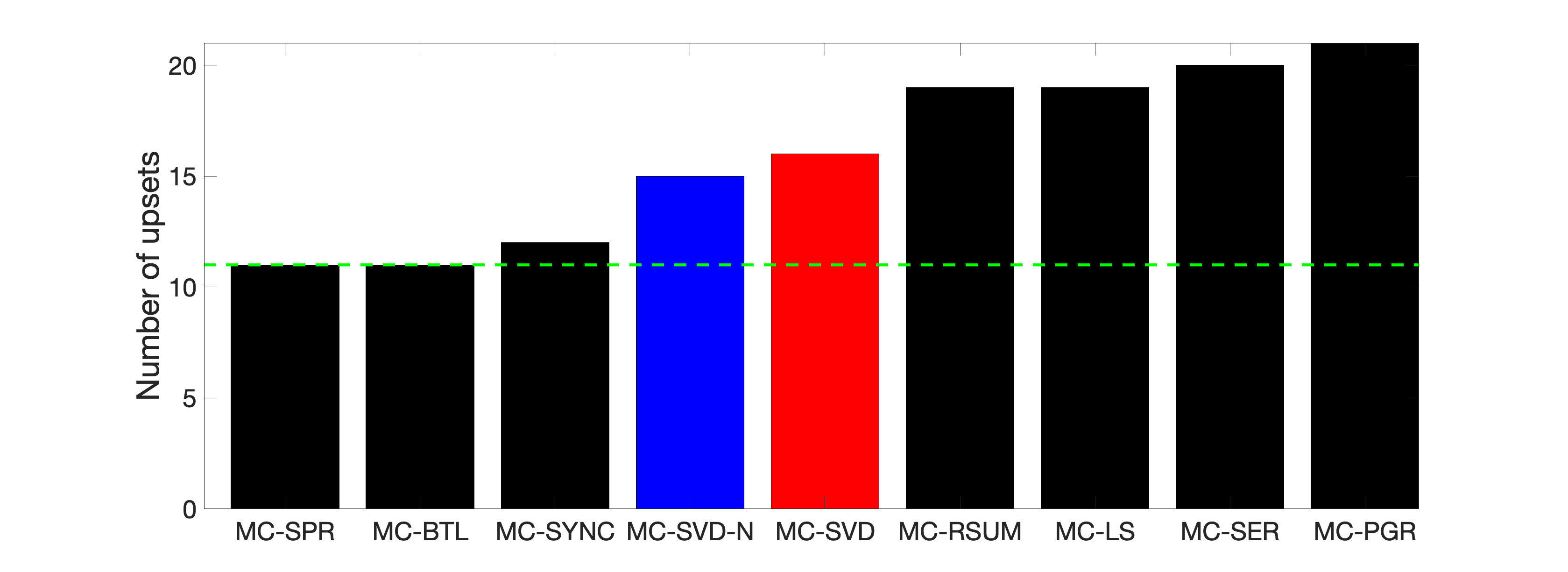}
& \includegraphics[width=0.248\columnwidth]{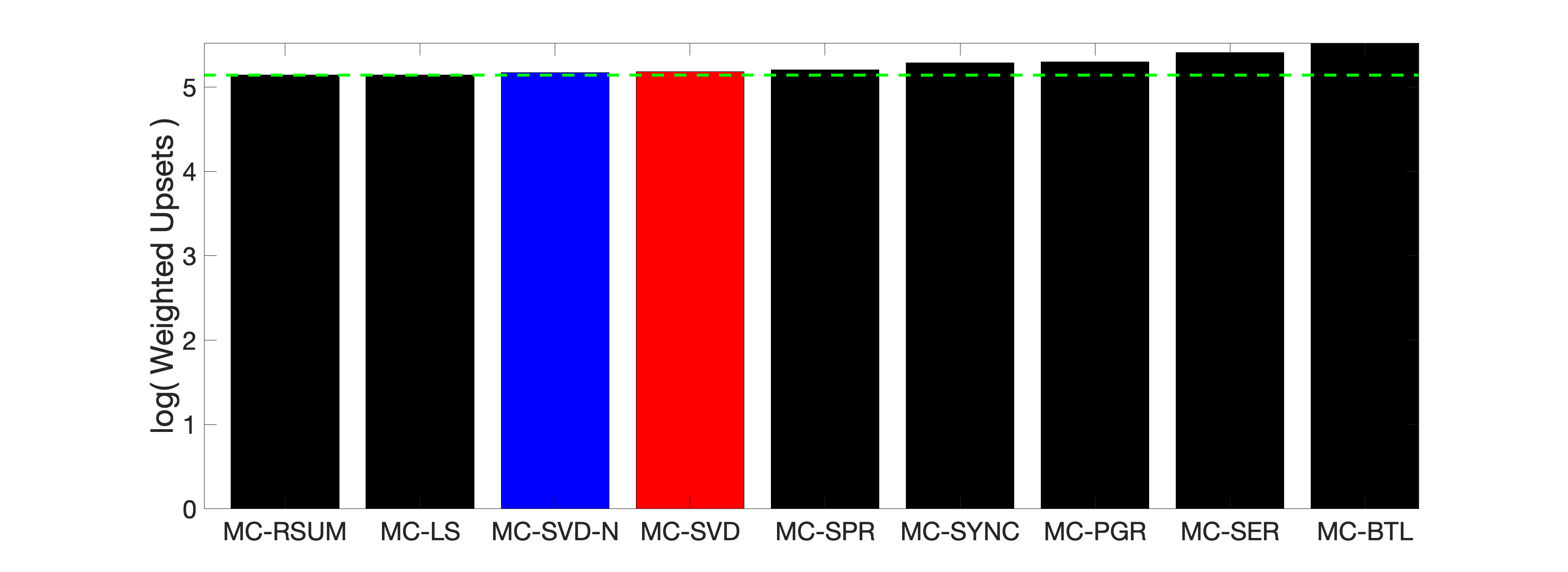}  \\ 
\scriptsize{G2-Q4} \hspace{-5mm}
& \includegraphics[width=0.248\columnwidth]{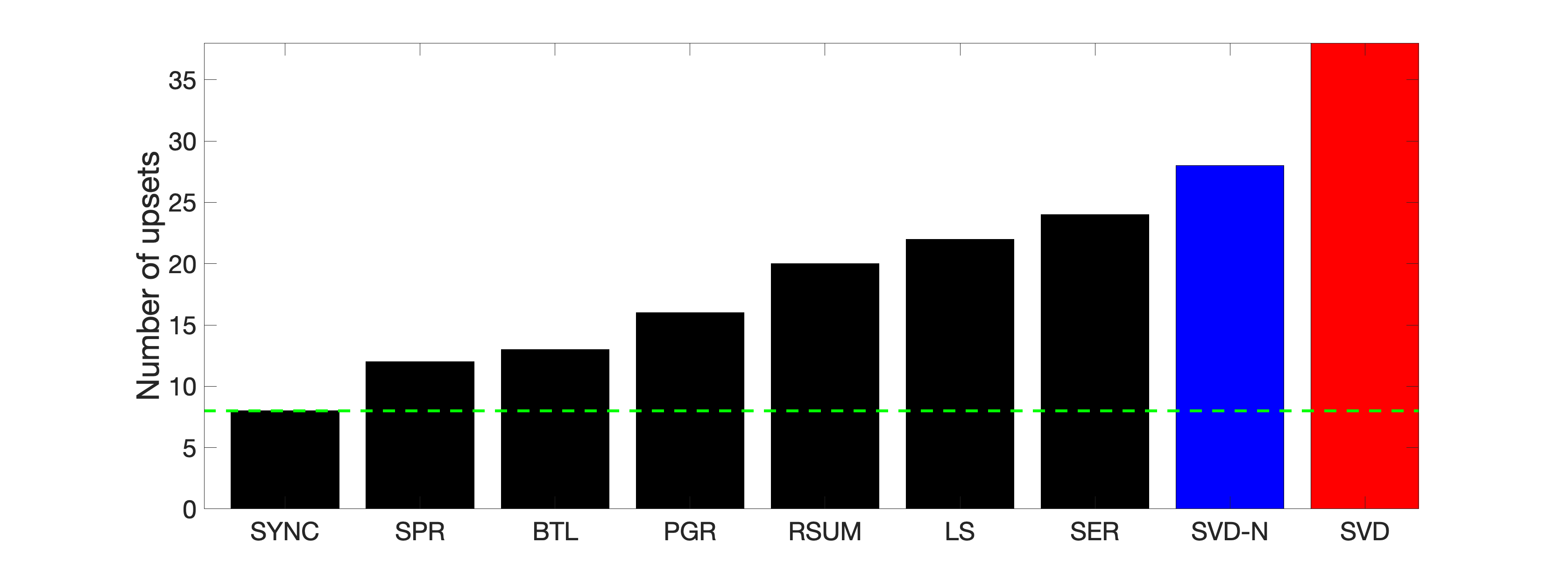}
& \includegraphics[width=0.248\columnwidth]{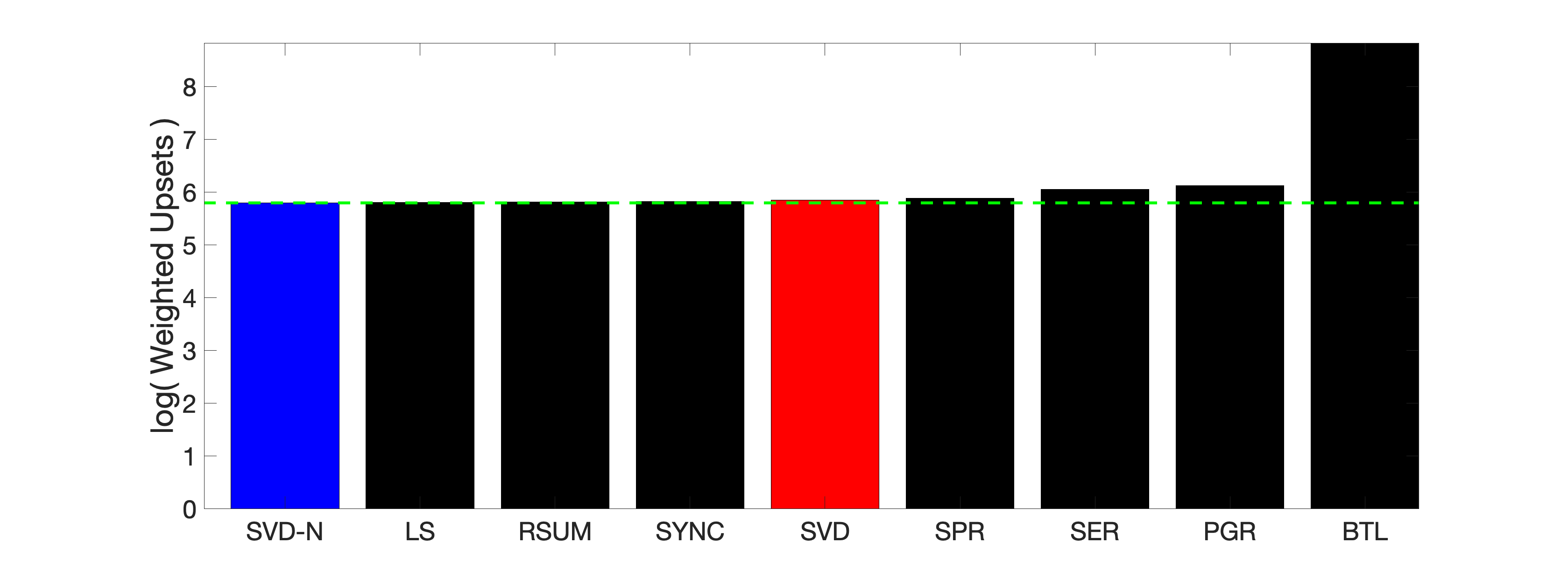}  
& \includegraphics[width=0.248\columnwidth]{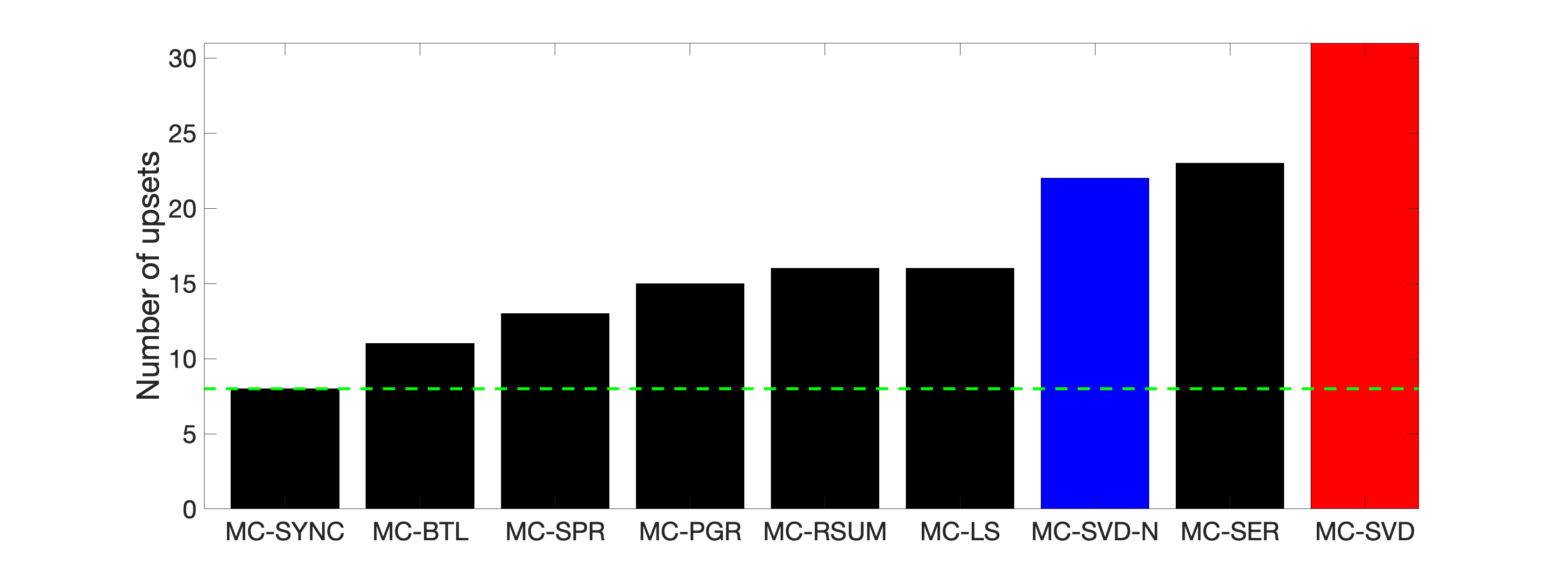}
& \includegraphics[width=0.248\columnwidth]{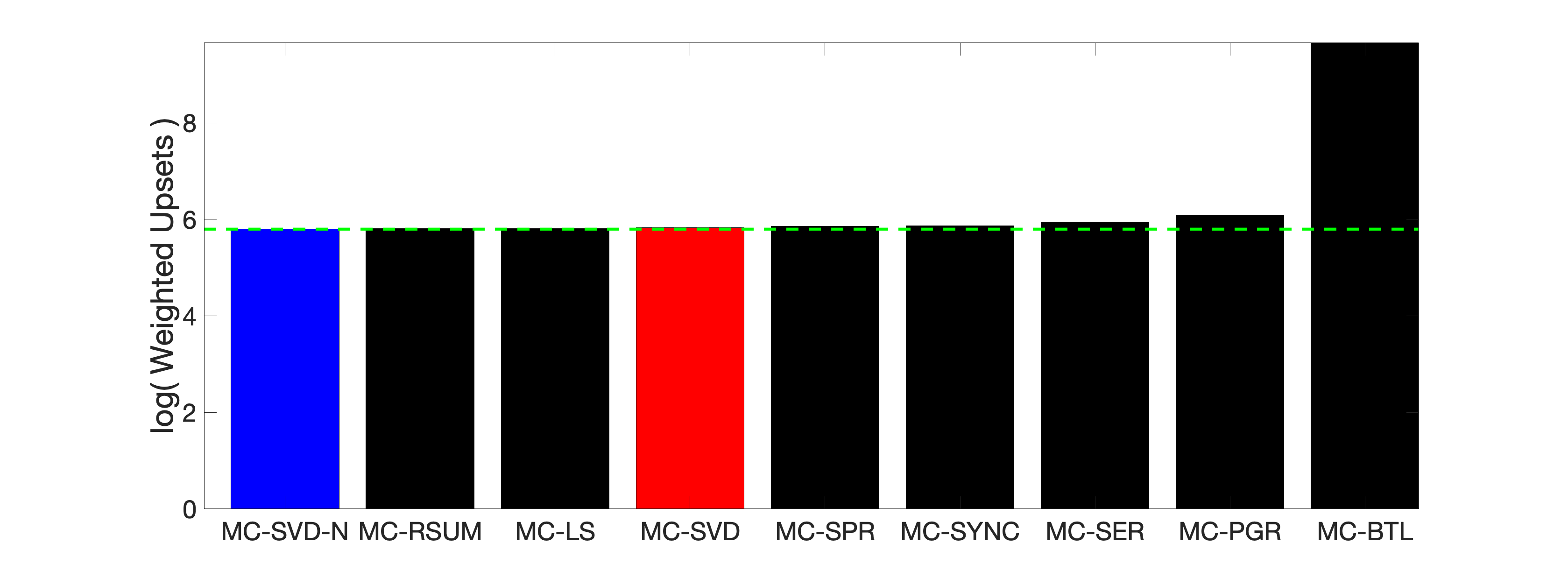}   
% \rotatebox{90}{\textsc{Upsets}}  
\end{tabular}
\captionsetup{width=0.99\linewidth}
\vspace{-3mm}
\captionof{figure}{Performance comparison in the animal dominance networks,  for two groups in Quarters 3-4, thus a total of four networks altogether. We compare both in terms of the number of upsets (first two columns) and weighted upsets (last two columns). The results without the matrix completion step pertain to the first and third columns, while columns two and four show results obtained after a low-rank matrix completion preprocessing step.
}
\label{fig:AnimalNetworks}
\vspace{-3mm}
\end{table*}

%\vspace{-2mm}
\paragraph{Faculty hiring networks.} Our next example covers  three North American academic hiring networks, that track the flow of academics between universities \cite{clauset2015systematic}. The flow is captured in a directed graph with 
% $n=415$ and 
adjacency matrix $A$, such that $A_{ij}$ is the number of faculty at university $j$ who received their doctorate from university $i$. We then consider the skew symmetric matrix $H = A - A^T$, capturing the net flow of faculty between a pair of institutions. Figure \ref{fig:Hiring_ALL_FM} shows the number of upsets attained by each algorithm, for three different disciplines: (a) Computer Science  ($n=206$),  % , $m=2658$ edges
(b) Business ($n=113$) and (c) History  ($n=145$), before and after the matrix completion step, for both upsets criteria. Again, 
we observe that \textsc{SVD-N} outperforms  \textsc{SVD}. Furthermore,  \textsc{SVD-N} typically ranks in the top half of the rankings, and for the case of weighted upsets after the matrix completion step, it is the best performer for the Business and History fields, and second best in Computer Science. 

% 
% Without matrix completion, \textsc{SVD-N} ranks first in History, comes second in the remaining two data sets. After matrix completion,  \textsc{SVD-N} ranks second in Business an third in History, while  \textsc{SVD} ranks third in Computer Science and Business. 

% Information is more granularly available by gender and seniority, but we defer the comparison to the appendix. There, we consider 12 sub-networks indexed by gender and seniority. On average, across all three disciplines and the resulting 36 sub-networks, we observe that \textsc{SVD} and/or \textsc{SVDN} are almost always among two best performers.
% for Computer Science (Figure \ref{tab:faculty_CS}, $n=206$),  % , $m=2658$ edges
% Business (\ref{tab:faculty_BS}, $n=113$) and 
% History (\ref{tab:faculty_HS}, $n=145$). % m=4408  

\begin{table*}\sffamily 
\hspace{-9mm} 
\begin{tabular}{l*4{C}@{}}
Field & Number of Upsets & Number of Upsets (Matrix Completion) & Weighted Upsets & Weighted Upsets (Matrix Completion)  \\
 \hline
% \rotatebox{90}{\textsc{G1-Q3}}  
% \parbox{1.2cm}{ G1-Q3}
\scriptsize{\parbox{1.2cm}{Computer Science}} \hspace{-5mm}
& \includegraphics[width=0.248\columnwidth]{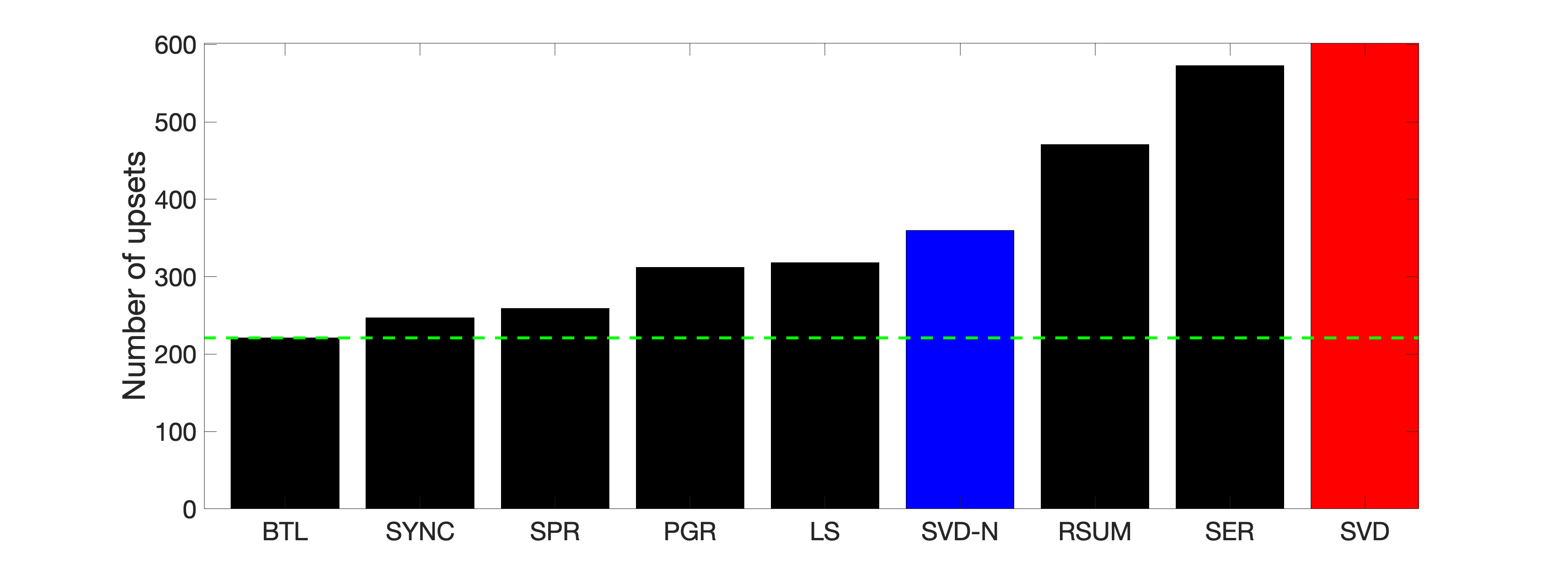}
& \includegraphics[width=0.248\columnwidth]{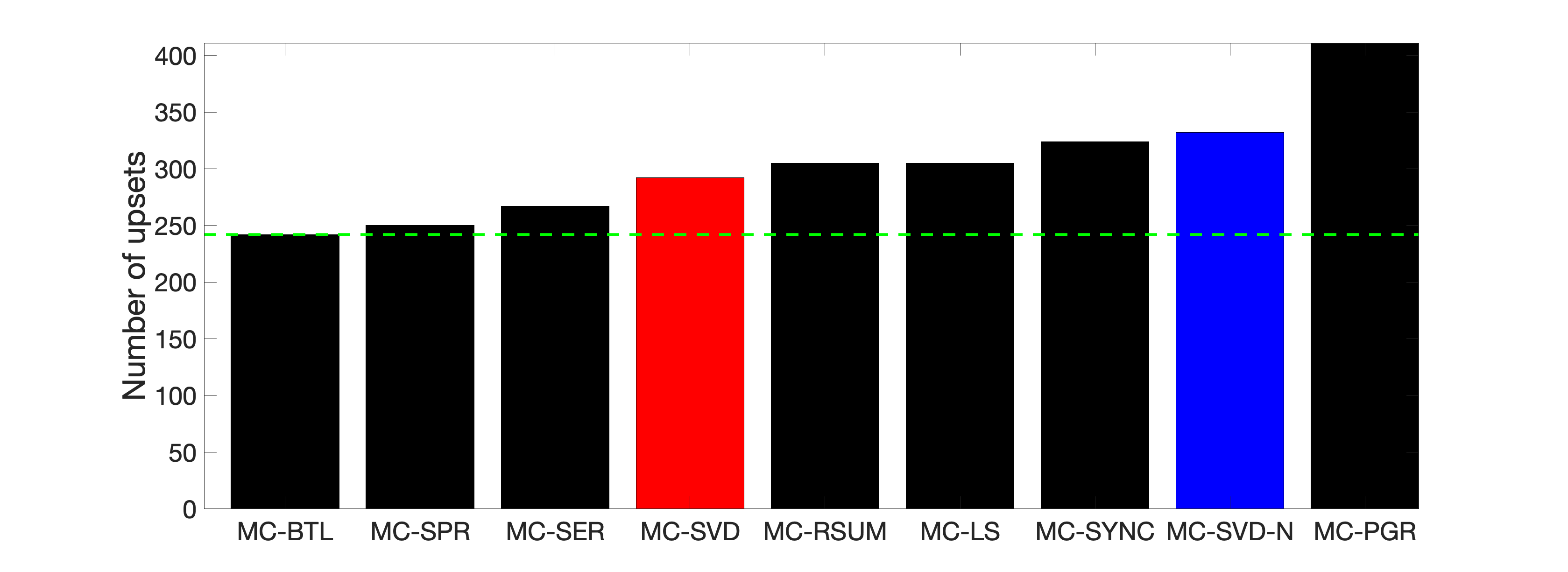}  
& \includegraphics[width=0.248\columnwidth]{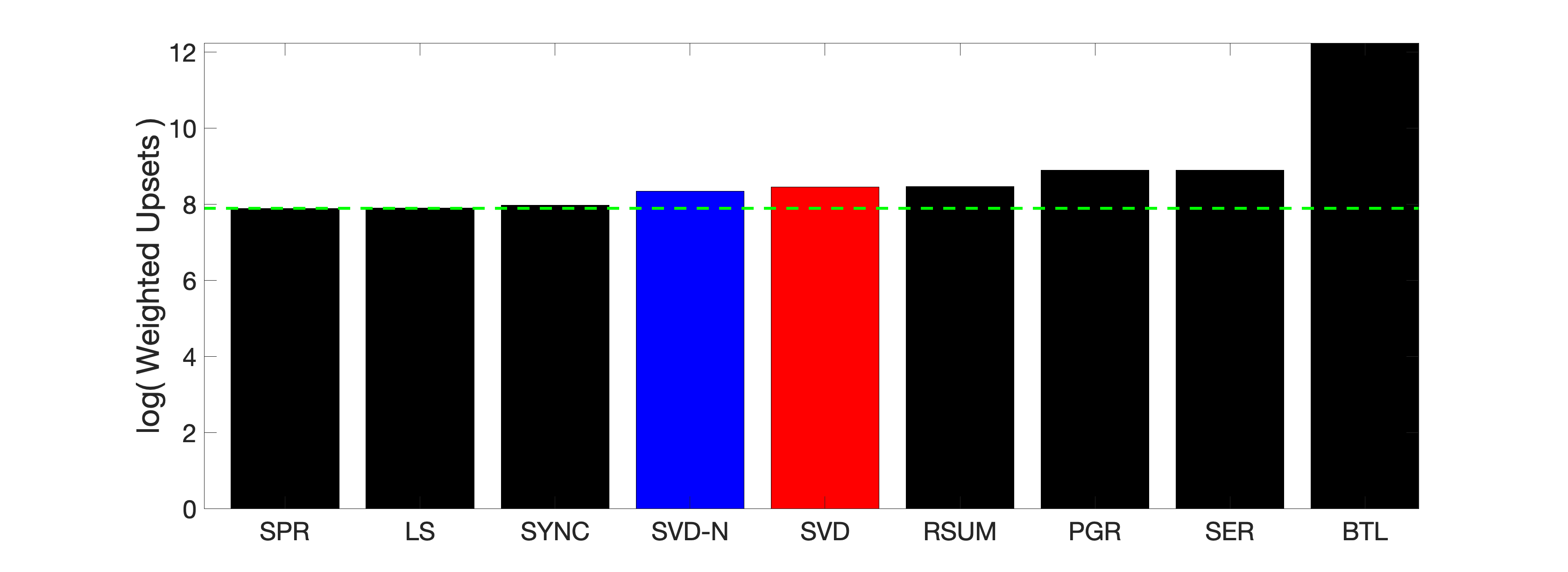}
& \includegraphics[width=0.248\columnwidth]{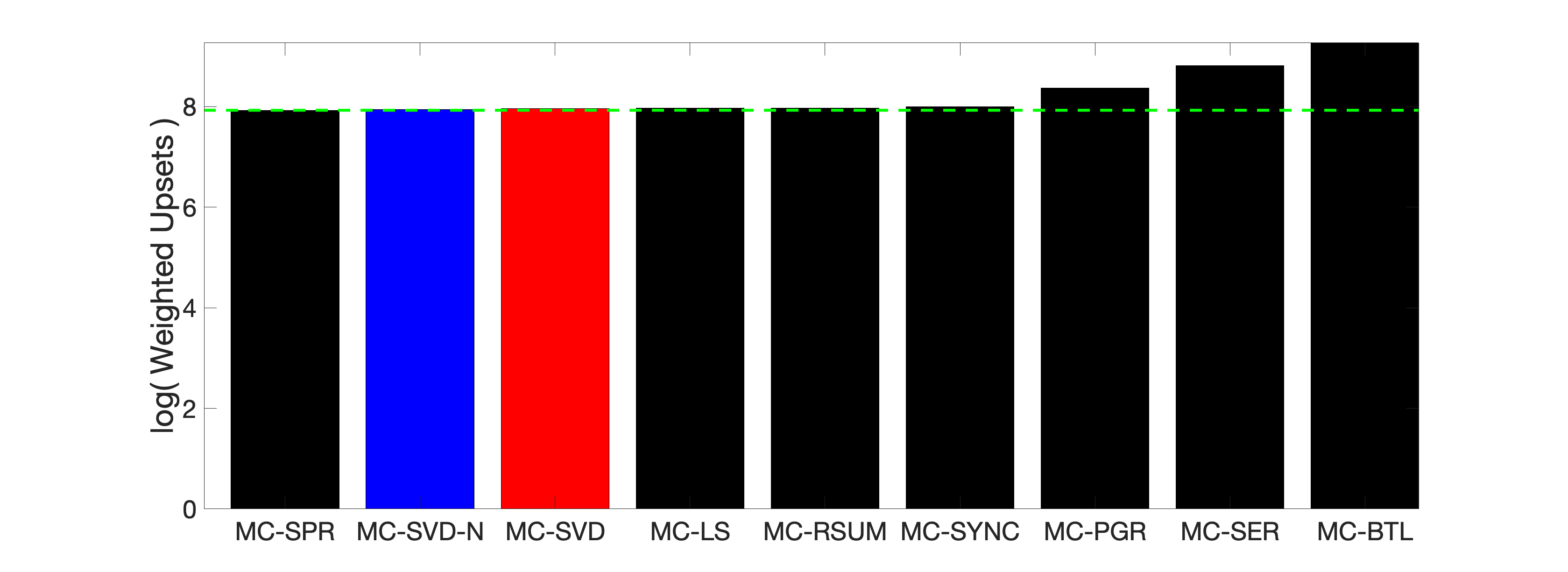}  \\ %
\scriptsize{Business} \hspace{-5mm}
& \includegraphics[width=0.248\columnwidth]{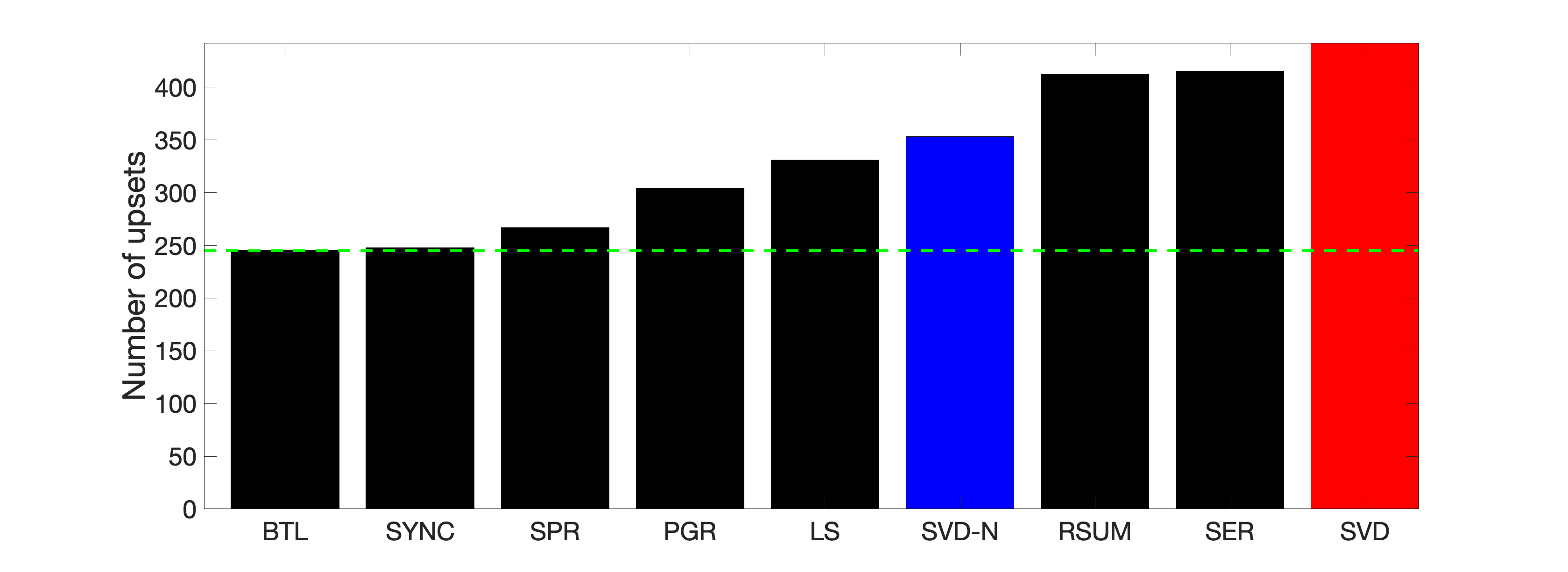}
& \includegraphics[width=0.248\columnwidth]{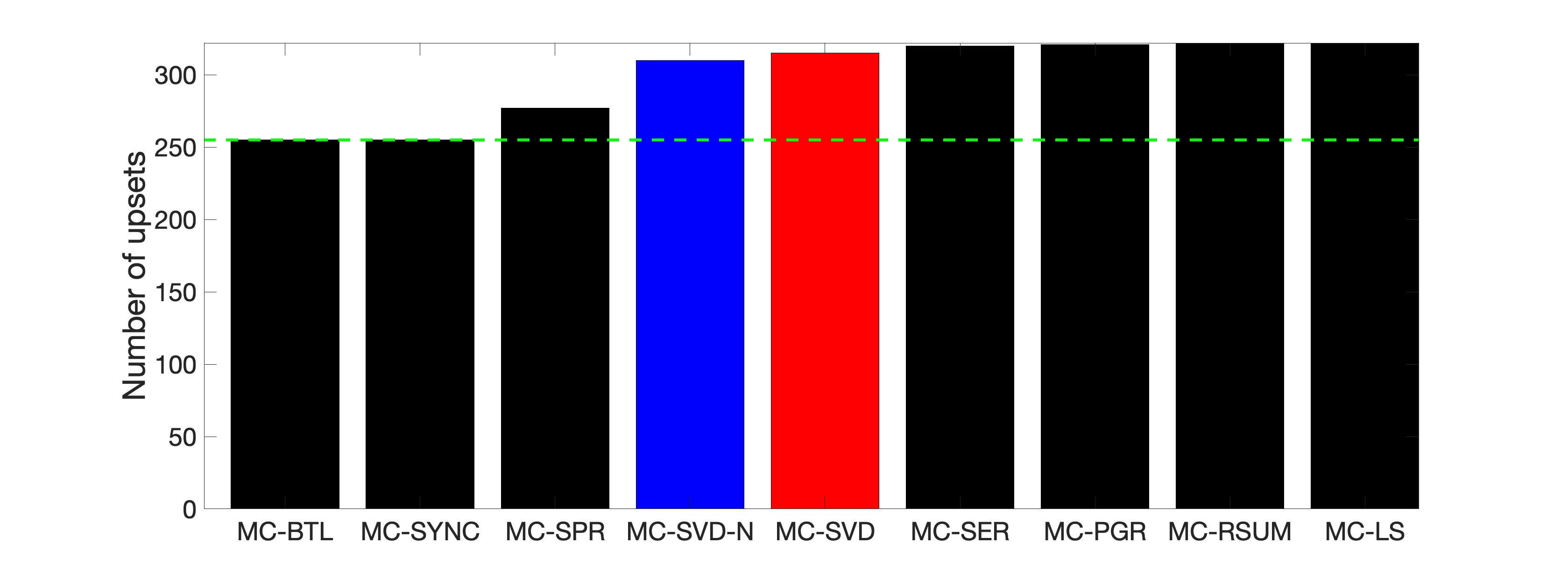}  
& \includegraphics[width=0.248\columnwidth]{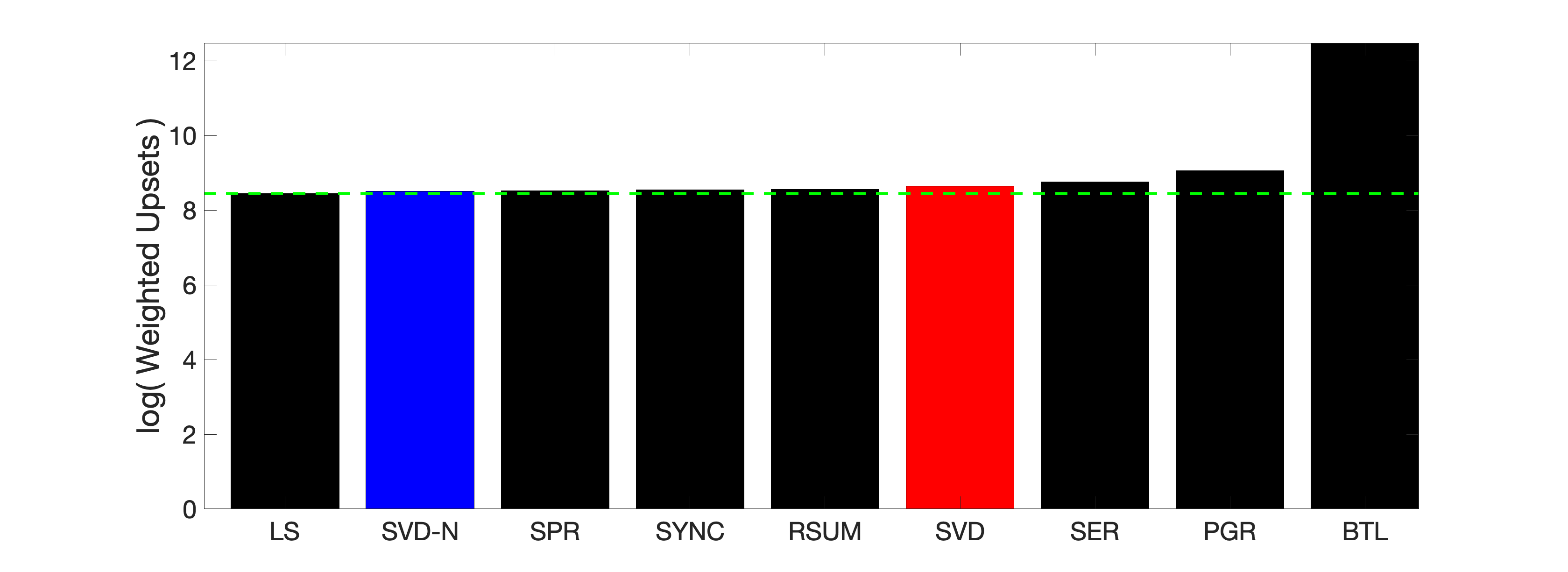}
& \includegraphics[width=0.248\columnwidth]{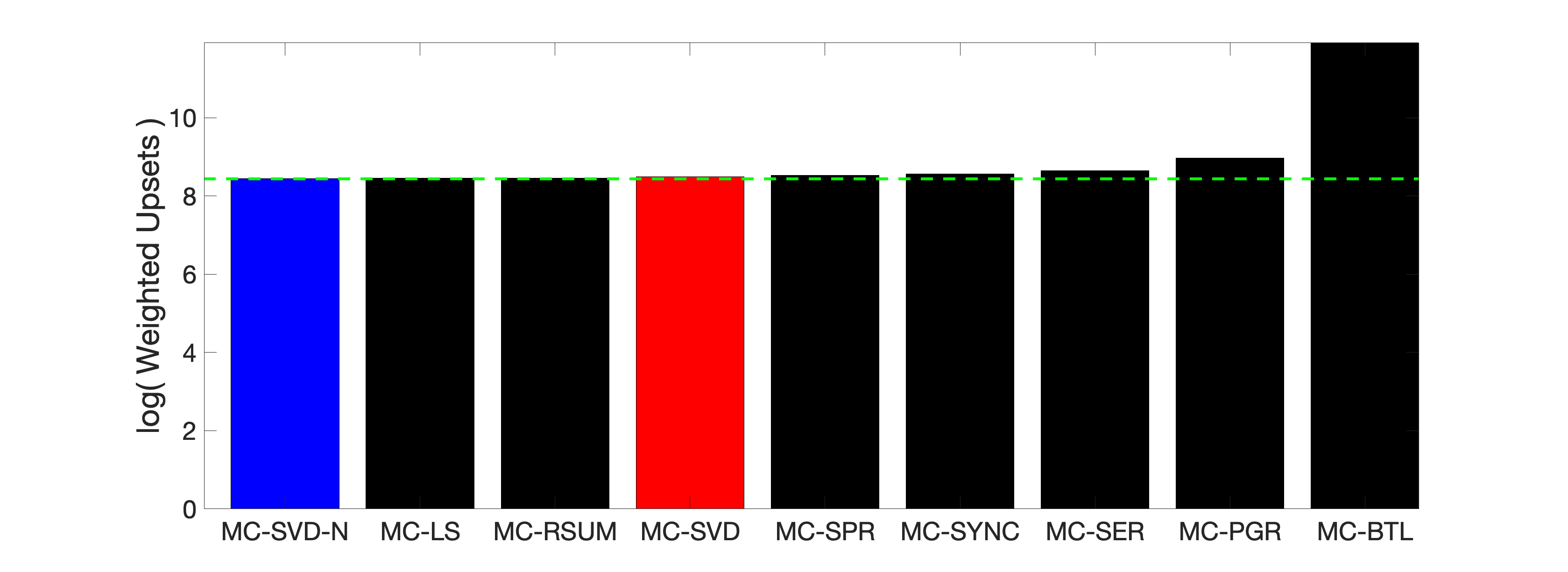}  \\ %
\scriptsize{History} \hspace{-5mm}
& \includegraphics[width=0.248\columnwidth]{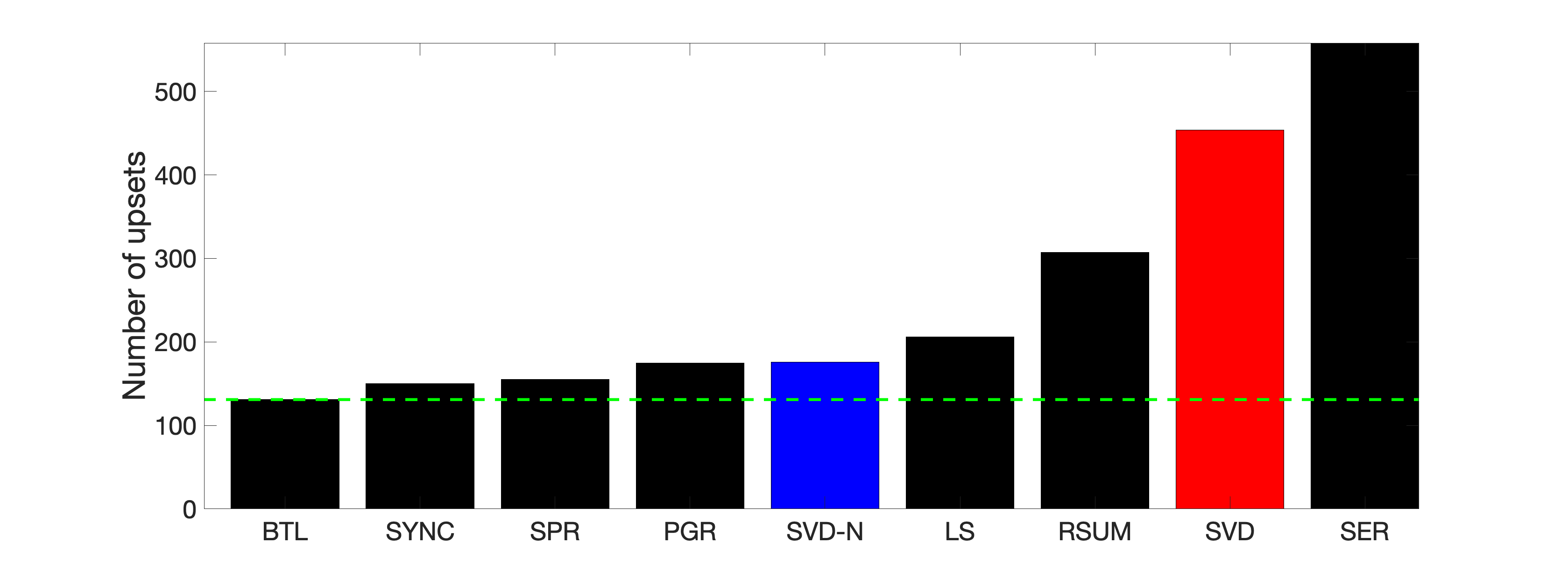}
& \includegraphics[width=0.248\columnwidth]{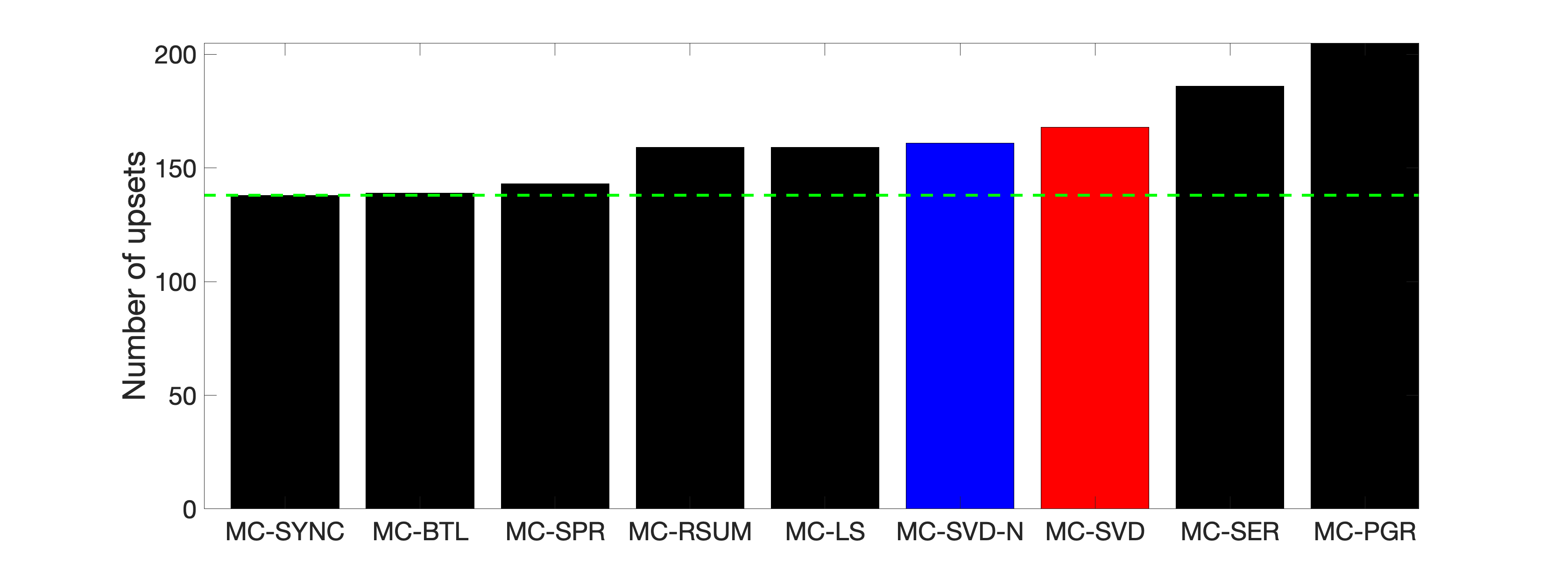}  
& \includegraphics[width=0.248\columnwidth]{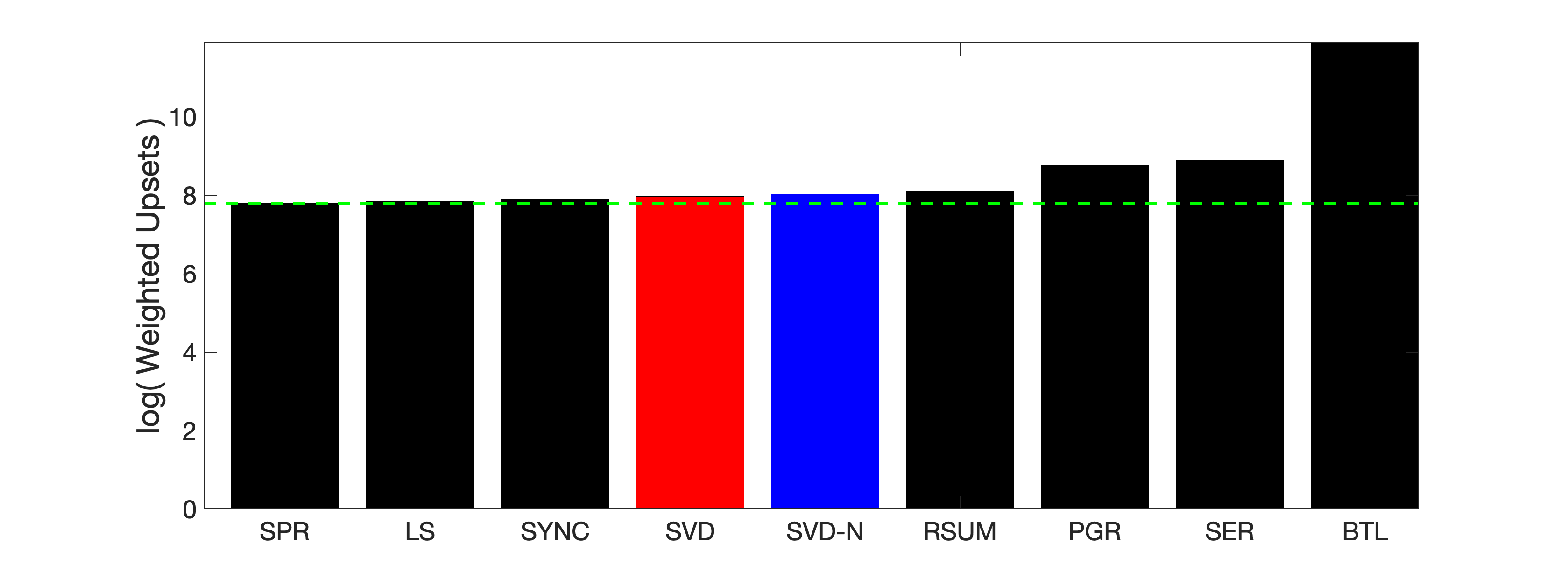}
& \includegraphics[width=0.248\columnwidth]{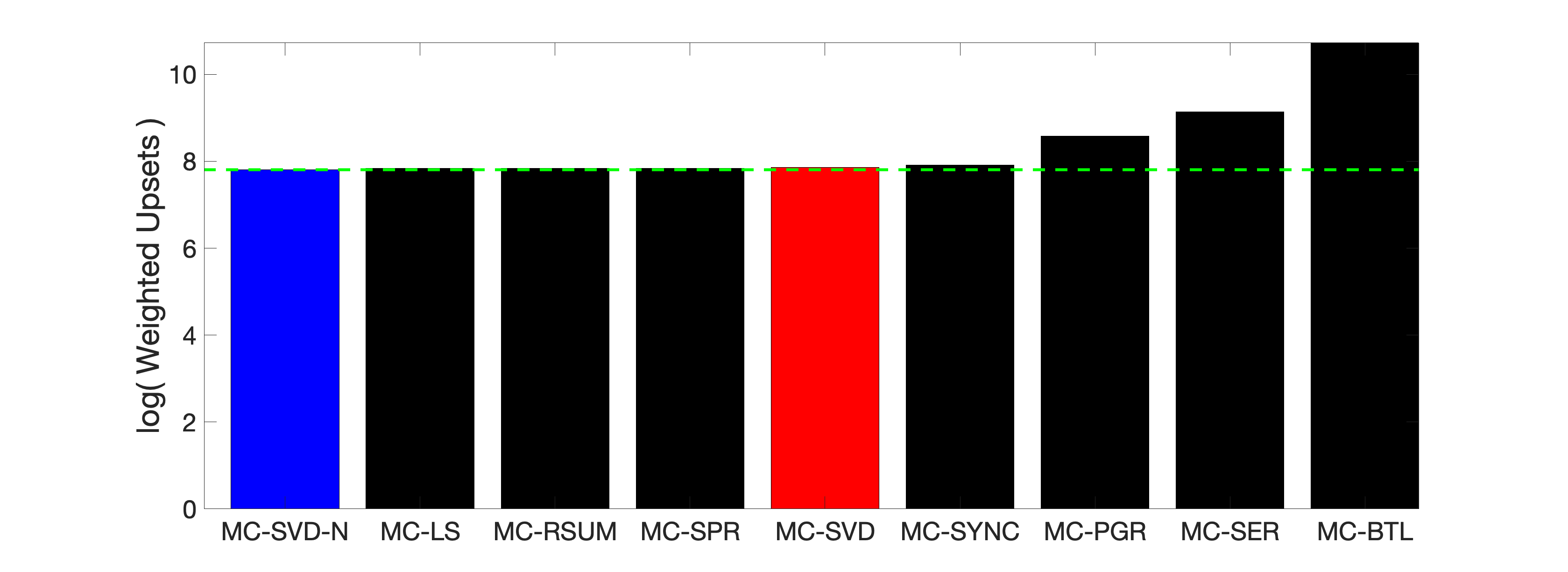}  \\ %
% 
% \rotatebox{90}{\textsc{Upsets}}  
\vspace{-2mm}
\end{tabular}
\captionsetup{width=0.99\linewidth}
\vspace{-2mm}
\captionof{figure}{Performance comparison in the faculty hiring networks, for three networks corresponding to different fields. We compare both in terms of the number of upsets (first two columns) and weighted upsets (last two columns). The results without the matrix completion step pertain to the first and third columns, while columns two and four show results obtained after a low-rank matrix completion preprocessing step.
}
\label{fig:Hiring_ALL_FM}
\end{table*}

\vspace{-1mm}
\paragraph{Microsoft Halo 2 Tournament.} 
Our third experiment was performed on a real data set of game outcomes collected during the Beta testing period for the Xbox game\footnote{Credits for using the Halo 2 Beta data set are given to Microsoft Research Ltd. and Bungie.} Halo 2. 
The graph has a total of $n=606$ players, and $6227$ head-to-head games. After removing the low-degree nodes, i.e., players who have played less than 3 games, we arrive at a  graph with $n = 535$ nodes and 6109 edges, with an average degree of roughly $23$. The skew-symmetric pairwise comparison matrix $H$ captures the net number of wins of player $i$ over player $j$. In this example illustrated in Figure \ref{fig:MSFT_Halo},  \textsc{SVD-N} outperforms  \textsc{SVD} in three out of four instances, and the two methods are always ranked in the interval 5-7 out of the 9 algorithms considered. 
% \text{LS} achieves the smallest number of upsets, followed by \textsc{RS} and \textsc{SVD-N}. After matrix completion, \textsc{RS}, \textsc{LS} and \textsc{SER} are the top three best performers.  

\begin{table*}\sffamily 
\hspace{-9mm} 
\begin{tabular}{l*4{C}@{}}
  & Number of Upsets & Number of Upsets (Matrix Completion) & Weighted Upsets & Weighted Upsets (Matrix Completion)  \\
\hline
% \rotatebox{90}{\textsc{G1-Q3}}  
% \parbox{1.2cm}{ G1-Q3}
% \scriptsize{signsumGoalDif} \hspace{-5mm}
& \includegraphics[width=0.248\columnwidth]{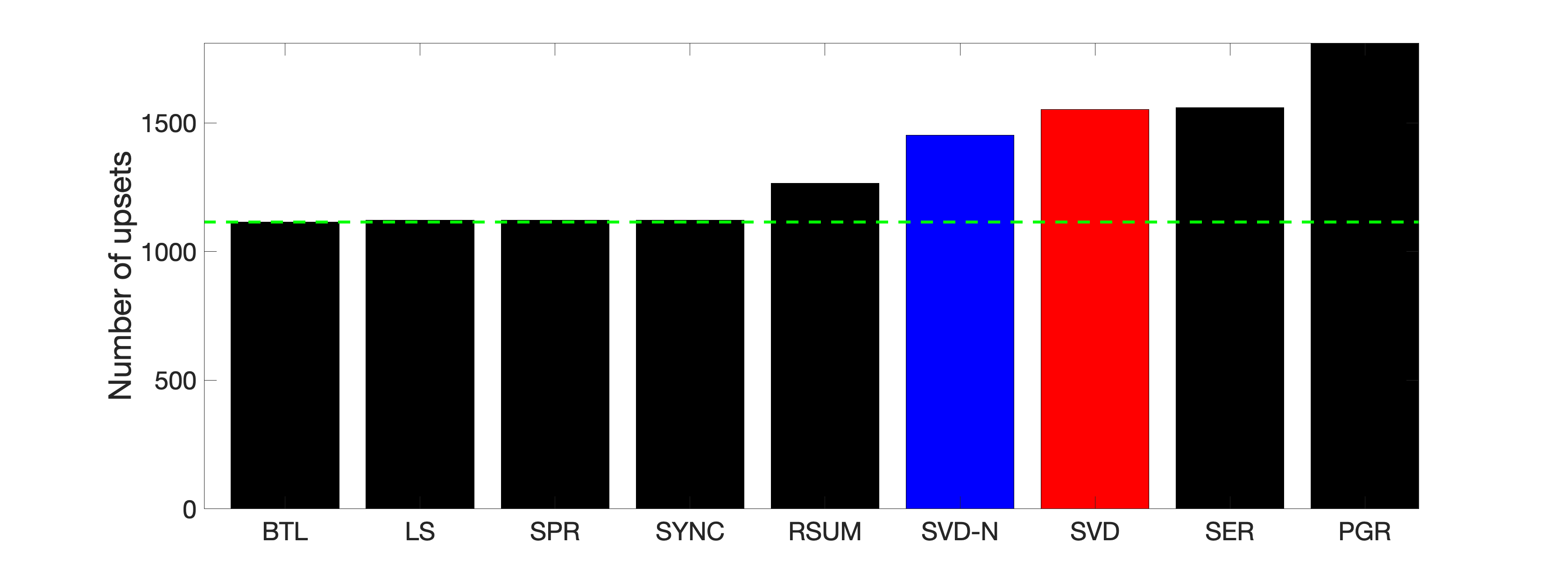}
& \includegraphics[width=0.248\columnwidth]{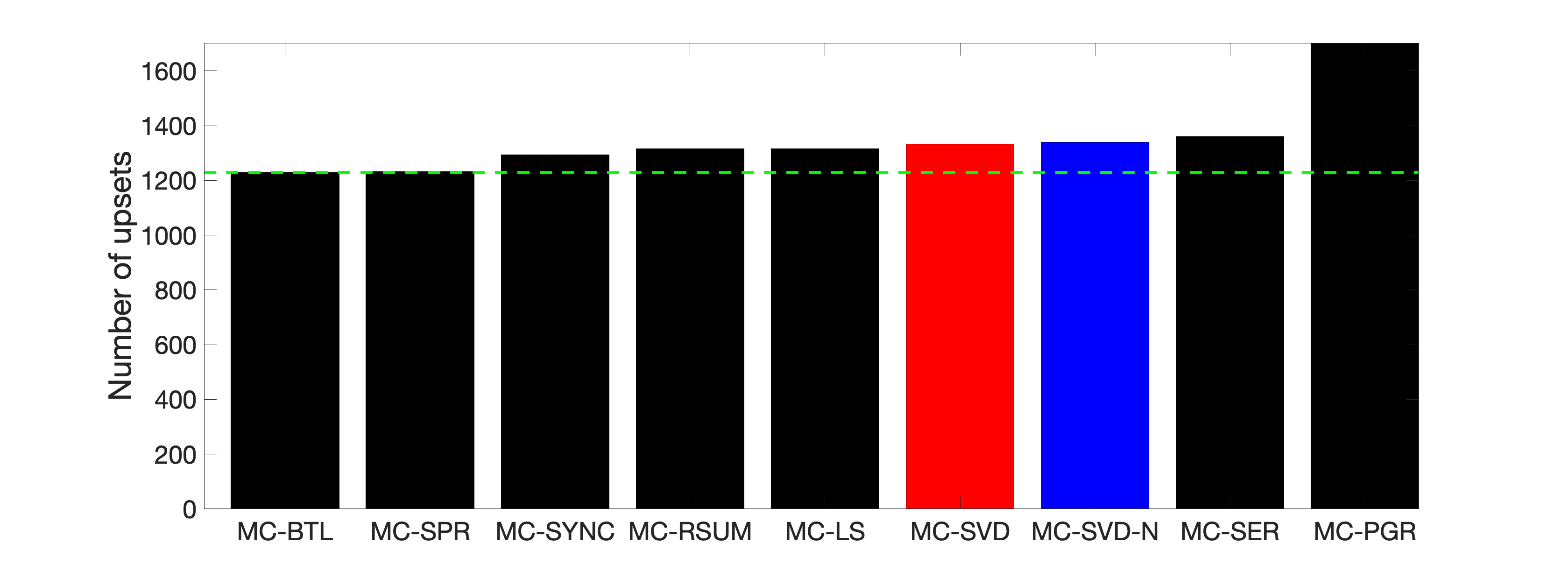}
& \includegraphics[width=0.248\columnwidth]{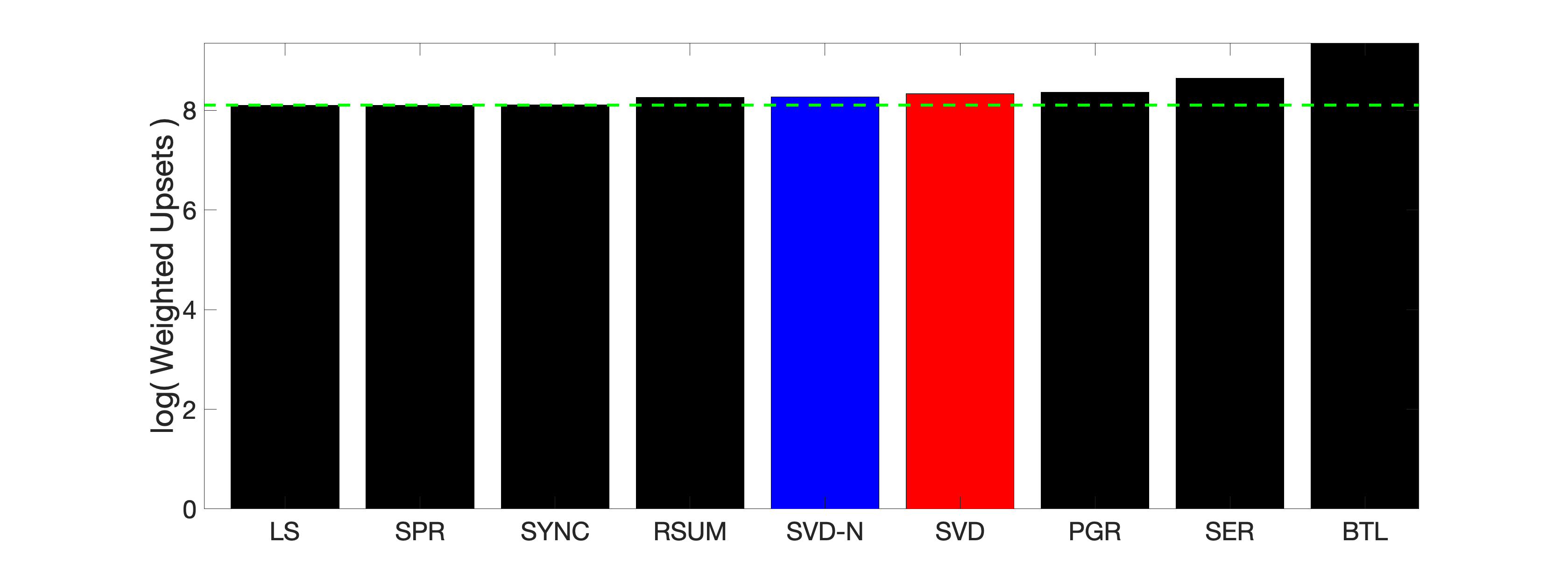}  
& \includegraphics[width=0.248\columnwidth]{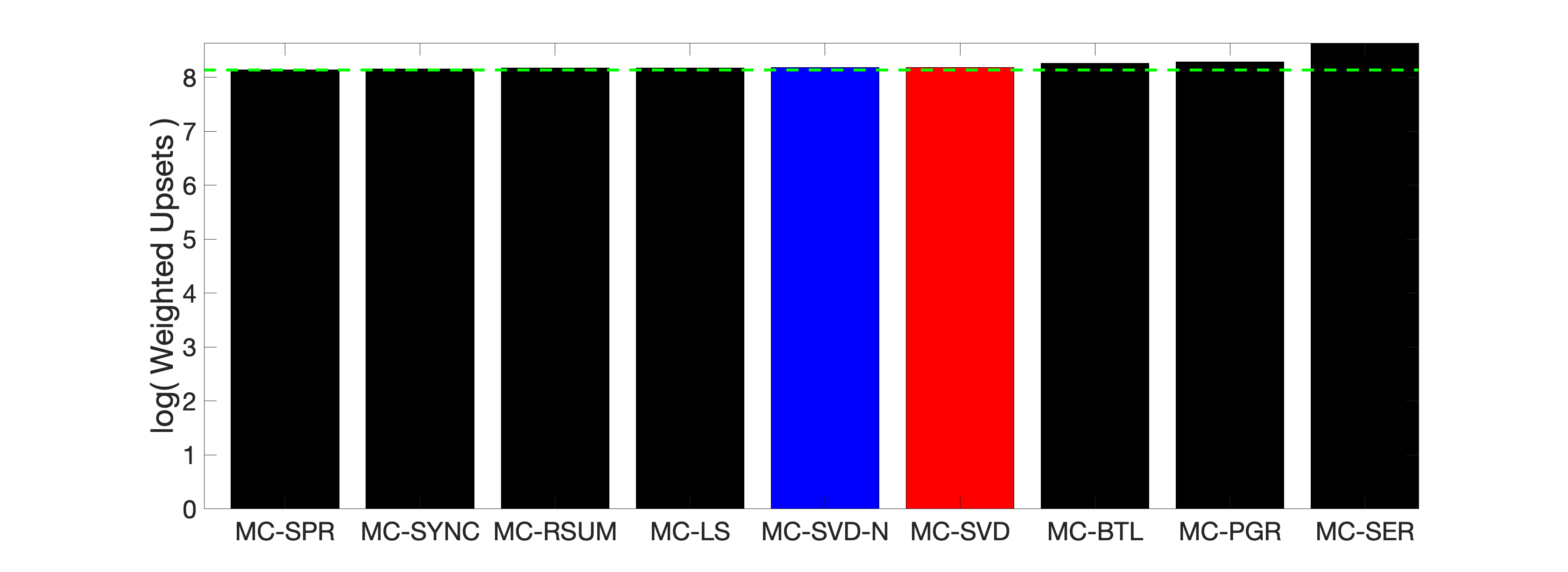}  \\ %
%
% \rotatebox{90}{\textsc{Upsets}}
\end{tabular}
\captionsetup{width=0.99\linewidth}
\vspace{-3mm}
\captionof{figure}{Recovery results for the Microsoft Halo Tournament data set, both in terms of the number of upsets (first two columns) and weighted upsets (last two columns). Columns one and three show results without the matrix completion step, while columns two and four show results obtained after a low-rank matrix completion preprocessing step.
}
\label{fig:MSFT_Halo}
\end{table*}

%\vspace{-3mm}
\paragraph{Premier League.}
Our final real example is the Premier League data set, for which we consider the four seasons during 2009-2013, shown in Figure \ref{fig:PremierLeague}. The skew-symmetric pairwise comparison matrix $H$ captures the net goal difference accrued in the two matches each pair of teams played against each other (home and away games). Across all the experimental setups considered and the different performance metrics, \textsc{SVD-N} outperforms  \textsc{SVD} in 12 out of 16 rankings shown in Figure \ref{fig:PremierLeague}, and the two methods typically score in the top half best performing methods. Furthermore,  \textsc{SVD-N} comes first in three instances.

\begin{table*}\sffamily 
\hspace{-9mm} 
\begin{tabular}{l*4{C}@{}}
Season & Number of Upsets & Number of Upsets (Matrix Completion) & Weighted Upsets & Weighted Upsets (Matrix Completion)  \\
 \hline
% \rotatebox{90}{\textsc{G1-Q3}}  
% \parbox{1.2cm}{ G1-Q3}
\scriptsize{2009-2010} \hspace{-5mm}
& \includegraphics[width=0.248\columnwidth]{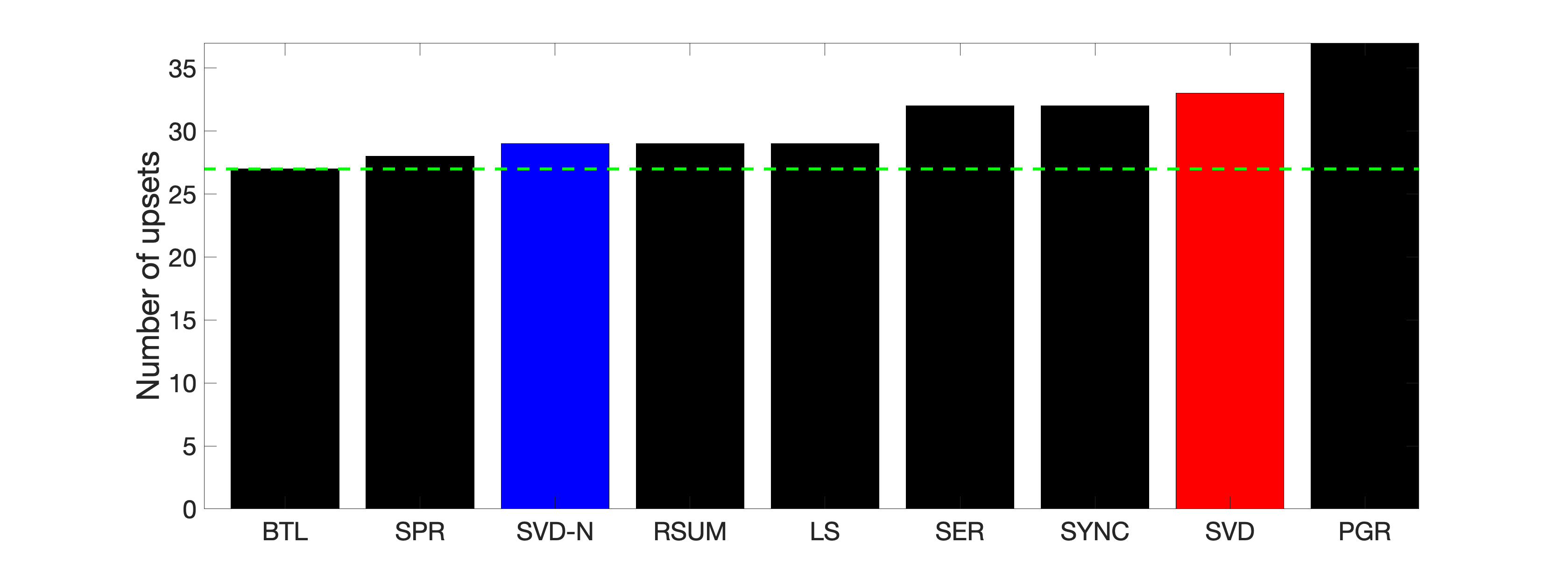}
& \includegraphics[width=0.248\columnwidth]{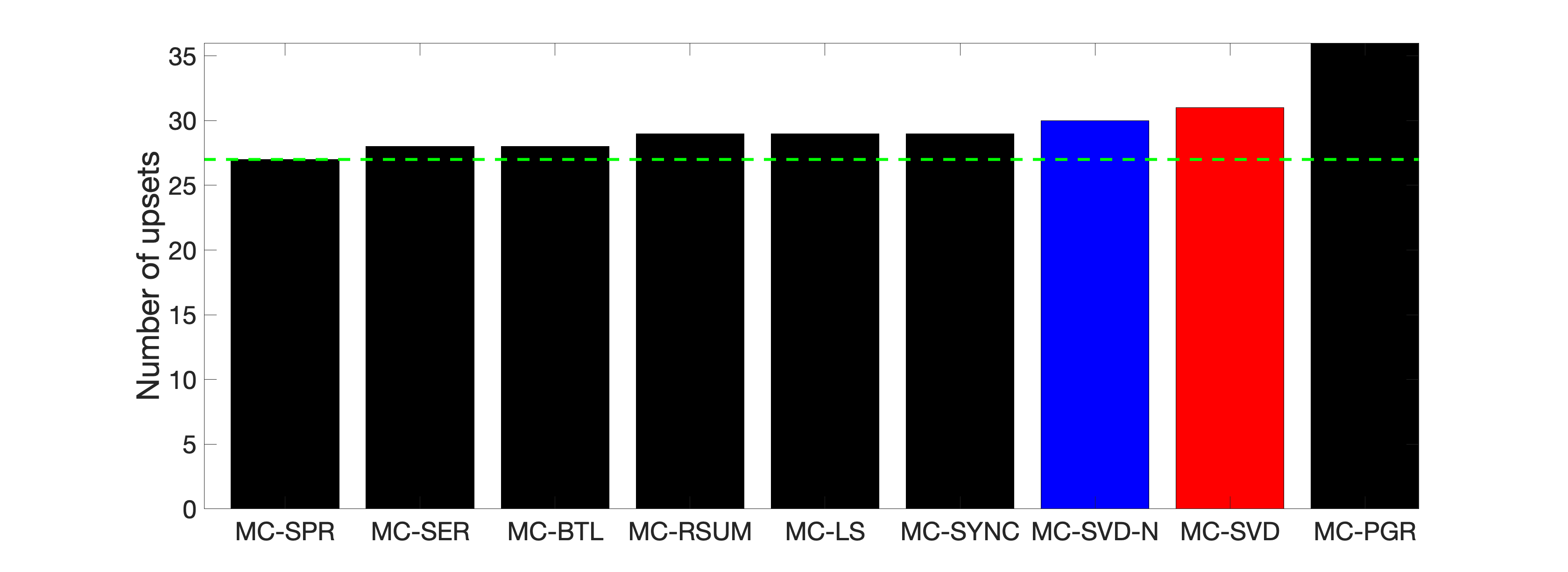}  
& \includegraphics[width=0.248\columnwidth]{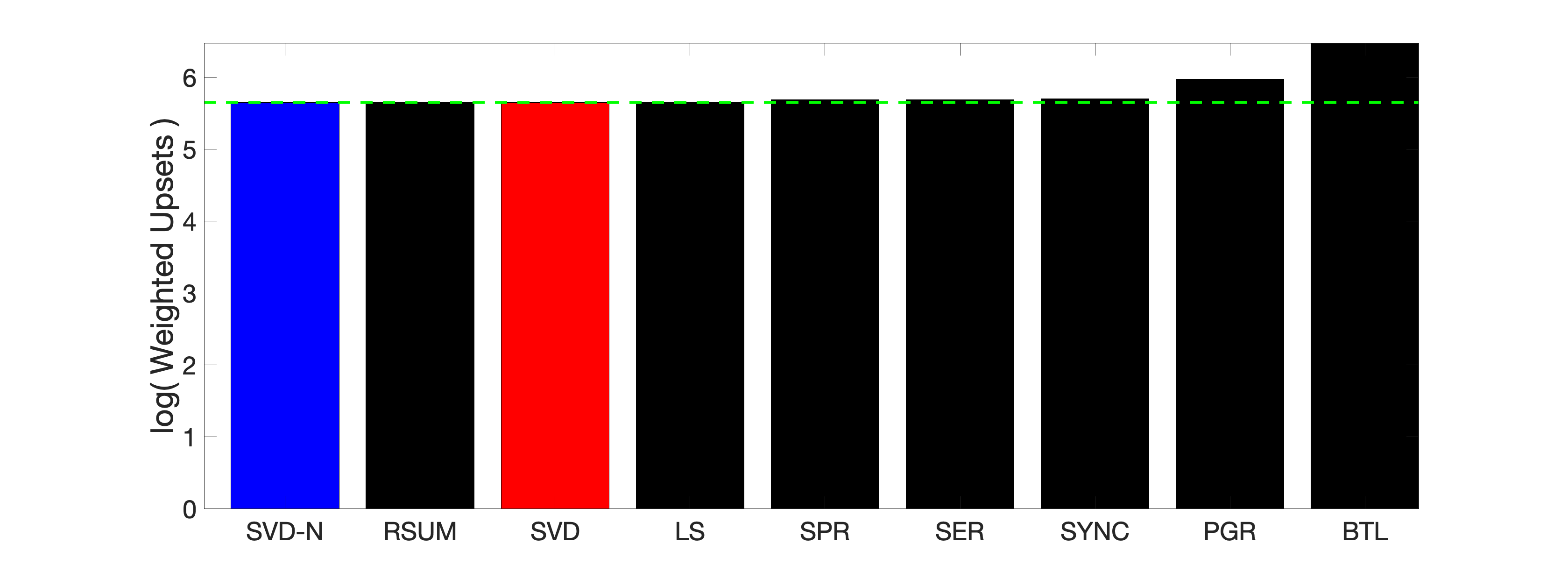}
& \includegraphics[width=0.248\columnwidth]{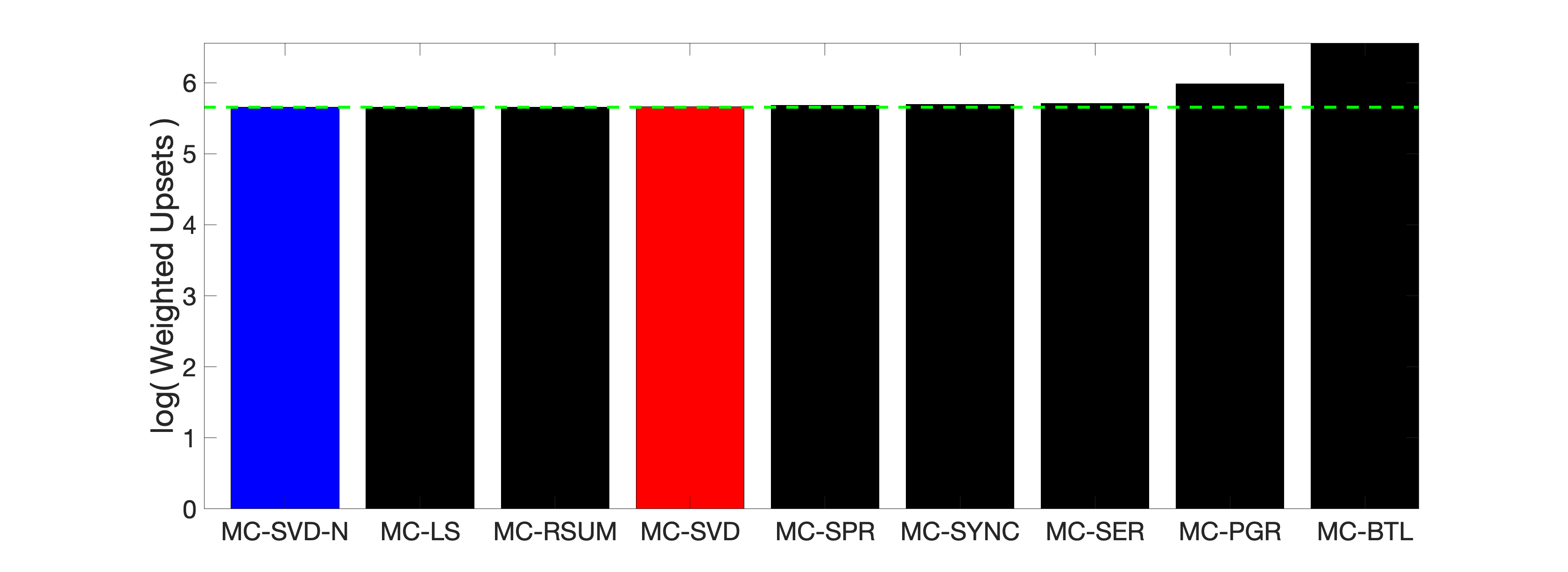}  \\ 
\scriptsize{2010-2011} \hspace{-5mm}
& \includegraphics[width=0.248\columnwidth]{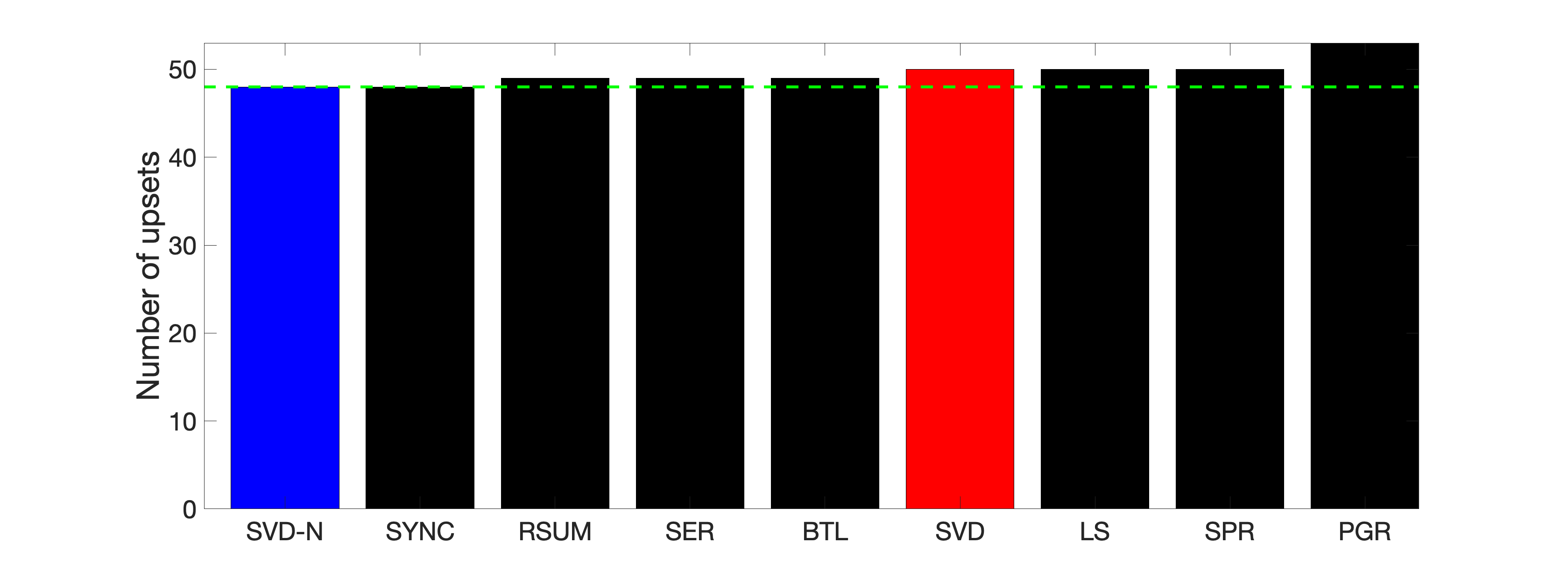}
& \includegraphics[width=0.248\columnwidth]{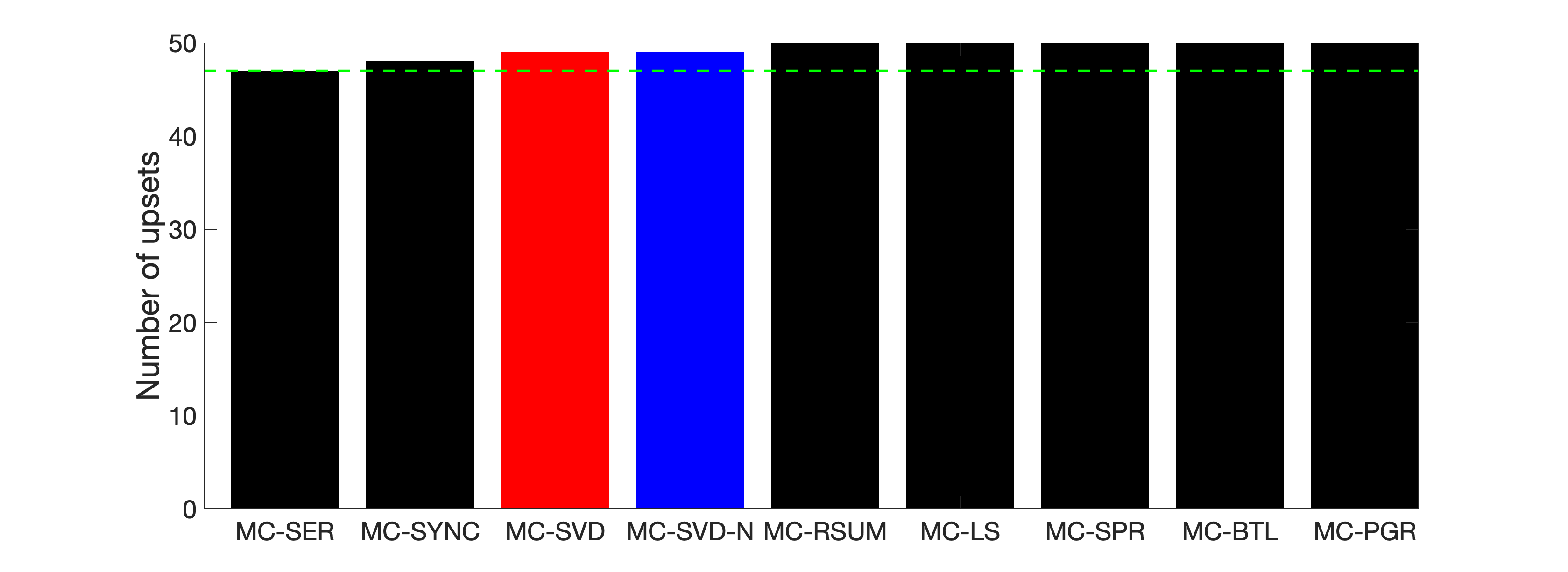}  
& \includegraphics[width=0.248\columnwidth]{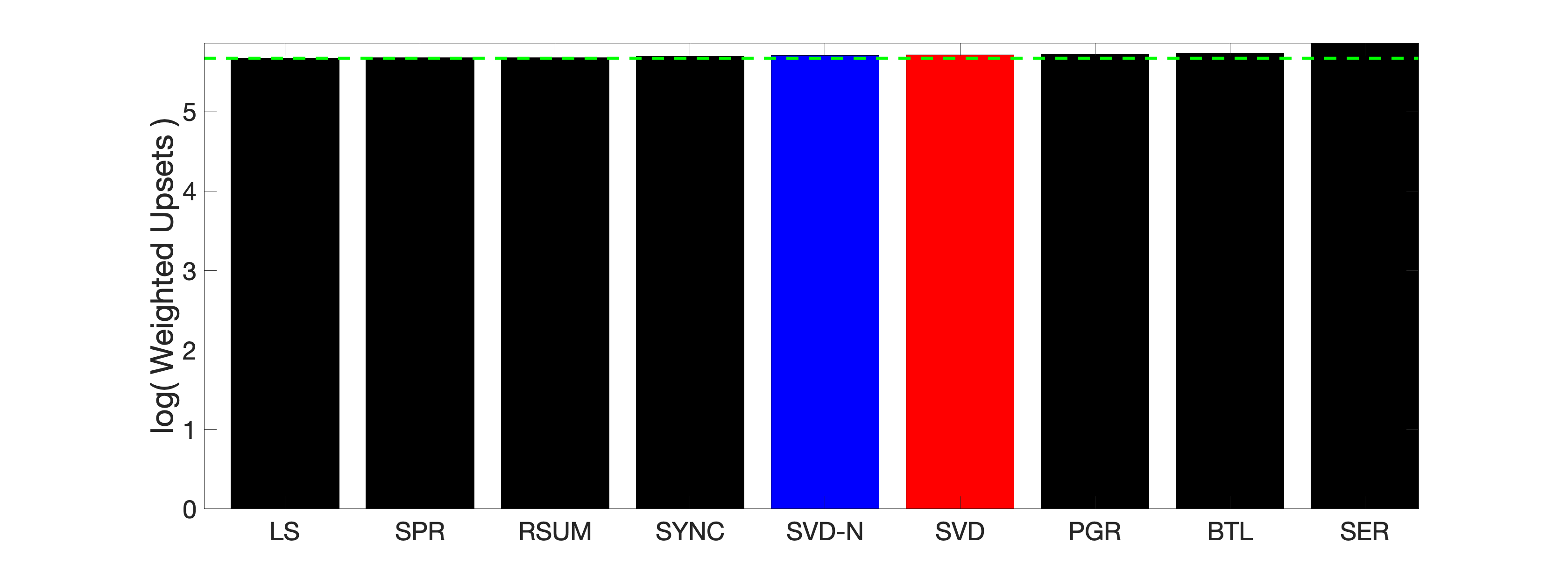}
& \includegraphics[width=0.248\columnwidth]{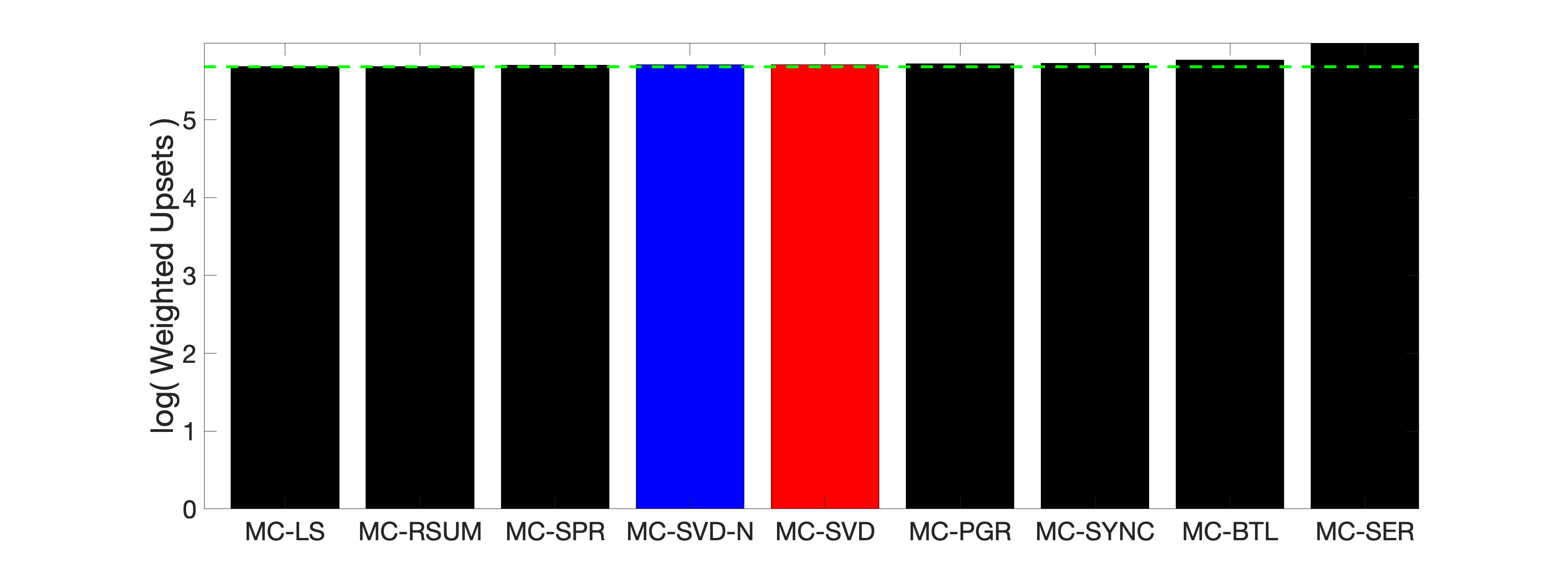}  \\ 
\scriptsize{2011-2012} \hspace{-5mm}
& \includegraphics[width=0.248\columnwidth]{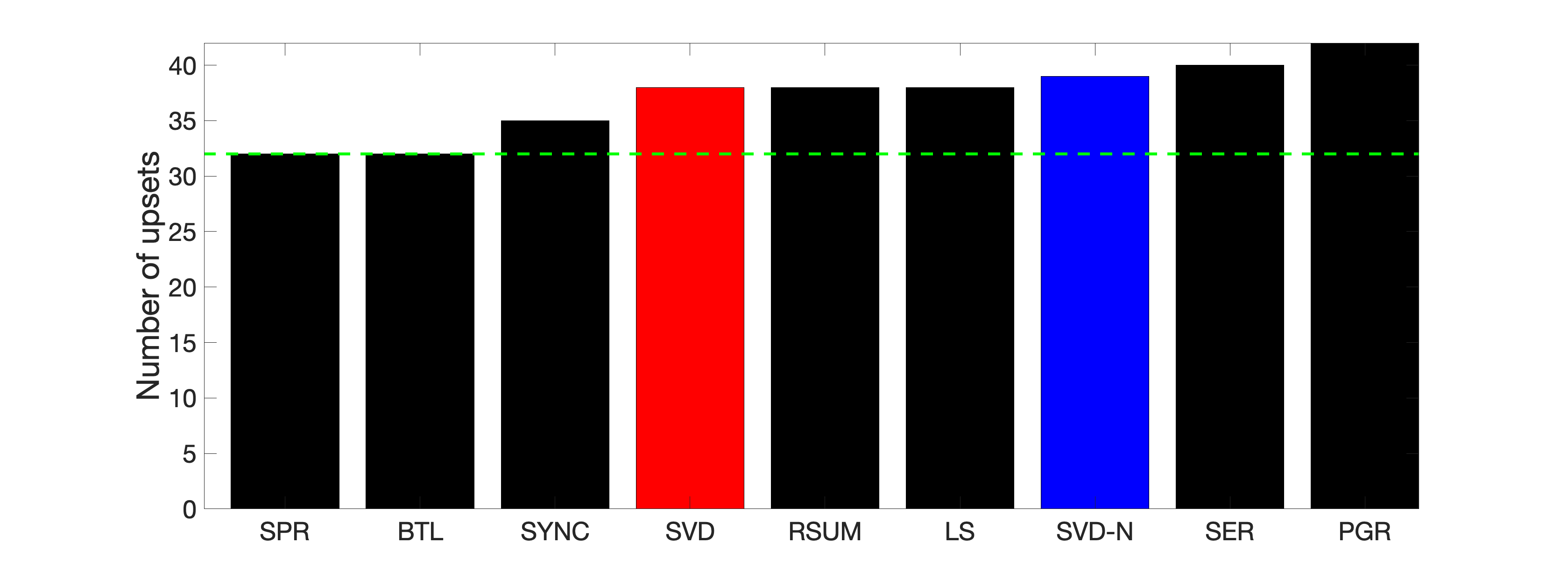}
& \includegraphics[width=0.248\columnwidth]{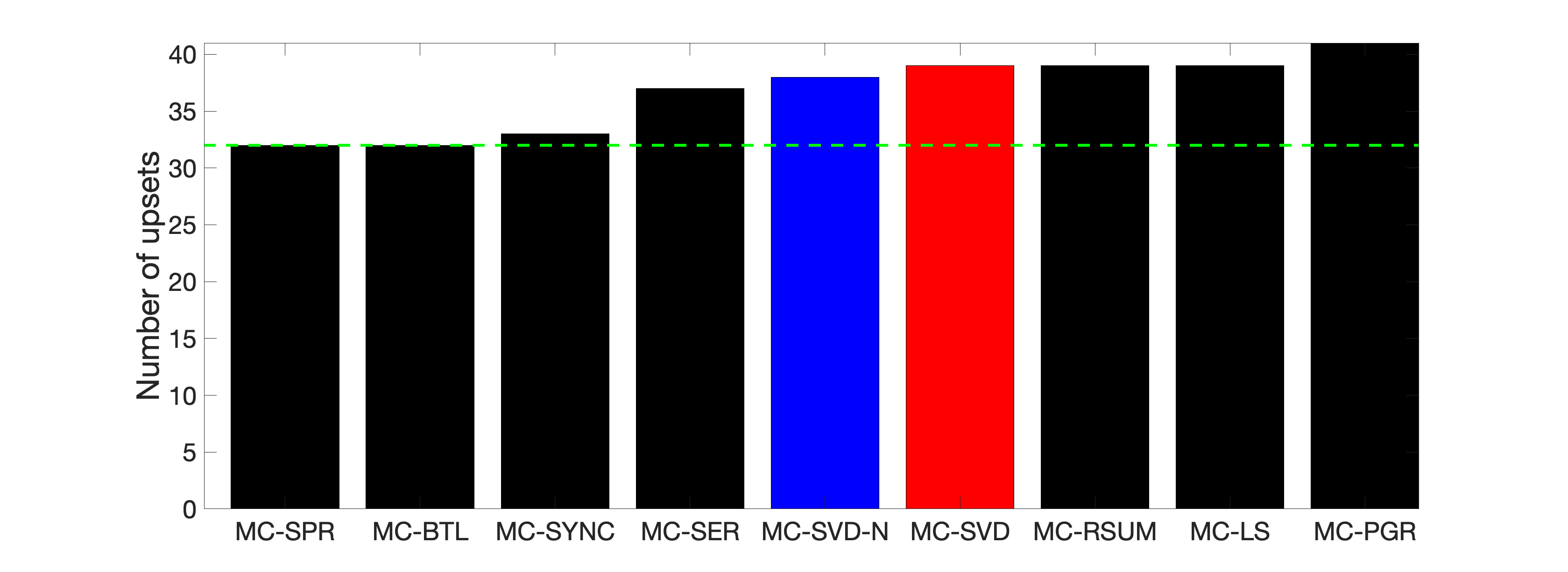}  
& \includegraphics[width=0.248\columnwidth]{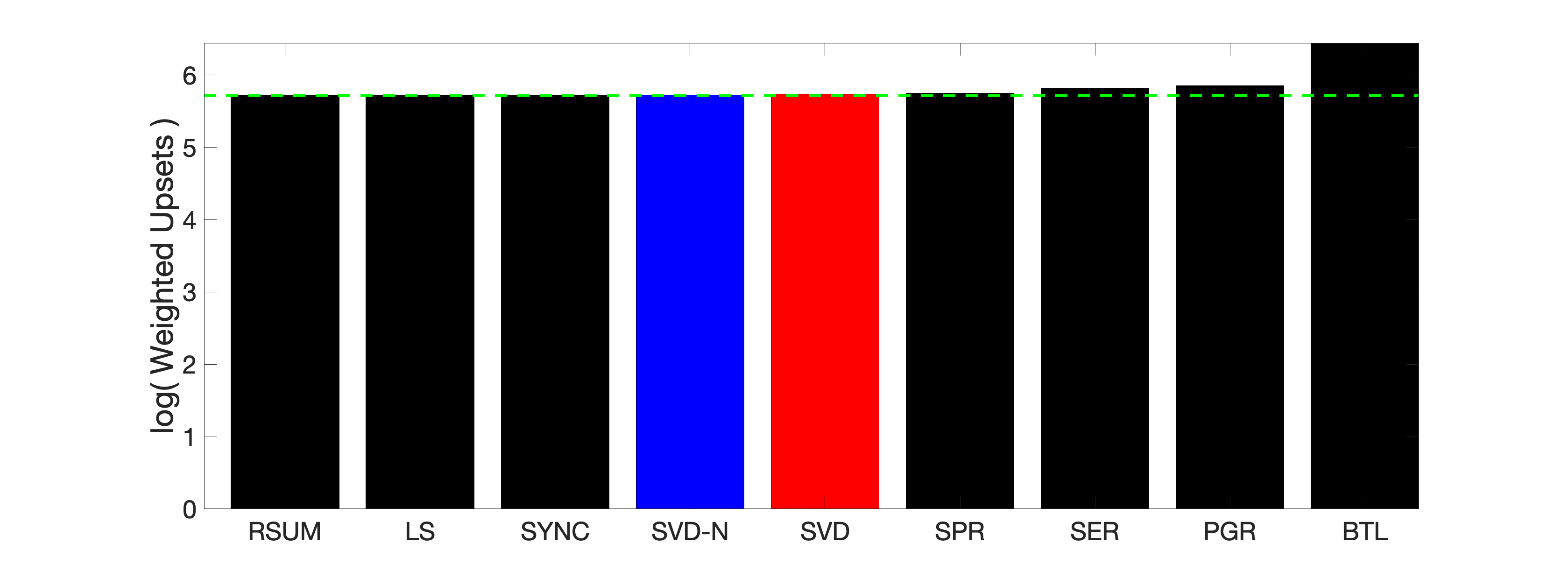}
& \includegraphics[width=0.248\columnwidth]{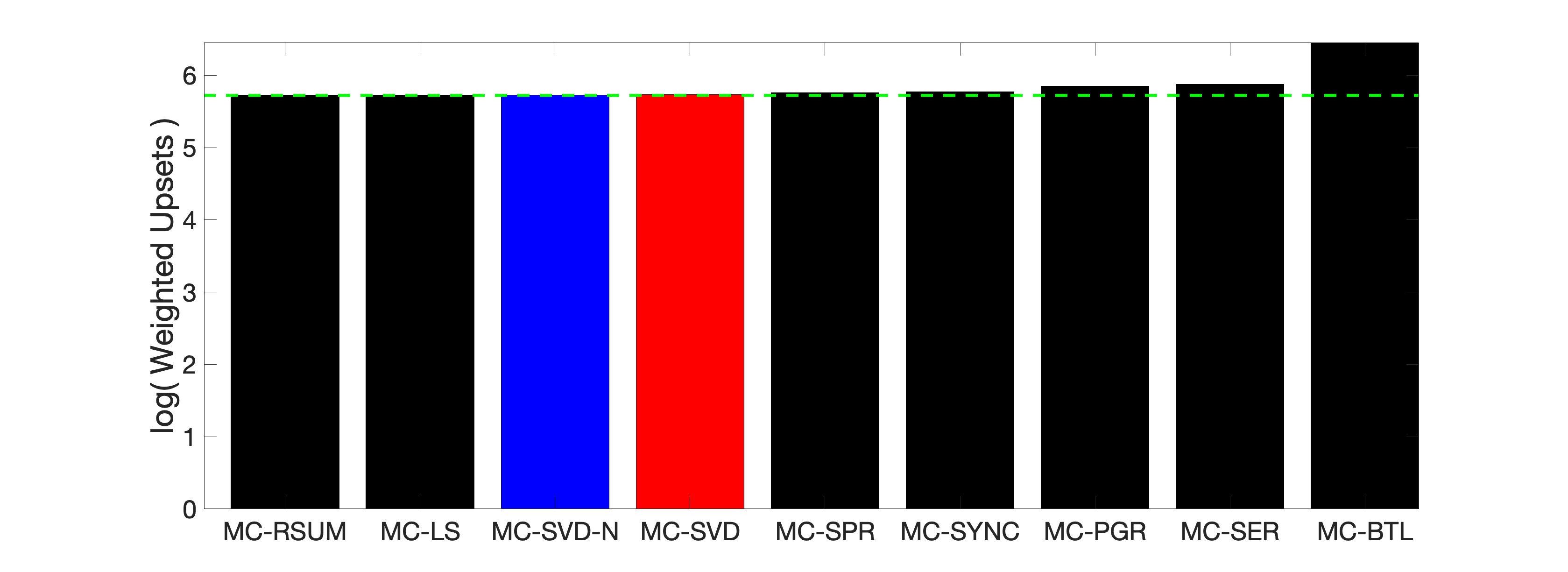}  \\ 
\scriptsize{2012-2013} \hspace{-5mm}
& \includegraphics[width=0.248\columnwidth]{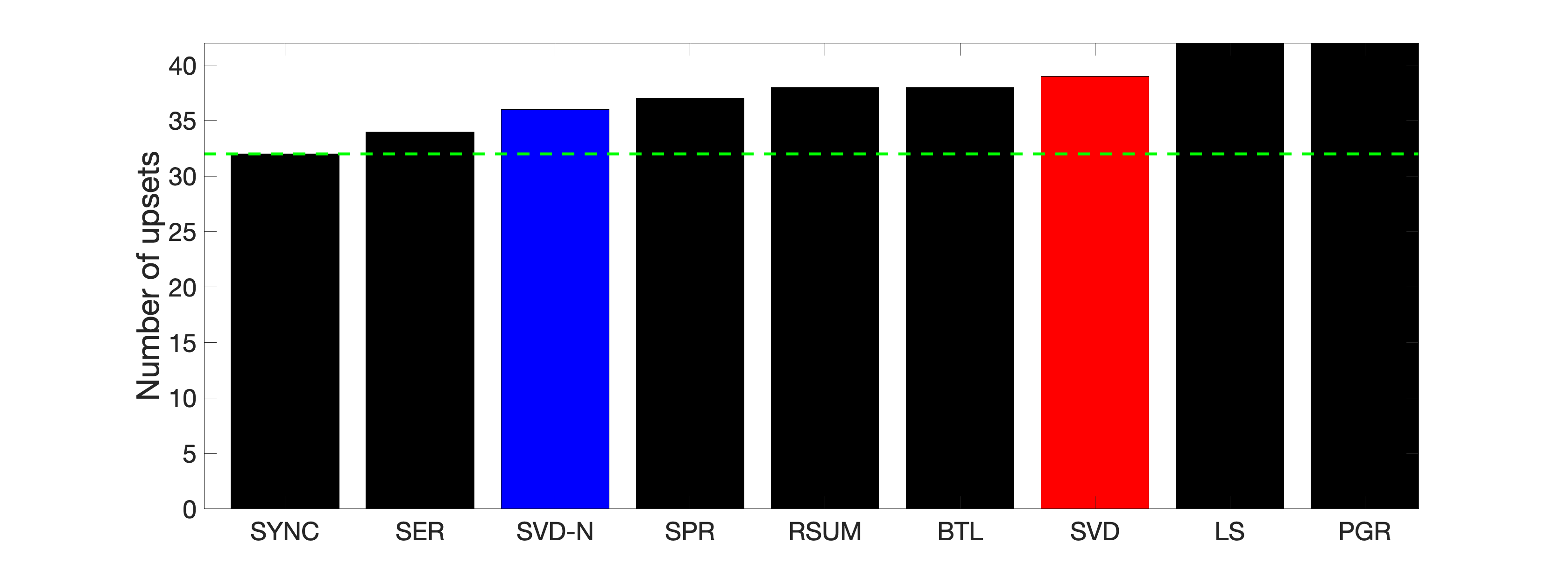}
& \includegraphics[width=0.248\columnwidth]{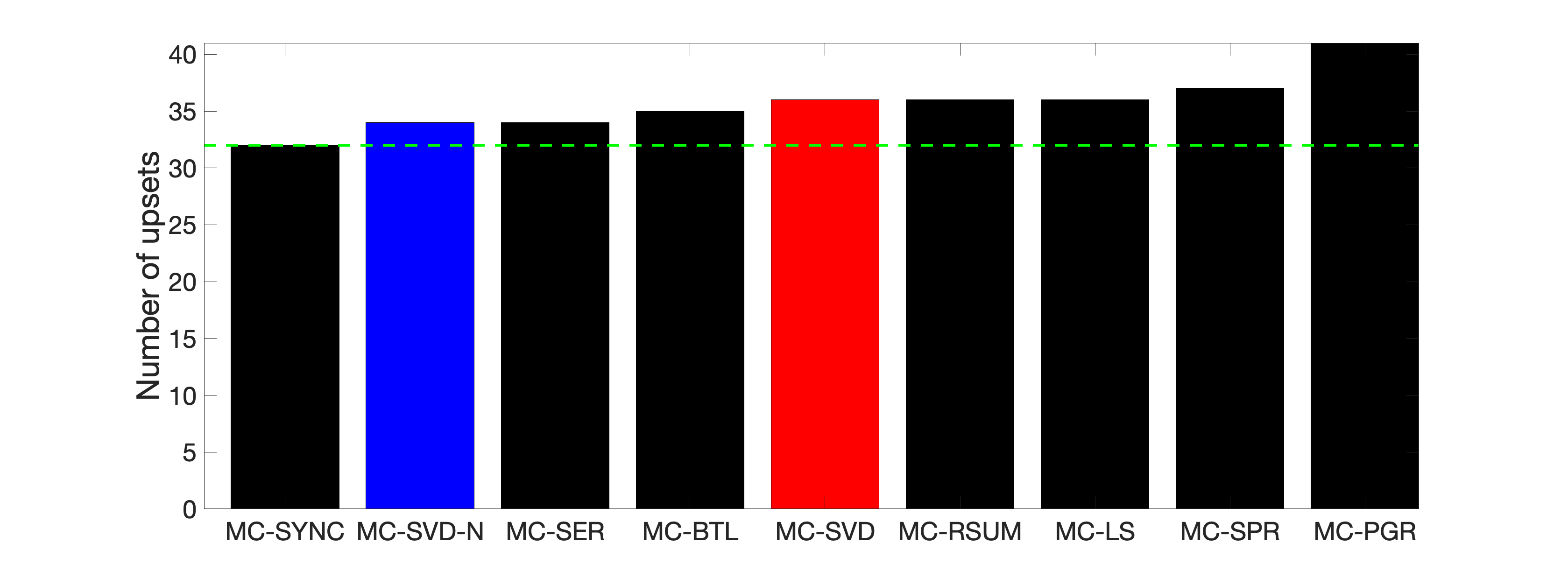}  
& \includegraphics[width=0.248\columnwidth]{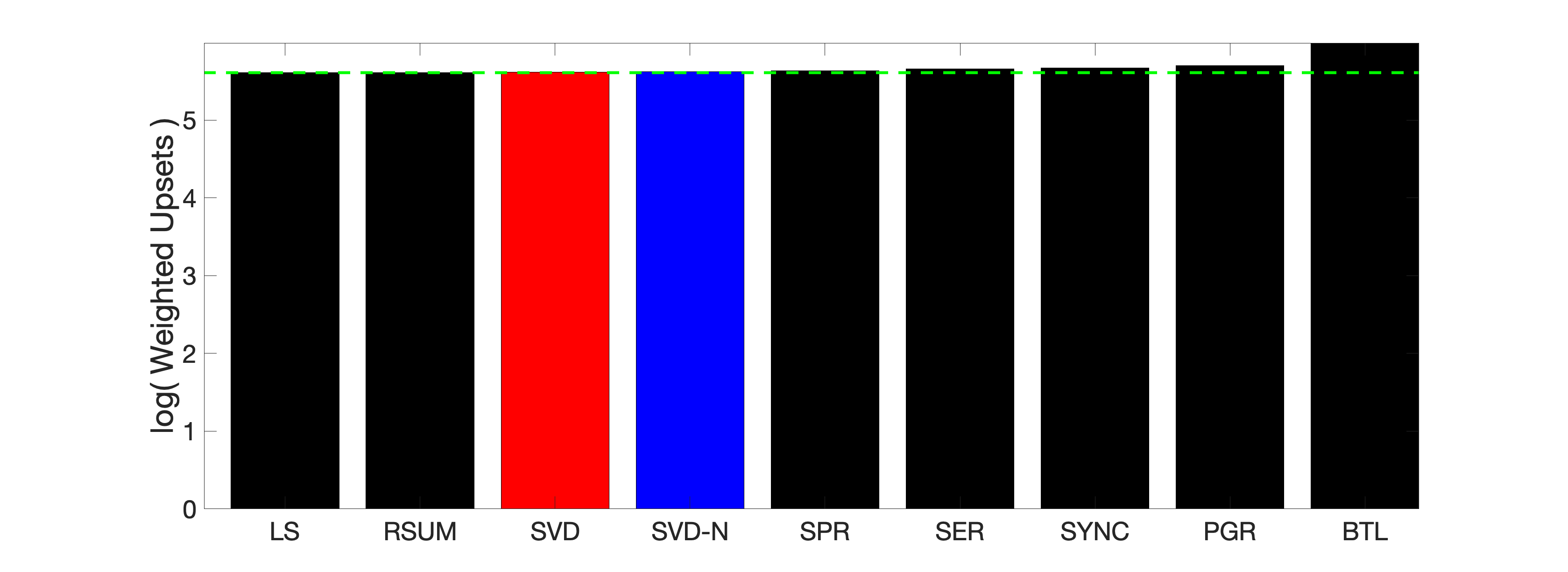}
& \includegraphics[width=0.248\columnwidth]{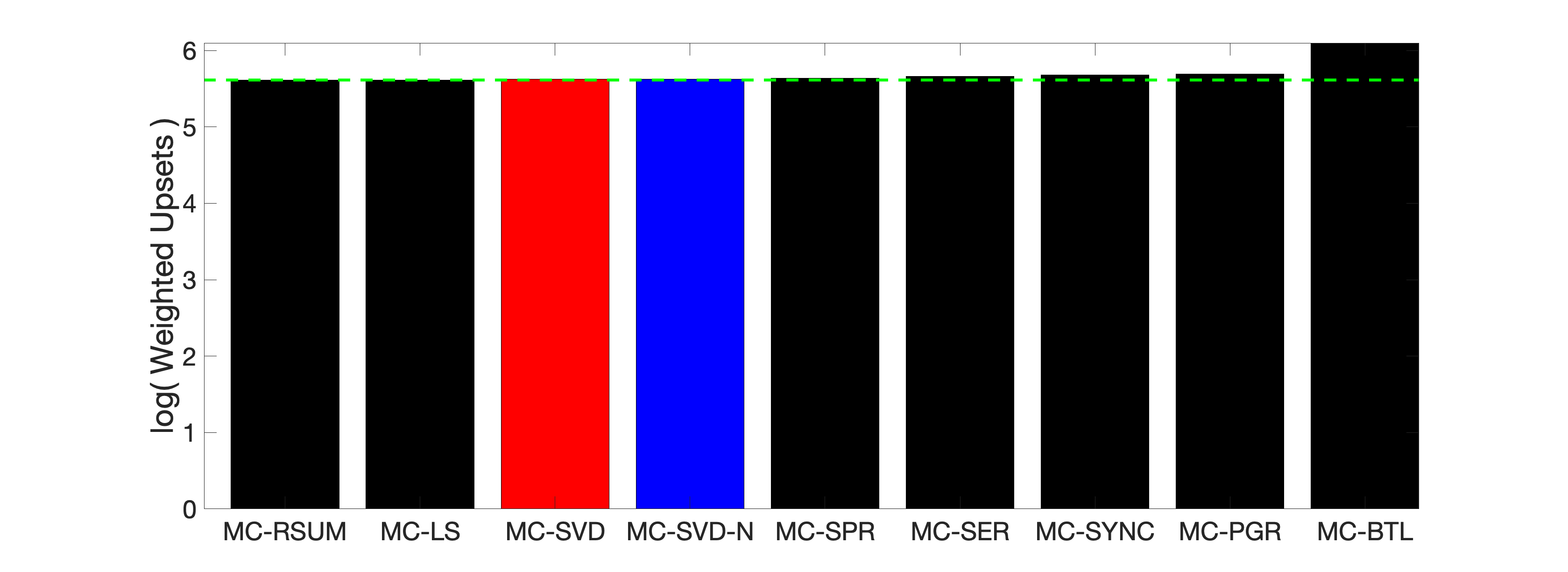}  \\ 
\end{tabular}
\captionsetup{width=0.99\linewidth}
\vspace{-4mm}
\captionof{figure}{Performance comparison for the Premier League data set, for four seasons during 2009-2013. We compare both in terms of the number of upsets (first two columns) and weighted upsets (last two columns). Recovery performance after the matrix completion step is shown in columns two and four, while columns one and three do not rely on matrix completion.
}
\label{fig:PremierLeague}
\vspace{-3mm}
\end{table*}

% Concluding remarks %% Update this section !!

\FloatBarrier

%\vspace{-2mm}
% \section{Concluding remarks and future directions}
\section{Conclusion and future directions}
% \section{Summary and future directions}
\label{sec:conclusion}
This paper considered the problems of ranking and time synchronization given a subset of noisy pairwise comparisons, and proposed an  \textsc{SVD}-based algorithmic pipeline to solve both tasks.
 
We analyzed the robustness of \textsc{SVD-RS} in the form of $\ell_2$  and $\ell_{\infty}$ recovery guarantees for the score vector $r$, against sampling sparsity and noise perturbation, using tools from matrix perturbation and random matrix theory. The $\ell_{\infty}$ analysis of \textsc{SVD-RS} leads to guarantees for rank recovery in terms of the maximum displacement error with respect to the  ground truth.
We also introduced \textsc{SVD-NRS}, a normalized version of \textsc{SVD-RS}, and provided $\ell_2$ recovery guarantees for the score vector $r$.

We have augmented our theoretical analysis with a comprehensive set of numerical experiments on both synthetic and real data (five different real data sets, which altogether contained 42 distinct problem instances/comparison graphs), showcasing the competitiveness of our approach when compared to other seven algorithms from the literature. In particular, \textsc{SVD-NRS} was shown to perform particularly well on many of the real data sets, often ranking in the top three algorithms in terms of performance, out of the nine algorithms considered. 
%, and which can also be analyzed theoretically using the same approach and tools mentioned above. 
% 
% Finally, we compared our algorithms experimentally on a variety of , with two other state-of-the art methods, and demonstrated the competitiveness of our algorithms.  

%Our work further opens several other 
There are several avenues for future work. 
% 
%Perhaps the most important one is 
\iffalse
An important direction would be to establish consistency in the $\ell_{\infty}$ norm 
%eigenvector perturbation error bounds (beyond those trivially implied by our existing $l_2$ bounds) 
using recent tools developed in   \cite{fan2016ell_infinity,abbe2017entrywiseFan},  
which would further induce robustness guarantees for recovering the underlying ground truth rankings. 
\fi 
%
% HEMANT: Maybe this part can be left out due to space reasons?
An interesting direction pertains to extending our % current 
analysis to the setting of very sparse graphs, with $p$ on the order of $\frac{1}{n}$, by leveraging recent regularization techniques \cite{joseph2013impactBinYu, le2015sparse_Vershynin}.
% used in the regime of bounded expected degrees, for which the highly irregular distribution of node degrees render both the graph adjacency matrix and its Laplacian not to concentrate around their expectations. 
%
% It would also be interesting to theoretically analyze  \textsc{SVD-Normalized-Rank} given that it performs very well on real data. 
Analysis of other, perhaps more realistic, noise models would also be of interest (such as the the multiplicative noise models), as well as obtaining theoretical guarantees for the pipeline with low-rank matrix completion as a pre-processing step.

\iffalse
open: $l_inf$ for SVD-NRS; the analysis of pipeline with low-rank matrix completion as a pre-processing step, and aim for theoretical guarantees. 
- sparse case
- ranking with covariate 
\fi 

\section*{Acknowledgements}
A.A. is at the d\'epartement d’informatique de l’ENS, Ecole normale sup\'erieure, UMR CNRS 8548, PSL Research University, 75005 Paris, France, and INRIA, and would like to acknowledge support from the ML and Optimisation joint research initiative with the fonds AXA pour la recherche and Kamet Ventures as well as a Google focused award. A part of this work was done while H.T. was affiliated to The Alan Turing Institute, London, and the School of Mathematics, University of Edinburgh. M.C. and H.T. also acknowledge support from the  EPSRC grant EP/N510129/1 at The Alan Turing Institute.

\bibliographystyle{plain}
\bibliography{A_references}

% \end{document}

%------------------------------------------------
% Appendix containing technical results needed 
% and maybe extra simulations
%-------------------------------------------------
% \clearpage 
\appendix
%---------------------------
% Input files for appendix
%---------------------------
%------------------------------------------
% Tools from Matrix perturbation analysis
%------------------------------------------
\section{Matrix perturbation analysis} 
Let $A, \tilde A \in \mathbb{C}^{m \times n}$ where we assume $m \geq n$ w.l.o.g. 
Let 
$$\sigma_1(A) \geq \dots \geq \sigma_n(A), \quad \sigma_1(\tilde A) \geq \dots \geq \sigma_n(\tilde A)$$ 
denote the singular values of $A$ and $\tilde A$ respectively, and denote $E = \tilde A - A$.

To begin with, we would like to quantify the perturbation of the singular values of $\tilde{A}$ 
with respect to those of $A$. Weyl's inequality \cite{Weyl1912} is a very useful result in this regard.
\begin{theorem}[Weyl's inequality \cite{Weyl1912}] \label{thm:weyl_sing_values}
It holds that 
$$\abs{\sigma_i(\tilde A) - \sigma_i(A)} \leq \norm{E}_2, \quad i=1,\dots,n.$$
\end{theorem}

Next, let us write the singular value decomposition of $A$ as
\begin{equation*}
 A = [U_1 \ U_2 \ U_3] \left[
  \begin{array}{cc}
              \Sigma_1 & 0 \\ 
               0 & \Sigma_2 \\
							 0 & 0
\end{array}
 \right]  [V_1 \ V_2]^{*}
\end{equation*}
and the same for $\tilde A$ (conformal partition) so that 
\begin{equation*}
 \tilde A = [\tilde U_1 \ \tilde U_2 \ \tilde U_3] \left[
  \begin{array}{cc}
              \tilde \Sigma_1 & 0 \\ 
               0 & \tilde \Sigma_2 \\
							 0 & 0
\end{array}
 \right]  [\tilde V_1 \ \tilde V_2]^{*}.
\end{equation*}
We will like to quantify the perturbation of $\calR(U_1), \calR(V_1)$, this is given precisely 
by Wedin's theorem \cite{Wedin1972}. Before introducing the theorem, we need some definitions. 
Let $U,\widetilde{U} \in \mathbb{C}^{n \times k}$ (for $k \leq n$) have orthonormal columns respectively and 
let $\sigma_1 \geq \dots \geq \sigma_k$ denote the singular values of $U^{*}\widetilde{U}$. 
Then the $k$ canonical angles between $\calR(U), \calR(\widetilde{U})$ are 
defined as $\theta_i := \cos^{-1}(\sigma_i)$ for $1 \leq i \leq k$, with each $\theta_i \in [0,\pi/2]$. 
It is usual to define $k \times k$ diagonal matrices 
$\Theta(\calR(U), \calR(\widetilde{U})) := \text{diag}(\theta_1,\dots,\theta_k)$ 
and $\sin \Theta(\calR(U), \calR(\widetilde{U})) := \text{diag}(\sin \theta_1,\dots,\sin \theta_k)$.
Denoting $||| \cdot |||$ to be any unitarily invariant norm (Frobenius, spectral, etc.), 
it is useful to know that the following relation 
holds (see for eg., \cite[Lemma 2.1]{li94}, \cite[Corollary I.5.4]{stewart1990matrix}).
\begin{equation*} 
|||  \sin \Theta(\calR(U), \calR(\widetilde{U}))  |||  =  ||| (I - \tilde{U}  \tilde{U}^{*} ) U |||.
\end{equation*}

Let $\Phi$ be the matrix of canonical angles between $\calR(U_1), \calR(\tilde U_1)$ 
and let $\Theta$ be the matrix of canonical angles between $\calR(V_1), \calR(\tilde V_1)$.
\begin{theorem}[Wedin \cite{Wedin1972}]  \label{thm:wedin_bd}
Suppose that there are numbers $\alpha,\delta > 0$ such that 
\begin{equation} \label{eq:wedin_cond}
\min \sigma(\tilde \Sigma_1) \geq \alpha + \delta \quad \text{and} \quad \max \sigma(\Sigma_2) \leq \alpha.
\end{equation}
Then, 
\begin{equation*}
\max \set{\norm{\sin \Theta}_2, \norm{\sin \Phi}_2} 
\leq \frac{\norm{E}_2}{\delta}.
\end{equation*}
\end{theorem}
Note that due to \eqref{eq:wedin_cond} the bounds are restricted to 
subspaces corresponding to the largest singular values. But 
one can derive bounds on $\max \set{\norm{\sin \Theta}_F, \norm{\sin \Phi}_F}$ 
without this restriction (see for eg. \cite[Theorem 4.1]{stewart1990matrix}). 
%------------------------------------
% Useful concentration inequalities
%------------------------------------
\section{Useful concentration inequalities} \label{sec:useful_conc}
We outline here some useful concentration inequalities that are used in our proofs.
%-----------------------------------
% Spectral norm of random matrices
%------------------------------------
\subsection{Spectral norm of random symmetric matrices} \label{app:sec_spec_rand_mat}
We will make use of the following result for bounding the spectral norm of 
random symmetric matrices with independent, centered and bounded entries.
\begin{theorem}[{\cite[Corollary 3.12, Remark 3.13]{bandeira2016}}] \label{app:thm_symm_rand}
Let $X$ be an $n \times n$ symmetric matrix whose entries $X_{ij}$ $(i \leq j)$ are 
independent, centered random variables. There there exists for any $0 < \varepsilon \leq 1/2$ 
a universal constant $c_{\varepsilon}$ such that for every $t \geq 0$, 
\begin{equation} \label{eq:afonso_conc}
\prob(\norm{X}_2 \geq (1+\varepsilon) 2\sqrt{2}\tilde{\sigma} + t) 
\leq n\exp\left(-\frac{t^2}{c_{\varepsilon}\tilde{\sigma}_{*}^2} \right)
\end{equation}
where
\begin{equation*} 
\tilde{\sigma}:= \max_{i} \sqrt{\sum_{j} \expec[X_{ij}^2]}, 
\quad \tilde{\sigma}_{*}:= \max_{i,j} \norm{X_{ij}}_{\infty}.
\end{equation*}
\end{theorem}
Note that it suffices to employ upper bound estimates on $\tilde{\sigma},\tilde{\sigma}$ in
\eqref{eq:afonso_conc}.

%------------------------
% Bernstein inequality
%------------------------
\subsection{Bernstein inequality} \label{app:sec_bern_ineq}
\begin{theorem}[{\cite[Corollary 2.11]{conc_book}}] \label{thm:bern_ineq}
Let $X_1,\dots,X_n$ be independent random variables with $\abs{X_i} \leq b$ for all $i$, 
and $v = \sum_{i=1}^n \expec[X_i^2]$. Then for any $t \geq 0$,
\begin{equation*}
	\prob \left(\abs{\sum_{i=1}^n (X_i - \expec[X_i])} \geq t \right) \leq 2\exp\left(-\frac{t^2}{2(v + \frac{bt}{3})}\right).
\end{equation*}
\end{theorem}
%

%----------------------------------
% Entry wise concentration result
%----------------------------------
\subsection{Product of a random matrix (raised to a power) and a fixed vector} \label{app:sec_randmat_deloc_vec}
Given a random symmetric matrix $X$ with independent entries on and above the diagonal, and a fixed vector $u$, 
Eldridge et al. \cite[Theorem 15]{Eldridge2018} provided an upper tail bound for $\abs{(X^k u)_i}$. The proof technique 
therein followed that of Erd\"os et al. \cite{erdos2013}. In fact, the proof goes through even if $X$ is a random skew-symmetric matrix, the version stated below is for such matrices.
\begin{theorem}[{\cite[Theorem 15]{Eldridge2018}}] \label{thm:eldridge_randmat_pow}
Let $X$ be a $n \times n$ skew-symmetric and centered random matrix. Let $u \in \matR^n$ be a fixed vector 
with $\norm{u}_{\infty} = 1$. Choose $\xi > 1$, $0 < \kappa < 1$ and define $\mu = \frac{2}{\kappa + 1}$. 
If $\expec[\abs{X_{ij}}^m] \leq 1/n$ for all $m \geq 2$, $1 \leq i,j \leq n$ and $k \leq \frac{\kappa}{8} (\log n)^{2\xi}$, 
then we have for any given $l \in [n]$ that 
\begin{equation*}
\prob(\abs{(X^k u)_{l}} \geq (\log n)^{k\xi}) \leq n^{-\frac{1}{4}(\log_{\mu} n)^{\xi-1} (\log_{\mu} e)^{-\xi}}.
\end{equation*}
\end{theorem}

%---------------------------------------------
% Proof of Theorem \ref{eq:main_thm_l2_ERO}
%---------------------------------------------
\section{Proof of Theorem \ref{thm:main_thm_l2_ERO}}  \label{sec:proofThm_l2_ERO}
%
%--------------------------------------------------------------------------
% Step 1: Analysis of the singular values and singular vectors of $C$
%--------------------------------------------------------------------------
\subsection{Proof of Lemma \ref{lem:sing_vals_C}} \label{subsec:step1_ero_proof}
%
%---------
% Proof
%---------
%\begin{proof}[Proof of Lemma \ref{lem:sing_vals_C}]
We begin by noting that for $\alpha = \frac{r^T e}{n}$, $\dotprod{e}{r-\alpha e} = 0$ and so $u_1, u_2$ (similarly $v_1,v_2$) 
are orthogonal. Moreover, $u_1,u_2,v_1,v_2$ also have unit $\ell_2$ norm. Finally, we note that 
\begin{equation*}
u_1 v_1^T + u_2 v_2^T 
= \frac{1}{\sqrt{n}\norm{r-\alpha e}_2}[-e(r-\alpha e)^T + (r-\alpha e) e^T] 
= \frac{1}{\sqrt{n}\norm{r-\alpha e}_2}[\underbrace{r e^T - e r^T}_{C}], 
\end{equation*}
which completes the proof.	
%\end{proof}

%------------------------------------------
% Step 2: Bounding the spectral norm of Z
%------------------------------------------
\subsection{Proof of Lemma \ref{lem:specnorm_Z_ERO}} \label{subsec:step2_ero_proof}

%
%
%\begin{proof}[Proof of Lemma \ref{lem:specnorm_Z_ERO}]
We begin by considering the second moment of $S_{ij}$ for $i < j$. Note that $\expec[N_{ij}] = 0$ and 
$\expec[N_{ij}^2] = M^2/3$. We thus obtain
\begin{align*}
\expec[ (S)_{ij}^2 ]   
& =  
\eta p (1- \eta p)^2 (r_i - r_j)^2  
+ (1-\eta) p  \expec[ N_{ij} - \eta p (r_i - r_j) ]^2 
+ (1- p) (r_i - r_j)^2 \eta^2 p^2  \\
& =
\eta p (1- \eta p)^2 (r_i - r_j)^2  
+ (1-\eta) p  \left[  \frac{M^2}{3}  + \eta^2 p^2 (r_i - r_j)^2 \right]  + (1- p) (r_i - r_j)^2 \eta^2 p^2 
\end{align*}
which together with $ (r_i - r_j)^2  \leq M^2$ yields 
\begin{align} \label{eq:Z_var_bound}
\expec[ (S)_{ij}^2 ]
& \leq  \eta p (1- \eta p)^2 M^2 + (1-\eta) p \left[ \frac{M^2}{3} + \eta^2 p^2 M^2 \right] + (1- p) \eta^2 p^2 M^2  \nonumber \\
& = [ \eta p (1-\eta p )^2 + (1-\eta) p \left(\frac{1}{3}+\eta^2p^2 \right) + (1-p) \eta^2 p^2  ] M^2 \nonumber \\
&\leq \left(\eta p + \frac{4}{3}p + \eta^2 p^2 \leq \frac{10}{3}p \right) M^2 = \frac{10p}{3} M^2.
\end{align}
We thus have that $ \expec[ (S)_{ij}^2 ] \leq  \frac{10p}{3} M^2$ for $i < j$; the same bound holds for $i > j$ as well. 
Hence we can bound the quantities $\tilde{\sigma}, \tilde{ \sigma }_*$ defined in Theorem \ref{app:thm_symm_rand} as 
follows.
\begin{align*}
\tilde{\sigma} & :=  \max_{i}  \sqrt{ \sum_{j}   \expec[ (\tilde{S})_{ij}^2 ]  }  
 \leq  \sqrt{ \frac{10p}{3} n  M^2 } = M \sqrt{\frac{10p}{3} n } , \\
   \tilde{ \sigma }_* &:=  \max_{i,j} ||  (\tilde{S})_{ij}   ||_{\infty}  \leq 2M. 
   %(n-1) (1+ \eta p)  \leq n (1+\eta p) .
\end{align*}
Then by invoking Theorem \ref{app:thm_symm_rand}, we obtain for any given $t \geq 0$, $0 < \varepsilon \leq 1/2$ that
\begin{align*}
\mathbb{P}( || \tilde{S} ||_2  \geq (1+\varepsilon) 2 \sqrt{2}  M \sqrt{\frac{10p}{3} n}  + t)
&\leq \mathbb{P}( || \tilde{S} ||_2  \geq (1+\varepsilon) 2 \sqrt{2}  \tilde{\sigma}  + t)    \\
& \leq 2n \exp  \left( \frac{-t^2}{c_{\varepsilon} \sigma_*^2 } \right) 
\leq 2n \exp  \left( \frac{-t^2}{4 c_{\varepsilon} M^2  } \right),  
\end{align*}
where $c_{\varepsilon} > 0$ depends only on $\varepsilon$. 
Plugging   $t = 2 \sqrt{2} M \sqrt{\frac{10p}{3} n}$ we obtain the stated bound. 
%
%\end{proof}
%
%------------------------------
% Step 3: Using Wedin's bound
%------------------------------ 
\subsection{Proof of Lemma \ref{lem:wedins_bd_ero}} \label{subsec:step3_ero_proof}
%
%
%
%\begin{proof}[Proof of Lemma \ref{lem:wedins_bd_ero}]
Note that Weyl's inequality \cite{Weyl1912} (see Theorem \ref{thm:weyl_sing_values}) 
for singular value perturbation readily yields 
\begin{equation*}
   \sigma_i(H) \in [ \sigma_i( \eta p C )  \pm \Delta ], \quad \forall  i =1,\ldots,n.
\end{equation*}
In particular, using the expressions for $\sigma_1(C), \sigma_2(C)$ from Lemma \ref{lem:sing_vals_C}, 
we have that 
\begin{align*}
   \sigma_1(H), \sigma_2(H)  &\geq  \eta p  \sqrt{n} \norm{r-\alpha e}_2 - \Delta, \\ 
	\sigma_3(H),  \ldots,  \sigma_n(H)  &\leq  \Delta.
\end{align*}
%
% 
%Let us denote $U = [u_1 \ u_2] \in \matR^{n \times 2}$ where we recall from Lemma \ref{lem:sing_vals_C} 
%that $ u_1 = \frac{e}{\sqrt{n}},  u_2 = - \frac{(r-\alpha e)}{ || r-\alpha e ||_2 }  $ are the left singular 
%vectors corresponding to the non-zero singular values of $C$. Also, let us denote 
%$\hat U = [\hat u_1 \ \hat u_2] \in \matR^{n \times 2}$ where $\hat u_1 , \hat u_2$ are the left singular vectors
%corresponding to the top two singular values of $H$.
Then using Wedin's bound for perturbation of singular subspaces \cite{Wedin1972} (see Theorem \ref{thm:wedin_bd}), 
we obtain
\begin{equation*} %\label{eq:sintheta_bd_1}
||  \sin \Theta(\calR(\hat{U}), \calR(U))  ||_2 = ||  (I - \hat{U} \hat{U}^T) U  ||_2  
 \leq  \frac{ \Delta }{  \eta p  \sqrt{n} \norm{r-\alpha e}_2 - \Delta }  \ (=: \delta) 
\end{equation*}
provided $\Delta <  \eta p  \sqrt{n} \norm{r-\alpha e}_2$ holds. 
%\end{proof}

%--------------------------------
% Analyzing the projection step
%--------------------------------
\subsection{Proof of Lemma \ref{lem:proj_analysis_Step_ERO}} \label{subsec:step4_ero_proof}
%
%
%\begin{proof}[Proof of Lemma \ref{lem:proj_analysis_Step_ERO}]
To begin with, note that $ \bar{u}_1 = \hat{U} \hat{U}^T u_1$  is the orthogonal projection 
of $u_1$ on $\calR(\hat{U})$. A unit vector orthogonal to $ \bar{u}_1 $ and lying in $\calR(\hat{U})$ is given by 
\begin{equation*}
	\tilde{u}_2 = \frac{1}{ || \bar{u}_1 ||_2 } \;  \hat{U}     \; 
\begin{bmatrix}
  \mp  \langle \hat{u}_2, \bar{u}_1  \rangle    \\
  \pm    \langle \hat{u}_1, \bar{u}_1  \rangle
\end{bmatrix} .
\end{equation*}
Denoting $ \bar{u}_2$ to be the orthogonal projection of $u_2$ on $ \calR(\hat{U})$,  we can decompose 
$\langle \tilde{u}_2, u_2  \rangle$ as
\begin{equation*}
 \langle \tilde{u}_2, u_2  \rangle =   \langle \tilde{u}_2,   
u_2 - \frac{ \bar{u}_2 }{|| \bar{u}_2 ||_2 } \rangle +   \langle \tilde{u}_2, 
\frac{ \bar{u}_2 }{|| \bar{u}_2 ||_2 } \rangle,
\end{equation*}
which leads to 
\begin{equation} \label{eq:lowbd_u2til_u2}
|  \langle \tilde{u}_2, u_2  \rangle  |     
\geq  
 		| \langle \tilde{u}_2,   \frac{ \bar{u}_2 }{|| \bar{u}_2 ||_2 } \rangle |
     -  |  \langle \tilde{u}_2,   u_2 - \frac{ \bar{u}_2 }{|| \bar{u}_2 ||_2 } \rangle  | .
\end{equation}
We will now bound the two terms in the RHS of \eqref{eq:lowbd_u2til_u2} starting with the first term.

%--------------------------
% Bounding the first term
%--------------------------
\underline{\textbf{Lower bounding the first term in RHS of \eqref{eq:lowbd_u2til_u2}.}} 
Since $\dotprod{\hat u_2}{\bar{u}_1} = \dotprod{\hat u_2}{u_1}$,  $\dotprod{\hat u_1}{\bar{u}_1} = \dotprod{\hat u_1}{u_1}$, we obtain
\begin{align*}
 \langle \tilde{u}_2,   \frac{ \bar{u}_2 }{|| \bar{u}_2 ||_2 } \rangle
& = \frac{u_2^T}{ || \bar{u}_1||_2  ||  \bar{u}_2 ||_2 }  \hat{U}       \underbrace{ \hat{U}^T    \hat{U}  }_{I}
\begin{bmatrix}
  \mp  \langle \hat{u}_2, \bar{u}_1  \rangle    \\
  \pm    \langle \hat{u}_1, \bar{u}_1  \rangle
\end{bmatrix}   \\
& = \frac{1}{\norm{\bar{u}_1}_2 \norm{\bar{u}_2}}_2 
[\mp\dotprod{u_2}{\hat u_1} \dotprod{\hat u_2}{u_1} \pm \dotprod{u_2}{\hat u_2}\dotprod{\hat u_1}{u_1}] \\
& = \frac{1}{\norm{\bar{u}_1}_2 \norm{\bar{u}_2}}_2 [\pm \det(U^T \hat U)]. \label{eq:lowbd}
\end{align*}
Recall the definition of $\delta$ in \eqref{eq:sintheta_bd_1}. The following simple claims will be useful for us. 
\begin{claim} \label{claim:detbd}
With $\delta$ as defined in \eqref{eq:sintheta_bd_1} we have $\abs{\det(U^T \hat U)} \geq 1 - \delta^2$.
\end{claim}
\begin{proof}
Since $U^T \hat U$ is a $2 \times 2$ matrix, we have that
\begin{equation*}
\abs{\det(U^T \hat U)} \geq \sigma_{min}^2(U^T \hat U) = 1 - ||  \sin \Theta(\calR(\hat{U}), \calR(U))  ||_2^2 \geq 1-\delta^2.
\end{equation*}
\end{proof}
\begin{claim} \label{claim:u1u2_proj_err}
It holds that $\norm{\bar{u}_1 - u_1}_2, \norm{\bar{u}_2 - u_2}_2 \leq \delta$.
\end{claim}
\begin{proof}
We will only show the bound for $\norm{\bar{u}_1 - u_1}_2$ as an identical argument holds for 
bounding $\norm{\bar{u}_2 - u_2}_2$. Using the facts $\bar u_1 = \hat U \hat U u_1$ and 
$u_1 = U U^T u_1$, we obtain
\begin{align*}
\norm{u_1 - \bar{u}_1}_2 
= \norm{(I -  \hat U \hat U^T)u_1}_2 = \norm{(I -  \hat U \hat U^T)U U^T u_1}_2 
\leq \norm{(I -  \hat U \hat U^T)U U^T}_2 = \norm{(I -  \hat U \hat U^T)U}_2 \leq \delta
\end{align*}
where the last inequality follows from \eqref{eq:sintheta_bd_1}.
\end{proof}
Using Claims \ref{claim:detbd}, \ref{claim:u1u2_proj_err} we can lower bound the 
first term in RHS of \eqref{eq:lowbd_u2til_u2} as follows.
\begin{equation}
\abs{ \langle \tilde{u}_2,   \frac{ \bar{u}_2 }{|| \bar{u}_2 ||_2 } \rangle }  
\geq \frac{1 - \delta^2}{(1+\delta)^2} = \frac{1-\delta}{1+\delta} = 1 - \frac{2 \delta}{1+\delta} \geq 1-2 \delta.
\label{eq:FirstTermFinalBound}
\end{equation}
%

%-----------------------------------
% Upper bounding the second term
%----------------------------------- 
\underline{\textbf{Upper bounding the second term in RHS of \eqref{eq:lowbd_u2til_u2}.}}
We can do this as follows.
\begin{align}
\abs {\langle \tilde{u}_2, u_2 - \frac{ \bar{u}_2  }{  || \bar{u}_2 ||_2 } \rangle }  
&\leq  \norm{   u_2 - \frac{ \bar{u}_2  }{  || \bar{u}_2 ||_2 }   }_2 \quad 
(\text{Cauchy-Schwartz and since } \norm{\tilde{u}_2}_2 = 1) \nonumber  \\
&  = \norm{  u_2 - \bar{u}_2 + \bar{u}_2 - \frac{ \bar{u}_2 }{ ||  \bar{u}_2 ||_2 } }_2  \nonumber \\ 
&\leq \norm{u_2 - \bar{u}_2}_2 + \norm{\bar{u}_2 - \frac{ \bar{u}_2 }{ ||  \bar{u}_2 ||_2}}_2 \nonumber  \\
&\leq  \delta + \abs{ 1 - \frac{ 1 }{ || \bar{u}_2 ||_2 }  }  \quad (\text{ Using Claim \ref{claim:u1u2_proj_err}}) \nonumber  \\
&= \delta + \frac{ \abs{ || \bar{u}_2 ||_2 - 1 } }{ || \bar{u}_2 ||_2 } \nonumber  \\
&  \leq \delta + \frac{\delta}{1-\delta}  \quad 
(\text{since } || \bar{u}_2 ||_2 \in [1-\delta,1+\delta]  \text{ from Claim \ref{claim:u1u2_proj_err}}) \nonumber \\
&\leq 3 \delta  \quad  (\mbox{whenever } \; \delta \leq 1/2).  \label{eq:SecTermFinalBound}
\end{align}
Altogether, we conclude that if $\delta \leq 1/2$ then using \eqref{eq:FirstTermFinalBound}, 
\eqref{eq:SecTermFinalBound} in \eqref{eq:lowbd_u2til_u2}, we obtain
\begin{equation*}
| \langle \tilde{u}_2,  u_2 \rangle | \geq  1 - 5 \delta,  
\end{equation*}
and consequently, there exists $ \beta \in \{ -1, 1 \}$ such that  
\begin{equation} \label{eq:l2normbd_1}
|| \tilde{u}_2 - \beta  u_2  ||_2^2 = 2 - 2 \langle \tilde{u}_2,  \beta u_2 \rangle  \leq 10 \delta.
\end{equation}
\section{Proof of Lemmas from Section \ref{sec:proofOutline_thm_linf_ERO}} \label{sec:proofs_lems_linf_svdrs}
\subsection{Proof of Lemma \ref{linf_ero_step1}} 
\begin{proof}
Recall from the proof of Lemma \ref{lem:proj_analysis_Step_ERO} that 
\begin{equation} \label{eq:linf_step1_temp1}
	\tilde{u}_2 = \frac{1}{ || \bar{u}_1 ||_2 } \;  \hat{U}     \; 
\begin{bmatrix}
  \mp  \langle \hat{u}_2, \bar{u}_1  \rangle    \\
  \pm    \langle \hat{u}_1, \bar{u}_1  \rangle
\end{bmatrix} 
= \frac{1}{ || \bar{u}_1 ||_2 } \;  \hat{U}     \; 
\begin{bmatrix}
  0 & \mp1 \\
	\pm1 & 0
\end{bmatrix} \Uhat^T \bar{u}_1 
= \frac{1}{ || \bar{u}_1 ||_2 } \;  \hat{U}     \; 
\begin{bmatrix}
  0 & \mp1 \\
	\pm1 & 0
\end{bmatrix} \Uhat^T u_1, 
\end{equation}
where the last equality follows from the fact 
$\dotprod{\hat u_2}{\bar{u}_1} = \dotprod{\hat u_2}{u_1}$,  $\dotprod{\hat u_1}{\bar{u}_1} = \dotprod{\hat u_1}{u_1}$.
Let $O$ be any $2 \times 2$ orthogonal matrix, clearly, $UO$ also corresponds to the two largest left singular vectors 
of $\expec[H]$ (since $\sigma_1(\expec[H]) = \sigma_2(\expec[H])$). Denote $P = \Uhat - UO$. Plugging in \eqref{eq:linf_step1_temp1}, we obtain 
\begin{align}
	\tilde{u}_2 &= \frac{1}{ || \bar{u}_1 ||_2 } \;  \hat{U}     \; 
\underbrace{\begin{bmatrix}
  0 & \mp1 \\
	\pm1 & 0
\end{bmatrix}}_{D} \Uhat^T \bar{u}_1  [O^T U^T + P] u_1 \nonumber \\
&= \pm \det(O) u_2 + \underbrace{\frac{1}{ || \bar{u}_1 ||_2 } UO D P^T u_1}_{w_1} + \underbrace{\frac{1}{ || \bar{u}_1 ||_2 } PDO^T U^T u_1}_{w_2} 
+ \underbrace{\frac{1}{ || \bar{u}_1 ||_2 } PDP^T u_1}_{w_3} \nonumber \\
\Rightarrow \norm{\tilde{u}_2 \mp \det(O) u_2}_{\infty} &\leq \norm{w_1}_{\infty} + \norm{w_2}_{\infty} + \norm{w_3}_{\infty}. \label{eq:u2til_inf_bd_init}
\end{align}
Denoting 
\begin{equation*}
P = [p_1 \ p_2], \quad O = \begin{bmatrix}
  o_{11} & o_{12} \\
	o_{21} & o_{22}
\end{bmatrix}
\end{equation*}
we will now bound $\norm{w_1}_{\infty} , \norm{w_2}_{\infty} , \norm{w_3}_{\infty}$. 

%------ Bounding w_1 infty norm --------
To begin with, 
\begin{align}
w_1 &= \frac{1}{\norm{\bar{u}_1}_2} [u_1 \ u_2] O 
\begin{bmatrix}
  0 & \mp1 \\
	\pm1 & 0
\end{bmatrix} 
\begin{bmatrix}
  \dotprod{p_1}{u_1} \\
	\dotprod{p_2}{u_1}
\end{bmatrix} \nonumber \\
&= \frac{1}{\norm{\bar{u}_1}_2} [u_1 \ u_2] O 
\begin{bmatrix}
  \mp \dotprod{p_2}{u_1} \\
	\pm \dotprod{p_1}{u_1}
\end{bmatrix} \nonumber \\ 
&= \frac{1}{\norm{\bar{u}_1}_2} [u_1 \ u_2]
\begin{bmatrix}
  \mp o_{11}\dotprod{p_2}{u_1} \pm o_{12} \dotprod{p_1}{u_1} \\
	\mp o_{21}\dotprod{p_2}{u_1} \pm o_{22} \dotprod{p_1}{u_1}
\end{bmatrix} \nonumber \\
\Rightarrow \norm{w_1}_{\infty} &\leq 
\frac{1}{\norm{\bar{u}_1}_2} \left[(\abs{\dotprod{p_2}{u_1}} + \abs{\dotprod{p_1}{u_1}})(\norm{u_1}_{\infty} + \norm{u_2}_{\infty}) \right]. \label{eq:linf_step1_temp2}
\end{align} 
Clearly, $\norm{u_1}_{\infty} = \frac{1}{\sqrt{n}}$ and $\norm{u_2}_{\infty} \leq \frac{M - \alpha}{\norm{r - \alpha e}_2}$. 
Using H\"older's inequality,
\begin{equation*}
	\abs{\dotprod{p_i}{u_1}} \leq \norm{P}_{\max} \norm{u}_1 = \norm{P}_{\max}\sqrt{n}; \quad i=1,2.
\end{equation*}
Moreover, from Claim \ref{claim:u1u2_proj_err}, we know that $\norm{\bar{u}_1}_2 \geq 1 - \delta$ where we recall $\delta$ defined in \eqref{eq:sintheta_bd_1}. 
We saw in \eqref{eq:Delta_bd_ERO_outline} that $\Delta \leq \frac{\eta p ||r-\alpha e||_2 \sqrt{n}}{3} \Leftrightarrow \delta \leq 1/2$, which implies 
$\norm{\bar{u}_1}_{2} \geq 1/2$. Plugging these bounds in \eqref{eq:linf_step1_temp2}, we obtain
\begin{equation} \label{eq:linf_w1_bd}
  \norm{w_1}_{\infty} \leq 4\norm{P}_{\max}\left(1 + \frac{\sqrt{n}(M-\alpha)}{\norm{r-\alpha e}_2} \right).
\end{equation}
%

%---------- Bounding w_2 infty norm --------
Next, we can bound $\norm{w_2}_{\infty}$ as follows.
\begin{align}
w_2 &= \frac{1}{\norm{\bar{u}_1}_2} [p_1 \ p_2] 
\begin{bmatrix}
  0 & \mp1 \\
	\pm1 & 0
\end{bmatrix} 
O^T
\begin{bmatrix}
  1 \\
	0
\end{bmatrix} \nonumber \\
&= \frac{1}{\norm{\bar{u}_1}_2} [p_1 \ p_2]
\begin{bmatrix}
  \mp o_{12} \\
	\pm o_{11}
\end{bmatrix} \nonumber \\ 
\Rightarrow \norm{w_2}_{\infty} &\leq \frac{2\norm{P}_{\max}}{\norm{\bar{u}_1}_2} \leq 4\norm{P}_{\max} \label{eq:w2_infty_bd}
\end{align}
since $\Delta \leq \frac{\eta p ||r-\alpha e||_2 \sqrt{n}}{3}$. 

%---------- Bounding w_3 infty norm --------
We now bound $\norm{w_3}_{\infty}$. 
\begin{align}
 w_3 &= \frac{1}{\norm{\bar{u}_1}_2} [p_1 \ p_2] 
\begin{bmatrix}
  0 & \mp1 \\
	\pm1 & 0
\end{bmatrix} 
\begin{bmatrix}
  \dotprod{p_1}{u_1} \\
	\dotprod{p_2}{u_1}
\end{bmatrix} \nonumber \\ 
&= 
\frac{1}{\norm{\bar{u}_1}_2} [p_1 \ p_2] 
\begin{bmatrix}
  \mp \dotprod{p_2}{u_1} \\
	\pm \dotprod{p_1}{u_1}
\end{bmatrix} \nonumber \\
\Rightarrow \norm{w_3}_{\infty} 
&\leq \frac{1}{\norm{\bar{u}_1}_2} ((\abs{\dotprod{p_2}{u_1}} + \abs{\dotprod{p_1}{u_1}}) \norm{P}_{\max})  \nonumber \\
&\leq 4\norm{P}_{\max}^2 \norm{u_1}_1 = 4 \sqrt{n} \norm{P}_{\max}^2 
\quad (\text{since } \Delta \leq \frac{\eta p ||r-\alpha e||_2 \sqrt{n}}{3} \text{ and using H\"older's inequality}). \label{eq:w3_infty_bd}
\end{align}
Plugging \eqref{eq:linf_w1_bd}, \eqref{eq:w2_infty_bd}, \eqref{eq:w3_infty_bd} in \eqref{eq:u2til_inf_bd_init}, we obtain
\begin{align*}
\norm{\tilde{u}_2 \mp \det(O) u_2}_{\infty} 
&\leq 4\norm{P}_{\max}\left(2 + \frac{\sqrt{n}(M-\alpha)}{\norm{r-\alpha e}_2} \right) + 4 \sqrt{n} \norm{P}_{\max}^2 \\
&= 4\norm{\Uhat - UO}_{\max}\left(2 + \frac{\sqrt{n}(M-\alpha)}{\norm{r-\alpha e}_2} \right) + 4 \sqrt{n} \norm{\Uhat - UO}_{\max}^2.  
\end{align*}
The above bound is true for any orthogonal matrix $O$, it is not difficult to see that there exists an 
orthogonal matrix $O = O^*$ such that $\Uhat$ is ``aligned'' with $UO^*$. 
Indeed, for any orthogonal $O$, we first obtain via triangle inequality that 
\begin{align} \label{eq:linf_temp_4}
  \norm{\Uhat - UO}_2 \leq \norm{(I-UU^T) \Uhat}_2 + \norm{UU^T \Uhat - UO}_2 
	= \norm{\sin \Theta(\calR(U), \calR(\Uhat))}_2 + \norm{U^T \Uhat - O}_2.
\end{align}
Denoting $\tilde{U} \tilde{D} \tilde{V}^T$ to be the SVD of $U^T \Uhat$, we 
choose $O = O^* = \tilde{U} \tilde{V}^T$ (so $O^*$ is orthogonal). Denoting $\theta_p$ to be the 
principal angle between $\calR(U), \calR(\Uhat)$, we obtain   
\begin{align*}
  \norm{U^T \Uhat - O^*}_2 = \norm{I - \tilde{D}}_2 = 1- \cos \theta_p \leq \sin \theta_p = \norm{\sin \Theta(\calR(U), \calR(\Uhat))}_2.
\end{align*}
Plugging this in \eqref{eq:linf_temp_4} leads to 
\begin{align*}
\norm{\Uhat - UO^*}_2 
&\leq 2 \norm{\sin \Theta(\calR(U), \calR(\Uhat))}_2 \\
&\leq \frac{ 2\Delta }{  \eta p  \sqrt{n} \norm{r-\alpha e}_2 - \Delta } \quad \text{(see proof of Lemma \ref{lem:wedins_bd_ero})} \\
&\leq \frac{3\Delta}{\eta p \sqrt{n} \norm{r-\alpha e}_2}. \quad \text{(since } \Delta \leq \frac{\eta p ||r-\alpha e||_2 \sqrt{n}}{3} \text{)}
\end{align*}
\end{proof}

%--------------------------------------------
% Proof of Lemma \ref{lem:linf_ero_step2a}
%
\subsection{Proof of Lemma \ref{lem:linf_ero_step2a}}
\begin{proof}
\begin{enumerate}
\item \textbf{(Bounding $\norm{E_1}_{\max}$)} Denoting $(\Uhat - UO^*)_i$ to be the $i^{th}$ column of $\Uhat - UO^*$, we note that 
\begin{align*}
E_1 &= -\expec[H]^2 (\Uhat - U O^*) \Sighat^{-2} \\
&= \sigma^2 U U^T (\Uhat^T - UO^*) \Sighat^{-2} \\
&= (u_1 u_1^T + u_2 u_2^T) 
\begin{bmatrix}
(\Uhat - UO^*)_1 \frac{\sigma^2}{\sighat_1^{2}} & (\Uhat - UO^*)_2 \frac{\sigma^2}{\sighat_2^{2}}
\end{bmatrix} \\
&= \begin{bmatrix}
u_1 \dotprod{u_1}{(\Uhat - UO^*)_1} \frac{\sigma^2}{\sighat_1^{2}} + u_2 \dotprod{u_2}{(\Uhat - UO^*)_1} \frac{\sigma^2}{\sighat_1^{2}} 
& u_1 \dotprod{u_1}{(\Uhat - UO^*)_2} \frac{\sigma^2}{\sighat_2^{2}} + u_2 \dotprod{u_2}{(\Uhat - UO^*)_2} \frac{\sigma^2}{\sighat_2^{2}}
\end{bmatrix}.
\end{align*}
For the $i^{th}$ entry in the first column ($i=1,\dots,n$), we have that 
\begin{align*}
&\abs{u_{1,i} \dotprod{u_1}{(\Uhat - UO^*)_1} \frac{\sigma^2}{\sighat_1^{2}} + u_{2,i} \dotprod{u_2}{(\Uhat - UO^*)_1} \frac{\sigma^2}{\sighat_1^{2}}} \\
&\leq \norm{u_1}_{\infty} \underbrace{\norm{(\Uhat - UO^*)_1}_2}_{\leq \norm{\Uhat - UO^*}_2} \frac{\sigma^2}{\sighat_1^{2}} 
+ \norm{u_2}_{\infty} \underbrace{\norm{(\Uhat - UO^*)_1}_2}_{\leq \norm{\Uhat - UO^*}_2} \frac{\sigma^2}{\sighat_1^{2}}.
\end{align*}
Since $\norm{u_1}_{\infty} = \frac{1}{\sqrt{n}}$, $\norm{u_2}_{\infty} \leq \frac{M-\alpha}{\norm{r-\alpha e}_2}$, and 
$\sighat_1 \geq \sigma - \Delta \geq \frac{2\sigma}{3}$, we obtain
\begin{align*}
&\abs{u_{1,i} \dotprod{u_1}{(\Uhat - UO^*)_1} \frac{\sigma^2}{\sighat_1^{2}} + u_{2,i} \dotprod{u_2}{(\Uhat - UO^*)_1} \frac{\sigma^2}{\sighat_1^{2}}} \\
&\leq \frac{9}{4} \norm{\Uhat - UO^*}_2 \left(\frac{1}{\sqrt{n}} + \frac{M-\alpha}{\norm{r-\alpha e}_2}\right) \\
&\leq \frac{27 \Delta}{4 \sigma} \left(\frac{1}{\sqrt{n}} + \frac{M-\alpha}{\norm{r-\alpha e}_2}\right)
\end{align*}
where in the last inequality, we used $\norm{\Uhat - UO^*}_2 \leq \frac{3\Delta}{\sigma}$ (from Lemma \ref{linf_ero_step1}). 
The same entry-wise bound holds for the second column, leading to the stated bound on $\norm{E_1}_{\max}$.

%--------------------
\item \textbf{(Bounding $\norm{E_2}_{\max}$)} Denoting the $(i,j)^{th}$ entry of $O^*$ by $o^*_{ij}$, we have that 
\begin{align*}
E_2 &= \expec[H]^2 UO^{*} (\sigma^{-2} I - \Sighat^{-2}) \\
&= -UU^T  \sigma^2 (U O^*) 
\begin{bmatrix}
\sigma^{-2} - \sighat_1^{-2} & 0 \\
0 & \sigma^{-2} - \sighat_2^{-2}
\end{bmatrix} \\
&= -\sigma^2 \begin{bmatrix}
(\sigma^{-2} - \sighat_1^{-2}) (u_1 o^{*}_{11} + u_2 o^{*}_{21}) & 
(\sigma^{-2} - \sighat_2^{-2}) (u_1 o^{*}_{12} + u_2 o^{*}_{22})
\end{bmatrix}.
\end{align*}
Hence the $i^{th}$ entry of the first column can be bounded as follows.
\begin{align*}
 \sigma^2 \abs{\sigma^{-2} - \sighat_1^{-2}} \abs{u_{1,i} o^{*}_{11} + u_{2,i} o^{*}_{12}} 
&\leq \frac{\Delta(2\sigma + \Delta)}{(\sigma - \Delta)^2} (\norm{u_1}_{\infty} + \norm{u_2}_{\infty}) \quad \text{(using $\abs{\sighat_1 - \sigma} \leq \Delta$)} \\
&\leq \frac{21\Delta}{4\sigma} \left(\frac{1}{\sqrt{n}} + \frac{M-\alpha}{\norm{r-\alpha e}_2}\right) \quad \text{(using $\Delta \leq \frac{\sigma}{3}$).}
\end{align*}
The same entry-wise bound holds for the second column, leading to the stated bound on $\norm{E_2}_{\max}$.

%--------------------
\item \textbf{(Bounding $\norm{E_3}_{\max}$)} We have that 
\begin{align*}
E_3 &= -\expec[H] Z \Uhat \Sighat^{-2} \\
&= -\sigma (u_1 v_1^T + u_2 v_2^T) Z 
\begin{bmatrix}
\sighat_1^{-2} \uhat_1 & \sighat_2^{-2} \uhat_2
\end{bmatrix} \\
&= -\sigma \begin{bmatrix}
u_1\dotprod{v_1}{Z \uhat_1}\sighat_1^{-2} + u_2\dotprod{v_2}{Z \uhat_1}\sighat_1^{-2}
 & u_1\dotprod{v_1}{Z \uhat_2}\sighat_2^{-2} + u_2\dotprod{v_2}{Z \uhat_2}\sighat_2^{-2}
\end{bmatrix}.
\end{align*}
The $i^{th}$ entry of the first column can be bounded as follows.
\begin{align*}
\sigma\abs{u_{1,i} \dotprod{v_1}{Z \uhat_1}\sighat_1^{-2} + u_{2,i}\dotprod{v_2}{Z \uhat_1}\sighat_1^{-2}} 
&\leq \frac{\sigma}{\sighat_1^2}(\norm{u_1}_{\infty} \abs{\dotprod{v_1}{Z \uhat_1}} + \norm{u_2}_{\infty} \abs{\dotprod{v_2}{Z \uhat_1}}) \\
&\leq \frac{\sigma}{\sighat_1^2} \norm{Z}_2 \left(\frac{1}{\sqrt{n}} + \frac{M-\alpha}{\norm{r-\alpha e}_2}\right) \quad \text{(using Cauchy-Schwarz)} \\
&\leq \frac{9\Delta}{4\sigma} \left(\frac{1}{\sqrt{n}} + \frac{M-\alpha}{\norm{r-\alpha e}_2}\right) \quad \text{(using $\est{\sigma}_1 \geq 2\sigma/3$)}.
\end{align*}
The same entry-wise bound holds for the second column, leading to the stated bound on $\norm{E_3}_{\max}$.

%--------------------
\item \textbf{(Bounding $\norm{E_4}_{\max}$)} We have that 
\begin{align*}
E_4 = Z \expec[H] \Uhat \Sighat^{-2} 
&= \sigma Z (u_1 v_1^T + u_2 v_2^T) [\sighat_1^{-2} \uhat_1 \ \sighat_2^{-2} \uhat_2] \\
&= \sigma \begin{bmatrix} 
Z u_1 \frac{v_1^T \uhat_1}{\sighat_1^2} + Z u_2 \frac{v_2^T \uhat_1}{\sighat_1^2}
& Z u_1 \frac{v_1^T \uhat_2}{\sighat_2^2} + Z u_2 \frac{v_2^T \uhat_2}{\sighat_2^2}
\end{bmatrix}.
\end{align*}
The magnitude of the $i^{th}$ entry of the first column can be bounded as 
\begin{align*}
\frac{\sigma}{\sighat_1^2} \abs{(v_1^T \uhat_1) (Z u_1)_i + (v_2^T \uhat_1) (Z u_2)_i} 
&\leq \frac{\sigma}{\sighat_1^2} (\norm{Z u_1}_{\infty} + \norm{Z u_2}_{\infty}) \quad \text{(Using Cauchy-Schwarz)} \\
&\leq \frac{9}{4\sigma} (\norm{Z u_1}_{\infty} + \norm{Z u_2}_{\infty}) \quad \text{(using $\est{\sigma}_1 \geq 2\sigma/3$)}.
\end{align*}
The same entry-wise bound holds for the second column, leading to the stated bound on $\norm{E_4}_{\max}$. 

%--------------------
\item \textbf{(Bounding $\norm{E_5}_{\max}$)} We can write
\begin{align*}
 E_5 = Z^2 \Uhat \Sighat^{-2} = Z^2 (\Uhat - UO^*)\Sighat^{-2} + Z^2 (UO^*)\Sighat^{-2} 
\end{align*}
which in turn implies $\norm{E_5}_{\max} \leq \norm{Z^2 (\Uhat - UO^*)\Sighat^{-2}}_{\max} + \norm{Z^2 (UO^*)\Sighat^{-2}}_{\max}$. 
We will now bound these terms individually, below. To begin with, 
\begin{align*}
\norm{Z^2 (\Uhat - UO^*)\Sighat^{-2}}_{\max} 
&\leq \norm{Z^2 (\Uhat - UO^*)\Sighat^{-2}}_2 \\
&\leq \Delta^2 \norm{\Uhat - UO^*}_2 \norm{\Sighat^{-2}}_2 \quad \text{(sub-multiplicativity of $\norm{\cdot}_2$)} \\
&\leq \frac{3\Delta^3}{\sigma} \norm{\Sighat^{-2}}_2 \quad \text{(from Lemma \ref{linf_ero_step1})} \\
&\leq \frac{27\Delta^3}{4\sigma^3} \quad \text{(using $\est{\sigma}_i \geq 2\sigma/3$)}.
\end{align*}
Additionally, denoting $O^* = [o^*_1 \ o^*_2]$, we can write 
\begin{align*}
Z^2 (UO^*) \Sighat^{-2} 
= \underbrace{[Z^2 u_1 \ \ Z^2 u_2]}_{\widetilde{A}} [\sighat_1^{-2} o^{*}_1 \ \ \sighat_2^{-2} o^{*}_2] 
= [\sighat_1^{-2} \widetilde{A} o^{*}_1 \ \ \sighat_2^{-2}\widetilde{A} o^{*}_2]
\end{align*}
which in turn implies 
\begin{align*}
\norm{Z^2 (UO^*) \Sighat^{-2}}_{\max} 
&\leq \max\set{\frac{\norm{\widetilde{A} o^{*}_1}_{\infty}}{\sighat_1^2}, \frac{\norm{\widetilde{A} o^{*}_2}_{\infty}}{\sighat_2^2}} \\
&\leq \frac{9}{4\sigma^2} \max\set{\norm{\widetilde{A} o^{*}_1}_{\infty}, \norm{\widetilde{A} o^{*}_2}_{\infty}} \quad \text{(using $\est{\sigma}_i \geq 2\sigma/3$)} \\
&\leq \frac{9}{4\sigma^2}(\norm{Z^2 u_1}_{\infty} + \norm{Z^2 u_2}_{\infty}).
\end{align*}
\end{enumerate}
\end{proof}

%--------------------------------------------
% Proof of Lemma \ref{lem:inf_conc_bd1}
%--------------------------------------------
\subsection{Proof of Lemma \ref{lem:inf_conc_bd1}}
\begin{proof}
\begin{enumerate}
\item For a fixed $i$, consider $\sum_{j=1}^n Z_{ij} u_{1j}$ which is a sum of independent random variables. 
From the definition of $Z_{ij}$, we can see that $\abs{Z_{ij} u_{1j}} \leq \frac{2M}{\sqrt{n}}$ for each $j$. 
Moreover, $\expec[(Z_{ij} u_{1j})^2] =\frac{1}{n} \expec[Z_{ij}^2] \leq \frac{10p M^2}{3n}$ where the last inequality was 
shown in the proof of Lemma \ref{lem:specnorm_Z_ERO} (see \eqref{eq:Z_var_bound}). Hence $\sum_{i=1}^n \expec[(Z_{ij} u_{1j})^2] \leq \frac{10p M^2}{3}$ 
and it follows from Bernstein's inequality (see Theorem \ref{app:sec_bern_ineq}) and the union bound that
\begin{equation} \label{eq:ineq_bd_temp1}
  \prob(\norm{Z u_1}_{\infty} \geq t) \leq 2n \exp \left(-\frac{3t^2}{20 p M^2 + \frac{4t M}{\sqrt{n}}} \right).
\end{equation}
We want to choose $t$ such that
\begin{equation} \label{eq:ineq_bd_temp3}
3t^2 \geq 2\log n \left(20 p M^2 + \frac{4t M}{\sqrt{n}} \right) \Leftrightarrow 
3t^2 - 40p M^2 \log n - \frac{8tM \log n}{\sqrt{n}} \geq 0.
\end{equation} 
Since $t \geq 0$, the above inequality is achieved iff 
\begin{equation} \label{eq:ineq_bd_temp2}
		t \geq \frac{\frac{8M \log n}{\sqrt{n}} + \sqrt{\frac{64 M^2 \log^2 n}{n} + 480 pM^2 \log n}}{6}.
\end{equation}
Hence if 
\begin{equation*}
 \frac{64 M^2 \log^2 n}{n} \leq 480 pM^2 \log n \Leftrightarrow p \geq \frac{2\log n}{15 n}
\end{equation*}
holds, then the RHS of \eqref{eq:ineq_bd_temp2} is upper bounded by
\begin{align*}
 \frac{\frac{8M \log n}{\sqrt{n}} + \sqrt{960 pM^2 \log n}}{6} \leq \frac{\frac{8M \log n}{\sqrt{2\log n}} \sqrt{15 p} + 8M \sqrt{15p \log n} }{6} = \frac{2\sqrt{2} + 4}{3} M \sqrt{15 p \log n}.
\end{align*}
Therefore if $t \geq \frac{2\sqrt{2} + 4}{3} M \sqrt{15 p \log n}$ is satisfied, then it implies that 
\eqref{eq:ineq_bd_temp3} holds. Hence plugging $t = \frac{2\sqrt{2} + 4}{3} M \sqrt{15 p \log n}$ in \eqref{eq:ineq_bd_temp1} 
leads to the statement of the Lemma.

\item For a fixed $i$, consider now $\sum_{j=1}^n Z_{ij} u_{2j}$ which is also a sum of independent random variables. 
One can verify that $\abs{Z_{ij} u_{2j}} \leq 2M \sqrt{B}$ for each $j$, with $B$ as defined in the Lemma. 
Moreover, $\expec[(Z_{ij} u_{2j})^2] \leq B \expec[Z_{ij}^2] \leq \frac{10p M^2 B}{3}$ (see \eqref{eq:Z_var_bound}) and hence, 
$\sum_{i=1}^n \expec[(Z_{ij} u_{2j})^2] \leq \frac{10p M^2 B n}{3}$ and it follows from Bernstein's inequality 
(see Theorem \ref{app:sec_bern_ineq}) and the union bound that
\begin{equation} \label{eq:ineq_bd_temp21}
  \prob(\norm{Z u_2}_{\infty} \geq t) \leq 2n \exp \left(-\frac{3t^2}{20 p M^2 Bn + 4t \sqrt{B} M} \right).
\end{equation}
Proceeding identically as before, the reader is invited to verify that if $p \geq \frac{2 \log n}{15 n}$, then the stated bound on $\norm{Z u_2}_{\infty}$ 
is obtained by plugging $t = \frac{2\sqrt{2} + 4}{3} M \sqrt{15 p B n \log n}$ in \eqref{eq:ineq_bd_temp21}. 
\end{enumerate}
\end{proof}

%--------------------------------------------
% Proof of Lemma \ref{lem:inf_conc_bd2}
%--------------------------------------------
\subsection{Proof of Lemma \ref{lem:inf_conc_bd2}}
\begin{proof}
\begin{enumerate}
\item For a given $l \in [n]$, we can write 
\begin{align}
\abs{(Z^2 u_1)_l} 
&= \norm{u_1}_{\infty} \abs{(Z^2 \underbrace{\frac{u_1}{\norm{u_1}_{\infty}}}_{\util_1})_l} \nonumber \\
&= \frac{1}{\sqrt{n}} \abs{(Z^2 \util_1)_l} \quad \text{(with $\norm{\util_1}_{\infty} = 1$)} \nonumber \\
&= \frac{\phi^2}{\sqrt{n}} \abs{\left((Z/\phi)^2 \util_1\right)_l} \label{eq:inf_conc_bd2_tem1}
\end{align}
for any $\phi > 0$. For any integer $m \geq 2$, and $1 \leq i,j \leq n$, we have that 
\begin{align*}
\expec\left[\left(\frac{\abs{Z_{ij}}}{\phi}\right)^m\right] 
&= \frac{1}{\phi^m}\left[\eta p\abs{r_i - r_j}^m (1-\eta p)^m 
+ (1-\eta)p \expec[\abs{N_{ij} - \eta p(r_i - r_j)}^m] + (1-p)\abs{\eta p (r_i - r_j)}^{m}\right] \\
&\leq \frac{1}{\phi^m} [\eta p M^m + p (2M)^m + (\eta p)^m M^m] \\
&\leq  \frac{1}{\phi^m}[2^{m+1} p M^m].
\end{align*}
Setting $\phi = 2^{3/2} \sqrt{pn} M$ leads to
\begin{equation*}
\expec\left[\left(\frac{\abs{Z_{ij}}}{\phi}\right)^m\right] 
\leq \frac{2^{m+1} p M^m}{(2)^{3m/2} p^{m/2} M^m n^{m/2}} 
= \frac{1}{2^{\frac{m}{2} - 1}}  \frac{1}{p^{\frac{m}{2} - 1}}  \frac{1}{n^{\frac{m}{2}}}.
\end{equation*}
It is easy to check that $p \geq \frac{1}{2n}$ implies 
$\expec\left[\left(\frac{\abs{Z_{ij}}}{\phi}\right)^m\right] \leq 1/n$. Therefore invoking 
Theorem \ref{thm:eldridge_randmat_pow} for $k = 2$, we obtain for the stated choices of 
$\xi,\mu,\kappa$ that $\abs{\left((Z/\phi)^2 \util_1\right)_l} \leq (\log n)^{2\xi}$ holds with probability 
at least $1-n^{-\frac{1}{4}(\log_{\mu} n)^{\xi-1} (\log_{\mu} e)^{-\xi}}$. Plugging this in \eqref{eq:inf_conc_bd2_tem1} 
with the expression for $\phi$, and using the union bound, we obtain the statement of the Lemma.

\item The steps are identical to those above, with the only difference being $\norm{u_2}_{\infty} = \frac{M-\alpha}{\norm{r-\alpha e}_2}$. 

\end{enumerate}
\end{proof}

%---------------------------------------------
% Proof of Theorem \ref{thm:main_svdn_l2_ERO}
%---------------------------------------------
\section{Proofs of Lemmas from Section \ref{sec:proofOutline_thm_svdn_l2_ERO}}  \label{sec:proof_app_N_Thm_l2_ERO}
\subsection{Proof of Lemma \ref{lem:Dbar_conc}}
Note that $\Dbar_{ii} = \sum_{j=1}^n \abs{H_{ij}}$ is the sum of independent random variables, where 
$\abs{H_{ij}} \leq M$ $\forall i,j$. Moreover, 
\begin{align*}
\sum_{j=1}^n \expec[\abs{H_{ij}}^2] 
= \eta p \sum_{j=1}^n (r_i - r_j)^2 + (1-\eta) \frac{np M}{2} 
\leq \eta p n M^2 + (1-\eta) np \frac{M}{2} 
= np A(\eta,M).
\end{align*}
Invoking Bernstein's inequality, it follows for any given $i$ and $t \geq 0$ that 
\begin{equation*}
 \prob(\abs{\Dbar_{ii} - \expec[\Dbar_{ii}]} \geq t) \leq 2\exp\left(-\frac{t^2}{2(np A(\eta,M) + \frac{M t}{3})} \right).
\end{equation*}
In order to get a high probability bound, we require $t^2 \geq 4\log n (np A(\eta,M) + \frac{Mt}{3})$. Since $t \geq 0$, 
this is equivalent to saying that 
\begin{equation} \label{eq:tcond_temp1}
t \geq \frac{\frac{4M}{3} \log n + \sqrt{(\frac{4M}{3} \log n)^2 + 16 pn \log n A(\eta,M)}}{2}.
\end{equation}
If $(\frac{4M}{3} \log n)^2 \leq 6 pn \log n A(\eta,M)$ or equivalently $p \geq \frac{M^2}{9 A(\eta,M)} \frac{\log n}{n}$, 
then the RHS of \eqref{eq:tcond_temp1} is bounded by 
\begin{equation*}
\sqrt{16 pn \log n A(\eta,M)} \frac{\sqrt{2} + 1}{2} = 2(\sqrt{2} + 1) \sqrt{pn \log n} \sqrt{A(\eta,M)}.
\end{equation*}
Hence taking $t = 2(\sqrt{2} + 1) \sqrt{pn \log n} \sqrt{A(\eta,M)}$ and applying the union bound, we obtain the statement of the Lemma.

\end{document}